\documentclass[10pt,journal,compsoc]{IEEEtran}
\ifCLASSOPTIONcompsoc
  \usepackage[nocompress]{cite}
\else
  \usepackage{cite}
\fi
\usepackage{amsmath,amssymb,epsfig,subfigure,color}
\usepackage{multirow,helvet,booktabs,courier,bm,cite}
\usepackage{latexsym,amsmath,amssymb,graphicx}
\usepackage{amsthm}
\usepackage{url}
\usepackage{algorithm}
\usepackage{algorithmic}
\usepackage{setspace}
\usepackage{color}
\definecolor{c1}{rgb}{0.859,0.933,0.953}
\newtheorem{definition}{Definition}
\newtheorem{theorem}{Theorem}

\newtheorem{remark}{Remark}
\newtheorem{corollary}{Corollary}

\newtheorem{proposition}{Proposition}
\newcommand{\tabincell}[2]{\begin{tabular}{@{}#1@{}}#2\end{tabular}}

\ifCLASSINFOpdf
\else
\fi

\begin{document}
%
\title{Low-Rank Tensor Function Representation for Multi-Dimensional Data Recovery}
%
%
%
%

\author{Yisi Luo, Xile Zhao,~\IEEEmembership{Member, IEEE}, Zhemin Li, Michael K. Ng,~\IEEEmembership{Senior Member, IEEE},\\Deyu Meng,~\IEEEmembership{Member, IEEE}
\IEEEcompsocitemizethanks{\IEEEcompsocthanksitem Yisi Luo and Deyu Meng are with the School of Mathematics and Statistics, Xi'an Jiaotong University, Xi'an, P.R.China (e-mail: yisiluo1221@foxmail.com, dymeng@mail.xjtu.edu.cn).\protect
\IEEEcompsocthanksitem Xile Zhao is with the School of Mathematical Sciences, University of Electronic Science and Technology of China, Chengdu, P.R.China (e-mail: xlzhao122003@163.com).\protect
\IEEEcompsocthanksitem Zhemin Li is with the Department of Mathematics, National University of Defense Technology, Changsha, P.R.China (e-mail: lizhemin@nudt.edu.cn).\protect
\IEEEcompsocthanksitem Michael K. Ng is with the Department of Mathematics, The University of Hong Kong, Hong Kong (e-mail: mng@maths.hku.hk).}
\thanks{(Corresponding authors: Xile Zhao and Deyu Meng)}}

%
%

\markboth{}%
{}
%



\IEEEtitleabstractindextext{%
\begin{sloppypar}
\begin{abstract}
Since higher-order tensors are naturally suitable for representing multi-dimensional data in real-world, e.g., color images and videos, low-rank tensor representation has become one of the emerging areas in machine learning and computer vision. However, classical low-rank tensor representations can only represent data on finite meshgrid due to their intrinsical discrete nature, which hinders their potential applicability in many scenarios beyond meshgrid. To break this barrier, we propose a low-rank tensor function representation (LRTFR), which can continuously represent data beyond meshgrid with infinite resolution. Specifically, the suggested tensor function, which maps an arbitrary coordinate to the corresponding value, can continuously represent data in an infinite real space. Parallel to discrete tensors, we develop two fundamental concepts for tensor functions, i.e., the tensor function rank and low-rank tensor function factorization. We theoretically justify that both low-rank and smooth regularizations are harmoniously unified in the LRTFR, which leads to high effectiveness and efficiency for data continuous representation. Extensive multi-dimensional data recovery applications arising from image processing (image inpainting and denoising), machine learning (hyperparameter optimization), and computer graphics (point cloud upsampling) substantiate the superiority and versatility of our method as compared with state-of-the-art methods. Especially, the experiments beyond the original meshgrid resolution (hyperparameter optimization) or even beyond meshgrid (point cloud upsampling) validate the favorable performances of our method for continuous representation. 
\end{abstract}
\end{sloppypar}

\begin{IEEEkeywords}
Tensor factorization, multi-dimensional data, data recovery
\end{IEEEkeywords}}

\maketitle

\IEEEdisplaynontitleabstractindextext

%
\IEEEpeerreviewmaketitle

\IEEEraisesectionheading{\section{Introduction}\label{sec:introduction}}
Recently, due to the advance of technology, various types of multi-dimensional data (e.g., color images, multispectral images, point clouds, traffic flow data, user-item data, etc.) are increasingly emerging. Mathematically, higher-order tensors are naturally suitable for multi-dimensional data modeling and processing, which is one of the focus areas in machine learning, computer vision, and scientific computing\cite{PNAS_tensor,TIT_sparse,SIAM_19,SIAM_21}.\par
Most real-world data intrinsically exhibits low-dimensional structures, for example, images\cite{IJCV_WNN}, videos\cite{Train_TIP}, point clouds\cite{TIM_22}, and so on. Hence, low-rank modeling of matrices/tensors has been widely studied for data processing and representation\cite{PAMI_13,TSP,NGMeet,PAMI_13_lowrank,PAMI_15_Zhao,PAMI_16_Lin,KBR}. Different from matrices, the rank definition of higher-order tensors is not unique. The most classic tensor ranks are the Tucker rank (defined by the rank of unfolding matrices)\cite{SIAM_review} and CANDECOMP/PARAFAC (CP) rank (defined as the smallest number of rank one tensor decomposition)\cite{SIAM_review}. It is shown that solving the low-Tucker/CP-rank programming is effective to obtain a compact and meaningful representation of data\cite{TSP_17,LRTDTV}. Meanwhile, the low-Tucker/CP-rank models have been applied to facilitate the efficiency of modern deep learning algorithms\cite{HPO_PAMI,NTM,AAAI_2021_transfer,TFNN}, which reveals their wide and promising applicabilities. Another type of tensor rank is based on the so-called tensor network decomposition\cite{TN_ICML}, such as tensor train rank\cite{TT}, tensor ring rank\cite{TSP_ring}, and tensor tree rank\cite{SIAM_tree}. These more sophisticated tensor ranks were proven to be highly representative for multi-dimensional modeling; see\cite{Ring_ICCV,Train_TIP,FCTN}. More recently, the tensor tubal-rank\cite{SIAM_13}, which is based on the tensor singular value decomposition (t-SVD)\cite{LAA_11}, has attracted much attention due to its strong relationships to the definition of matrix rank; see for example\cite{PAMI_TRPCA,PAMI_factorization,PAMI_Wang_21,PAMI_Wang_22}. Besides, the effctiveness of tubal-rank minimization for various signal processing tasks has been validated\cite{cvpr_14,CVPR_19,TLRR,FTNN}. In summary, low-rank tensor modeling has become increasingly popular for multi-dimensional data representation and processing.\par
Except for the low-rankness, smoothness is another frequently considered regularization in multi-dimensional data representation. Many literatures incorporated the smoothness into the low-rank representation\cite{LRTDTV,AAAI_2017,HLRTF,SMF,DFR,CTV}. In such a hybrid model, low-rankness was usually revealed by low-rank factorization \cite{DFR,PAMI_factorization,AAAI_tensor_1} or surrogate functions\cite{STD,TIP_22_Luo,PAMI_Zhang,AAAI_tensor_2}. Besides, there are two main categories of smooth regularization forms, i.e., explicit and implicit ones. The explicit smooth regularizations are mainly based on the total variation (TV) and its variants \cite{AAAI_2017,E3DTV,l0l1HTV,HLRTF,LRTDTV,CTV}. The implicit smooth regularizations are revealed by using basis functions to parameterize the model, which naturally extends the model to function representations that are related to this work. For example, the pioneer works \cite{sp,spl_smooth} utilized non-negative matrix factorization parameterized by basis functions to reveal the implicit smoothness. It was further extended to higher-order tensor models\cite{STD} by using the low-Tucker-rank regularization. More recently, the CP factorization was elegantly generalized to multivariate functions by using Fourier
series\cite{TSP_CP}. The main aim of these implicit smooth representations is to employ some basis functions (e.g., Gaussian basis\cite{sp} or Fourier basis\cite{TSP_CP}) to represent the low-rank matrix/tensor, which implicitly induces the smoothness between adjacent elements of the matrix/tensor. Nevertheless, these shallow basis functions may sometimes not be suitable to capture the complex fine details of real-world data. Meanwhile, these basis function-based representations still rely on the low-rank regularization defined on original meshgrid to handle plain data analysis\cite{sp} or regression problems\cite{TSP_CP}, which cannot straightforwardly extend to more challenging multi-dimensional data recovery problems on or beyond meshgrid.\par
 \begin{figure}[t]
  \scriptsize
  \setlength{\tabcolsep}{0.9pt}
    \begin{center}
    \begin{tabular}{c}
  \includegraphics[width=0.44\textwidth]{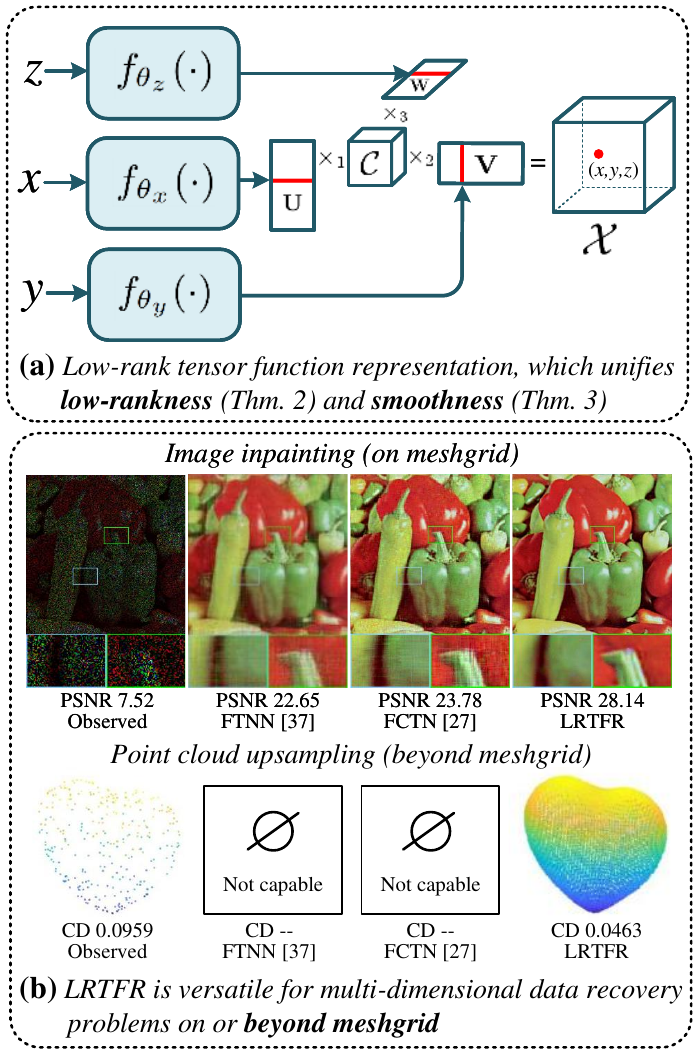}
    \end{tabular}
    \end{center}
  \vspace{-0.4cm}
  \caption{(a) The proposed LRTFR can continuously represent multi-dimensional data via the low-rank tensor function factorization, which intrinsically unifies the low-rankness and smoothness into the representation. (b) The continuous LRTFR is versatile for general multi-dimensional data recovery problems on or beyond meshgrid.\label{fig_1}\vspace{-0.4cm}}
\end{figure}
To sum up, low-rank tensor representations are suitable for representing data on discrete meshgrid. However, how to extend low-rank tensor models for continuous data representation beyond meshgrid is a pressing challenge driven by real-world applications. For example, point cloud representation, where classical low-rank representations can not represent such signals beyond meshgrid. An important question naturally arises: Can we develop a multi-dimensional data representation that can not only preserve the compact low-rank structure, but also continuously represent data beyond meshgrid with infinite resolution?\par 
To meet this challenge, this work presents the low-rank tensor function representation (LRTFR) for multi-dimensional data continuous representation. Specifically, we consider representing data with a tensor function, which maps a multi-dimensional coordinate to the corresponding value to continuously represent data. Parallel to classical low-rank tensor representations, we develop two fundamental concepts for tensor functions, i.e., the tensor function rank (Definition \ref{def_rank}) and tensor function factorization (Theorem \ref{Pro_1}), which intrinsically encodes the low-rankness into the continuous representation. The factor functions of the factorization can be readily parameterized by multilayer perceptrons (MLPs), which makes the model highly expressive for real-world data representation. Moreover, we theoretically justify a Lipschitz smooth regularization hidden in the model. The low-rankness and smoothness are harmoniously unified in the representation, making LRTFR effective and efficient for data continuous representation. Compared to classical low-rank tensor representations, LRTFR is more versatile for representing data on or beyond meshgrid; see Fig. \ref{fig_1}. In summary, this work makes the following contributions:\par
{\bf (i)} To more intrinsically explore the inside information underlying multi-dimensional data, we develop a new tensor function formulation for continuously representing multi-dimensional data in an infinite real space and explore two fundamental concepts under tensor functions, i.e., the tensor function rank and tensor function factorization, which establish insightful connections between discrete and continuous representations for multi-dimensional data.\par
{\bf (ii)} On the basis of the tensor function so formulated, we propose the LRTFR method for handling general multi-dimensional data recovery tasks. The referred low-rank tensor function, which maps an arbitrary coordinate to the corresponding value, allows us to recover multi-dimensional data on meshgrid beyond the original resolution or even beyond meshgrid.\par
{\bf (iii)} We theoretically reveal that the low-rank and smooth regularizations are implicitly unified in LRTFR, which justifies the potential effectiveness of LRTFR for multi-dimensional data recovery.\par
{\bf (iv)} The proposed LRTFR is versatile for various multi-dimensional data recovery tasks on or beyond meshgrid, including multi-dimensional image inpainting and denoising (on meshgrid with original resolution), hyperparameter optimization (on meshgrid beyond original resolution), and point cloud upsampling (beyond meshgrid). All of these problems are addressed in an unsupervised manner by solely using the observed data. Extensive experiments validate the broad applicability and superiority of our method as compared with state-of-the-art methods.
\subsection{Related Work}\label{Rela}
\subsubsection{Data Continuous Representation}
The topic of continuous representation of data has very recently attracted considerable attention \cite{LIIF,DeepSDF,SAL,3dshape,ICCV_19}. One of the most famous techniques for continuous data representation is the implicit neural representation (INR)\cite{nips_INR,sine,NeRF}. The INR constructs differentiable functions (deep neural networks) w.r.t. the coordinates (inputs) to implicitly represent the continuous data (outputs). The implicit representation is obtained by feeding each coordinate to the implicit function and outputting the corresponding value. Since the implicit function is defined on a continuous domain, the resulting data representation is also continuous. The success of INR-based methods for different tasks has been witnessed, e.g., neural rendering \cite{sota_rendering,NeRF,real_time,nips_rendering,cvpr_without}, image/shape generation \cite{GAN_CVPR,AISTATS,shape_GAN,cvpr_3d}, and scene representation \cite{DeepSDF,RGB-D,oc_net,nips_INR}. Besides, some important improvements over INR (e.g., more effective training strategy \cite{implicit2,SAP,real_time} and more expressive INR structures \cite{sine,NIPS_Four,ICML_21}) have also been studied recently.\par 
Despite these great efforts, INR still suffers from two limitations. First, INR always requires relatively large memory and computational cost, mainly because the size of the input coordinate matrix is too large (See Sec. \ref{advantage_INR} for detailed computational analyses). The computational cost tremendously increases when dealing with multi-dimensional data, e.g., multispectral images and videos. Second, INR itself is not stable enough to directly learn a valid continuous representation from raw data, which usually results in overfitting and restricts its applicability in many real-world scenarios beyond meshgrid. The underlying reason of the two limitations is that INR ignores the important domain knowledge of data, which results in large training costs and inevitable overfitting phenomenons. As compared, our LRTFR disentangles the continuous representation into several much simpler factor representations by introducing insightful domain knowledge (i.e., low-rankness); see Fig. \ref{fig_1} (a). Consequently, LRTFR achieves lower computational costs and increases the stability of the continuous representation against overfitting; see Sec. \ref{advantage_INR} for details and also extensive experiments in Sec. \ref{Sec_exp}.
\subsubsection{Data Recovery via Continuous Representation}
Several works have studied INRs for image recovery problems, which are related to our work. The local implicit image function \cite{LIIF} was proposed for image super-resolution. Its improved versions in terms of network capacity \cite{INR_SR_NIPS} and details recovery \cite{INR_SR_TIP} were proposed to learn more realistic image continuous representations. The meta-learning strategies \cite{metaSR,sine} were also developed to learn mappings that generate INRs for image recovery. However, these image recovery methods were entirely dependent on supervised learning with pairs of images to train the INR. Comparatively, our LRTFR is a model-based and unsupervised method that finely encodes the low-rankness into the continuous representation. Hence, our method can be more easily and conveniently generalized over different applications; see Sec. \ref{Sec_app}. A very recent work \cite{denoise_INR_1} proposed to use INR for zero-shot blind image denoising, where decent denoising results were obtained. However, its accompanied theoretical explanations for image denoising are lacking. It also has high computational costs because it uses the standard INR. In comparison, our LRTFR implicitly and theoretically encodes the low-rankness into the continuous representation via the tensor function factorization, which is more efficient with clearer interpretations for data recovery.\par
\section{Main Results}
\subsection{Preliminaries}
In this paper, scalars, vectors, matrices, and tensors are denoted by $x$, $\bf x$, $\bf X$, and $\mathcal X$. The $i$-th element of $\bf x$ is denoted by ${\bf x}_{(i)}$, and it is similar for matrices, i.e., ${\bf X}_{(i,j)}$, and tensors, i.e., ${\mathcal X}_{(i,j,k)}$. When we use the index $:$, e.g., ${\bf X}_{(:,j)}$, we mean the $j$-th column of $\bf X$. The tensor Frobenius norm of ${\mathcal X}\in{\mathbb R}^{n_1\times n_2\times n_3}$ is defined as $\left\lVert{\mathcal X}\right\rVert_F:=\sqrt{\langle {\mathcal X},{\mathcal X}\rangle}=\sqrt{\sum_{ijk}{\mathcal X}_{(i,j,k)}^2}$. The tensor $\ell_1$-norm is defined as $\left\lVert {\mathcal X}\right\rVert_{\ell_1}:=\sum_{ijk}|{\mathcal X}_{(i,j,k)}|$. The unfolding operator of a tensor ${\mathcal X}\in{\mathbb R}^{n_1\times n_2\times n_3}$ along the $i$-th mode ($i=1,2,3$) is defined as ${\tt unfold}_i(\cdot):{\mathbb R}^{n_1\times n_2\times n_3}\rightarrow{\mathbb R}^{n_i\times \prod_{j\neq i}n_j}$, which returns the unfolding matrix along the mode $i$, and the unfolding matrix is denoted by ${\bf X}^{(i)}:={\tt unfold}_i({\mathcal X})$. ${\tt fold}_i(\cdot)$ denotes the inverse operator of ${\tt unfold}_i(\cdot)$. The mode-$i$ ($i=1,2,3$) tensor-matrix product is defined as ${\mathcal X}\times_i{\bf A}:={\tt fold}_i({\bf A}{\bf X}^{(i)})$, which returns a tensor. The symbol $\lfloor x\rfloor$ denotes the rounded-down of a constant $x$. Next, we introduce the tensor Tucker rank and Tucker factorization.
\begin{definition}(Tucker rank\cite{SIAM_review})\label{def_Tucker}
The Tucker rank of a tensor ${\mathcal X}\in{\mathbb R}^{n_1\times n_2\times n_3}$ is a vector defined as 
\begin{equation}
{\rm rank}_T({\mathcal X}):=({\rm rank}({\bf X}^{(1)}), {\rm rank}({\bf X}^{(2)}),{\rm rank}({\bf X}^{(3)})).
\end{equation}
For simplicity, we sometimes use the notation $({\rm rank}_T({\mathcal X}))_{(i)} := {\rm rank}({\bf X}^{(i)})$ in the main body.
\end{definition}
\begin{theorem}(Tucker factorization \cite{SIAM_review})\label{th_Tucker}
Let ${\mathcal X}\in{\mathbb R}^{n_1\times n_2\times n_3}$.
\begin{itemize}
\item[\bf (i)] If ${\rm rank}_T({\mathcal X})=(r_1,r_2,r_3)$, then there exist a core tensor ${\mathcal C}\in{\mathbb R}^{r_1\times r_2\times r_3}$ and three factor matrices ${\bf U}\in{\mathbb R}^{n_1\times r_1}$, ${\bf V}\in{\mathbb R}^{n_2\times r_2}$, and ${\bf W}\in{\mathbb R}^{n_3\times r_3}$ such that ${\mathcal X}={\mathcal C}\times_1{\bf U}\times_2{\bf V}\times_3{\bf W}$.
\item[\bf (ii)] Let ${\mathcal C}\in{\mathbb R}^{r_1\times r_2\times r_3}$ be an arbitrary tensor, ${\bf U}\in{\mathbb R}^{n_1\times r_1}$, ${\bf V}\in{\mathbb R}^{n_2\times r_2}$, ${\bf W}\in{\mathbb R}^{n_3\times r_3}$ be arbitrary matrices ($r_i\leq n_i$ for $i=1,2,3$). Then $\big{(}{\rm rank}_T({\mathcal C}\times_1{\bf U}\times_2{\bf V}\times_3{\bf W})\big{)}_{(i)}\leq r_i$ ($i=1,2,3$).   
\end{itemize}
\end{theorem}
\subsection{Low-Rank Tensor Function Representation}
In this section, we detailedly introduce the proposed continuous representation for multi-dimensional data. Without loss of generality, we consider the three-dimensional case, while it can be easily generalized to higher-dimensional cases. Let $f(\cdot):X_f\times Y_f\times Z_f\rightarrow{\mathbb R}$ be a bounded real function, where $X_f,Y_f,Z_f\subset{\mathbb R}$ are definition domains in three dimensions. The function $f(\cdot)$ gives the value of data at any coordinate in $D_f:=X_f\times Y_f\times Z_f$. We interpret $f(\cdot)$ as a {\sl tensor function} since it maps a three-dimensional coordinate to the corresponding value, implicitly representing third-order tensor data. Compared with classical tensor formulation, the tensor function intrinsically allows us to process and analyse multi-dimensional data on meshgrid beyond the original resolution or even beyond meshgrid. When $D_f$ is a discrete set of some constants, the output form of $f(\cdot)$ degrades to the discrete case (i.e., tensors).\par 
Based on tensor functions, we can naturally define the following sampled tensor set, which covers all tensors that can be sampled from the tensor function with different sampling coordinates.
\begin{definition}\label{def_S}
For a tensor function $f(\cdot):D_f\rightarrow{\mathbb R}$, we define the sampled tensor set $S[f]$ as
\begin{equation}
\begin{split}
S[f]:=\{{\mathcal T}|{\mathcal T}_{(i,j,k)}&=f({\bf x}_{(i)},{\bf y}_{(j)},{\bf z}_{(k)}),{\bf x}\in X_f^{{n}_1},\\&{\bf y}\in Y_f^{{n}_2},{\bf z}\in Z_f^{{n}_3},{n}_1,{n}_2,{n}_3\in{\mathbb N}_+\},
\end{split}
\end{equation}
where $\bf x$, $\bf y$, $\bf z$ denote the coordinate vector variables and $n_1$, $n_2$, $n_3$ are positive integer variables that determine the sizes of the sampled tensor $\mathcal T$.
\end{definition}
Tensor function is a promising and potential tool for multi-dimensional data processing. An interesting and fundamental question is whether we can analogously define the ``rank'' and develop the ``tensor factorization'' for tensor functions. Regarding the rank definition, a reasonable expectation is that any tensor sampled from $S[f]$ is a low-rank tensor. Thus, we can naturally define the function rank (F-rank) of $f(\cdot)$ as the supremum of the tensor rank in $S[f]$.
\begin{definition}\label{def_rank}
Given a tensor function $f:D_f=X_f\times Y_f\times Z_f\rightarrow {\mathbb R}$, we define a measure of its complexity, denoted by ${\rm F}$-${\rm rank}[f]$ (function rank of $f(\cdot)$), as the supremum of Tucker rank in the sampled tensor set $S[f]$:
\begin{equation}
\begin{split}
{\rm F\text{-}rank}[f]:=(r_{1},r_{2},r_{3}),\;{\rm where}\;r_{i} = \sup_{{\mathcal T}\in S[f]}{\rm rank}({\bf T}^{(i)}).
\end{split}
\end{equation}
\end{definition}
We call a tensor function $f({\cdot})$ with ${\rm F\text{-}rank}[f]=(r_1,r_2,r_3)$ ($r_i<\infty$ for $i=1,2,3$) as a low-rank tensor function since the Tucker rank of any ${\mathcal T}\in S[f]$ is bounded by $(r_1,r_2,r_3)$. When $f(\cdot)$ is defined on certain discrete sets, the F-rank degenerates into the discrete case, i.e., the classical Tucker rank, as stated in Proposition \ref{lemma_Tucker}.
\begin{proposition}\label{lemma_Tucker}
Let ${\mathcal X}\in{\mathbb R}^{m_1\times m_2\times m_3}$ be an arbitrary tensor. Let $X_f=\{1,2,\cdots,m_1\},Y_f=\{1,2,\cdots,m_2\},Z_f=\{1,2,\cdots,m_3\}$ be three discrete sets, and denote $D_f=X_f\times Y_f\times Z_f$. Define the tensor function $f(\cdot):D_f\rightarrow{\mathbb R}$ by $f(v_1,v_2,v_3)={\mathcal X}_{(v_1,v_2,v_3)}$ for any $(v_1,v_2,v_3)\in D_f$. Then we have ${\rm F\text{-}rank}[f]={\rm rank}_T({\mathcal X})$.  
\end{proposition}
\begin{proof}
First, it is ordinary to see that $\sup_{{\mathcal T}\in S[f]}{\rm rank}({\bf T}^{(i)})\geq{\rm rank}({\bf X}^{(i)})$ ($i=1,2,3$) by setting ${\bf x}=(1,2,\cdots,m_1)$, ${\bf y}=(1,2,\cdots,m_2)$, and ${\bf z}=(1,2,\cdots,m_3)$ in (\ref{def_S}).\par 
Second, we will show that $\sup_{{\mathcal T}\in S[f]}{\rm rank}({\bf T}^{(i)})\leq{\rm rank}({\bf X}^{(i)})$ ($i=1,2,3$). Let ${\mathcal T}\in S[f]\cap{\mathbb R}^{n_1\times n_2\times n_3}$ with $n_i$s $(i=1,2,3)$ being some integers. According to the definition of $S[f]$, we can see that for each column vector of ${\bf T}^{(1)}$, which is denoted by ${\bf T}^{(1)}_{(:,p)}$ $(p\in\{1,2,\cdots,n_2n_3\})$, there exists a constant $l_p\in\{1,2,\cdots,m_2m_3\}$ depending on $p$ such that ${\bf T}^{(1)}_{(:,p)}$ is a rearrangement of the elements of ${\bf X}^{(1)}_{(:,l_p)}$, where repeated sampling is allowed. In another word, for any ${\bf T}^{(1)}_{(:,p)}$, there exists a rearranging matrix ${\bf A}\in{\{0,1\}^{n_1\times m_1}}$ and a column of ${\bf X}^{(1)}$ depending on $p$ (i.e., ${\bf X}^{(1)}_{(:,l_p)}$) such that ${\bf T}^{(1)}_{(:,p)}={\bf A}{\bf X}^{(1)}_{(:,l_p)}$. From Definition \ref{def_S}, we can see that the rearranging matrix for all ${\bf T}^{(1)}_{(:,p)}$s $(p=1,2,\cdots,n_2n_3)$ are consistent, i.e., ${\bf T}^{(1)}_{(:,p)}={\bf A}{\bf X}^{(1)}_{(:,l_p)}$ for all $p=1,2,\cdots,n_2n_3$. Let ${\widetilde{\bf X}^{(1)}}=[{\bf X}^{(1)}_{(:,l_1)},{\bf X}^{(1)}_{(:,l_2)},\cdots,{\bf X}^{(1)}_{(:,l_{n_2n_3})}]\in{\mathbb R}^{m_1\times n_2n_3}$, then we have ${\rm rank}({\widetilde{\bf X}^{(1)}})\leq{\rm rank}({{\bf X}^{(1)}})$ and ${\bf T}^{{(1)}}={\bf A}\widetilde{\bf X}^{{(1)}}$, and thus ${\rm rank}({\bf T}^{(1)})\leq{\rm rank}({{\bf X}^{(1)}})$ holds. Similarly, we can prove ${\rm rank}({\bf T}^{(2)})\leq{\rm rank}({{\bf X}^{(2)}})$ and ${\rm rank}({\bf T}^{(3)})\leq{\rm rank}({{\bf X}^{(3)}})$. Note that these results do not depend on $n_i$s $(i=1,2,3)$. In summary, for any ${\mathcal T}\in S[f]$, the inequalities ${\rm rank}({\bf T}^{(i)})\leq{\rm rank}({\bf X}^{(i)})$ ($i=1,2,3$) hold and they deduce $\sup_{{\mathcal T}\in S[f]}{\rm rank}({\bf T}^{(i)})\leq{\rm rank}({\bf X}^{(i)})$ ($i=1,2,3$). 
\end{proof}
Lemma \ref{lemma_Tucker} establishes the connection between F-rank and the classical tensor rank in the discrete case. Otherwise, if the definition domains are continuous, e.g., $X_f=[1,n_1]$, $Y_f=[1,n_2]$, and $Z_f=[1,n_3]$ for some constants $n_i$s ($i=1,2,3$), $f(\cdot)$ can represent data beyond meshgrid with infinite resolution. Therefore, ${\rm F\text{-}rank}$ is an extension of Tucker rank from discrete tensors to tensor functions for continuous representations.\par 
Similar to classical tensor representations, it is meaningful to think over whether a low-rank tensor function $f(\cdot)$ with $({\rm F\text{-}rank}[f])_{(i)}<\infty$ ($i=1,2,3$) can also have some tensor factorization strategies to encode the low-rankness. We present a positive answer that a tensor function $f(\cdot)$ with ${\rm F\text{-}rank}[f]=(r_1,r_2,r_3)$ can be factorized as the product of a core tensor $\mathcal C$ and three factor functions $f_x(\cdot)$, $f_y(\cdot)$, and $f_z(\cdot)$, where their output dimensions are related to the F-rank $r_i$ ($i=1,2,3$). On the contrary, the products of a core tensor $\mathcal C$ and three factor functions $g_x(\cdot)$, $g_y(\cdot)$, and $g_z(\cdot)$ form a low-rank representation $g(\cdot)$, where ${\rm F\text{-}rank}[g]$ is bounded by the output dimensions of $g_x(\cdot)$, $g_y(\cdot)$, and $g_z(\cdot)$. The theory is formally given as follows. 
\begin{theorem}\label{Pro_1}
(Low-rank tensor function factorization, see proof in supplementary) Let $f(\cdot):D_f=X_f\times Y_f\times Z_f\rightarrow {\mathbb R}$ be a bounded tensor function, where $X_f,Y_f,Z_f\subset{\mathbb R}$. Then
\begin{itemize}
\item[\bf (i)] If ${\rm F\text{-}rank}[f]=(r_1,r_2,r_3)$, then there exist a tensor ${\mathcal C}\in{\mathbb R}^{r_1\times r_2\times r_3}$ and three bounded functions $f_x(\cdot):X_f\rightarrow {\mathbb R}^{r_1}$, $f_y(\cdot):Y_f\rightarrow {\mathbb R}^{r_2}$, and $f_z(\cdot):Z_f\rightarrow {\mathbb R}^{r_3}$ such that for any ${(v_1,v_2,v_3)}\in D_f$, $f(v_1,v_2,v_3)={\mathcal C}\times_1f_x({v}_1)\times_2f_y({v}_2)\times_3f_z({v}_3)$.
\item[\bf (ii)] On the other hand, let ${\mathcal C}\in{\mathbb R}^{r_1\times r_2\times r_3}$ be an arbitrary tensor and $g_x(\cdot):X_g\rightarrow {\mathbb R}^{r_1}$, $g_y(\cdot):Y_g\rightarrow {\mathbb R}^{r_2}$, and $g_z(\cdot):Z_g\rightarrow {\mathbb R}^{r_3}$ be arbitrary bounded functions defined on $X_g$, $Y_g$, $Z_g\subset{\mathbb R}$. Then we have $({\rm F\text{-}rank}[g])_{(i)}\leq r_i$ ($i=1,2,3$), where $g(\cdot):D_g=X_g\times Y_g\times Z_g\rightarrow{\mathbb R}$ is defined by $g(v_1,v_2,v_3)={\mathcal C}\times_1 g_x(v_1)\times_2 g_y(v_2)\times_3 g_z(v_3)$ for any $(v_1,v_2,v_3)\in D_g$.
\end{itemize}
\end{theorem}
Theorem \ref{Pro_1} is a natural extension of Tucker factorization (Theorem \ref{th_Tucker}) from discrete meshgrid to the continuous domain. It inherits the nice property of Tucker factorization that the low-rank tensor function $f(\cdot)$ can be factorized into some simpler factor functions (in terms of number of function variables).
\begin{remark}
Tucker factorization (Theorem \ref{th_Tucker}) is a special case of our Theorem \ref{Pro_1} when $D_f$ (or $D_g$) is a certain discrete set representing meshgrid. This can be easily derived by incorporating Lemma \ref{lemma_Tucker} into Theorem \ref{Pro_1}. 
\end{remark}
Based on the low-rank tensor function factorization, we can compactly represent the multi-dimensional data by a low-rank tensor function formulated as 
\begin{equation}\label{eq_LRTFR}
{\textbf [}{\mathcal C};f_x,f_y,f_z{\textbf ]}({\bf v}):={\mathcal C}\times_1f_x({\bf v}_{(1)})\times_2f_y({\bf v}_{(2)})\times_3f_z({\bf v}_{(3)}), 
\end{equation}
which is parameterized by the core tensor ${\mathcal C}$ and factor functions $f_x(\cdot)$, $f_y(\cdot)$, and $f_z(\cdot)$. The representation implicitly encodes the low-rankness of the tensor function by the low-rank function factorization, i.e., any tensor sampled from the tensor function representation must be a low-rank tensor, as stated in Theorem \ref{Pro_1}. This property allows us to more compactly and stably represent the multi-dimesional data in a continuous manner.\par 
In the LRTFR (\ref{eq_LRTFR}), we further suggest to use the MLP to parameterize the factor functions due to its powerful universal approximation abilities\cite{Uni}. Specifically, we employ three MLPs $f_{\theta_x}(\cdot)$, $f_{\theta_y}(\cdot)$, and $f_{\theta_z}(\cdot)$ with parameters $\theta_x$, $\theta_y$, and $\theta_z$ to parameterize the factor functions $f_{x}(\cdot)$, $f_{y}(\cdot)$, and $f_{z}(\cdot)$. Taking $f_{\theta_x}(\cdot)$ as an example, it is formulated as
\begin{equation}\label{MLP}
f_{\theta_x}({\bf x}):={\bf H}_d(\sigma({\bf H}_{d-1}\cdots\sigma({\bf H}_1{\bf x}))):X_f\rightarrow {\mathbb R}^{r_1},
\end{equation}
where $\sigma(\cdot)$ is the nonlinear activation function and $\theta_x:=\{{\bf H}_i\}_{i=1}^d$ are learnable weight matrices of the MLP. With these in mind, the MLP-parameterized LRTFR is formulated as ${\textbf [}{\mathcal C};f_{\theta_x},f_{\theta_y},f_{\theta_z}{\textbf ]}({\bf v}):={\mathcal C}\times_1f_{\theta_x}({\bf v}_{(1)})\times_2f_{\theta_y}({\bf v}_{(2)})\times_3f_{\theta_z}({\bf v}_{(3)})$, which is parameterized by the core tensor ${\mathcal C}$ and MLP weights $\theta_x$, $\theta_y$, and $\theta_z$.
\par 
Compared to existing continuous representations using the INR\cite{sine,DeepSDF,NIPS_Four}, the advantages of our LRTFR are that {\bf (i)} the suggested domain knowledge, i.e., low-rankness, is helpful to alleviate the overfitting phenomenon for learning continuous representations, and {\bf (ii)} the tensor function factorization is expected to more evidently reduce the computational complexity for learning continuous representations; see details in Sec. \ref{advantage_INR}.
\subsection{Implicit Smooth Regularization of LRTFR}
Since smoothness is another common property in multi-dimensional data, e.g., the temporal smoothness of videos\cite{IJCV_Dong} and the spectral smoothness of hyperspectral images\cite{LRTDTV}, it is interesting to explore the smooth property in LRTFR besides the low-rankness. Next, we theoretically justify that LRTFR encodes an implicit smooth regularization brought from the specific structures of MLPs. 
\begin{theorem}\label{Pro_2}
(See proof in supplementary) Let ${\mathcal C}\in{\mathbb R}^{r_1\times r_2\times r_3}$, and $f_{\theta_x}(\cdot):{X_f}\rightarrow{\mathbb R}^{r_1}$, $f_{\theta_y}(\cdot):{Y_f}\rightarrow{\mathbb R}^{r_2}$, $f_{\theta_z}(\cdot):{Z_f}\rightarrow{\mathbb R}^{r_3}$ be three MLPs structured as in (\ref{MLP}) with parameters $\theta_x$, $\theta_y$, $\theta_z$, where $X_f,Y_f,Z_f\subset{\mathbb R}$. Suppose that the MLPs share the same activation function $\sigma(\cdot)$ and depth $d$. Besides, we assume that
\begin{itemize}
\item $\sigma(\cdot)$ is Lipschitz continuous with Lipschitz constant $\kappa$.
\item The $\ell_1$-norm of $\mathcal C$ is bounded by $\eta_1>0$.
\item The $\ell_1$-norm of each weight matrix ${\bf H}_{i}$ in the three MLPs is bounded by $\eta_2>0$. Let $\eta=\max\{\eta_1,\eta_2\}$.
\end{itemize}
Define a tensor function $f(\cdot):D_f=X_f\times Y_f\times Z_f\rightarrow {\mathbb R}$ as $f(\cdot):={\textbf{\rm [}}{\mathcal C};f_{\theta_x},f_{\theta_y},f_{\theta_z}{\textbf{\rm ]}}(\cdot)$. Then, the following inequalities hold for any $(x_1,y_1,z_1)$, $(x_2,y_2,z_2)\in{D_f}$:
    \begin{equation}\label{ineq_f}
  \left\{
  \begin{aligned}
&|f(x_1,y_1,z_1)-f(x_2,y_1,z_1)|\leq \delta|x_1-x_2|\\
&|f(x_1,y_1,z_1)-f(x_1,y_2,z_1)|\leq \delta|y_1-y_2|\\
&|f(x_1,y_1,z_1)-f(x_1,y_1,z_2)|\leq \delta|z_1-z_2|,
  \end{aligned}
  \right.
    \end{equation}
where $\delta=\eta^{3d+1}\kappa^{3d-3}{\zeta}^2$ and ${\zeta}=\max\{|x_1|,|y_1|,|z_1|\}$.
\end{theorem}
Theorem \ref{Pro_2} gives a Lipschitz type smoothness guarantee of the MLP-parameterized LRTFR. The smoothness is implicitly encoded under mild assumptions of nonlinear activation function and weight matrices, which are easy to achieve in real implementation. For example, most widely-used activation functions are Lipschitz continuous, such as ReLU, LeakyReLU, Sine, and Tanh. Moreover, we can expediently control the degree of smoothness by controlling the structures of MLPs, as explained in the following remark.
\begin{remark}
In Theorem \ref{Pro_2}, we can see that the degree of implicit smoothness (the constant $\delta$) is related to the Lipschitz constant $\kappa$ of the activation function and the upper bound $\eta$ of weight matrices and core tensor. Thus, we can control two variables in practice to balance the implicit smoothness:
\begin{itemize}
\item[1)] First, we use the sine function $\sigma(\cdot)=\sin(\omega_0 \cdot)$ as the nonlinear activation function in the MLPs. The sine function is Lipschitz continuous. An efficient way to adjust its Lipschitz constant $\kappa$ is to change the value of $\omega_0$, i.e., the smaller $\omega_0$ is, the smaller Lipschitz constant $\kappa$ can be got and the smoother result can be obtained.
\item[2)] Second, to control the upper bound $\eta$ of weight matrices and core tensor, we can tune the trade-off parameter of the energy regularization on MLP weights and the core tensor, known as the weight decay in modern deep learning optimizers. This strategy controls the intensity of $\eta$.
\end{itemize}
\end{remark}
Based on Theorem \ref{Pro_2}, we can deduce the following corollary concluding the smoothness of any sampled tensor on the continuous representation $f(\cdot)$.
\begin{corollary}\label{ineq}
Suppose the assumptions in Theorem \ref{Pro_2} hold. Define $f(\cdot):={\textbf{\rm [}}{\mathcal C};f_{\theta_x},f_{\theta_y},f_{\theta_z}{\textbf{\rm ]}}(\cdot)$. Then, for any ${\mathcal T}\in S[f]\cap{\mathbb R}^{n_1\times n_2\times n_3}$ (with $n_i$s being any positive numbers) sampled by coordinate vectors ${\bf x}\in{X_f}^{n_1}$, ${\bf y}\in{Y_f}^{n_2}$, and ${\bf z}\in{Z_f}^{n_3}$ (see Definition \ref{def_S}), the following inequalities hold for all ($i,j,k$)s $(i=1,2,\cdots,n_1,\;j=1,2,\cdots,n_2,\;k=1,2,\cdots,n_3)$:
    \begin{equation}
  \left\{
  \begin{aligned}
&|{\mathcal T}_{({\bf x}_{(i)},{\bf y}_{(j)},{\bf z}_{(k)})}-{\mathcal T}_{({\bf x}_{(i-1)},{\bf y}_{(j)},{\bf z}_{(k)})}|\leq \delta|{\bf x}_{(i)}-{\bf x}_{(i-1)}|\\
&|{\mathcal T}_{({\bf x}_{(i)},{\bf y}_{(j)},{\bf z}_{(k)})}-{\mathcal T}_{({\bf x}_{(i)},{\bf y}_{(j-1)},{\bf z}_{(k)})}|\leq \delta|{\bf y}_{(j)}-{\bf y}_{(j-1)}|\\
&|{\mathcal T}_{({\bf x}_{(i)},{\bf y}_{(j)},{\bf z}_{(k)})}-{\mathcal T}_{({\bf x}_{(i)},{\bf y}_{(j)},{\bf z}_{(k-1)})}|\leq \delta|{\bf z}_{(k)}-{\bf z}_{(k-1)}|,
  \end{aligned}
  \right.
    \end{equation}
    where $\delta=\eta^{3d+1}\kappa^{3d-3}{\tilde{\zeta}}^2$ and ${\tilde{\zeta}}=\max\{\lVert{\bf x}\rVert_\infty,\lVert{\bf y}\rVert_\infty,\lVert{\bf z}\rVert_\infty\}$.
\end{corollary}
Corollary \ref{ineq} claims that for any sampled tensor $\mathcal T\in S[f]$, the difference between its adjacent elements is bounded by the distance between adjacent coordinates up to a constant. Hence, our LRTFR implicitly and efficiently unifies the low-rankness and smoothness in all three dimensions, making it effective and efficient for continuous data representation.
\subsection{Advantages over Implicit Neural Representation}\label{advantage_INR}
We next discuss the relationships and advantages of our method over INR. The classical INRs (e.g., \cite{sine,DeepSDF,ICML_21,NIPS_Four}) learn an implicit function parameterized by deep neural networks $f_\theta(\cdot)$ to map a vector-form coordinate $(x,y,z)\in{\mathbb R}^3$ to the value of interest, i.e., $f_\theta(x,y,z)$. Our LRTFR (see illustrations in Fig. \ref{fig_1} (a)) maps some independent coordinates $x$, $y$, and $z$ to the corresponding value by disentangling the continuous representation into three simpler factor functions, implicitly encoding the low-rankness into the representation. Both INR and our method learn a continuous representation of data in the infinite real space. However, we additionally introduce the low-rank domain knowledge into the LRTFR, while INR-based methods mostly ignore the intrinsic structure of data. By virtue of the introduced domain knowledge, our LRTFR enjoys two intrinsic superiorities over INR.\par
      \begin{figure}[t]
          \scriptsize
          \setlength{\tabcolsep}{0.9pt}
          \begin{center}
          \begin{tabular}{cccc}
\includegraphics[width=0.111\textwidth]{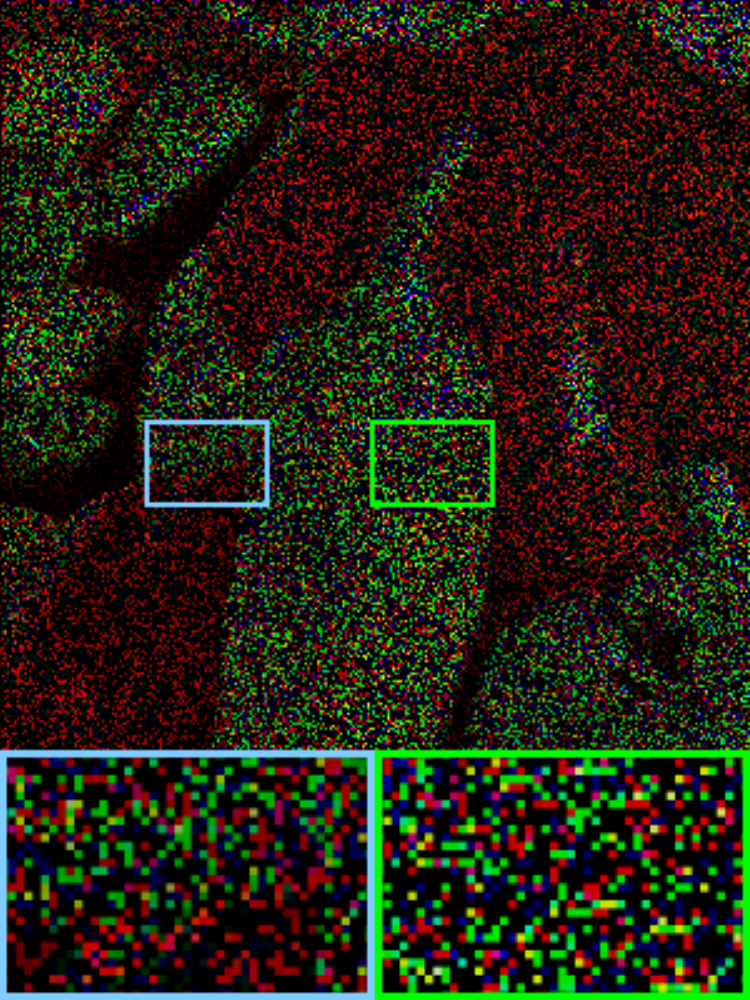}&
\includegraphics[width=0.111\textwidth]{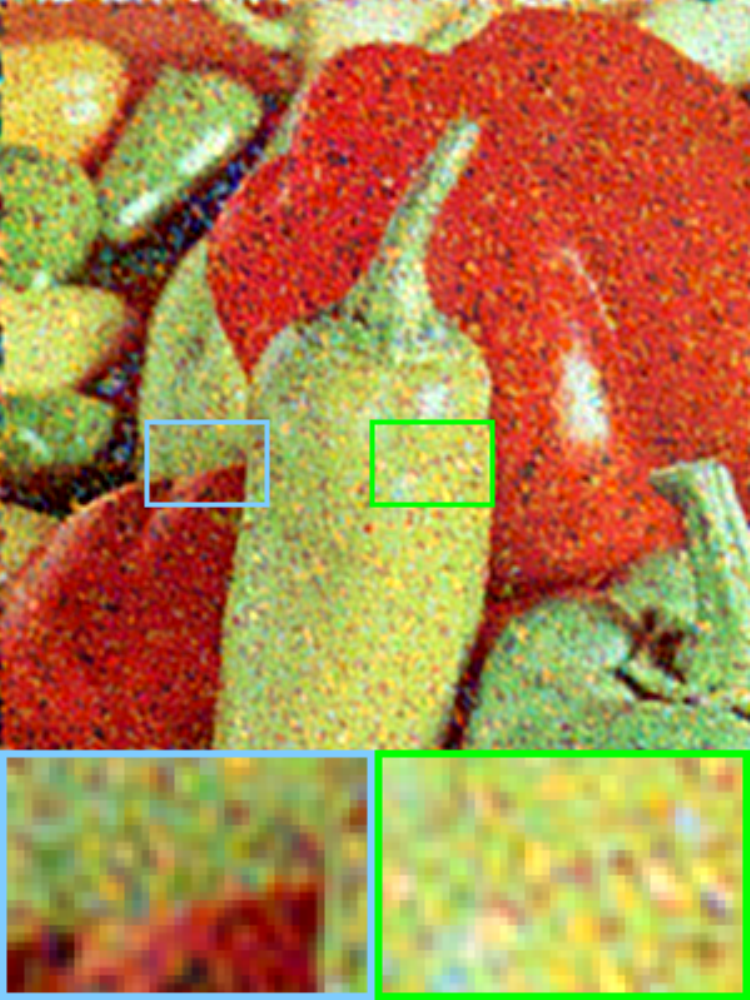}&
\includegraphics[width=0.111\textwidth]{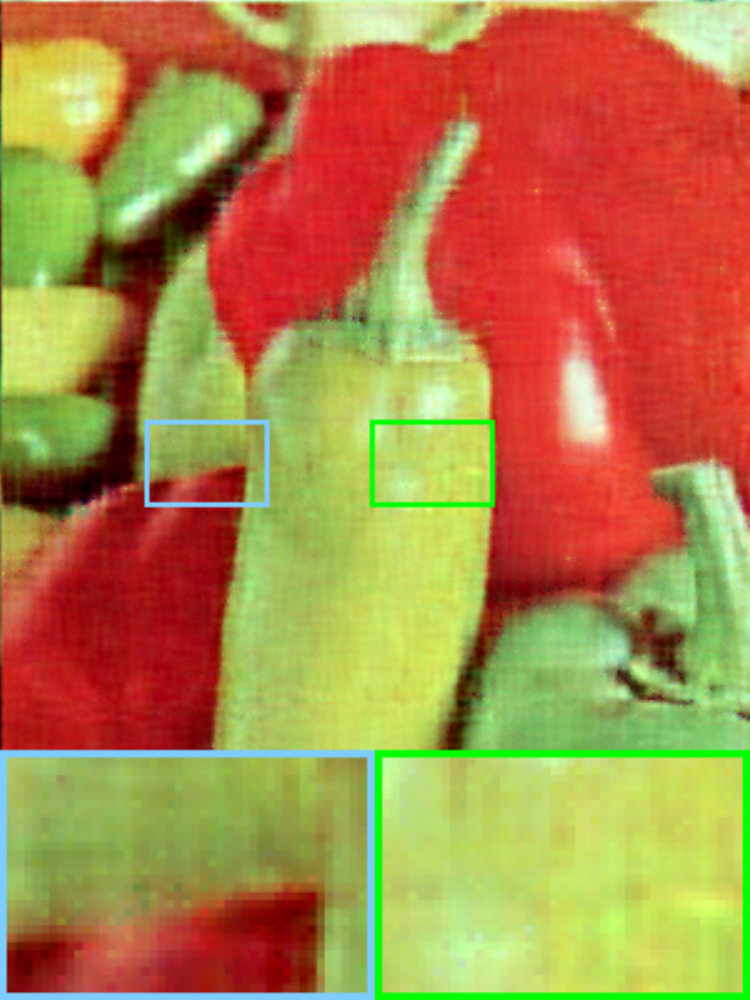}&
\includegraphics[width=0.111\textwidth]{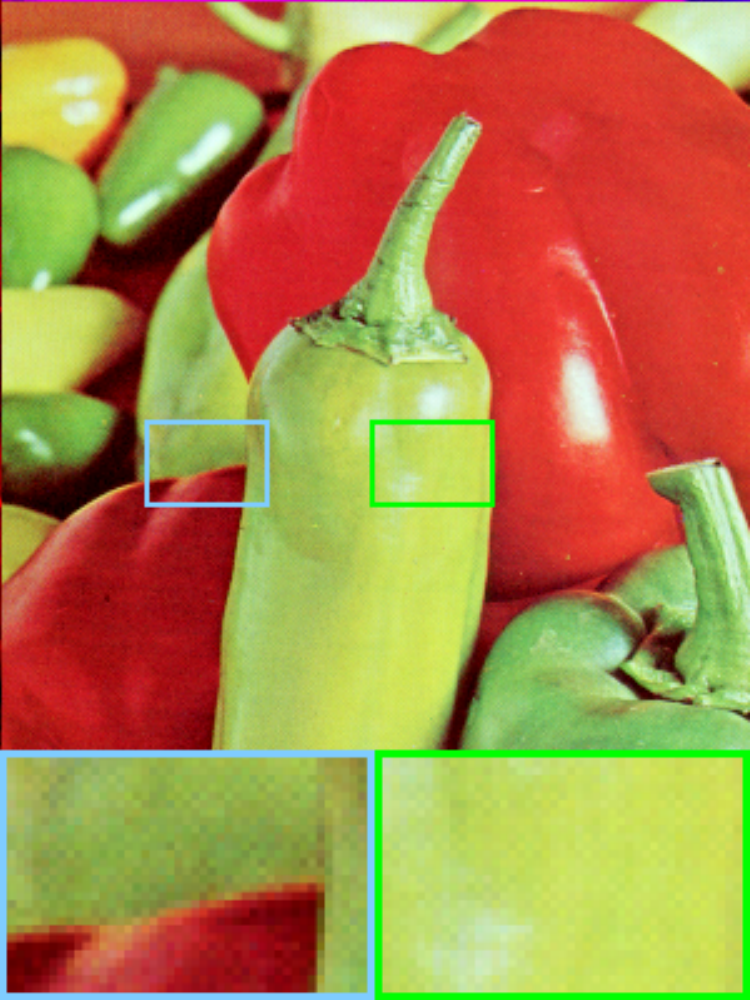}\\
PSNR 7.66&
PSNR 20.62&
PSNR 25.87&
PSNR Inf\\
Time \--\--&
Time 53.22 s&
Time 8.78 s&
Time \--\--\\
\includegraphics[width=0.111\textwidth]{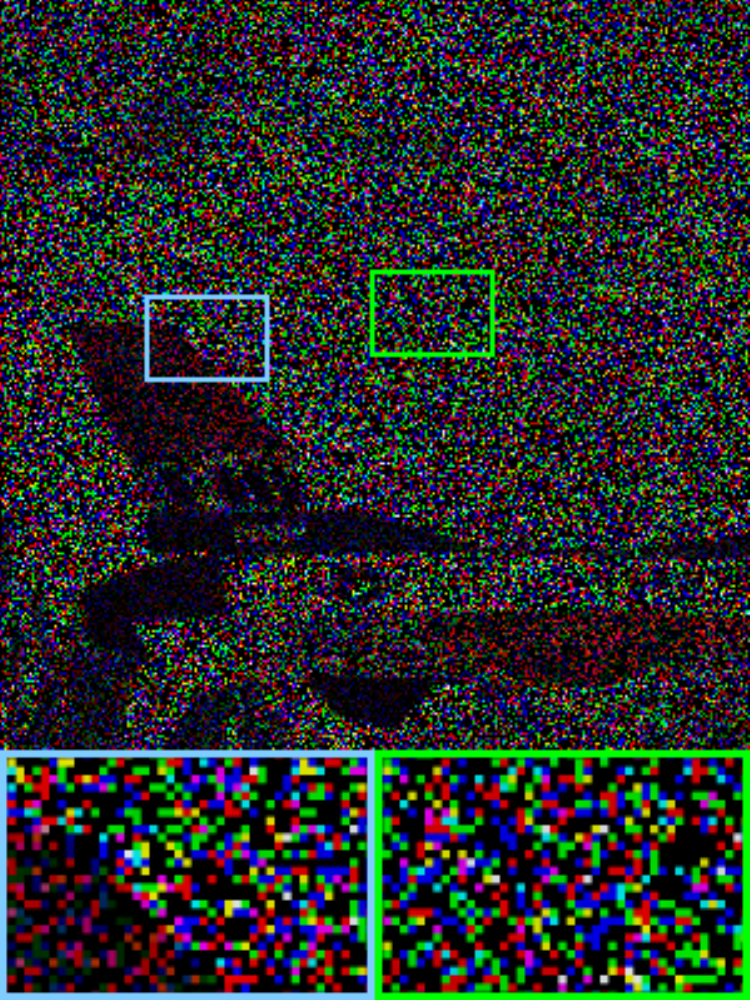}&
\includegraphics[width=0.111\textwidth]{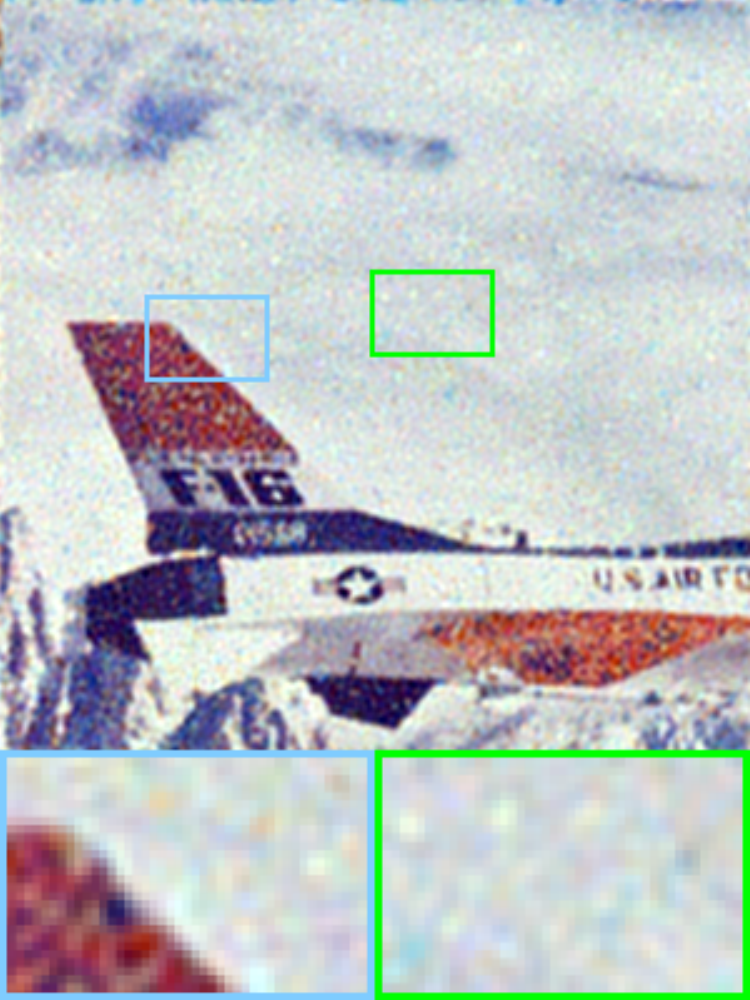}&
\includegraphics[width=0.111\textwidth]{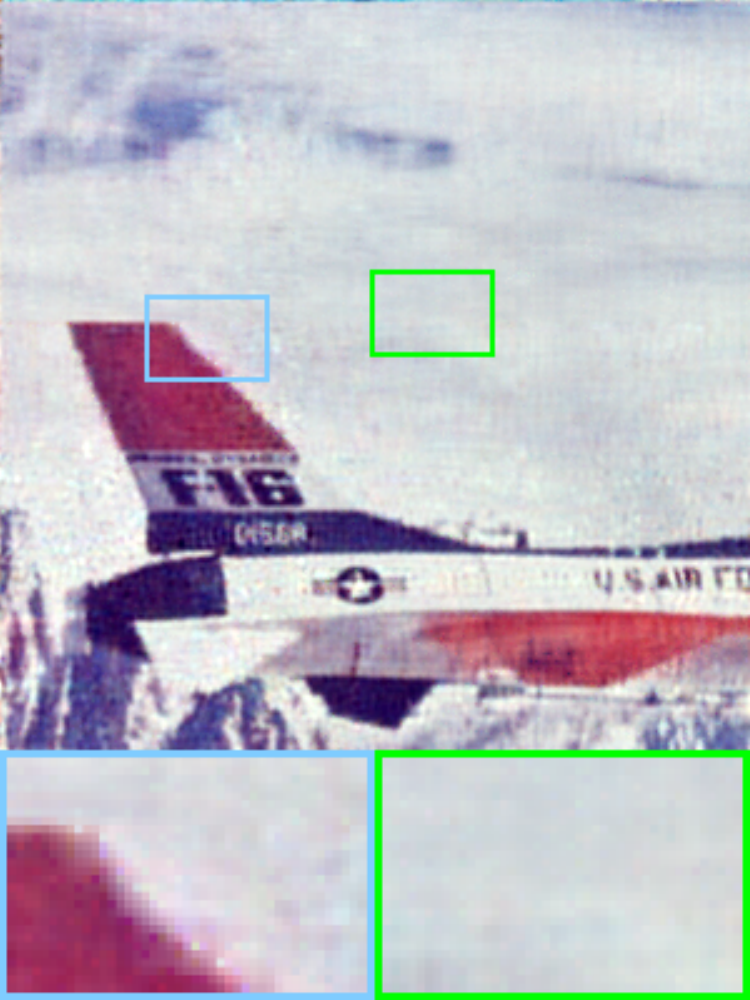}&
\includegraphics[width=0.111\textwidth]{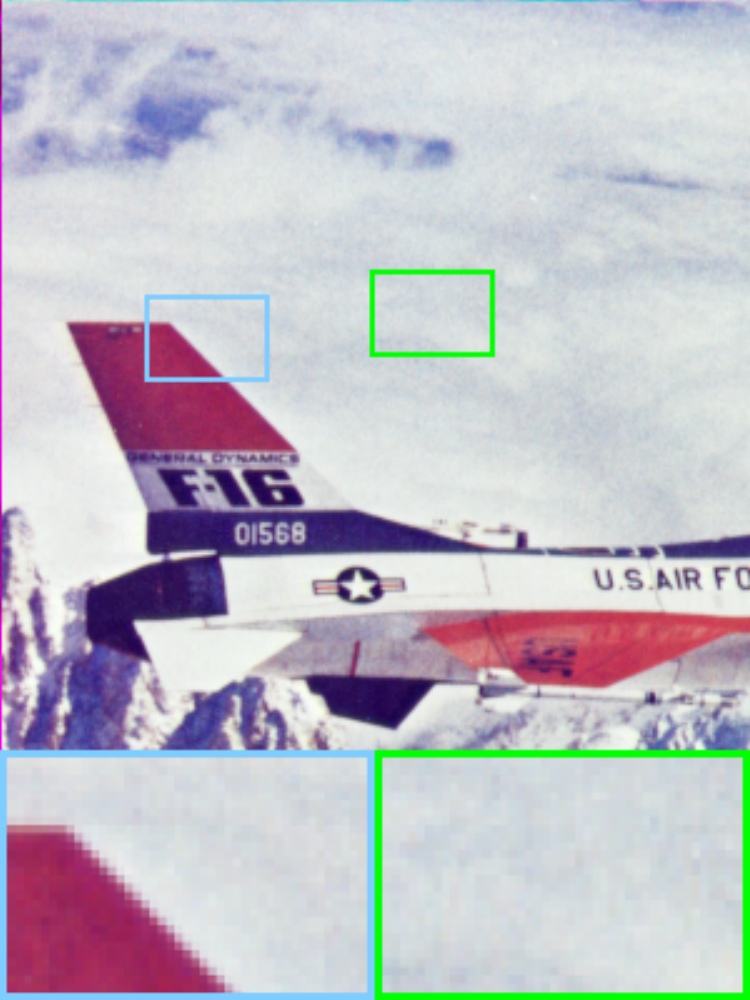}\\
PSNR 3.80 &
PSNR 24.42&
PSNR 28.32 &
PSNR Inf\\
Time \--\--&
Time 53.17 s&
Time 8.84 s&
Time \--\--\\
Observed&
INR\cite{sine}&
LRTFR&
Original\\
          \end{tabular}
          \end{center}
          \vspace{-0.4cm}
          \caption{The image inpainting results and the running time of INR and our LRTFR on {\it Peppers} and {\it Plane} with sampling rate 0.2. Our method can obtain a more stable continuous representation than INR with less running time.\label{fig_continuous_representation}\vspace{-0.5cm}}
          \end{figure}
First, our method is more stable since the regularization role imposed by low-rankness tends to help alleviate the overfitting issue of standard INR. As an example, we directly apply INR\cite{sine} and our LRTFR to the image inpainting task; see Fig. \ref{fig_continuous_representation}. Specifically, we directly use INR and LRTFR (with the same MLP structure) to fit the observed entries of the image and use the learned continuous representations to predict the unobserved entries. We have cropped the original image from $512\times512\times 3$ to $300\times300\times 3$ due to the large memory costs of INR. The results show that our method is able to obtain a more stable continuous representation according to the inpainting results. Specifically, the result of INR has notable noise due to overfitting. In contrast, our method is more stable and has cleaner results due to the low-rank regularization, which considerably alleviates the overfitting issue of INR. 
\par
Second, by virtue of tensor function factorization, the parameters of our LRTFR can be efficiently trained with significantly lower computational costs than INR. Given the observed data ${\mathcal O}\in{\mathbb R}^{n_1\times n_2\times n_3}$, INR\cite{sine} needs an input coordinate matrix of size $n_1n_2n_3\times 3$ to train the network, and the computational cost in each forward pass is $O(m^2dn_1n_2n_3)$, where $m$ denotes the number of hiden units of the MLP and $d$ denotes the depth. In practice, our GPU with 12 GB memory already fails to run the INR\cite{sine} with an image of size $512\times512\times3$ due to insufficient memory. In our LRTFR, the continuous representation (tensor function) is disentangled into three simpler factor functions, which could more compactly and efficiently represent the data. Our method only needs three input vectors of sizes $n_1\times1$, $n_2\times1$, and $n_3\times1$. The computational cost of our method is $O(m{\hat r}d(n_1+n_2+n_3)+{\hat r}n_1n_2n_3)$ in each forward pass, where ${\hat r}=\max\{r_1,r_2,r_3\}$ and $r_i$s ($i=1,2,3$) denote the preset F-ranks. This cost is much lower than that of INR, i.e., $O(m^2dn_1n_2n_3)$, since ${\hat r}$ is much smaller than $m^2d$ in practice. The running time comparisons in Fig. \ref{fig_continuous_representation} substantially verify the higher efficiency of our method.
\subsection{LRTFR for Multi-Dimensional Data Recovery}\label{Sec_app}
Since LRTFR can continuously represent data and capture the compact low-rank structure, it is versatile for data processing and analysis on or beyond meshgrid. In this work, we deploy the LRTFR in multi-dimensional data recovery problems to examinate its effectiveness. Specifically, we first establish a general data recovery model by using LRTFR and then more detailedly introduce four data recovery problems, including image inpainting and image denoising (on meshgrid with original resolution), hyperparameter optimization (on meshgrid beyond original resolution), and point cloud upsampling (beyond meshgrid).\par 
We remark that our LRTFR characterizes the low-rank structure of data and obtains a continuous representation beyond meshgrid. Thus, it is not limited to the above tasks. For other tasks on or beyond meshgrid (e.g., video frame interpolation and hyperspectral image fusion), with suitable formulations, our method is believed to perform well.
\subsubsection{General Data Recovery Model}
In this section, we introduce a general model for multi-dimensional data recovery by using LRTFR. Suppose that there is an observed multi-dimensional data defined in a function form $h(\cdot):D_h=X_h\times Y_h\times Z_h\rightarrow {\mathbb R}$ with $D_h\subset{\mathbb R}^3$ being a set with measure less than infinity. We assume that the underlying ``clean'' low-rank tensor function is continuous and has F-rank bounded by $(r_1;r_2;r_3)$. Using the MLP-parameterized LRTFR, we can formulate the following multi-dimensional data recovery model
\begin{equation}\label{min_f_x}
\min_{\substack{{\mathcal C}\in{\mathbb R}^{r_1\times r_2\times r_3},\\\theta_x,\theta_y,\theta_z}}\int_{D_h}|h({\bf v})-{{\textbf [}}{\mathcal C};f_{\theta_x},f_{\theta_y},f_{\theta_z}{\textbf ]}({\bf v})|^2d{\bf v},
\end{equation} 
where the object takes the Lebesgue integral over $D_h$, $f_{\theta_x}(\cdot):X_f\rightarrow {\mathbb R}^{r_1}$, $f_{\theta_y}(\cdot):Y_f\rightarrow {\mathbb R}^{r_2}$, and $f_{\theta_z}(\cdot):Z_f\rightarrow {\mathbb R}^{r_3}$ are three factor MLPs, and $\mathcal C$ is the core tensor. The recovered low-rank tensor function is $f(\cdot):={{\textbf [}}{\mathcal C};f_{\theta_x},f_{\theta_y},f_{\theta_z}{\textbf ]}({\cdot})$. According to Theorem \ref{Pro_1}, the real low-rank continuous tensor function must lies in the optimization space of (\ref{min_f_x}) and the recovered tensor function $f(\cdot)$ must be a low-rank tensor function with $({\rm F\text{-}rank}[f])_{(i)}\leq r_i$.\par 
The observed multi-dimensional data is usually in a discrete manner on meshgrid (e.g., images) or even not on meshgrid (e.g., point clouds). In both situations, the observed function $h(\cdot)$ is defined on a certain discrete set $D_h$. In this discrete case, our optimization model (\ref{min_f_x}) becomes
\begin{equation}\label{min_MLP}
\begin{split}
\min_{\substack{{\mathcal C}\in{\mathbb R}^{r_1\times r_2\times r_3},\\\theta_x,\theta_y,\theta_z}}\sum_{{\bf v}\in D_h}\big{|}h({\bf v})-{{\textbf [}}{\mathcal C};f_{\theta_x},f_{\theta_y},f_{\theta_z}{\textbf ]}({\bf v})\big{|}^2.
\end{split}
\end{equation}
Here, the optimization variables are the core tensor $\mathcal C$ and the weights of MLPs, and the objective function is a square error term, which is differentiable w.r.t. all MLP weights and the core tensor. Hence, we can use easily attachable deep learning optimizers based on the gradient descent to tackle model (\ref{min_MLP}). In this work, we consistently use the efficient adaptive moment estimation (Adam) algorithm\cite{Adam}. Next, we more detailedly introduce four applications in multi-dimensional data recovery, which are some specific examples of the general model (\ref{min_MLP}).
\subsubsection{Multi-Dimensional Image Inpainting}
The multi-dimensional image inpainting\cite{FTNN,FCTN}, which is a classical problem on the original meshgrid, aims to recover the underlying image from the observed incompleted image. Given an observed incompleted image ${\mathcal O}\in{\mathbb R}^{n_1\times n_2\times n_3}$ with observed set $\Omega\subset\Psi$, where $\Psi:=\{(i,j,k)|i=1,2,\cdots,n_1,j=1,2,\cdots,n_2,k=1,2,\cdots,n_3\}$, the optimization model of our LRTFR for multi-dimensional image inpainting is formulated as
\begin{equation}\label{loss_inpainting}
                  \begin{split}
                  &\min_{\substack{{\mathcal C},\theta_x,\theta_y,\theta_z}}\lVert {\mathcal P}_\Omega({\mathcal O}-{\mathcal T})\rVert_{F}^2,\\
                  &{\mathcal T}_{ijk} ={{\textbf [}}{\mathcal C};f_{\theta_x},f_{\theta_y},f_{\theta_z}{\textbf ]}(i,j,k),\;{\forall}\;(i,j,k)\in\Psi.\\
                  \end{split}
                  \end{equation}
Here, $f_{\theta_x}(\cdot):X_f\rightarrow {\mathbb R}^{r_1}$, $f_{\theta_y}(\cdot):Y_f\rightarrow {\mathbb R}^{r_2}$, and $f_{\theta_z}(\cdot):Z_f\rightarrow {\mathbb R}^{r_3}$ are three factor MLPs, ${\mathcal C}\in{\mathbb R}^{r_1\times r_2\times r_3}$ is the core tensor, and $r_i$s ($i=1,2,3$) are preset F-ranks. ${\mathcal P}_\Omega(\cdot)$ denotes the projection operator that keeps the elements in $\Omega$ and makes others be zeros. The recovered result is obtained by ${\mathcal P}_\Omega({\mathcal O})+{\mathcal P}_{\Omega^C}({\mathcal T})$, where $\Omega^C$ denotes the complementary set of $\Omega$ in $\Psi$. We adopt the Adam optimizer to address the image inpainting model (\ref{loss_inpainting}) by optimizing the MLP parameters and core tensor.
\subsubsection{Multispectral Image Denoising}
Multispectral image (MSI) denoising\cite{KBR,LRTDTV} aims to recover a clean image from the noisy observation. It is another typical problem on the original meshgrid. In practice, MSI is corrupted with mixed noise, such as Gaussian noise, sparse noise, stripe noise, and deadlines (missing columns). Given the observed noisy MSI ${\mathcal O}\in{\mathbb R}^{n_1\times n_2\times n_3}$, the optimization model of LRTFR for MSI denoising is
    \begin{equation}\label{loss_denoising}
                  \begin{split}
                  &\min_{\substack{{\mathcal C},{\mathcal S},\\\theta_x,\theta_y,\theta_z}}\lVert {\mathcal O}-{\mathcal T}-{\mathcal S}\rVert_{F}^2+\gamma_1\lVert{\mathcal S}\rVert_{\ell_1}+\gamma_2\lVert{\mathcal T}\rVert_{\rm TV},\\
                  &{\mathcal T}_{ijk} ={{\textbf [}}{\mathcal C};f_{\theta_x},f_{\theta_y},f_{\theta_z}{\textbf ]}(i,j,k),\;{\forall}\;(i,j,k)\in\Psi.\\
                  \end{split}
                  \end{equation}
In this model, $f_{\theta_x}(\cdot):X_f\rightarrow {\mathbb R}^{r_1}$, $f_{\theta_y}(\cdot):Y_f\rightarrow {\mathbb R}^{r_2}$, and $f_{\theta_z}(\cdot):Z_f\rightarrow {\mathbb R}^{r_3}$ are three factor MLPs, ${\mathcal C}\in{\mathbb R}^{r_1\times r_2\times r_3}$ is the core tensor, ${\mathcal S}\in{\mathbb R}^{n_1\times n_2\times n_3}$ represents the sparse noise, ${\mathcal T}\in{\mathbb R}^{n_1\times n_2\times n_3}$ denotes the recovered result, and $\gamma_i$s ($i=1,2$) are trade-off parameters. We introduce a simple spatial TV regularization $\lVert{\mathcal T}\rVert_{\rm TV}:=\sum_{k=1}^{n_3}(\sum_{i=1}^{n_1-1} \sum_{j=1}^{n_2}|{\mathcal T}_{(i+1,j,k)}-{\mathcal T}_{(i,j,k)}|+\sum_{i=1}^{n_1} \sum_{j=1}^{n_2-1}|{\mathcal T}_{(i,j+1,k)}-{\mathcal T}_{(i,j,k)}|)$ to more faithfully remove the noise. Here, the Lipchitz smoothness in Theorem \ref{Pro_2} and the TV regularization reveal different types of smoothness: The Lipchitz smoothness delivers {\sl global} smooth structures that the gradient is bounded everywhere, while TV considers {\sl local} smoothness of adjacent pixels. The global and local smooth regularizations are complementary to each other to yield more promising denoising results.\par 
We utilize the alternating minimization algorithm to tackle the denoising model. Specifically, we tackle the following sub-problems in the $t$-th iteration:
\begin{equation}
\begin{split}
&\min_{{\mathcal C},\theta_x,\theta_y,\theta_z}\lVert {\mathcal O}-{\mathcal T}-{\mathcal S}^t\rVert_{F}^2+\gamma_2\lVert{\mathcal T}\rVert_{\rm TV},
\\&
\;\;\;\min_{{\mathcal S}}\;\;\;\lVert {\mathcal O}^t-{\mathcal T}^t-{\mathcal S}\rVert_{F}^2+\gamma_1\lVert{\mathcal S}\rVert_{\ell_1},\\
&{\mathcal T}_{ijk} ={{\textbf [}}{\mathcal C};f_{\theta_x},f_{\theta_y},f_{\theta_z}{\textbf ]}(i,j,k),\;{\forall}\;(i,j,k)\in\Psi.
\end{split}
\end{equation}
We utilize the Adam algorithm to tackle the $\{{\mathcal C},\theta_x,\theta_y,\theta_z\}$ sub-problem. In each iteration of the alternating minimization, we employ one step of the Adam algorithm to update $\{{\mathcal C},\theta_x,\theta_y,\theta_z\}$. The ${\mathcal S}$ sub-problem can be exactly solved by ${\mathcal S}={Soft}_{\frac{\gamma}{2}}({\mathcal O}^t-{\mathcal T}^t)$, where $Soft_{v}(\cdot):=sgn(\cdot)\max\{|\cdot|-v,0\}$ denotes the soft-thersholding operator applied on each element of the input.    
\subsubsection{Hyperparameter Optimization}
Hyperparameter optimization (HPO) \cite{HPO_PAMI,RS}, or hyperparameters searching, is a critical step in machine learning. Recent studies \cite{HPO_PAMI} cleverly modeled HPO as the classical low-rank tensor completion problem. It is even more interesting to explore the benefits of searching hyperparameter values in the continuous domain, and investigate what is the superiority of continuous representation over classical discrete tensor completion methods in HPO.\par
Specifically, the tensor completion-based HPO \cite{HPO_PAMI} formulates the hyperparameter search as the completion problem of higher-order tensors, whose elements represent the performances of the algorithm with different hyperparameter values. In practice, some partial observations have been given, i.e., there is an incompleted tensor ${\mathcal O}\in{\mathbb R}^{n_1\times n_2\times n_3}$ that gives the real performances under some configurations. We aim to complete the tensor to predict the performances among all configurations. This problem can be equivalently formulated as the tensor completion problem (\ref{loss_inpainting}). With the completion result $\mathcal T$, we select the configuration corresponding to the best predicted performance (the maximum value in $\mathcal T$) as the recommended hyperparameter values.\par
Since our method predicts a tensor function on a continuous domain, it is interesting to explore the benefits of seeking hyperparameter values beyond meshgrid. Intuitively, seeking hyperparameter values in a continuous domain is expected to obtain a better result, because the optimal configuration probably does not lie in the fixed meshgrid. To illustrate this claim, we sample $\times$2 and $\times$4 super-resolution results using the learned tensor function with evenly spaced sampling, which offers more candidate configurations and gives the corresponding predictions. Results show this strategy could suggest a better configuration than meshgrid methods in most cases; see Sec. \ref{sec_HPO}.
\subsubsection{Point Cloud Upsampling}
We further apply our LRTFR for point cloud representation to test the effectiveness of our method beyond meshgrid. Point cloud representation is a challenging task due to the unstructured and unordered nature of point clouds. It is difficult to use classical meshgrid-based low-rank representations\cite{PAMI_TRPCA,FTNN,HLRTF} for point cloud representation, since these unordered point clouds are defined beyond meshgrid. As compared, our LRTFR can represent point clouds by using the continuous representation, which validates its versatility as compared with classical low-rank representations.\par 
Specifically, we consider the point cloud upsampling task\cite{PCU_CVPR_18,PU-GAN}, which refers to upsampling a sparse point cloud into a dense point cloud that will benefit subsequent applications\cite{PCU_CVPR_21}. Suppose that we are given a sparse point cloud ${\bf P}\in{\mathbb R}^{p\times 3}$, where $p$ denotes the number of points. We use $\Omega=\{{\bf P}_{(m,:)}\}_{m=1}^p$ to denote the observed set. We borrow the signed distance function (SDF) \cite{DeepSDF} to learn a continuous representation from sparse point cloud. The loss function to learn the SDF is formulated as
    \begin{equation}\label{loss_point}
                         \begin{split}
                         \min_{{\mathcal C},\theta_{x},\theta_{y},\theta_{z}}\sum_{{{\bf v}}\in\Omega}|s({\bf v})|&+\gamma_1\int_{{\mathbb R}^3}\big{|}\lVert\frac{\partial s({\bf v})}{\partial ({\bf v})}\rVert_F^2-1\big{|}d{\bf v}\\&+\gamma_2\int_{{\mathbb R}^3\backslash \Omega}\exp(-|s({\bf v})|)d{\bf v},\\
                         s({\bf v}):={\mathcal C}{\rm\times_1}{f}_{\theta_{x}}(&{\bf v}_{(1)}){\rm\times_2}{f}_{\theta_{y}}({\bf v}_{(2)}){\rm\times_3}{f}_{\theta_{z}}({\bf v}_{(3)}),
                         \end{split}
                         \end{equation}
where $s(\cdot):{\mathbb R}^3{\rm\rightarrow}{\mathbb R}$ denotes the SDF represented by LRTFR, and $\gamma_i$s ($i=1,2$) are trade-off parameters. The first term enforces the SDF to be zero on observed points. The second term enforces the SDF gradients to be one everywhere. The third term restricts the SDF value to be far from zero outside the observed set\cite{sine}. In practice, we approximate the integral by randomly sampling large number of points in the space. The loss function can be minimized by using the Adam algorithm. The solution of $s({\bf v})=0$ forms a surface, which represents the underlying shape of the point cloud. We sample dense points $\bf v$ in the space such that $|s({\bf v})|<\tau$ by evenly spaced sampling, where $\tau$ is a pre-defined threshold. These points represent the desired upsampling result.
\section{Experiments}\label{Sec_exp}
In experiments, we have conducted comparison experiments and analysis on all the introduced tasks. We first introduce some important experimental settings. Then, we introduce baselines, datasets, and results for different tasks. Our method is implemented on Pytorch 1.7.0. with an i5-9400f CPU and an RTX 3060 GPU (12 GB GPU memory). \par 
{\bf Evaluation metrics} For inpainting and denoising, we use peak-signal-to-noisy ratio (PSNR), structural similarity (SSIM), and normalized root mean square error (NRMSE) for evaluations. For HPO, we report the average classification accuracy (ACA) and average recommendation accuracy (ARA)\cite{PR_HPO,HPO_PAMI} of different methods. For point cloud upsampling, we adopt the widely-used Chamfer distance (CD)\cite{CD} and F-Score\cite{F_Score} as evaluation metrics.\par 
{\bf Hyperparameters settings} For all tasks, we search the values of {\rm F\text{-}rank} $(r_1,r_2,r_3)$ in the following set
\begin{equation}
\{(\lfloor n_1/s \rfloor,\lfloor n_2/s\rfloor,\lfloor n_3/s_3\rfloor)| s,s_3=1,2,4,8,16,32\},
\end{equation}
where $n_i$s ($i=1,2,3$) denote the sizes of the observed data. Meanwhile, we adopt the sine activation function $\sigma(\cdot)=\sin(\omega_0\cdot)$ in the MLPs to learn the LRTFR, where $\omega_0$ is a hyperparameter. It was thoroughly demonstrated in literatures\cite{sine,piGAN,GAN_CVPR} that the sine function has superior representation abilities than other activations (e.g., ReLU and RBF) for continuous representation. We search the hyperparameter $\omega_0$ in $\{1,2,4,8,16,32\}$. The goal of the above hyperparameters search is to obtain the best PSNR (for inpainting and denoising), ACA (for HPO), and CD (for point cloud upsampling) values for different samples. The weight decay of Adam is set to $1$, $0.1$, $0.5$, and $0.5$ for inpainting, denoising, HPO, and point cloud upsampling. In the denoising model (\ref{loss_denoising}), we set $\gamma_2=10^{-5}$ for MSIs and $\gamma_2=10^{-6}$ for hyperspectral images. $\gamma_1$ is set to $0.1$ (with sparse noise) or $10$ (without sparse noise). In the point cloud upsampling model (\ref{loss_point}), we set $\gamma_1=10^{-6}$ and $\gamma_2=10^{-2}$ for all samples. The threshold $\tau$ is tuned such that the recovered dense point cloud has at least $10^4$ points. The depth of MLP is set to 4 for point cloud upsampling and 3 for other tasks. We analyze the influences of all hyperparameters in Sec. \ref{Sec_dis}.
\begin{figure*}[t]
    \scriptsize
    \setlength{\tabcolsep}{0.9pt}
    \begin{center}
    \begin{tabular}{cccccccc}
     \includegraphics[width=0.12\textwidth]{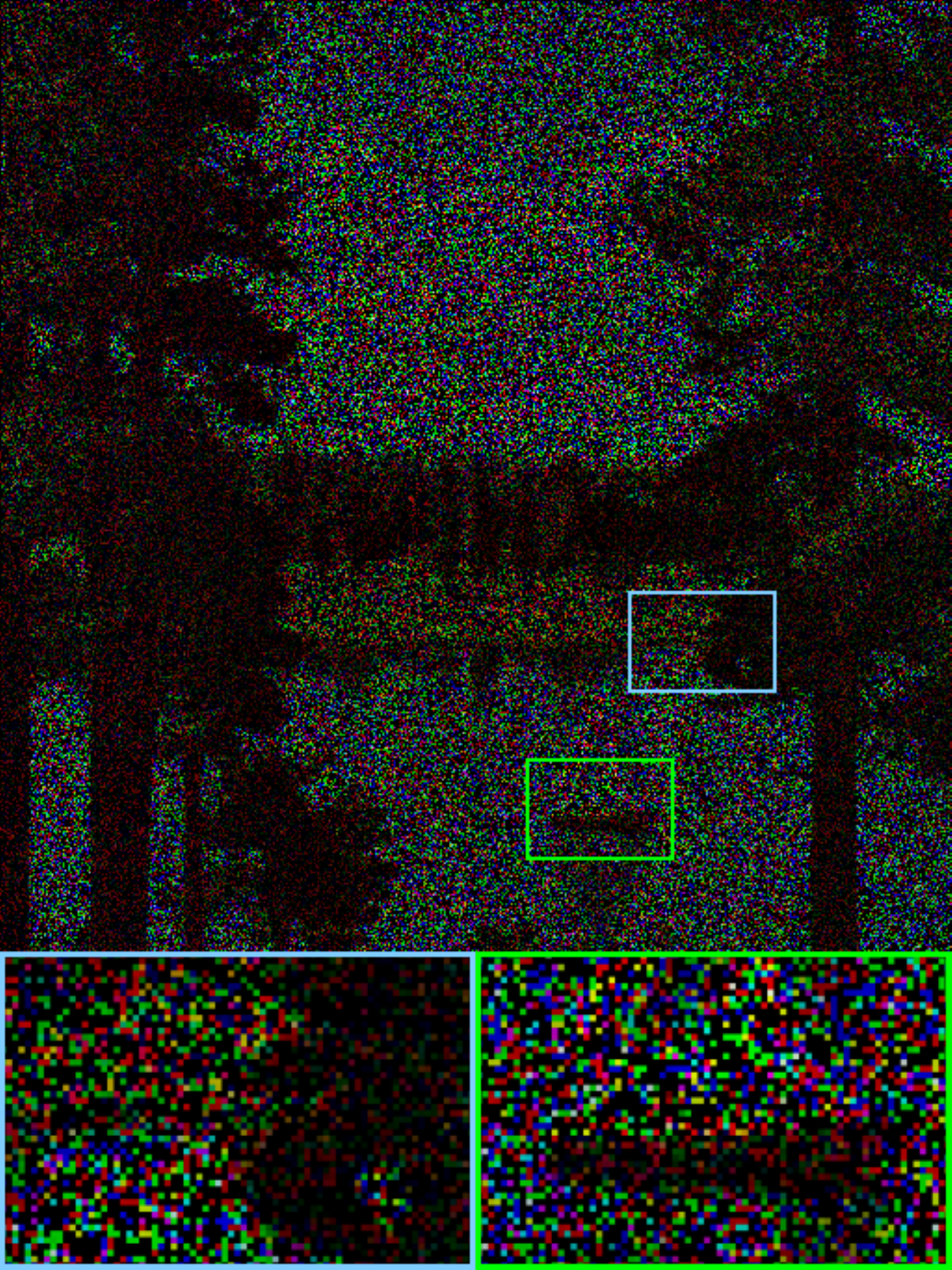}&
    \includegraphics[width=0.12\textwidth]{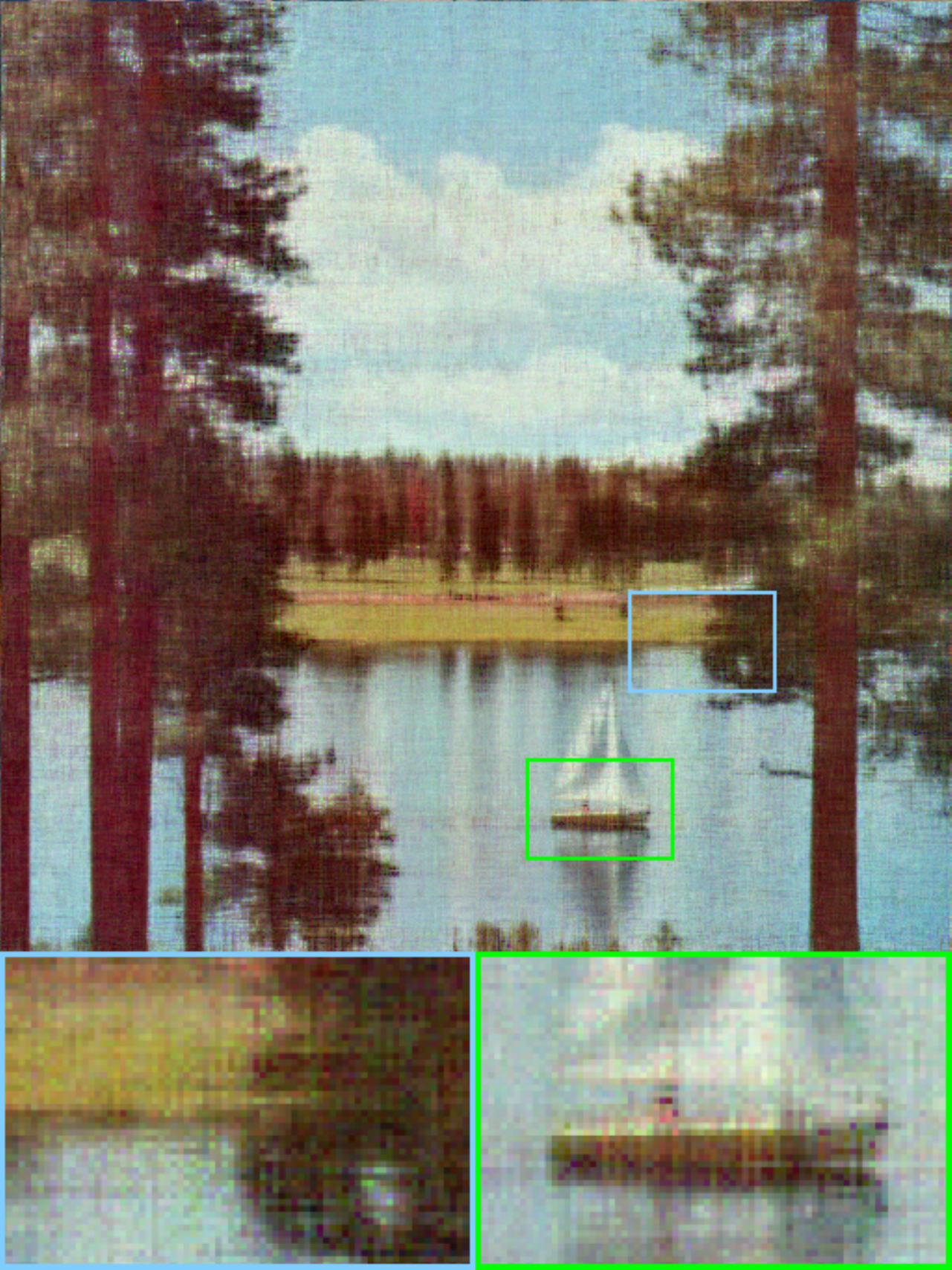}&
     \includegraphics[width=0.12\textwidth]{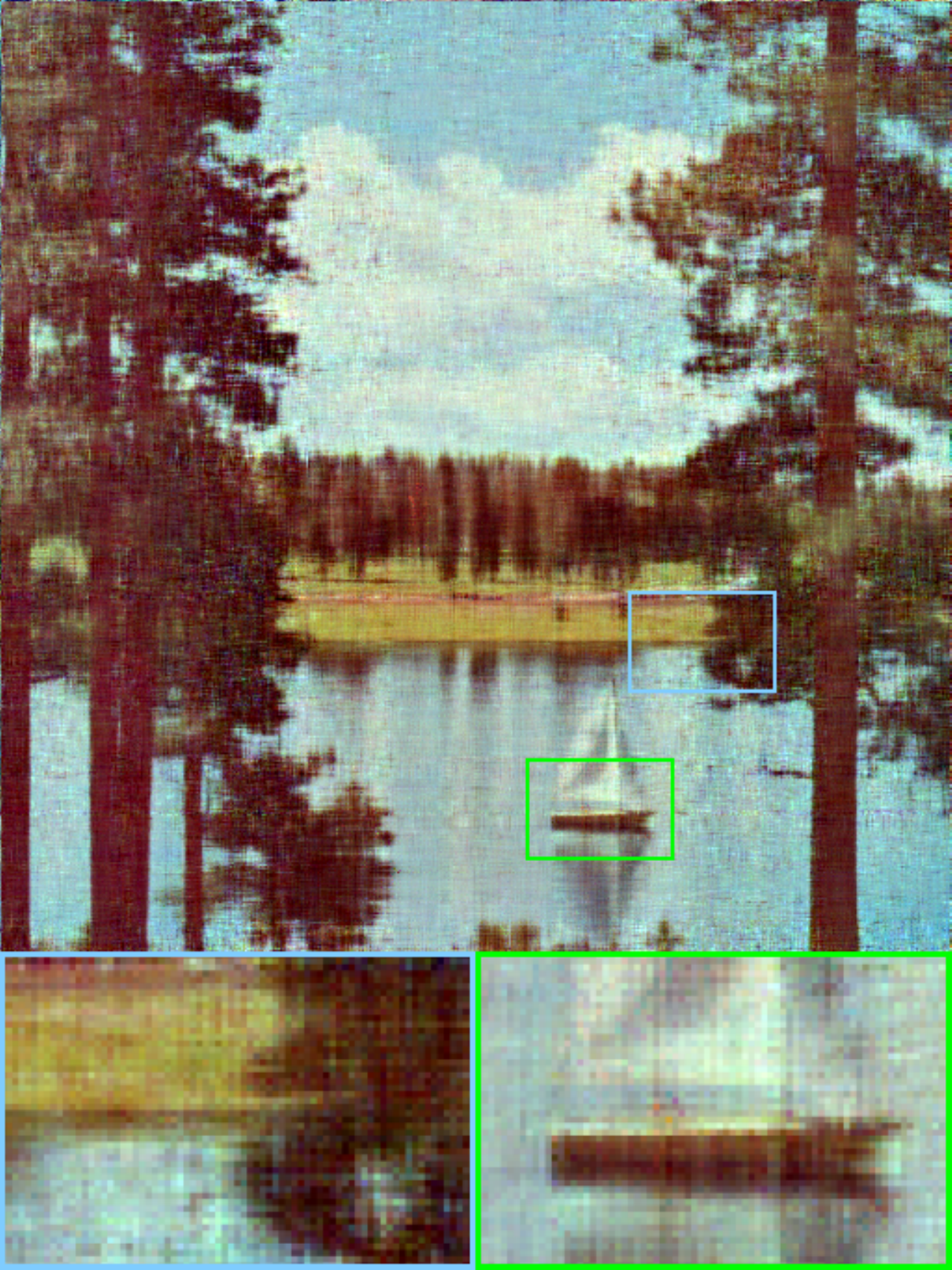}&
      \includegraphics[width=0.12\textwidth]{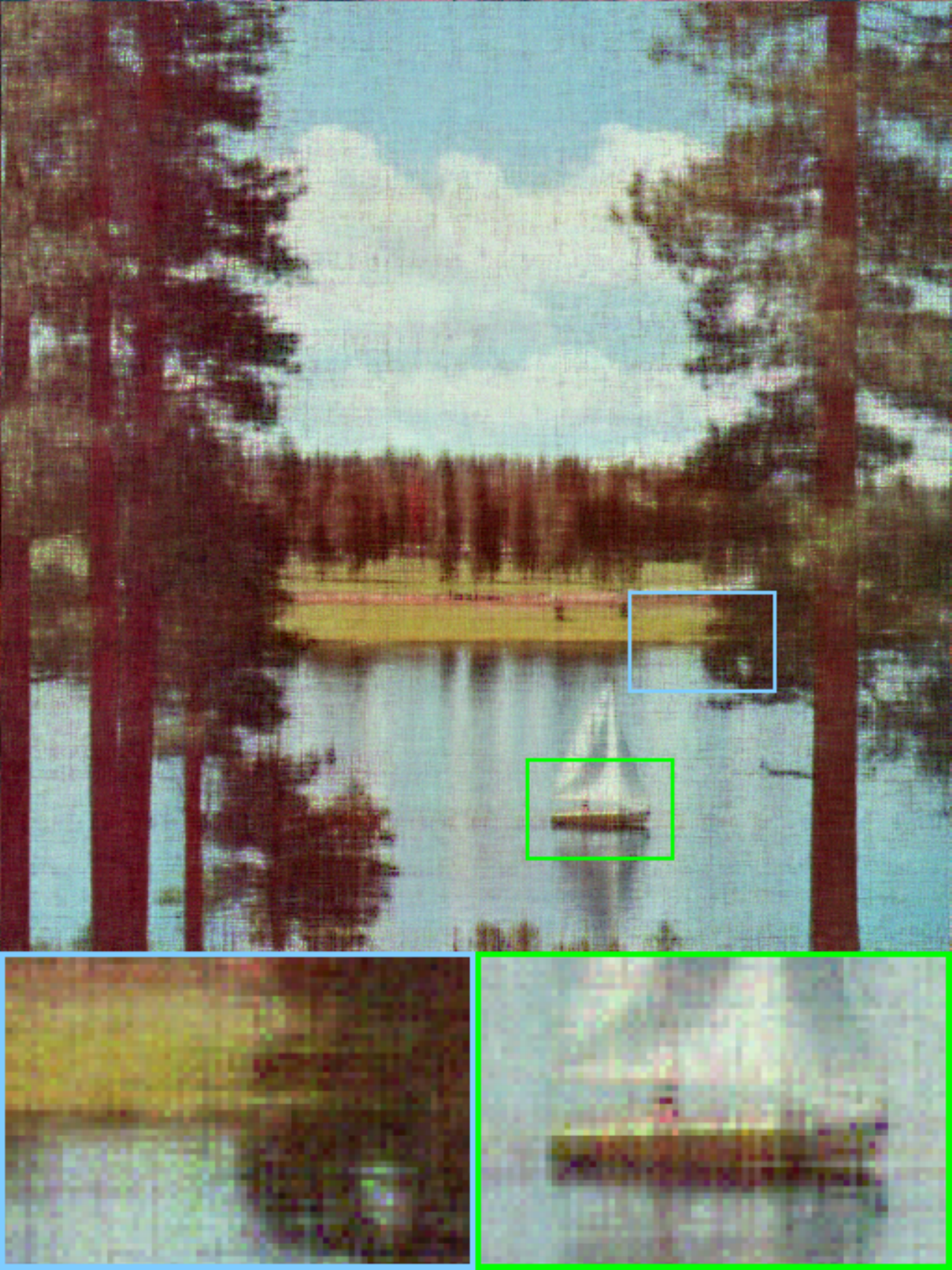}&
       \includegraphics[width=0.12\textwidth]{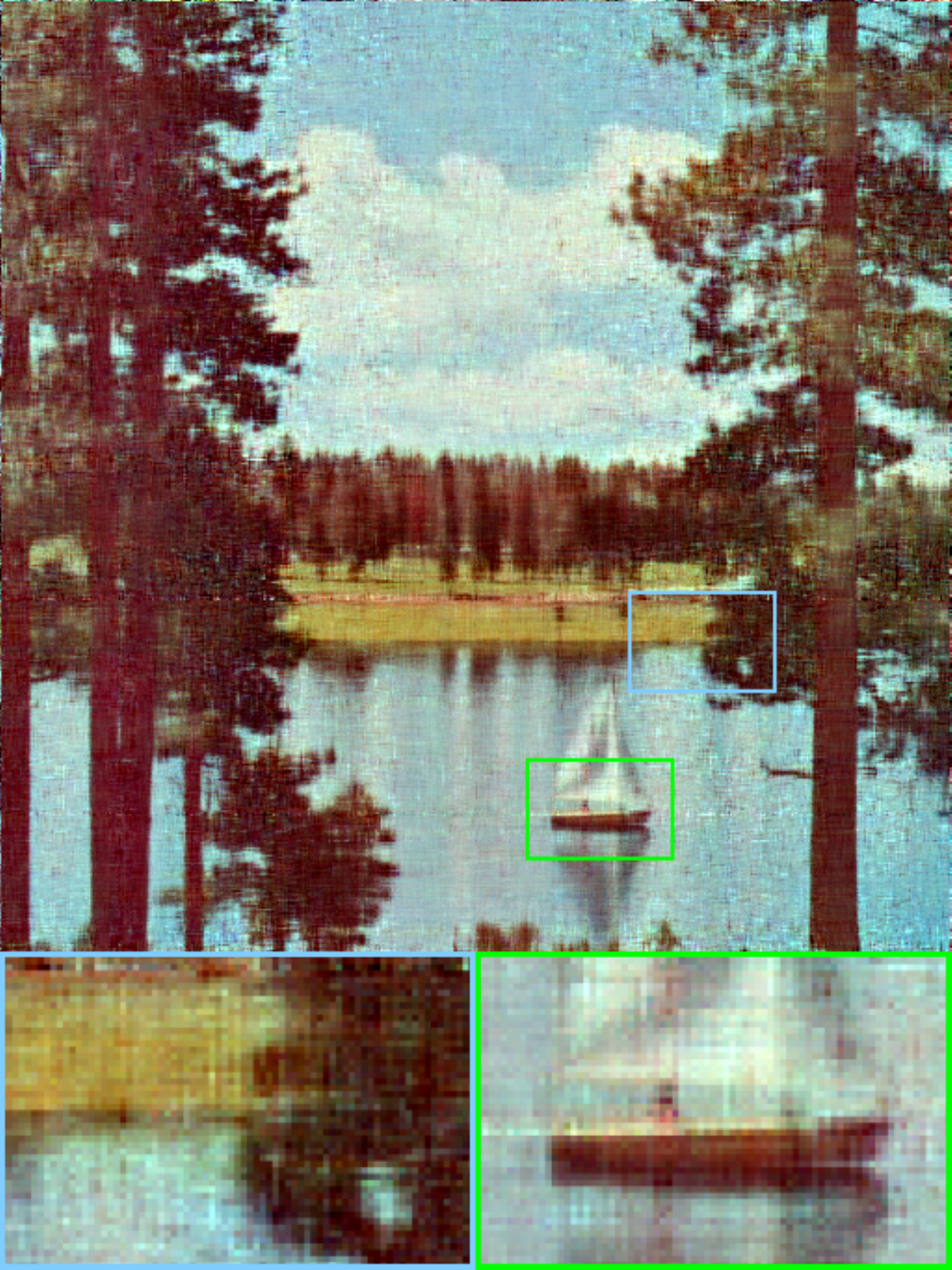}&
        \includegraphics[width=0.12\textwidth]{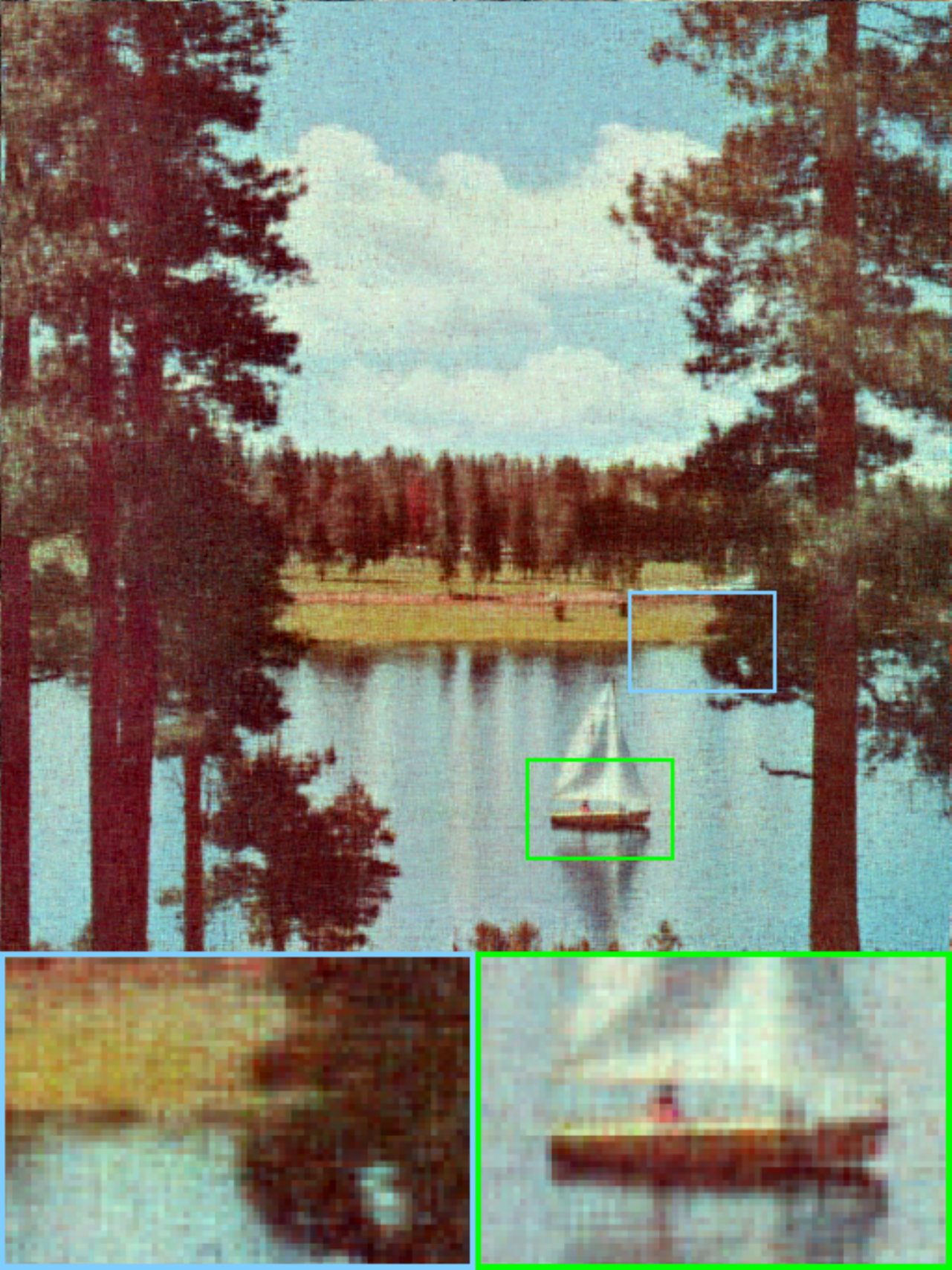}&
         \includegraphics[width=0.12\textwidth]{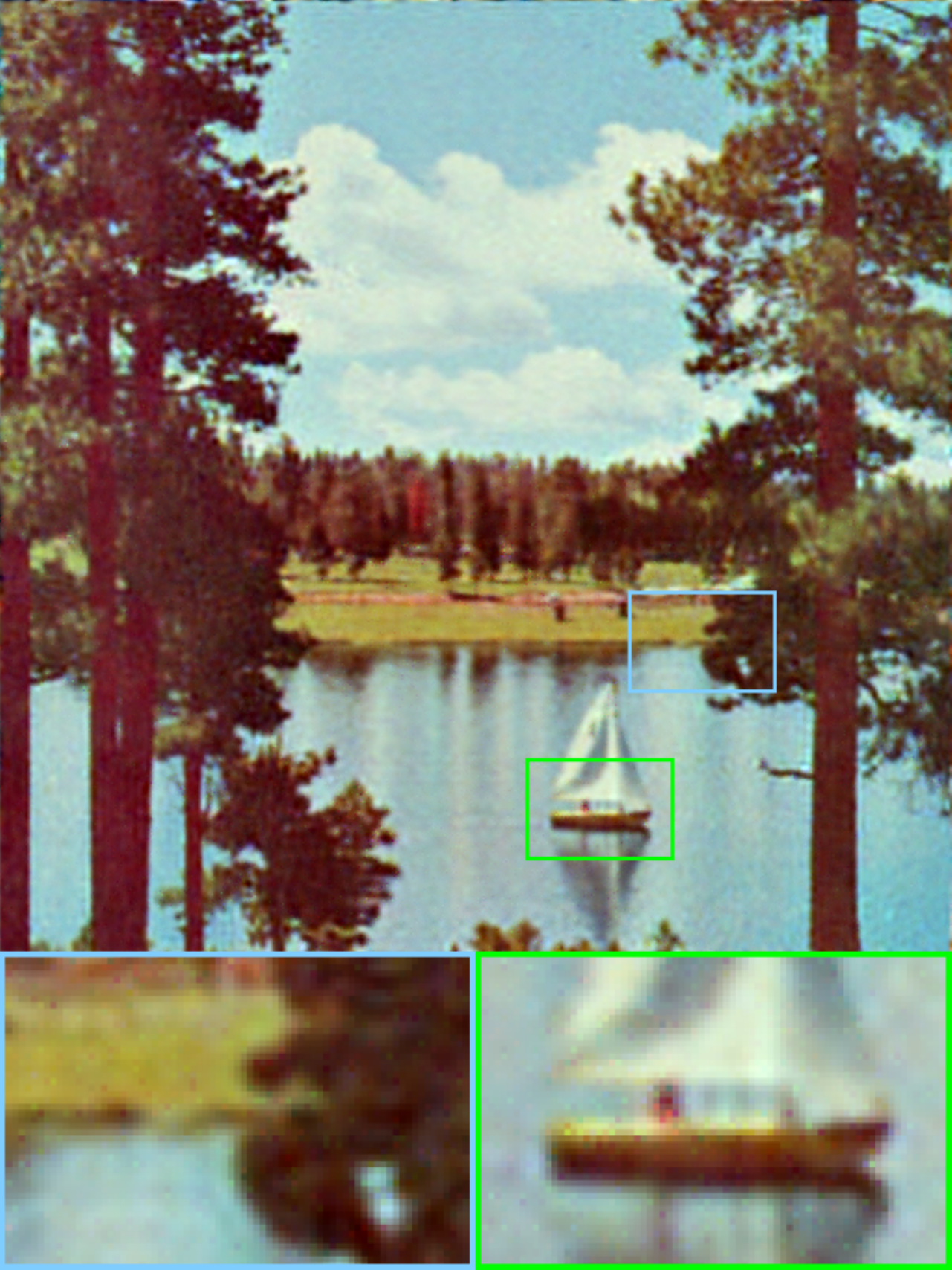}&
    \includegraphics[width=0.12\textwidth]{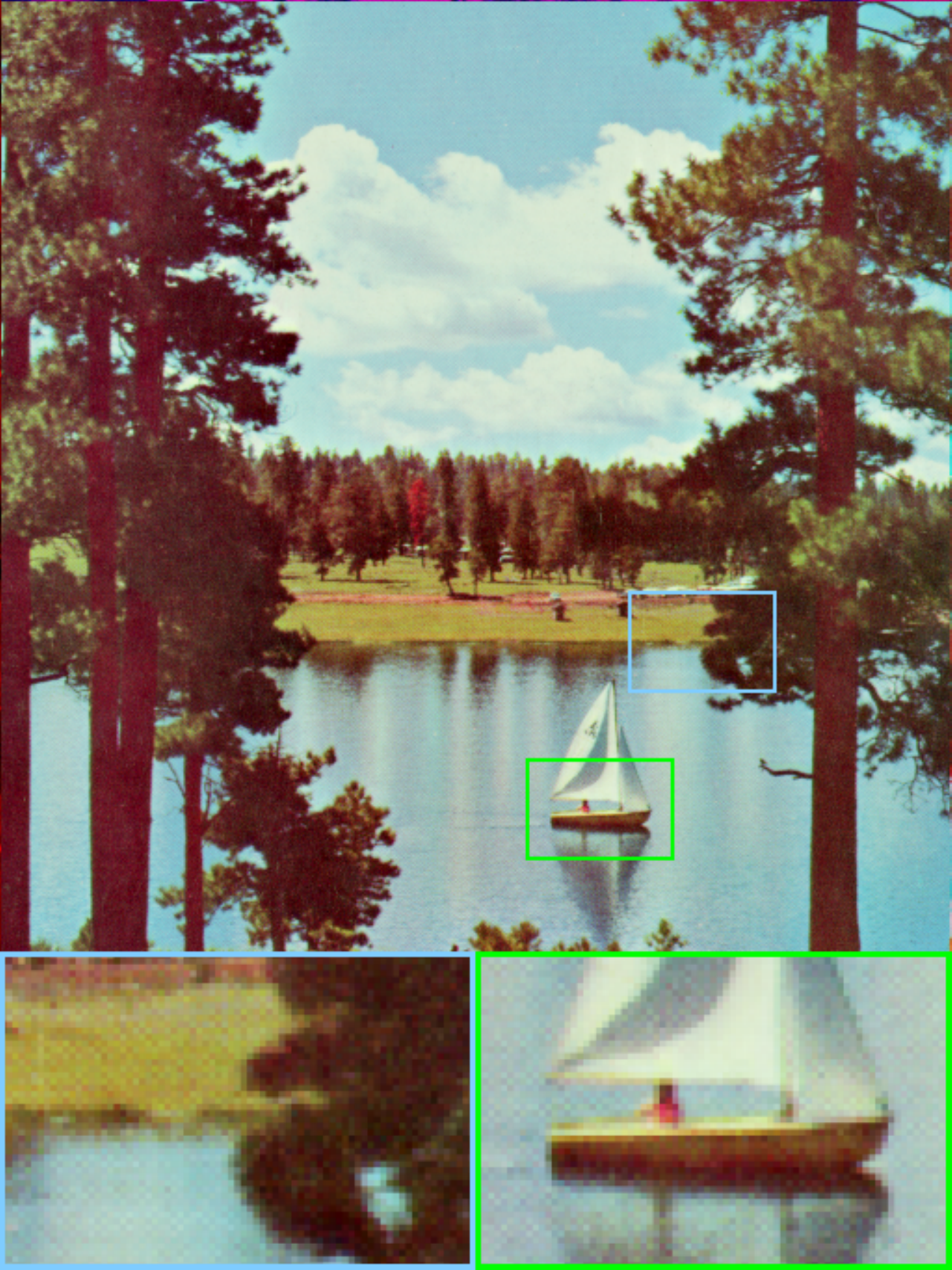}\\
PSNR 6.13  &
PSNR 22.18  &
PSNR 22.06  &
PSNR 22.40  &
PSNR 22.26  &
PSNR 24.35  &
PSNR 25.87  &
PSNR Inf\\
     \includegraphics[width=0.12\textwidth]{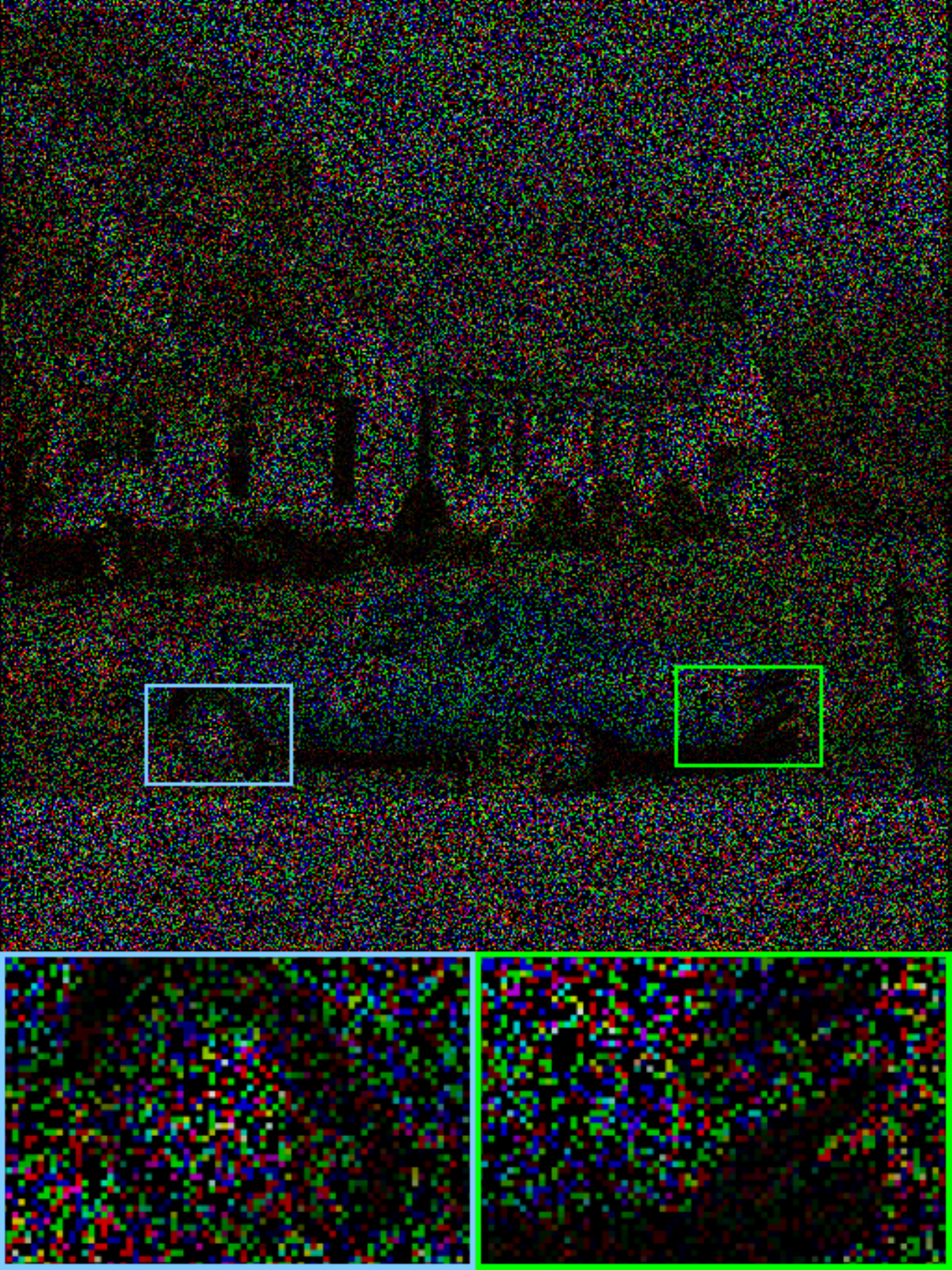}&
    \includegraphics[width=0.12\textwidth]{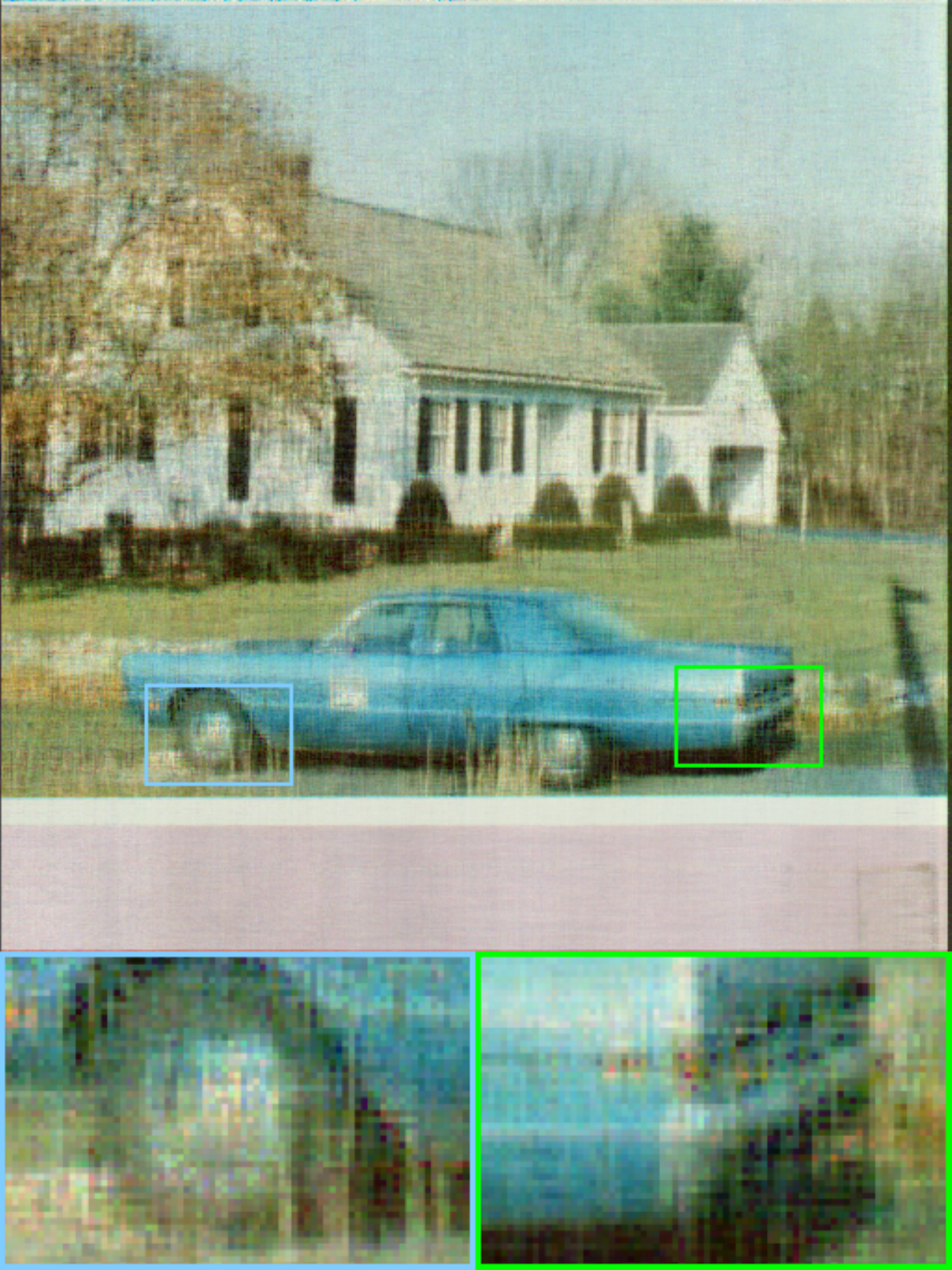}&
     \includegraphics[width=0.12\textwidth]{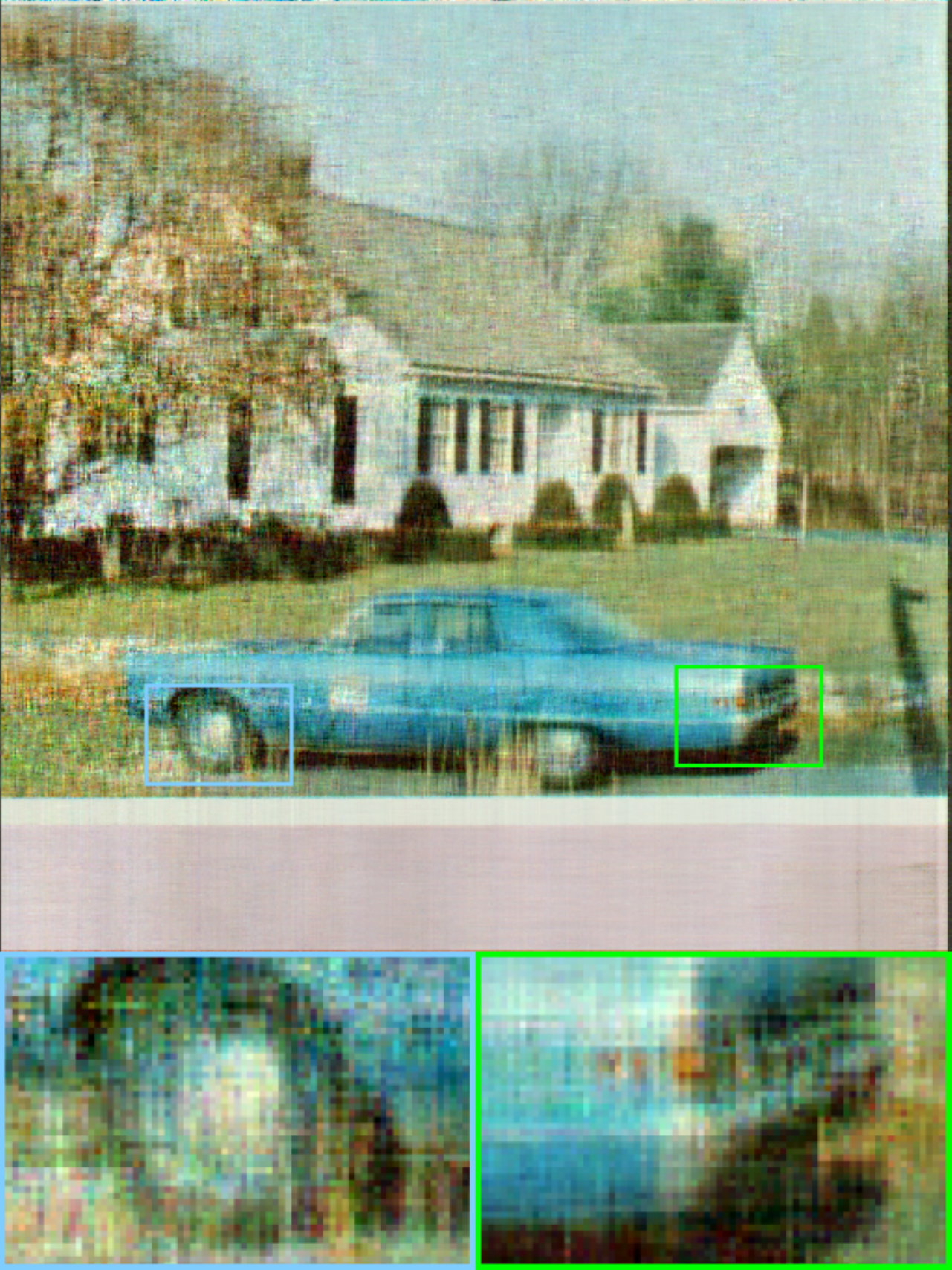}&
      \includegraphics[width=0.12\textwidth]{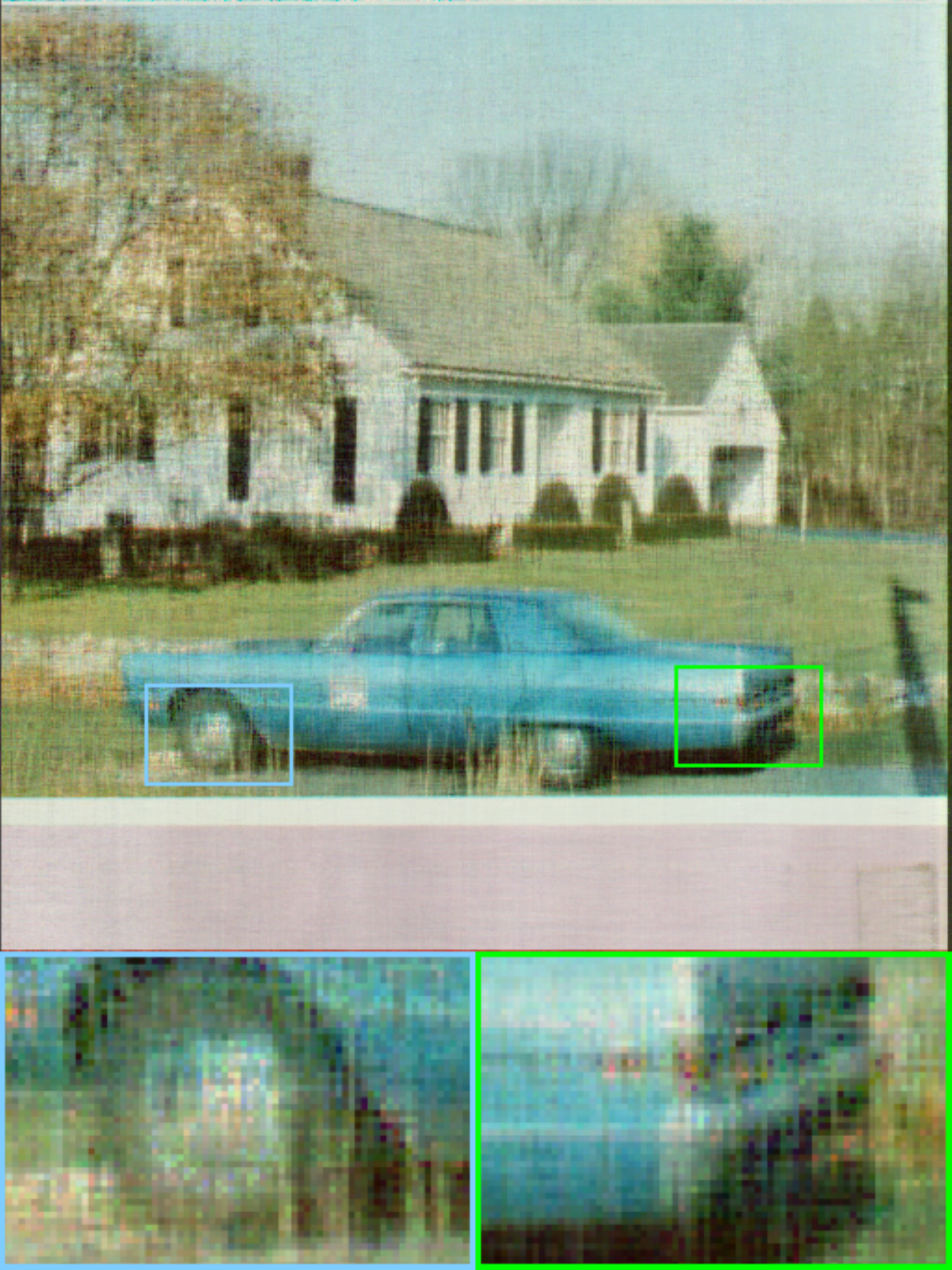}&
       \includegraphics[width=0.12\textwidth]{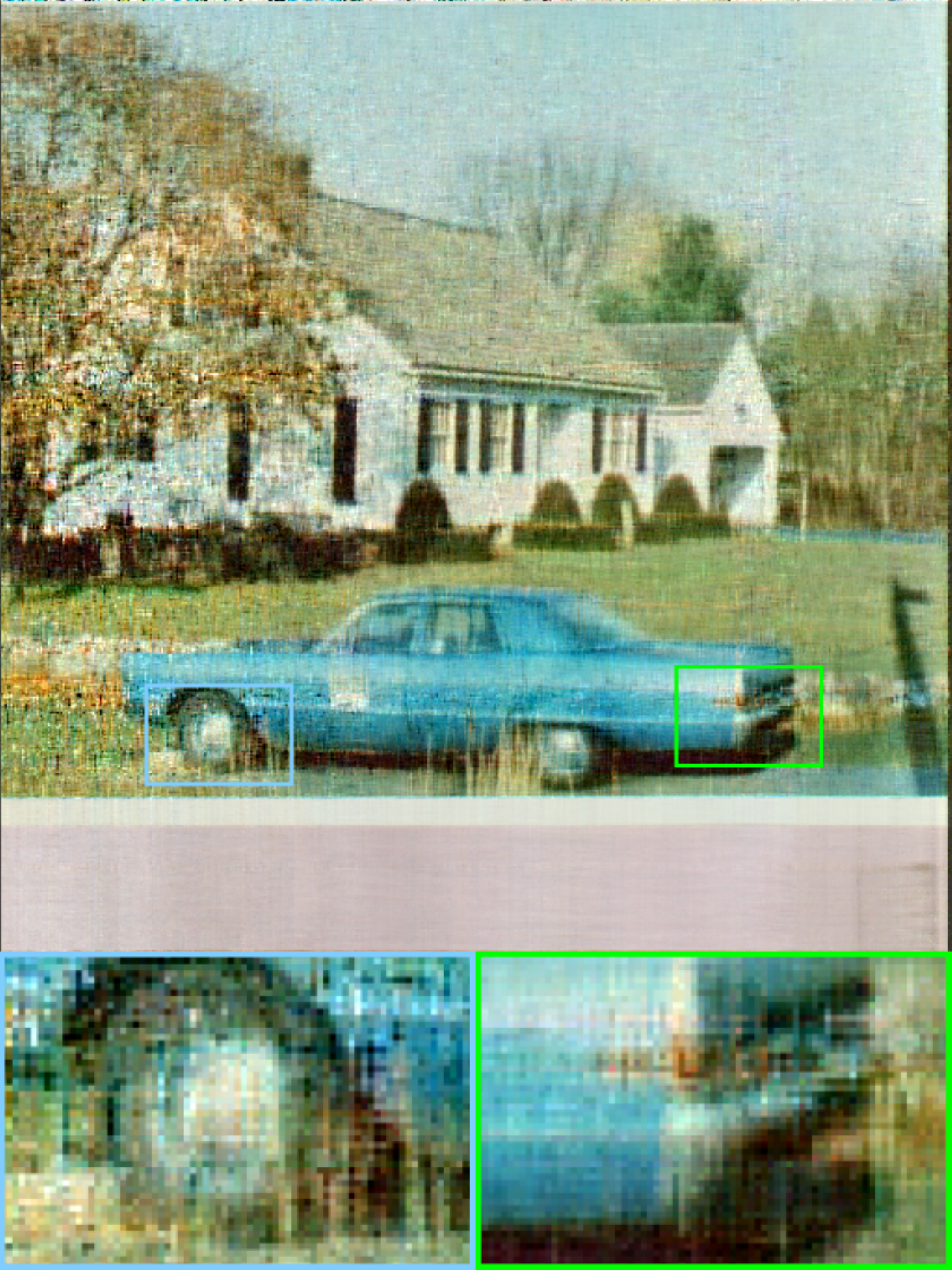}&
        \includegraphics[width=0.12\textwidth]{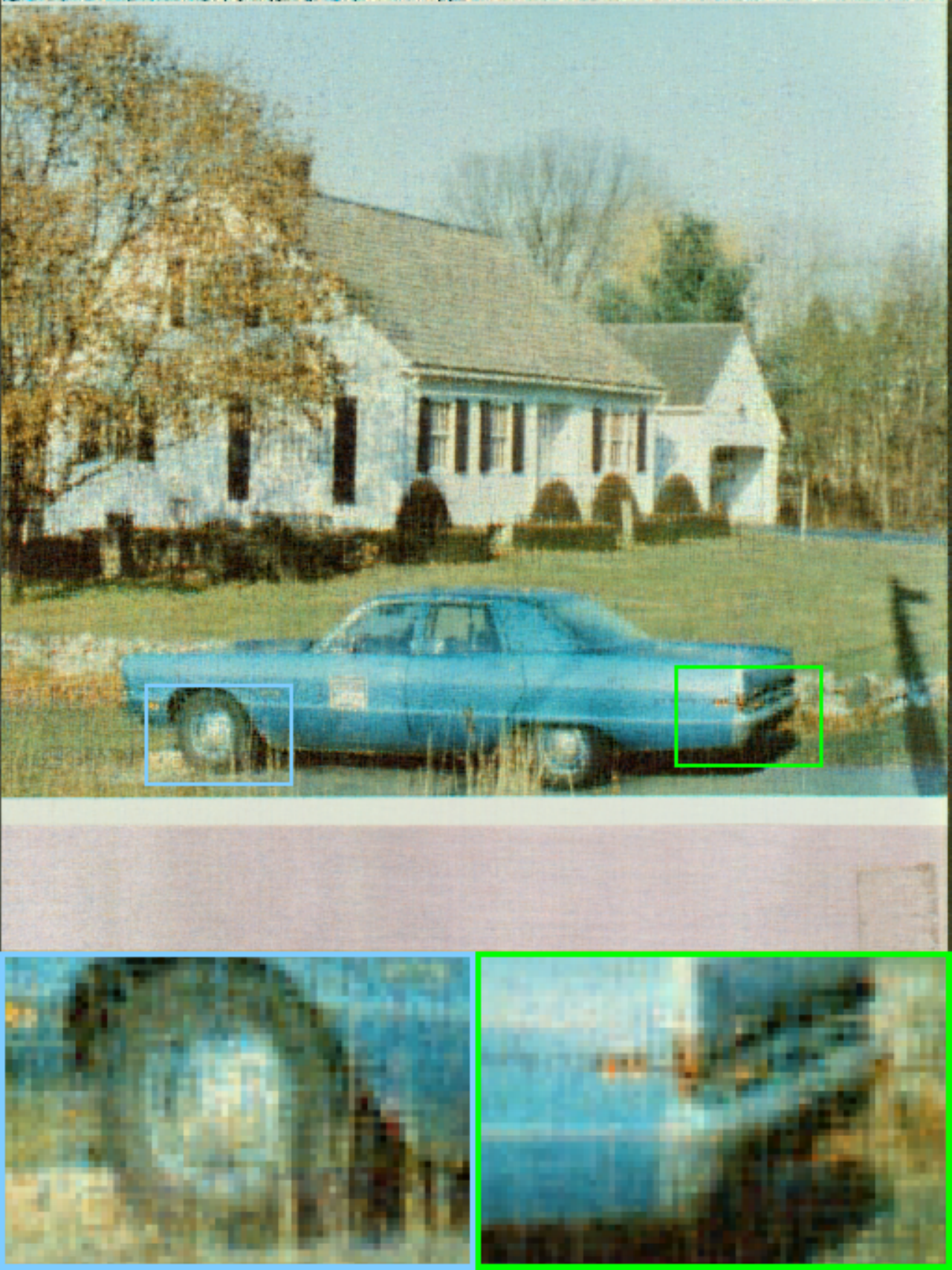}&
         \includegraphics[width=0.12\textwidth]{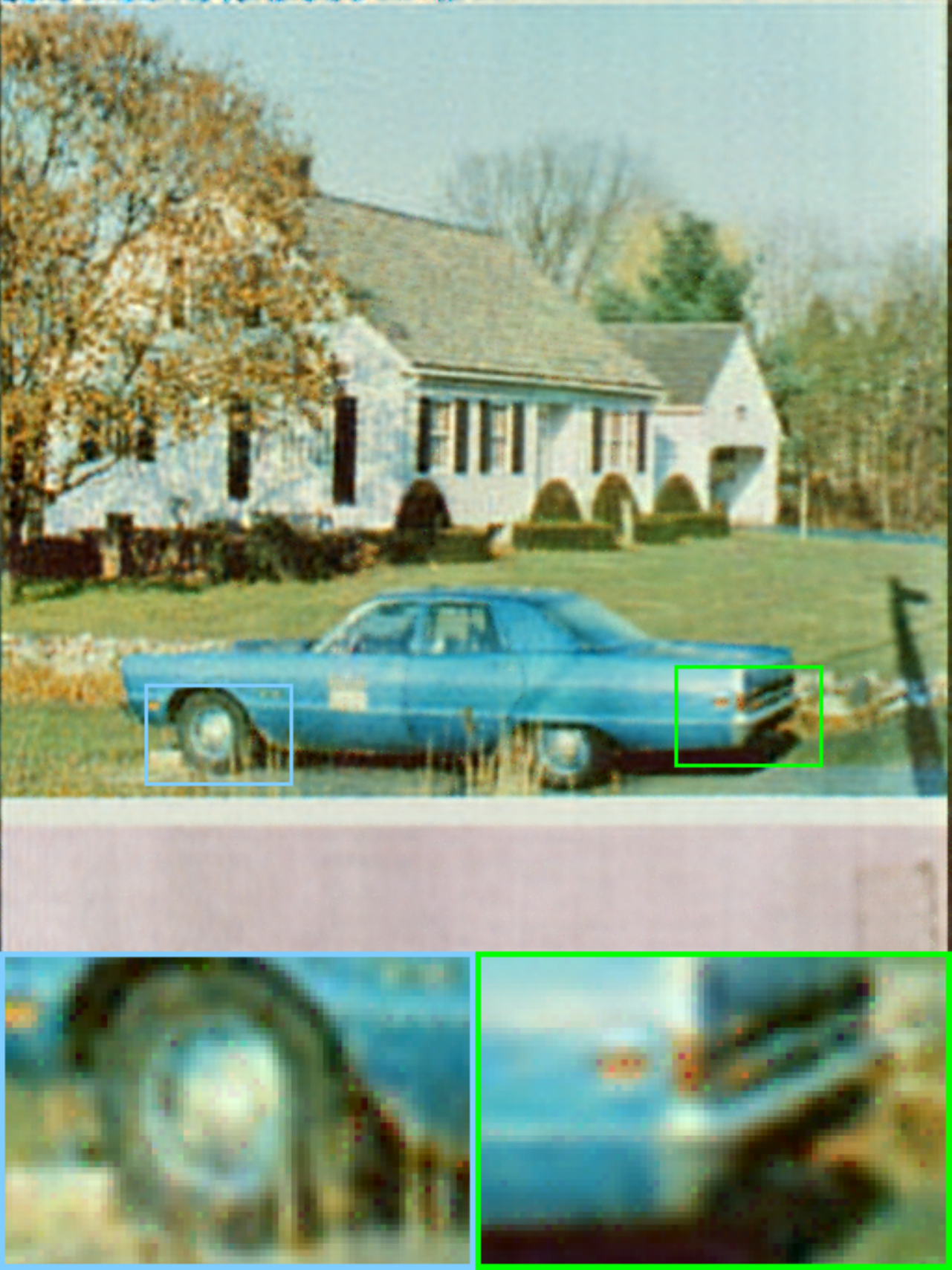}&
    \includegraphics[width=0.12\textwidth]{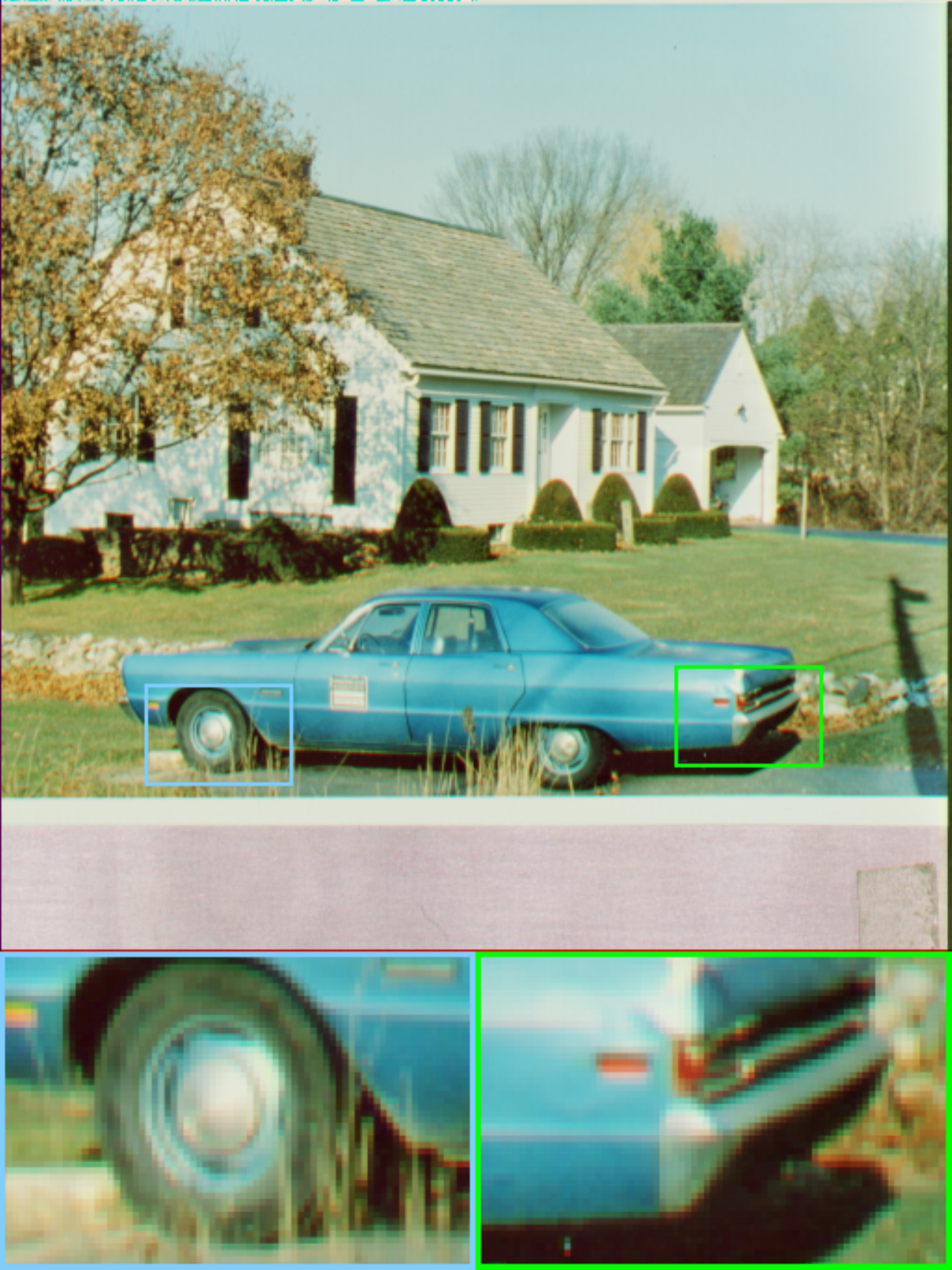}\\
PSNR 4.79  &
PSNR 23.47  &
PSNR 22.58  &
PSNR 23.86  &
PSNR 22.18  &
PSNR 25.04  &
PSNR 25.66  &
PSNR Inf\\
     \includegraphics[width=0.12\textwidth]{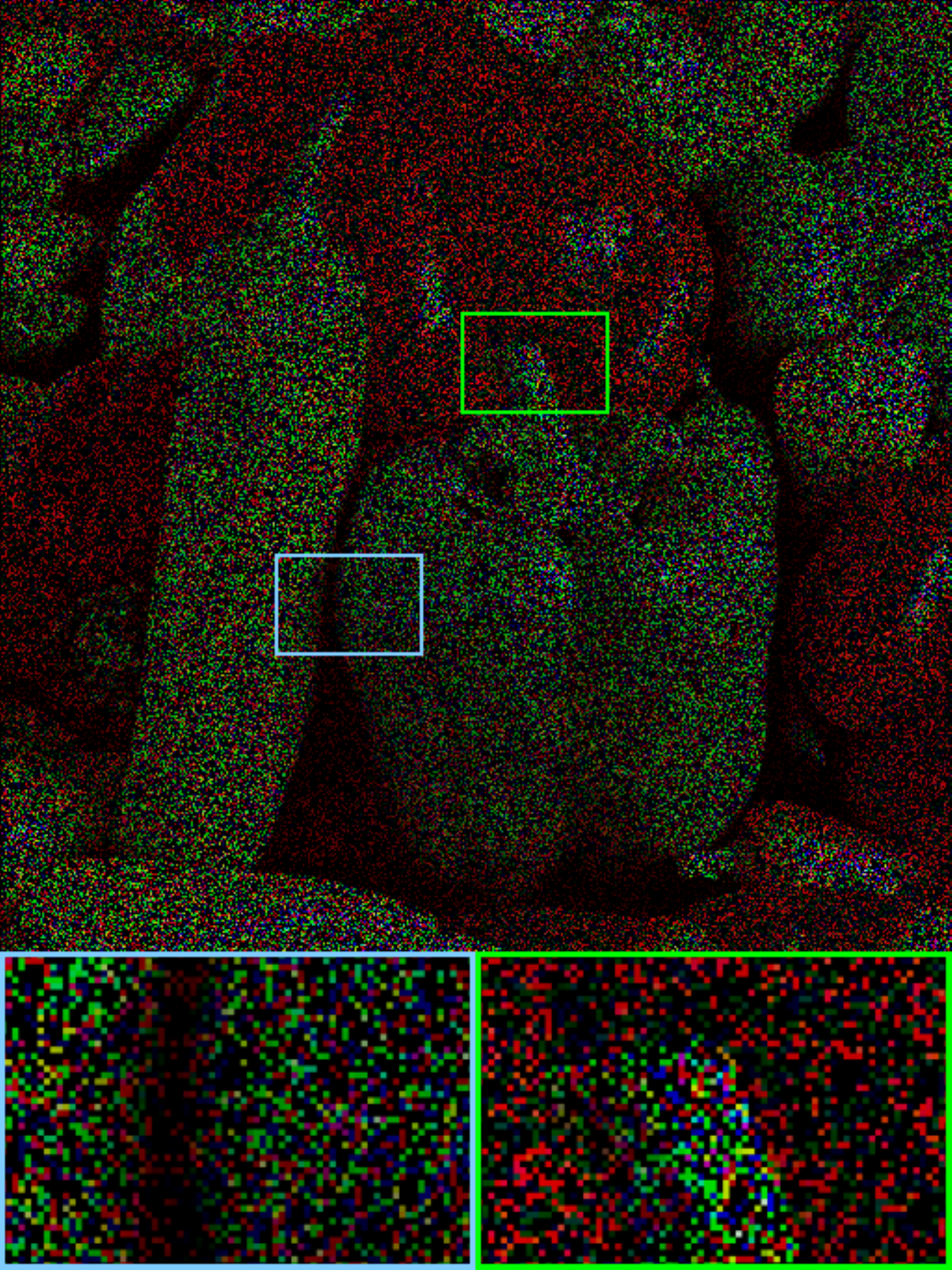}&
    \includegraphics[width=0.12\textwidth]{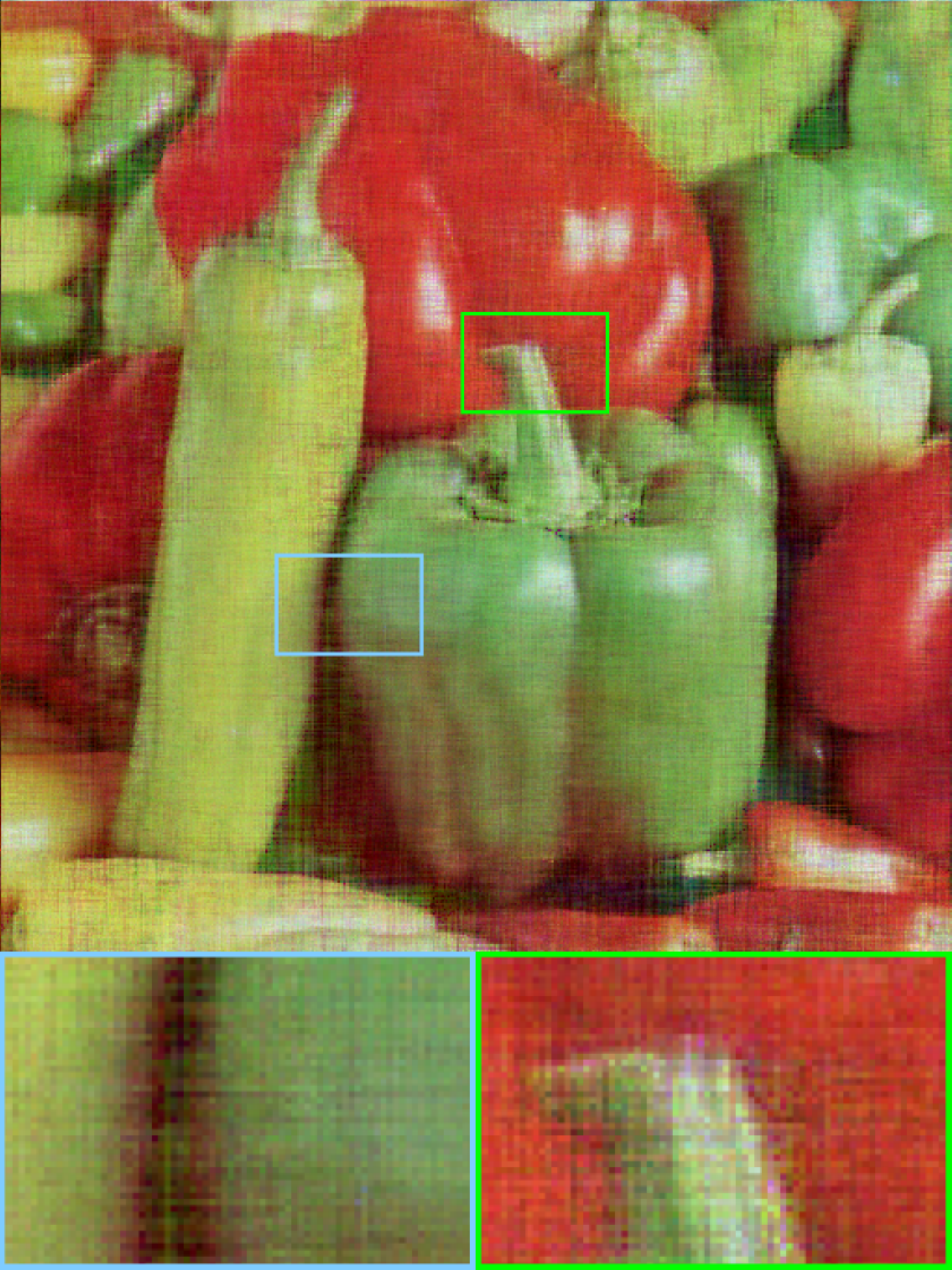}&
     \includegraphics[width=0.12\textwidth]{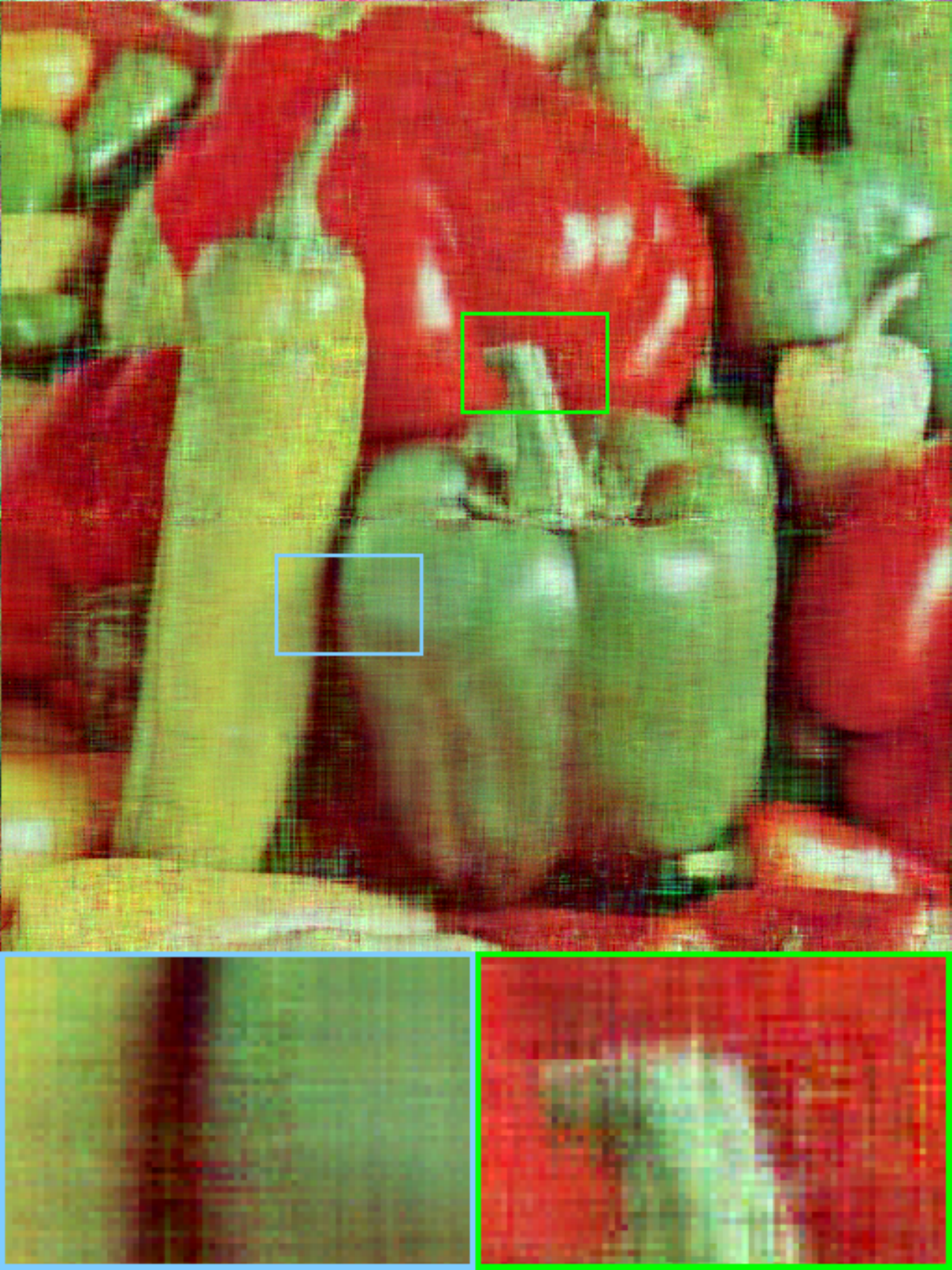}&
      \includegraphics[width=0.12\textwidth]{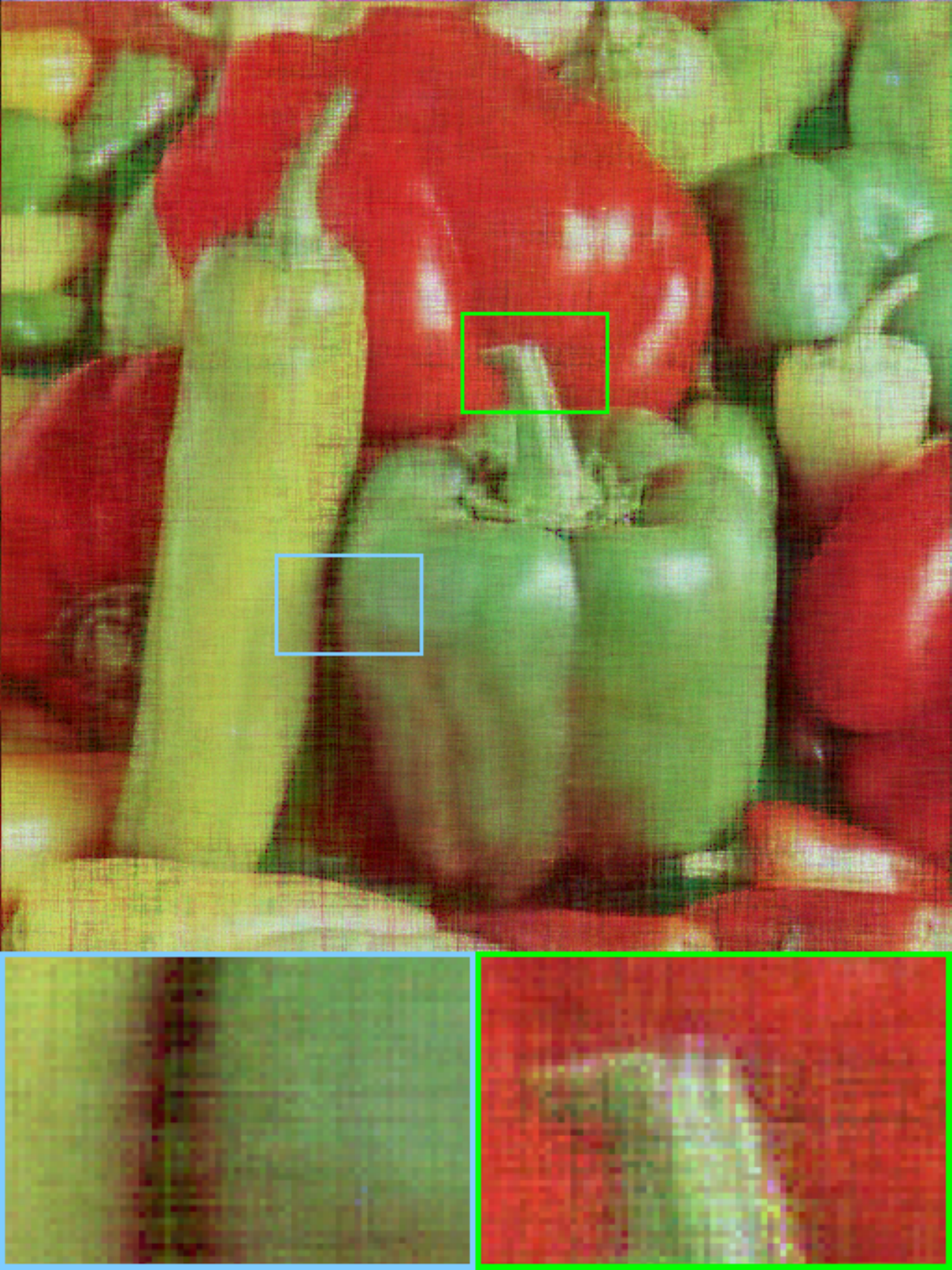}&
       \includegraphics[width=0.12\textwidth]{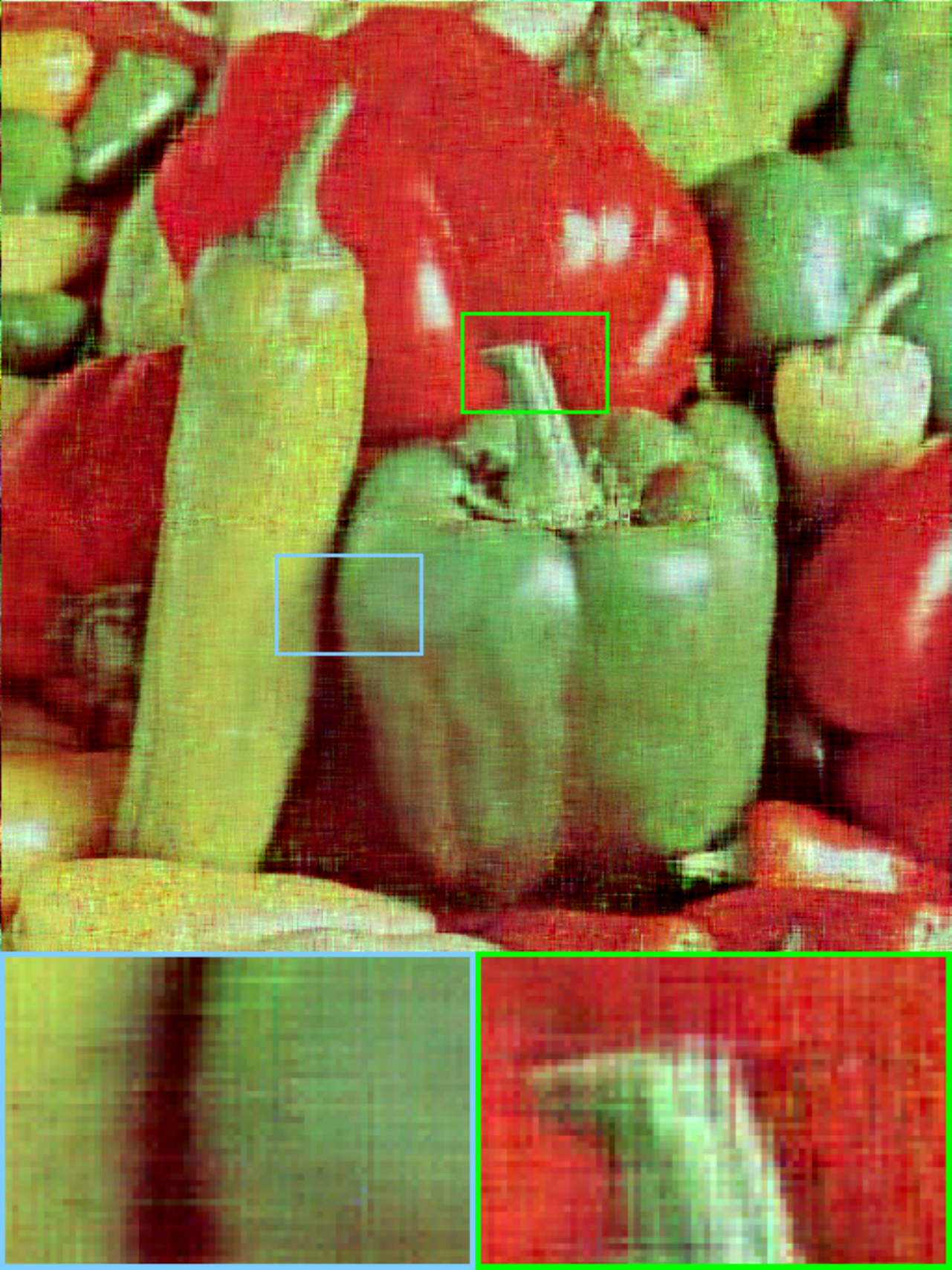}&
        \includegraphics[width=0.12\textwidth]{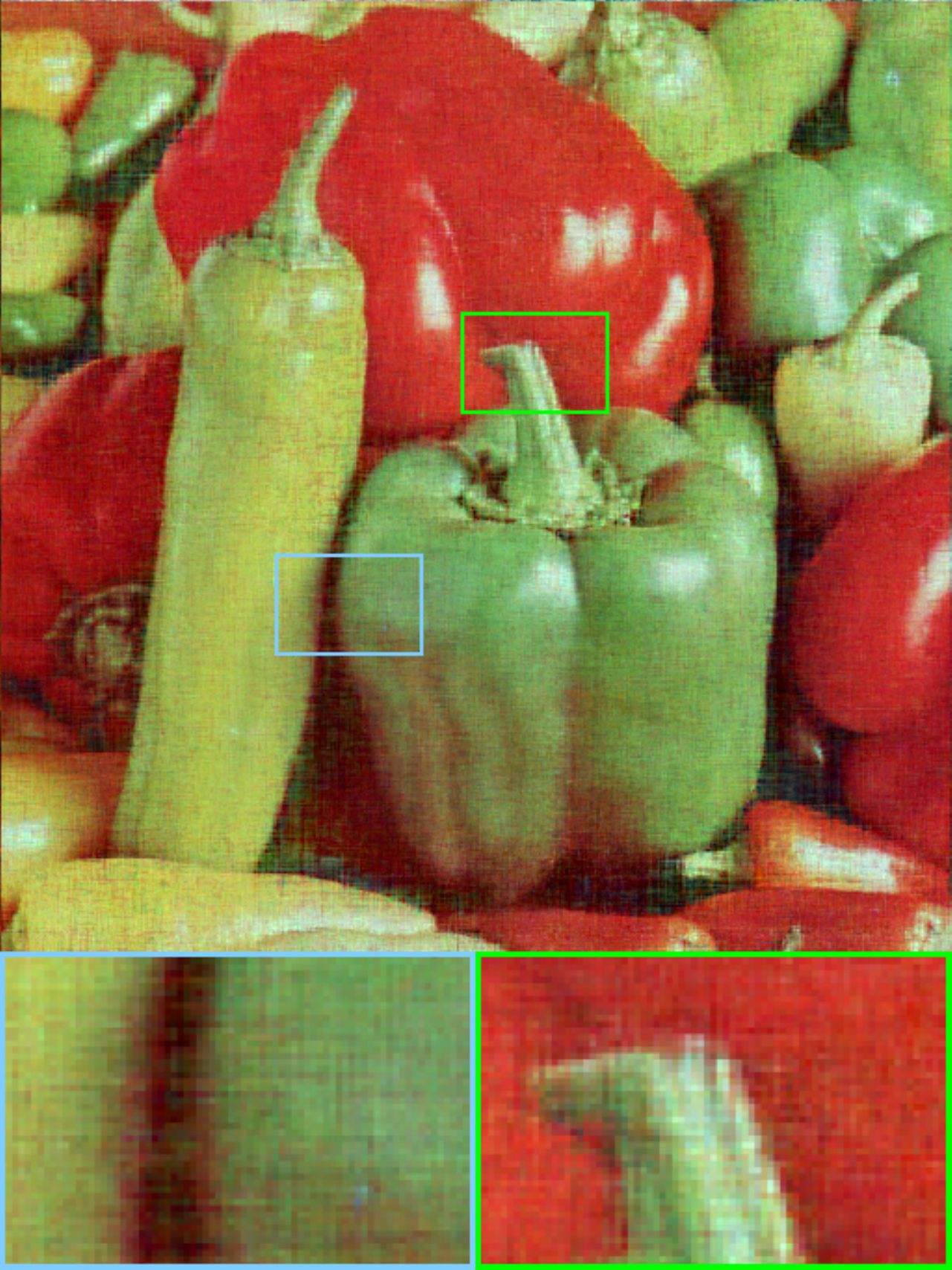}&
         \includegraphics[width=0.12\textwidth]{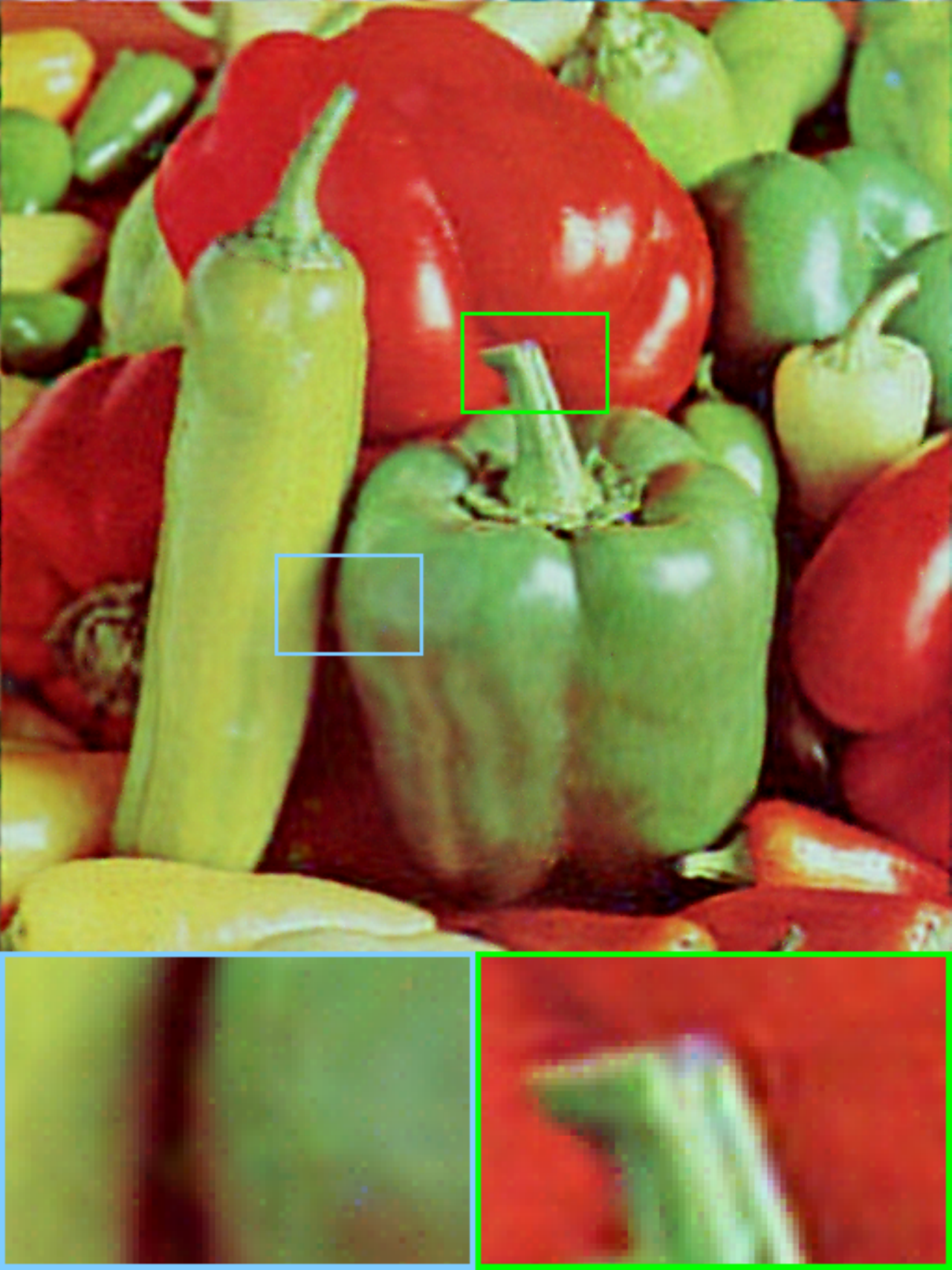}&
    \includegraphics[width=0.12\textwidth]{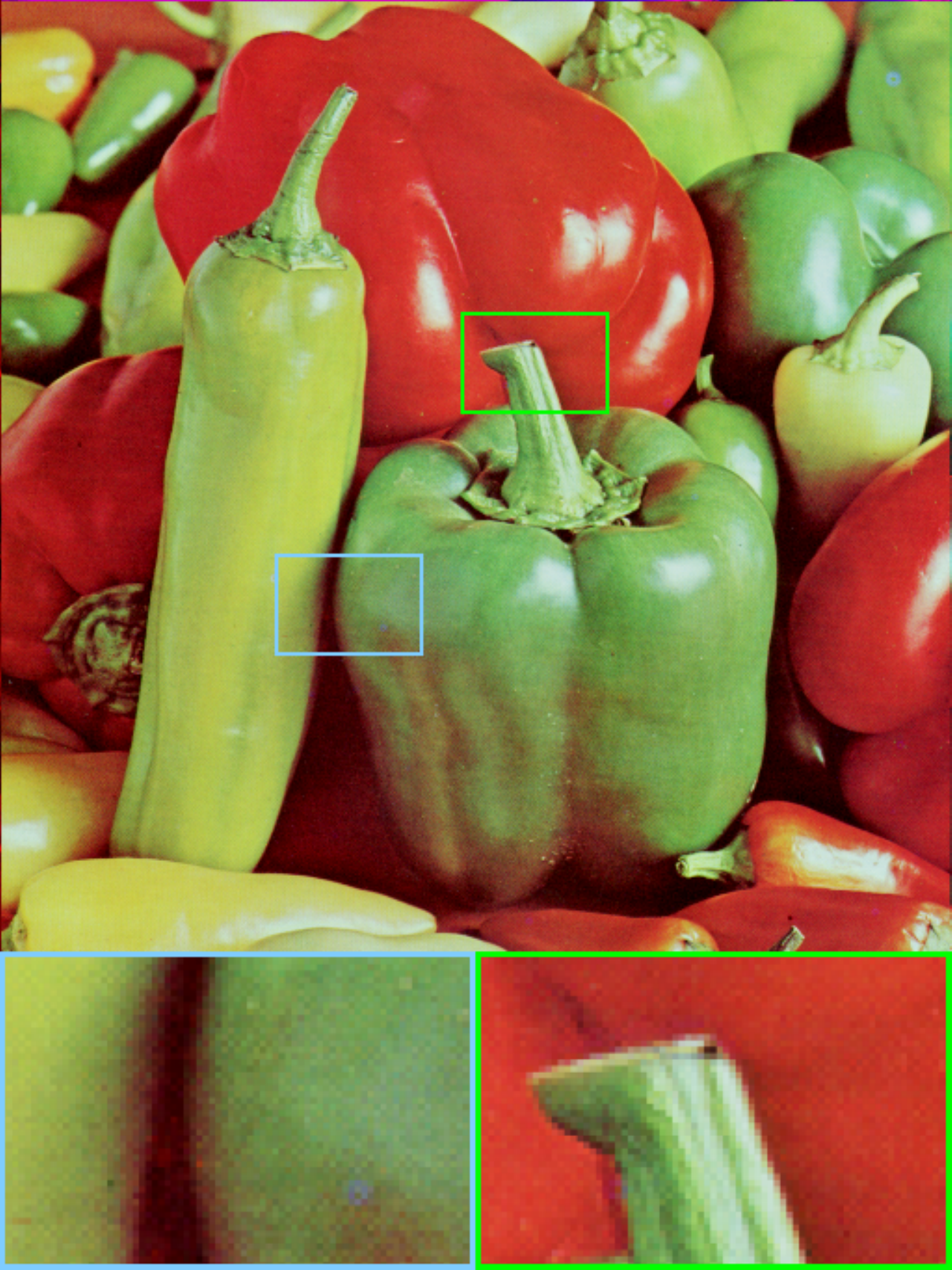}\\
PSNR 7.52  &
PSNR 22.34  &
PSNR 22.71  &
PSNR 22.65  &
PSNR 23.78  &
PSNR 24.40  &
PSNR 28.14  &
PSNR Inf\\
     \includegraphics[width=0.12\textwidth]{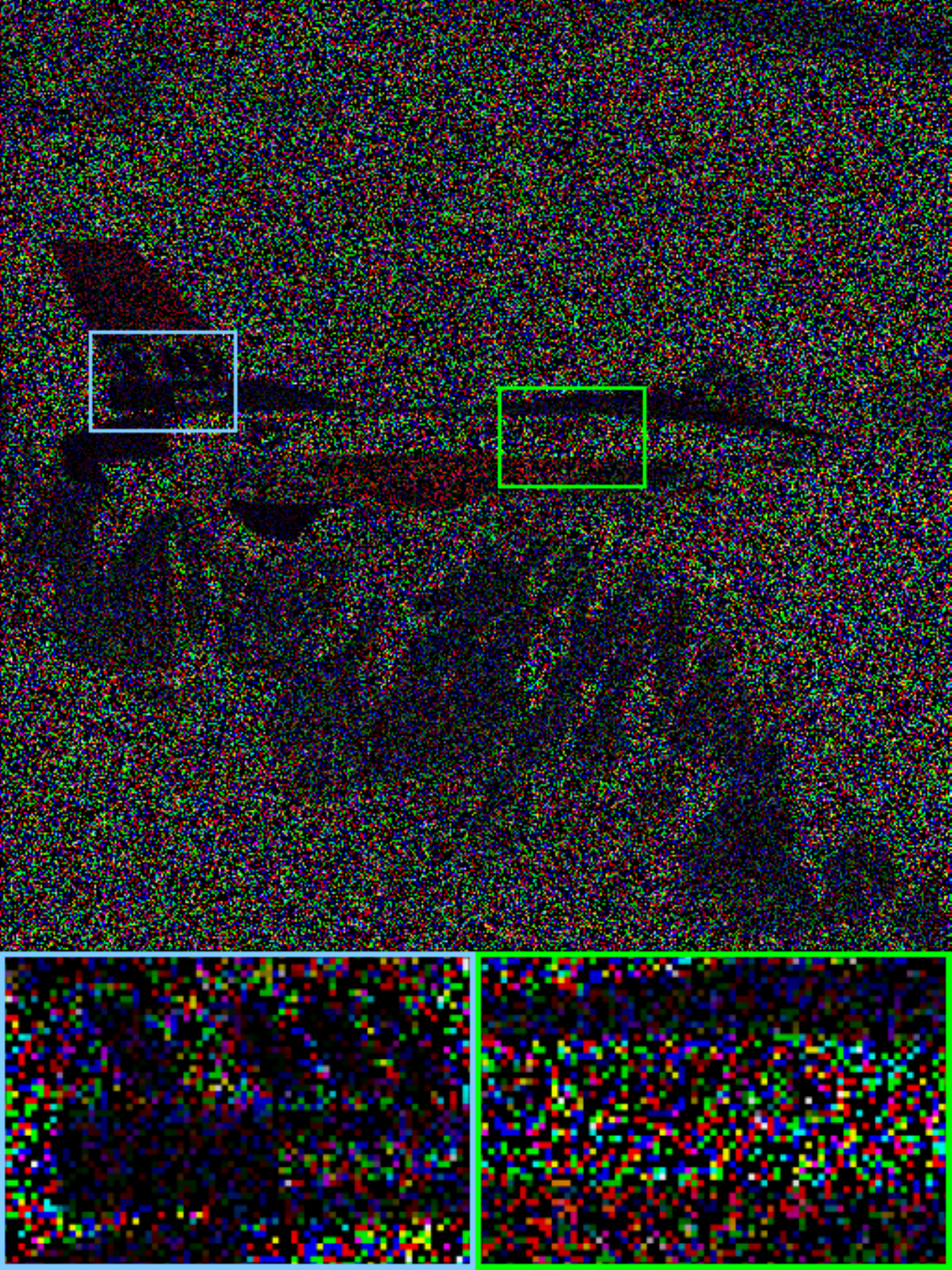}&
    \includegraphics[width=0.12\textwidth]{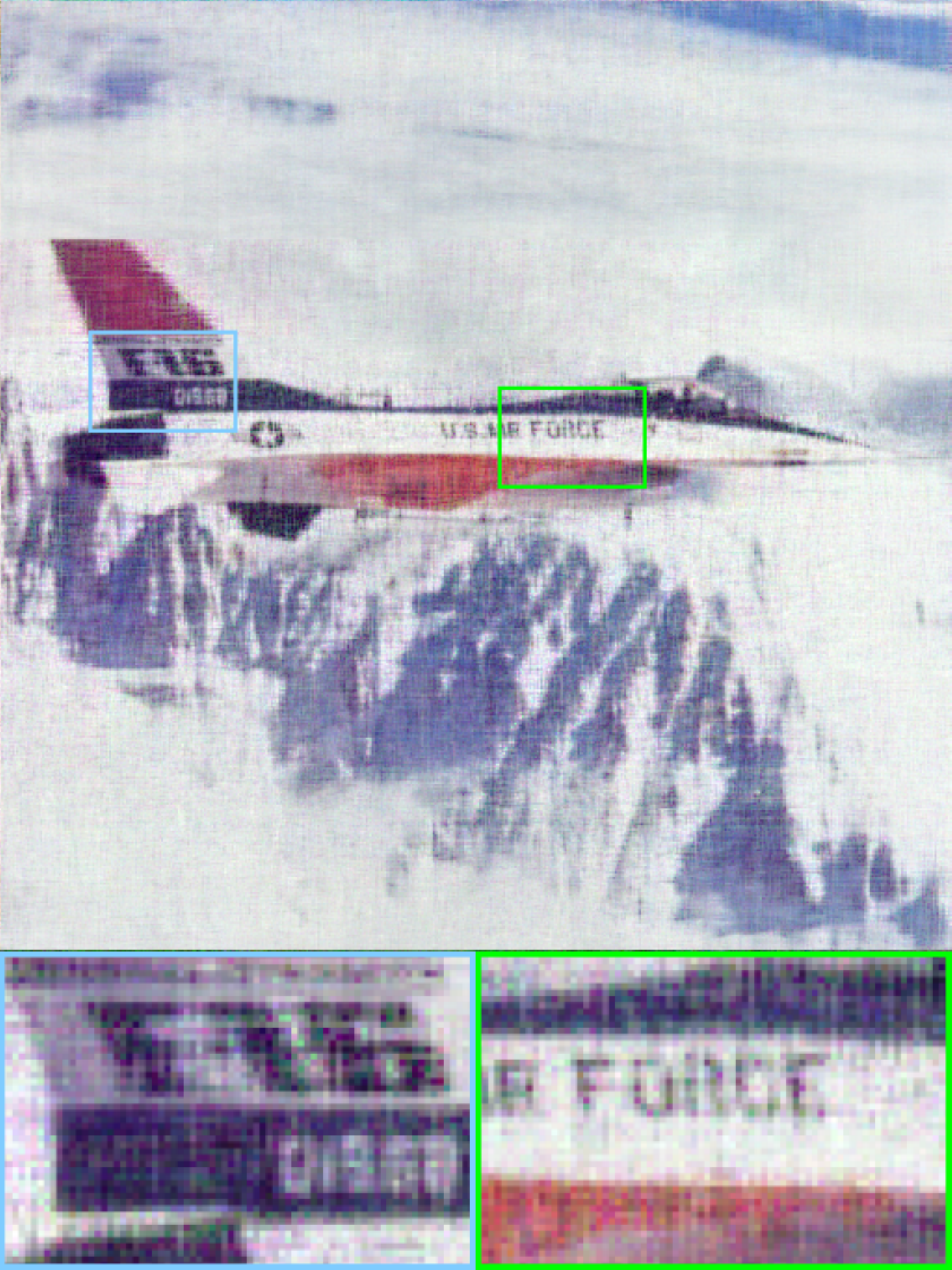}&
     \includegraphics[width=0.12\textwidth]{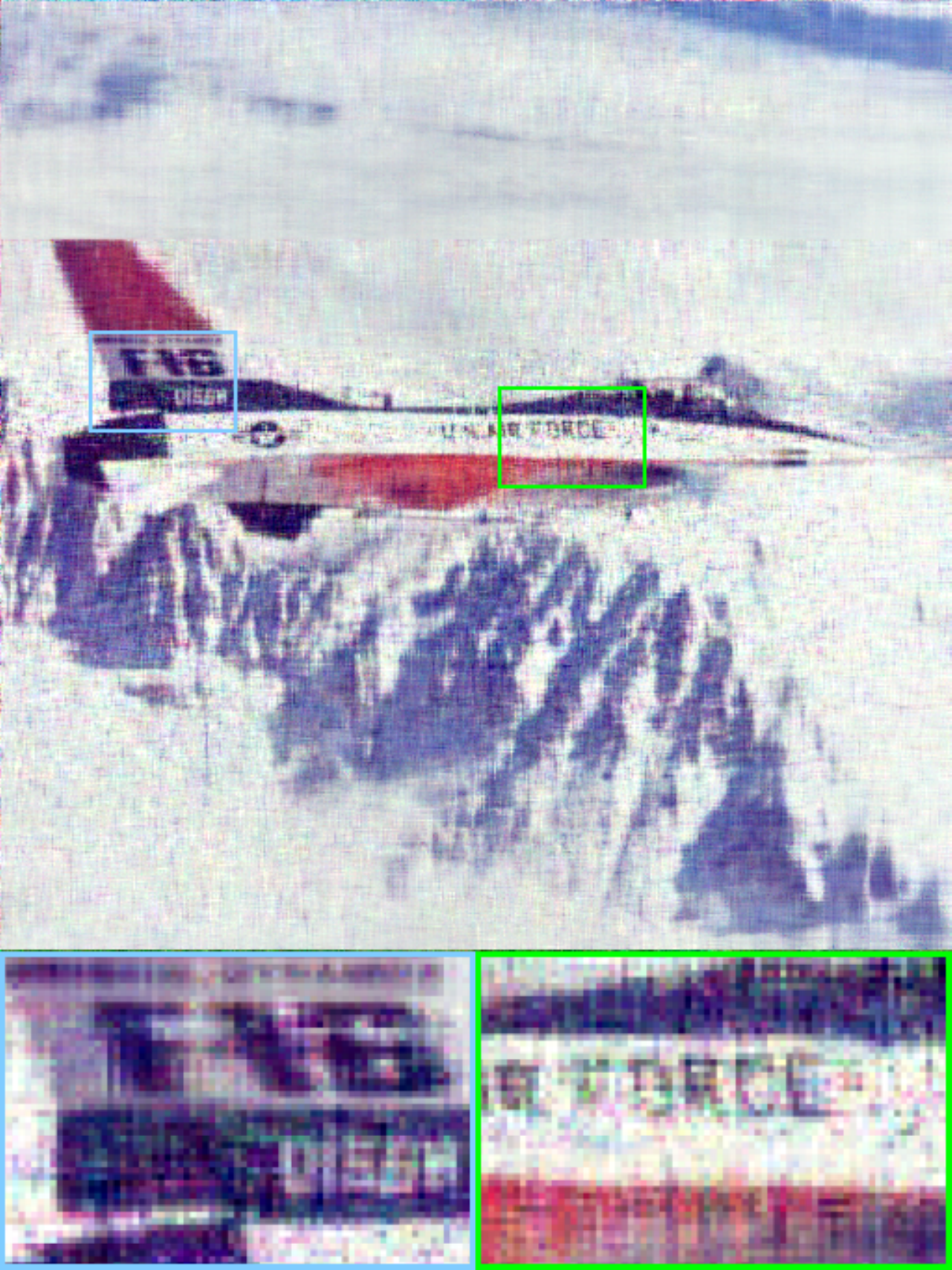}&
      \includegraphics[width=0.12\textwidth]{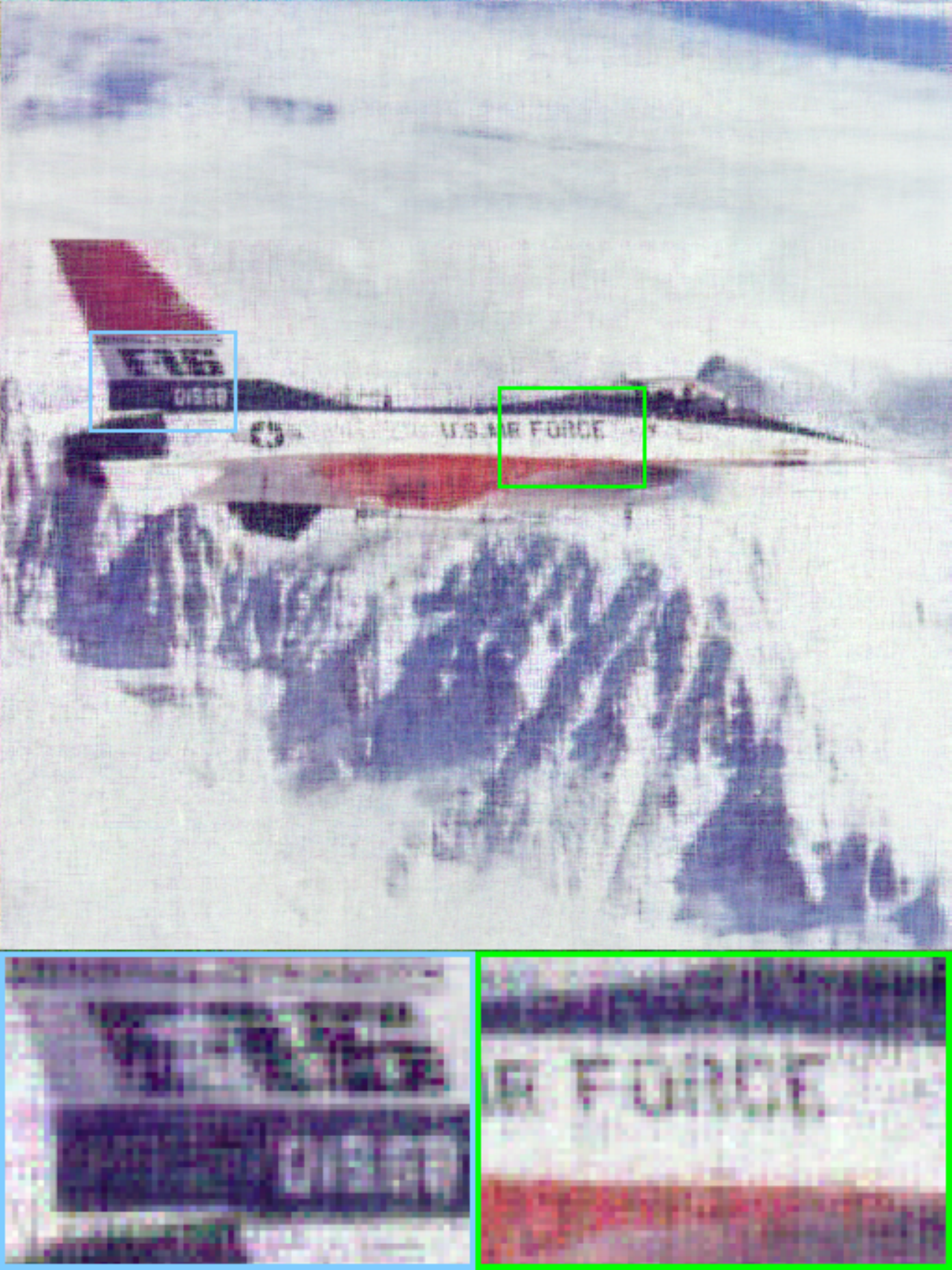}&
       \includegraphics[width=0.12\textwidth]{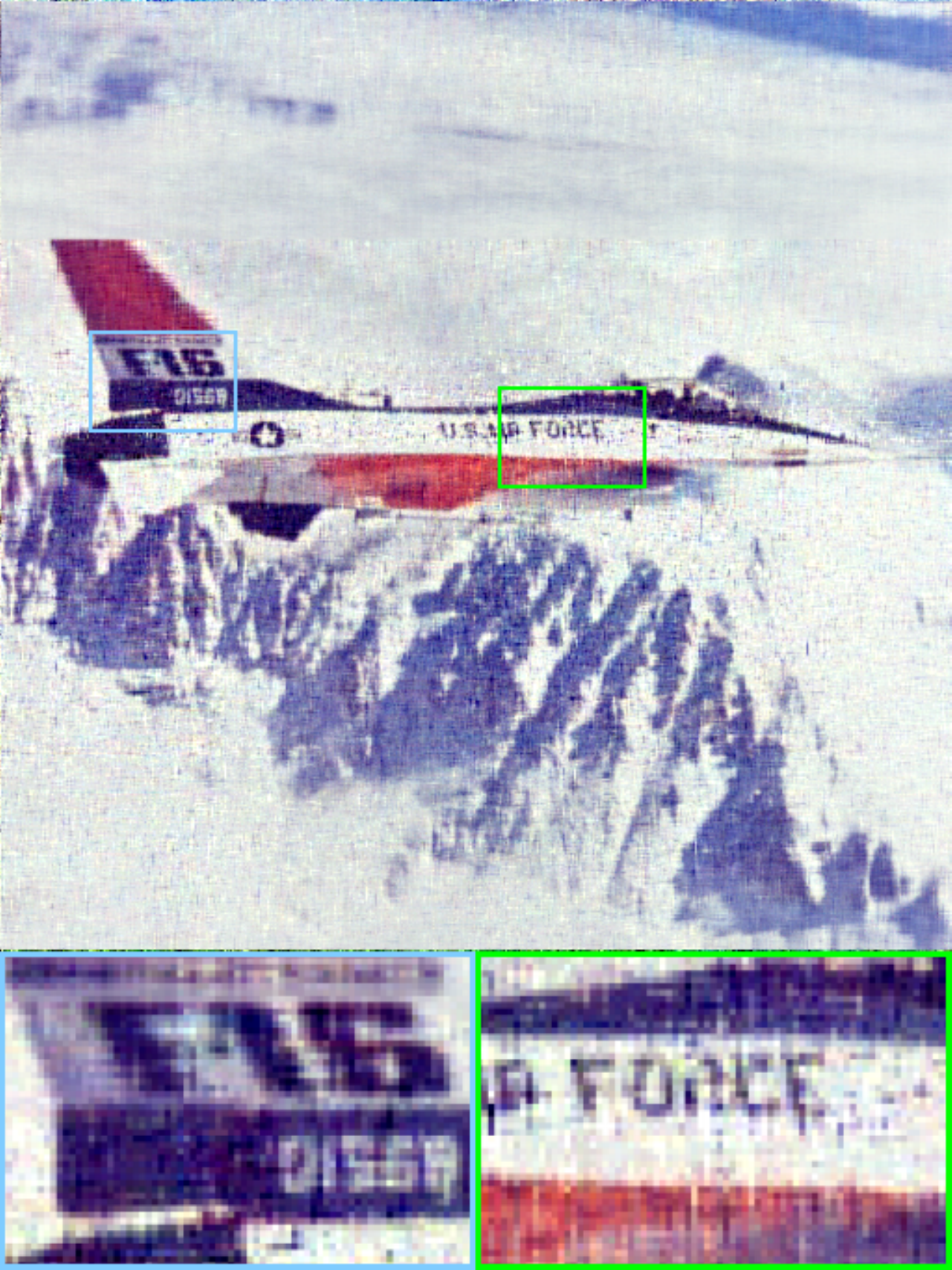}&
        \includegraphics[width=0.12\textwidth]{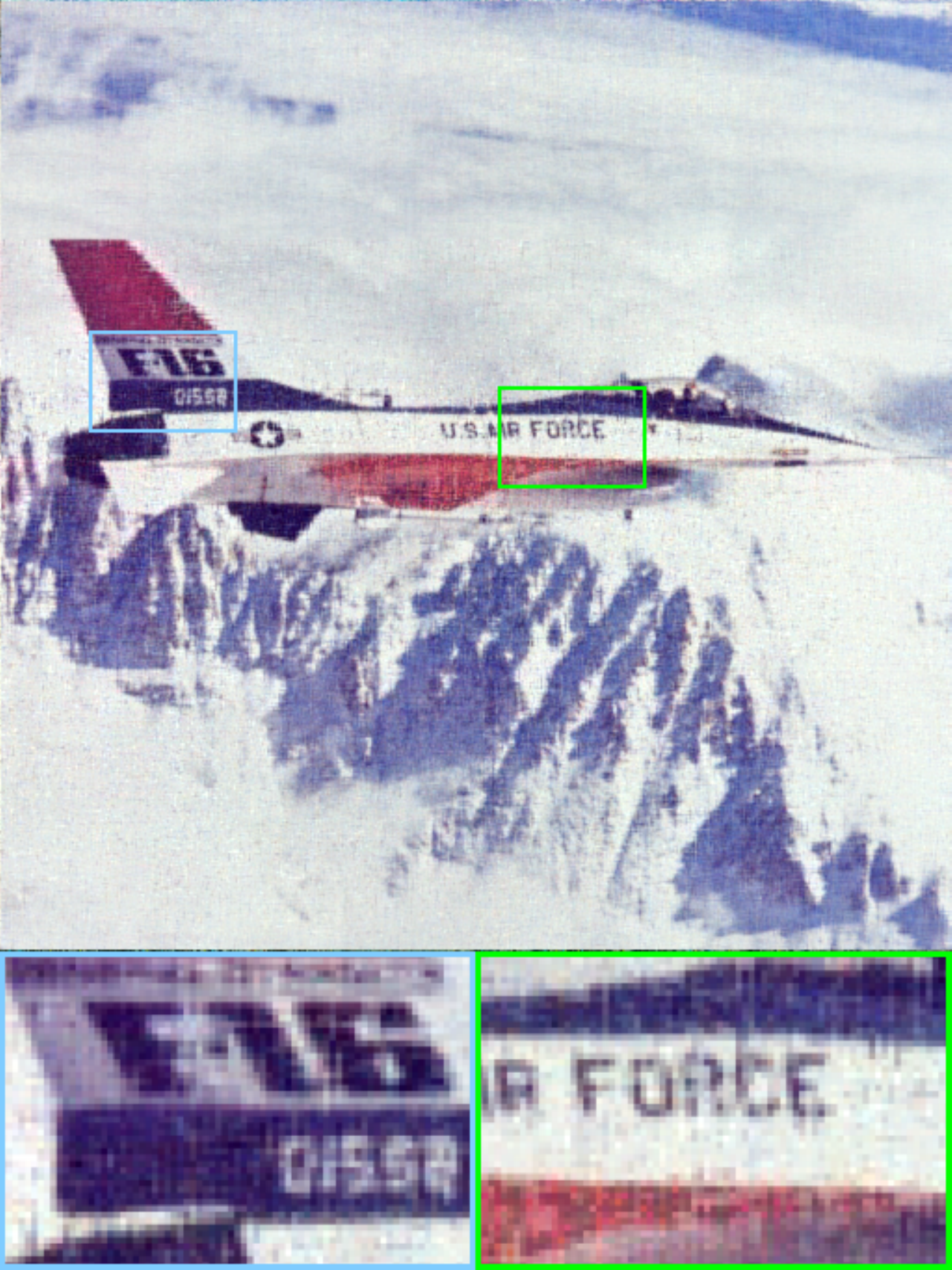}&
         \includegraphics[width=0.12\textwidth]{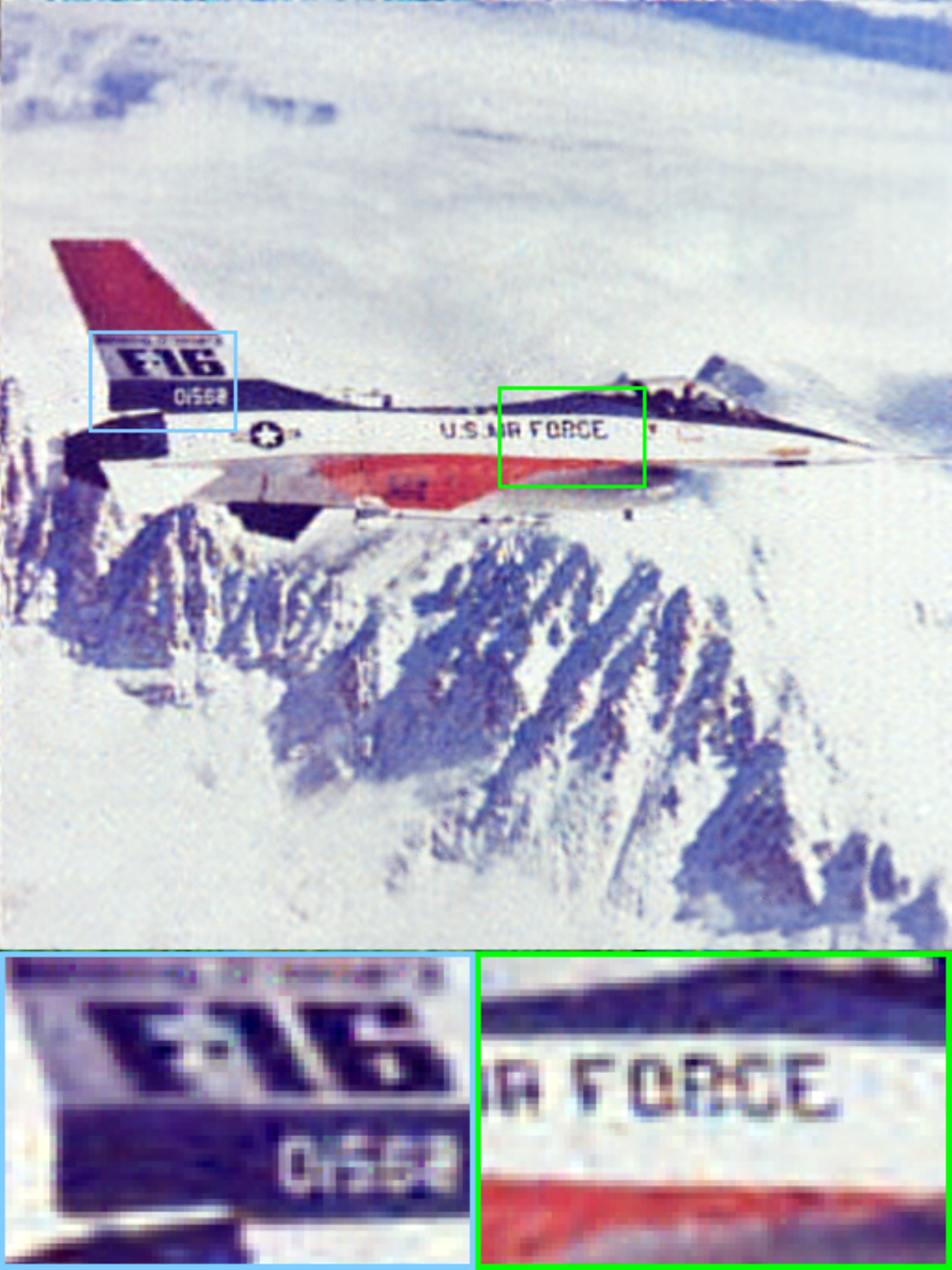}&
    \includegraphics[width=0.12\textwidth]{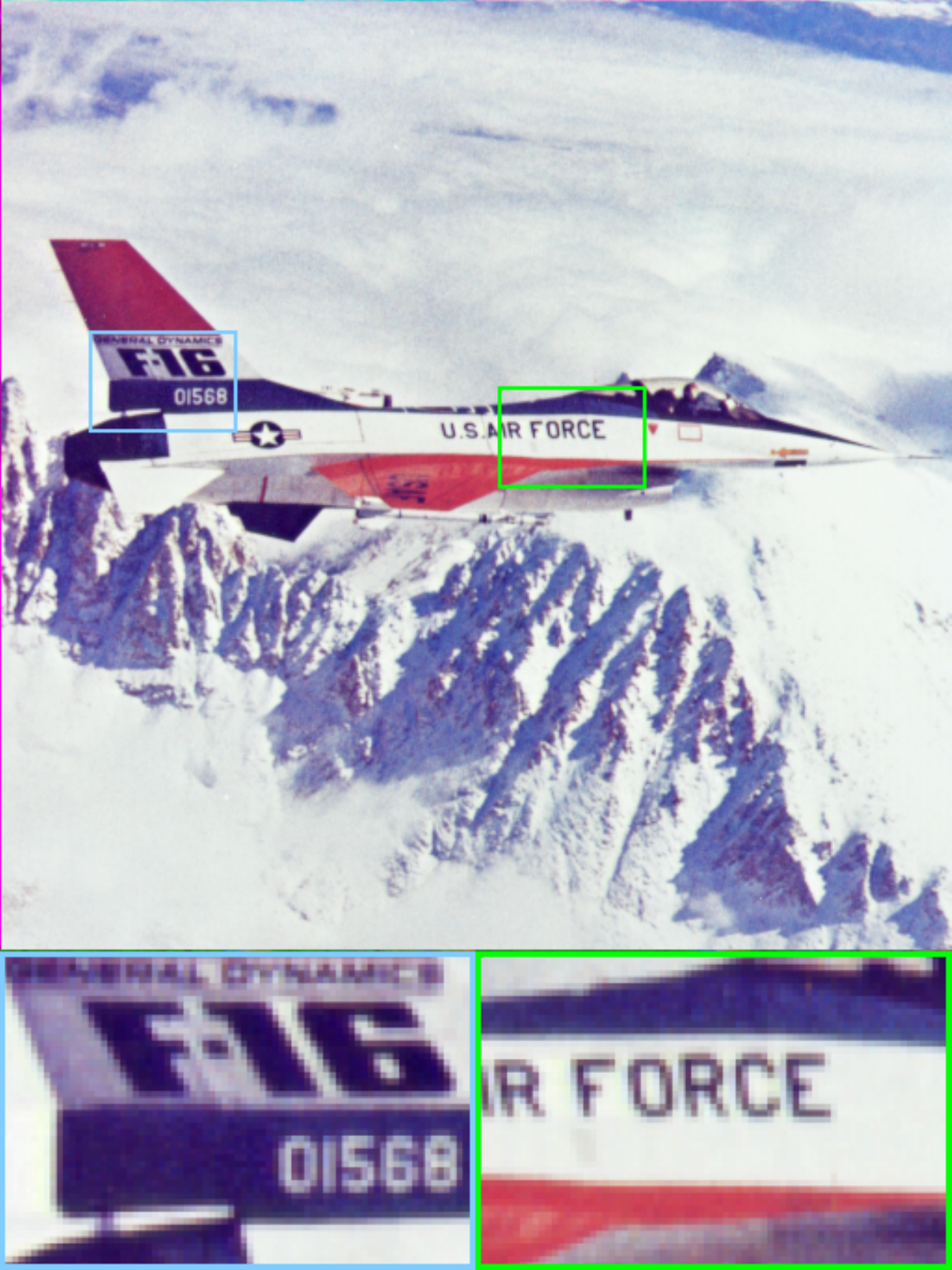}\\
PSNR 3.66  &
PSNR 25.24  &
PSNR 24.85  &
PSNR 25.66  &
PSNR 25.59  &
PSNR 27.78  &
PSNR 29.75  &
PSNR Inf\\
          Observed &DCTNN\cite{CVPR_19} &TRLRF\cite{TRLRF}&FTNN\cite{FTNN}&FCTN\cite{FCTN}&HLRTF\cite{HLRTF}&LRTFR & Original\\
    \end{tabular}
    \end{center}
    \vspace{-0.3cm}
    \caption{The results of multi-dimensional image inpainting by different methods on color images {\it Sailboat}, {\it House}, {\it Peppers}, and {\it Plane} (SR = 0.2).\label{fig_completion}\vspace{-0.3cm}}
    \end{figure*}
    \begin{table*}[!h]
    \caption{The average quantitative results by different methods for multi-dimensional image inpainting. The {\bf best} and \underline{second-best} values are highlighted. (PSNR $\uparrow$, SSIM $\uparrow$, and NRMSE $\downarrow$)\label{tab_completion}}\vspace{-0.4cm}
      \begin{center}
      \scriptsize
      \setlength{\tabcolsep}{2.5pt}
      \begin{spacing}{1.2}
      \begin{tabular}{clccccccccccccccc}
      \toprule
      \multicolumn{2}{c}{Sampling rate}&\multicolumn{3}{c}{0.1}&\multicolumn{3}{c}{0.15}&\multicolumn{3}{c}{0.2}&\multicolumn{3}{c}{0.25}&\multicolumn{3}{c}{0.3}\\
      \cmidrule{1-17}
      Data&Method&PSNR &SSIM &NRMSE \;\; &PSNR &SSIM &NRMSE \;\; &PSNR &SSIM &NRMSE \;\; &PSNR &SSIM&NRMSE \;\;&PSNR &SSIM&NRMSE\\
      \midrule
      \multirow{7}*{\tabincell{c}{
      Color images\\ {\it Sailboat}\\{\it House}\\{\it Peppers}\\{\it Plane}\\{(512$\times$512$\times$3)}}}
&Observed&{5.01}&{0.043}&{3.003}\;\;&{5.26}&{0.054}&{2.381}\;\;&{5.52}&{0.065}&{2.001}\;\;&{5.80}&{0.075}&{1.734}\;\;&{6.10}&{0.086}&{1.529}\\
&DCTNN&{20.02}&{0.595}&{0.187}\;\;&{21.85}&{0.681}&{0.149}\;\;&{23.31}&{0.743}&{0.126}\;\;&{24.51}&{0.790}&{0.109}\;\;&{25.59}&{0.827}&{0.097}\\
&TRLRF&{17.03}&{0.351}&{0.253}\;\;&{20.37}&{0.559}&{0.172}\;\;&{23.05}&{0.708}&{0.125}\;\;&{24.75}&{0.788}&{0.104}\;\;&{25.38}&{0.819}&{0.099}\\
&FTNN&{19.64}&{0.590}&{0.198}\;\;&{21.96}&{0.698}&{0.150}\;\;&{23.64}&{0.764}&{0.122}\;\;&{25.07}&{0.815}&{0.103}\;\;&{26.27}&{0.851}&{0.090}\\
&FCTN&{18.79}&{0.463}&{0.207}\;\;&{21.26}&{0.621}&{0.159}\;\;&{23.45}&{0.732}&{0.128}\;\;&{25.01}&{0.801}&{0.107}\;\;&{26.07}&{0.841}&{0.096}\\
&HLRTF&\underline{22.49}&\underline{0.700}&\underline{0.136}\;\;&\underline{24.41}&\underline{0.779}&\underline{0.110}\;\;&\underline{25.39}&\underline{0.811}&\underline{0.097}\;\;&\underline{26.34}&\underline{0.842}&\underline{0.086}\;\;&\underline{27.17}&\underline{0.868}&\underline{0.078}\\
&LRTFR&\bf{24.88}&\bf{0.827}&\bf{0.102}\;\;&\bf{26.00}&\bf{0.849}&\bf{0.089}\;\;&\bf{27.35}&\bf{0.892}&\bf{0.079}\;\;&\bf{27.65}&\bf{0.895}&\bf{0.075}\;\;&\bf{28.42}&\bf{0.916}&\bf{0.071}\\
      \midrule
      \multirow{7}*{\tabincell{c}{
      MSIs\\{\it Toys}\\{\it Flowers}\\{(256$\times$256$\times$31)}}}
&Observed&{12.32}&{0.442}&{2.999}\;\;&{12.57}&{0.472}&{2.386}\;\;&{12.83}&{0.500}&{1.999}\;\;&{13.10}&{0.527}&{1.740}\;\;&{13.42}&{0.554}&{1.526} \\
&DCTNN&{32.53}&{0.954}&{0.102}\;\;&{35.77}&{0.976}&{0.070}\;\;&{38.43}&{0.986}&{0.052}\;\;&{40.63}&{0.991}&{0.041}\;\;&{42.81}&{0.994}&{0.032} \\
&TRLRF&{27.83}&{0.797}&{0.165}\;\;&{35.60}&{0.966}&{0.070}\;\;&{37.03}&{0.975}&{0.060}\;\;&{37.85}&{0.979}&{0.054}\;\;&{38.49}&{0.982}&{0.050} \\
&FTNN&{35.08}&{0.973}&{0.079}\;\;&{38.09}&{0.985}&{0.057}\;\;&{40.61}&{0.990}&{0.045}\;\;&{42.65}&\underline{0.994}&{0.036}\;\;&{44.58}&{0.995}&{0.030} \\
&FCTN&{36.72}&{0.973}&{0.064}\;\;&{40.18}&{0.986}&{0.044}\;\;&{42.08}&{0.990}&{0.036}\;\;&{43.32}&{0.992}&{0.031}\;\;&{44.73}&{0.994}&{0.026} \\
&HLRTF&\underline{38.32}&\underline{0.985}&\underline{0.051}\;\;&\underline{41.64}&\underline{0.992}&\underline{0.036}\;\;&\underline{44.19}&\underline{0.995}&\underline{0.027}\;\;&\underline{46.17}&\bf{0.997}&\underline{0.021}\;\;&\underline{46.70}&\bf{0.997}&\underline{0.020} \\
&LRTFR&\bf{39.65}&\bf{0.988}&\bf{0.043}\;\;&\bf{43.46}&\bf{0.994}&\bf{0.028}\;\;&\bf{45.21}&\bf{0.996}&\bf{0.023}\;\;&\bf{47.16}&\bf{0.997}&\bf{0.019}\;\;&\bf{47.58}&\bf{0.997}&\bf{0.018} \\
      \midrule
      \multirow{7}*{\tabincell{c}{
      Videos\\ {\it Foreman}\\{\it Carphone}\\{(144$\times$176$\times$100)}}}
&Observed&5.20&0.044&3.004\;\;&5.45&0.059&2.381\;\;&5.72&0.073&1.998\;\;&6.03&0.087&1.730\;\;&6.30&0.101&1.527 \\
&DCTNN&24.89&0.817&0.101\;\;&26.41&0.862&0.084\;\;&27.66&0.891&0.073\;\;&28.79&0.913&0.064\;\;&29.76&0.929&0.057 \\
&TRLRF&\underline{25.07}&\underline{0.818}&\underline{0.098}\;\;&26.08&0.853&0.087\;\;&26.66&0.869&0.081\;\;&27.07&0.880&0.078\;\;&27.37&0.890&0.075 \\
&FTNN&21.37&0.723&0.156\;\;&23.59&0.810&0.119\;\;&25.38&0.864&0.096\;\;&26.89&0.899&0.081\;\;&28.19&0.922&0.069 \\
&FCTN&23.99&0.766&0.115\;\;&\underline{27.09}&\underline{0.867}&\underline{0.079}\;\;&\underline{28.87}&\underline{0.906}&\underline{0.064}\;\;&\underline{29.99}&\underline{0.925}&\underline{0.056}\;\;&30.71&\underline{0.936}&\underline{0.051}\\
&HLRTF&{24.66}&{0.768}&{0.104}\;\;&{26.49}&{0.830}&{0.085}\;\;&{28.10}&{0.877}&{0.071}\;\;&{29.52}&{0.908}&{0.060}\;\;&\underline{30.80}&{0.932}&{0.052} \\
&LRTFR&\bf{27.56}&\bf{0.901}&\bf{0.074}\;\;&\bf{29.29}&\bf{0.918}&\bf{0.060}\;\;&\bf{30.14}&\bf{0.930}&\bf{0.054}\;\;&\bf{30.96}&\bf{0.942}&\bf{0.050}\;\;&\bf{32.05}&\bf{0.959}&\bf{0.044} \\
      \bottomrule
      \end{tabular}
      \end{spacing}
      \end{center}
      \vspace{-0.7cm}
      \end{table*}
      \begin{figure*}[t]
          \scriptsize
          \setlength{\tabcolsep}{0.9pt}
          \begin{center}
          \begin{tabular}{cccccccc}
             \includegraphics[width=0.12\textwidth]{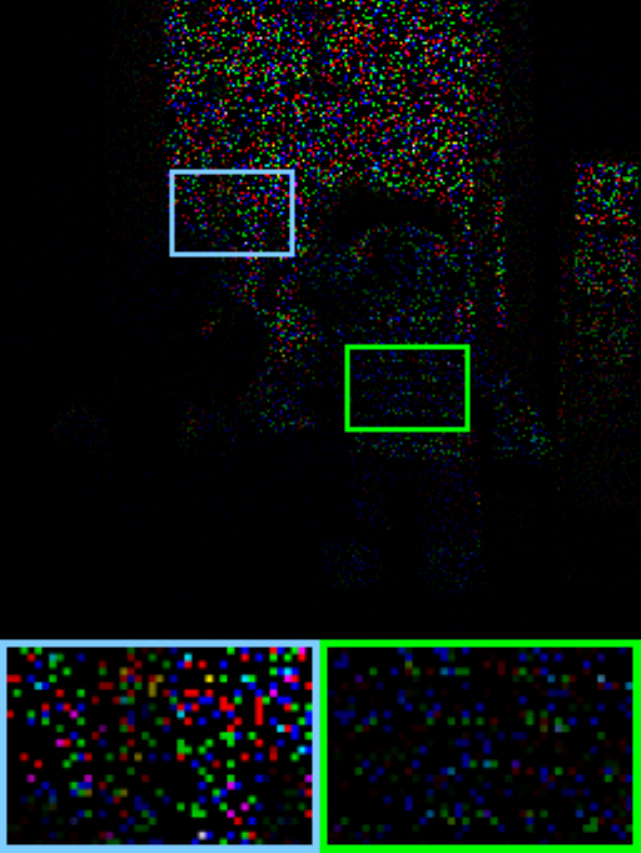}&
            \includegraphics[width=0.12\textwidth]{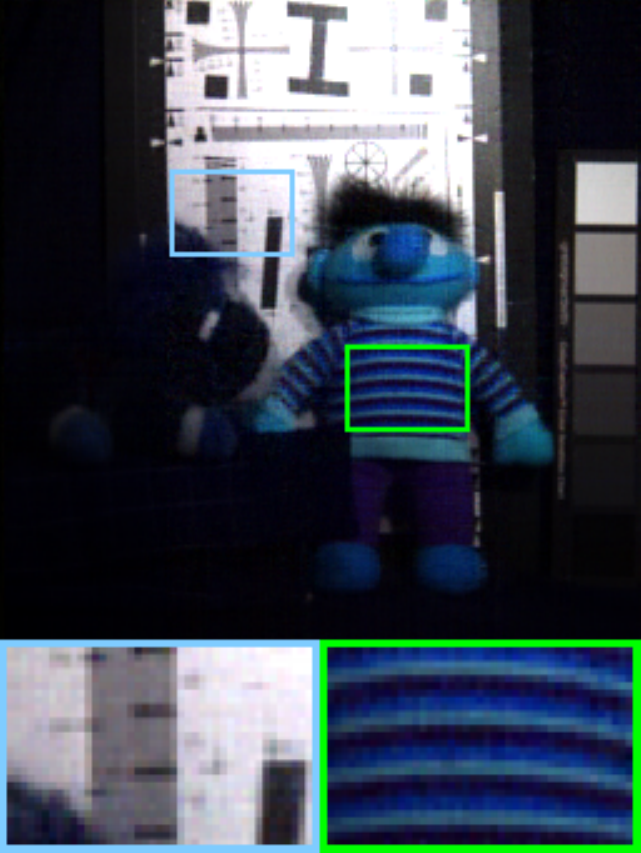}&
             \includegraphics[width=0.12\textwidth]{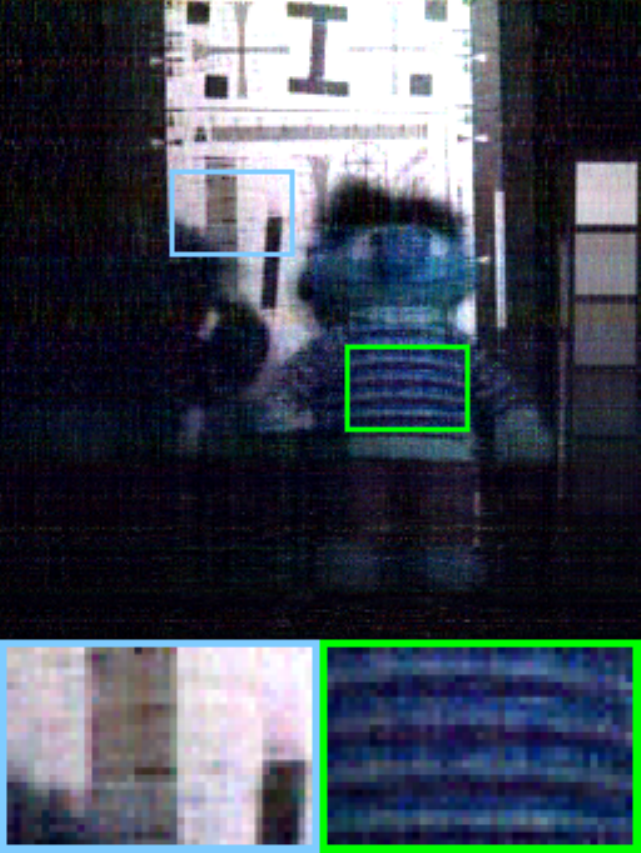}&
              \includegraphics[width=0.12\textwidth]{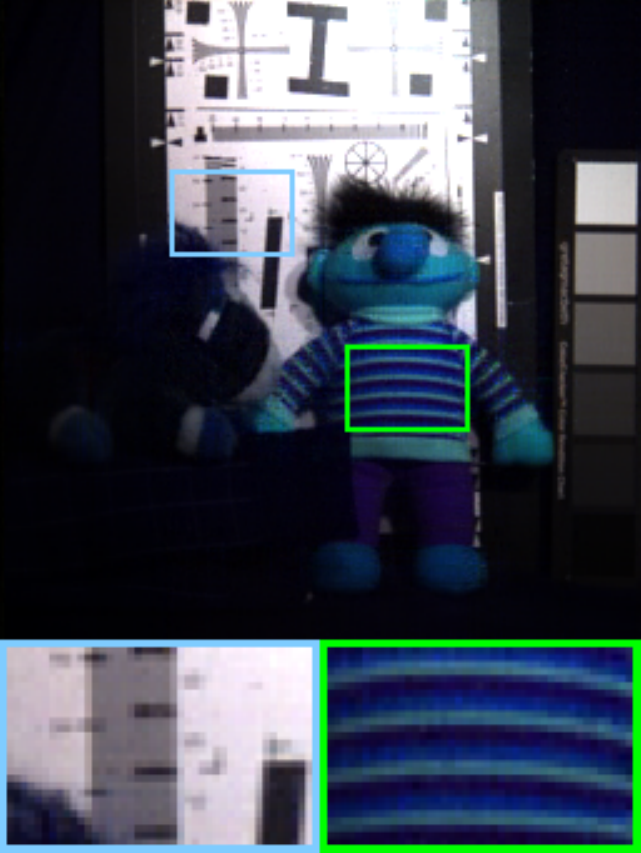}&
               \includegraphics[width=0.12\textwidth]{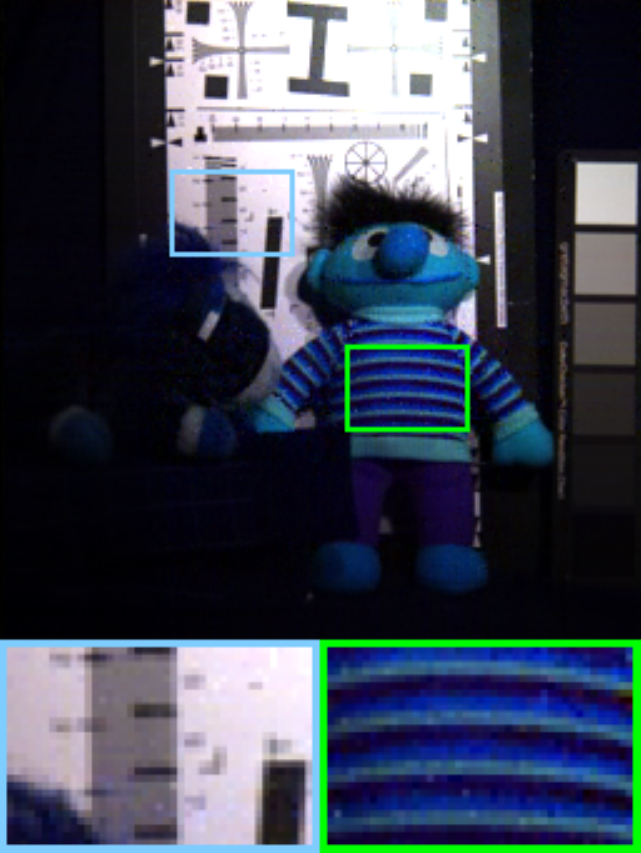}&
                \includegraphics[width=0.12\textwidth]{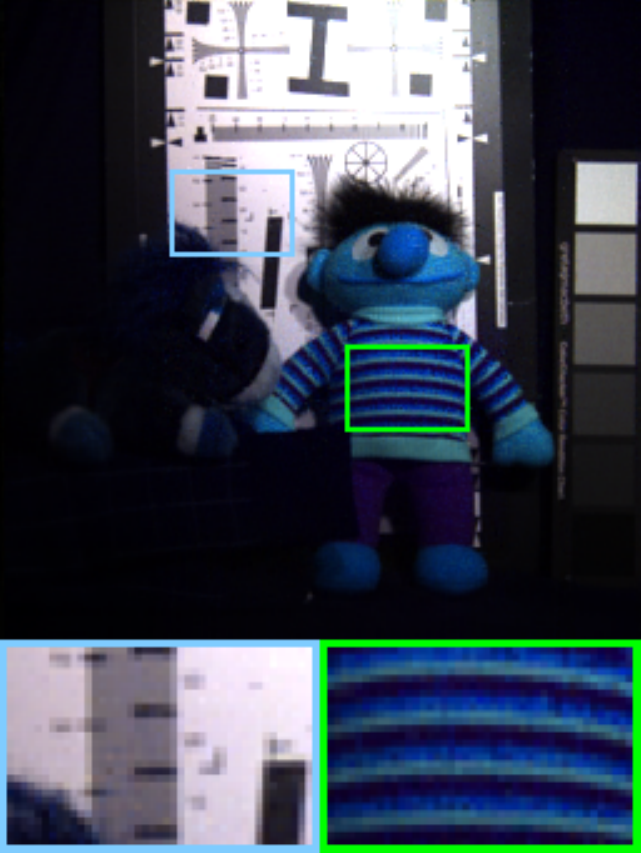}&
                 \includegraphics[width=0.12\textwidth]{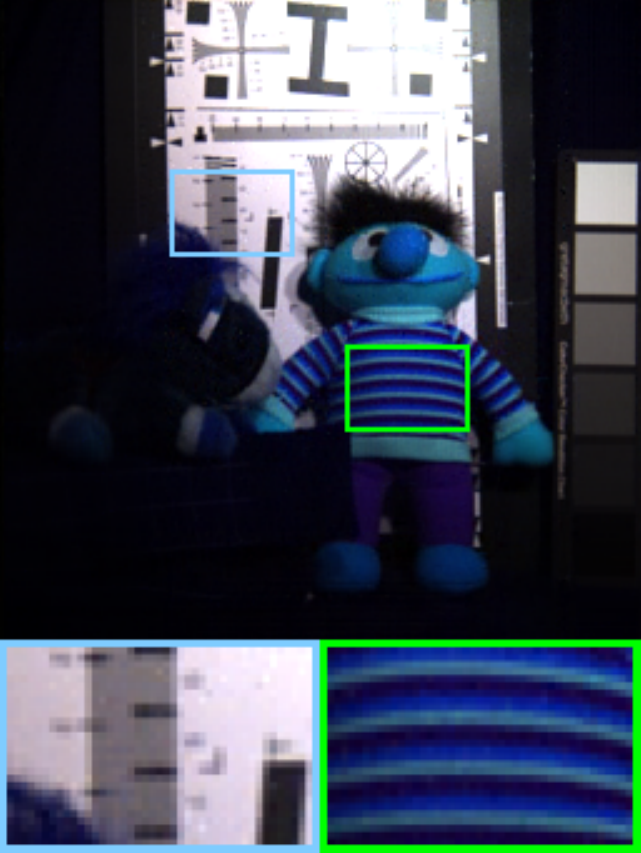}&
            \includegraphics[width=0.12\textwidth]{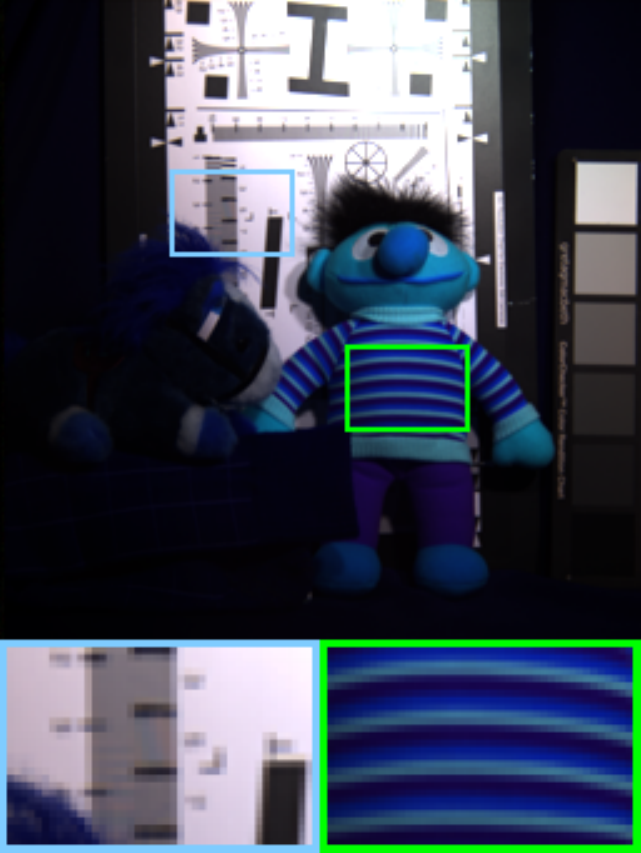}\\        
          PSNR 10.87  &
          PSNR 33.40  &
          PSNR 27.07  &
          PSNR 36.24  &
          PSNR 37.42  &
          PSNR 38.17  &
          PSNR 40.00  &
          PSNR Inf\\
             \includegraphics[width=0.12\textwidth]{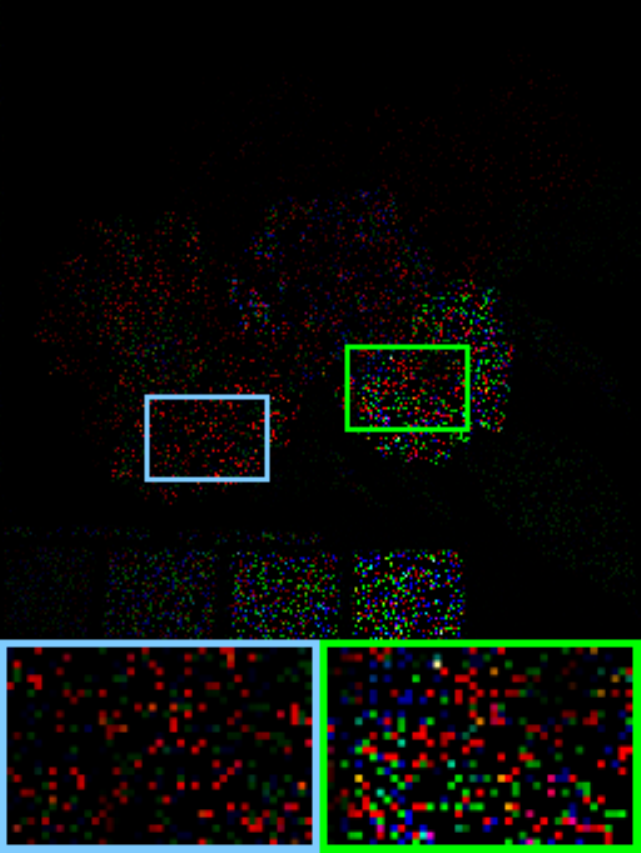}&
            \includegraphics[width=0.12\textwidth]{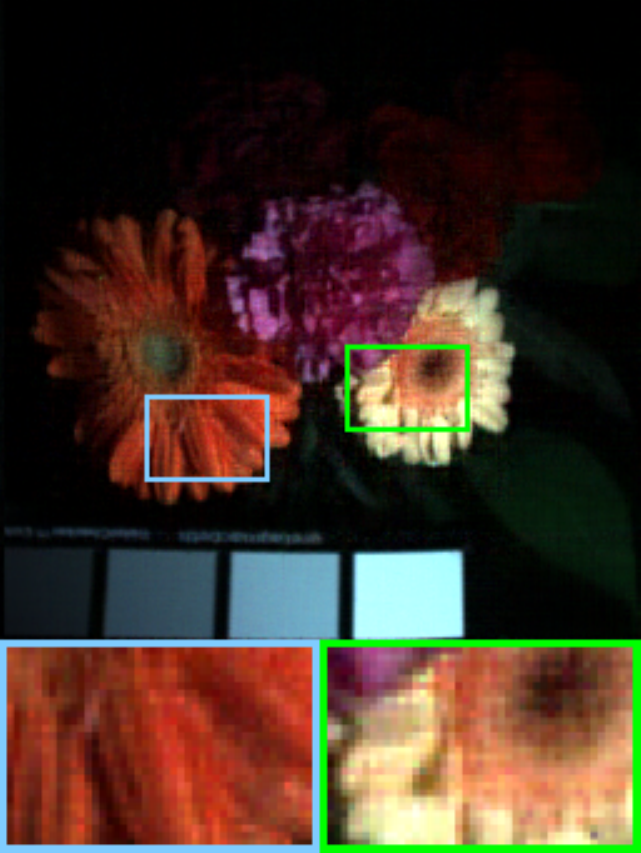}&
             \includegraphics[width=0.12\textwidth]{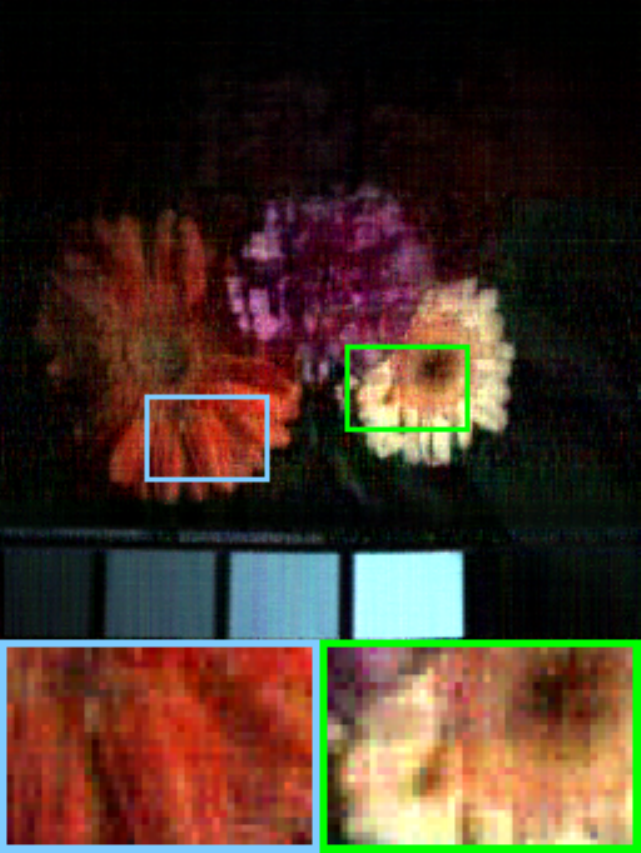}&
              \includegraphics[width=0.12\textwidth]{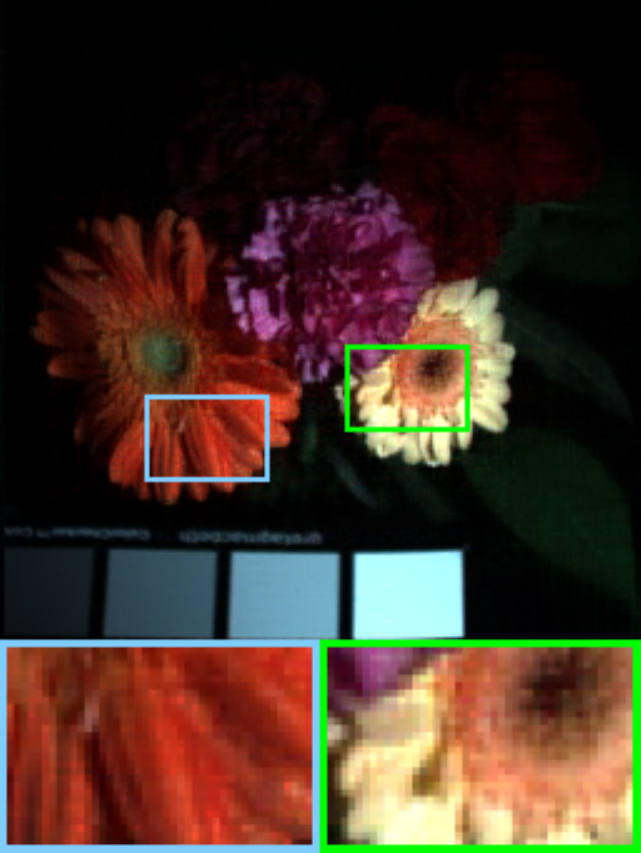}&
               \includegraphics[width=0.12\textwidth]{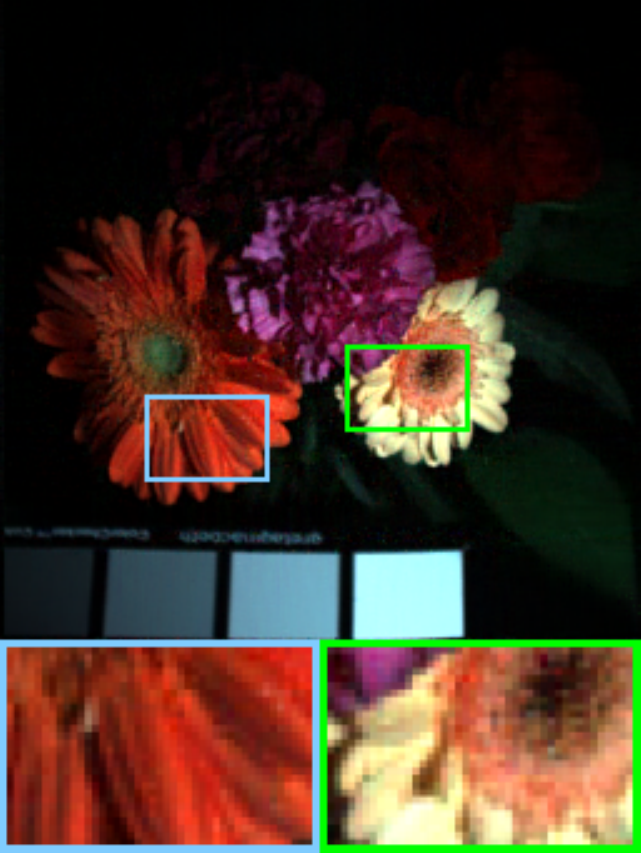}&
                \includegraphics[width=0.12\textwidth]{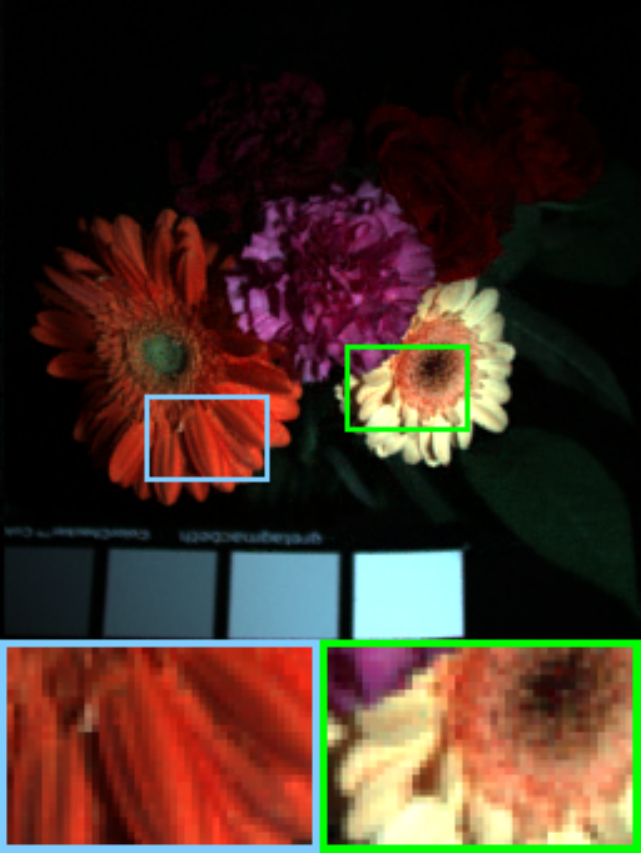}&
                 \includegraphics[width=0.12\textwidth]{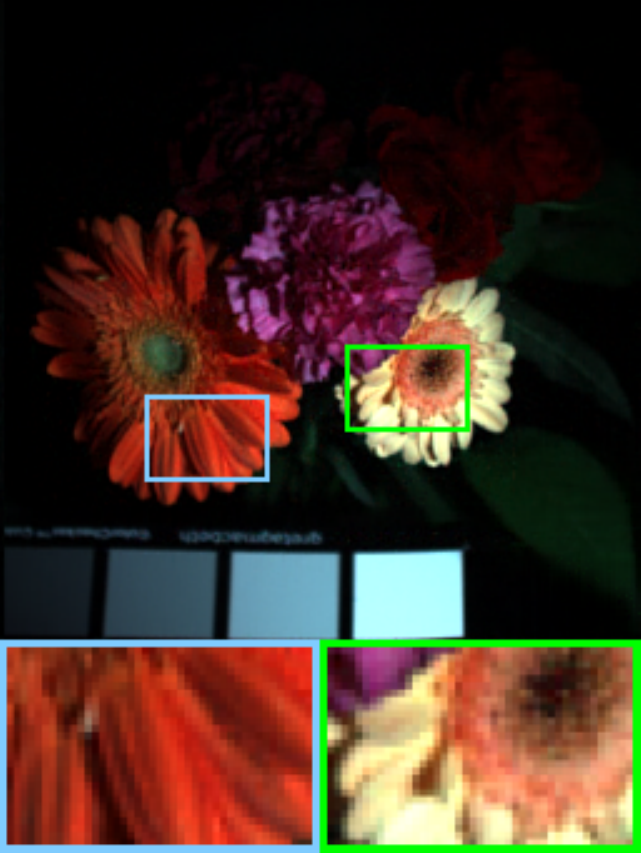}&
            \includegraphics[width=0.12\textwidth]{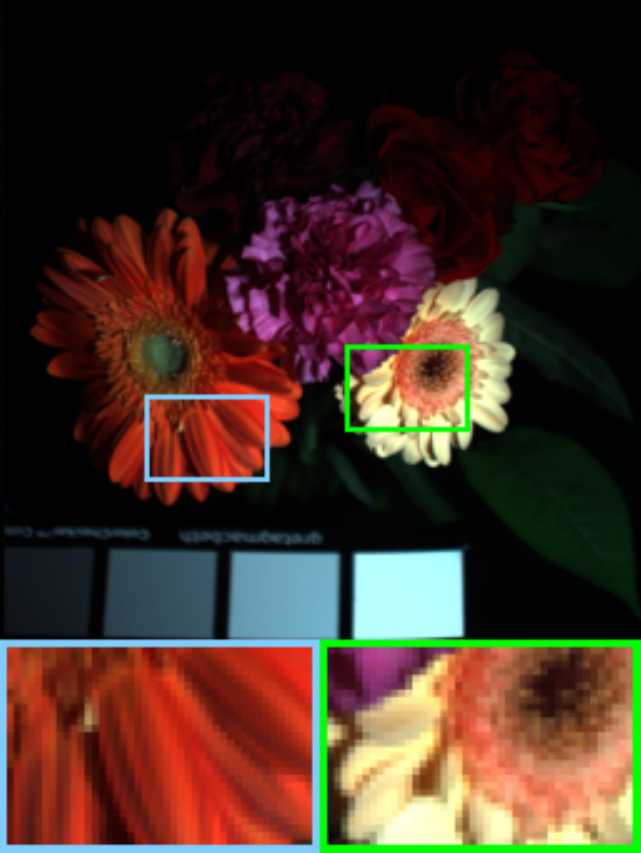}\\              
          PSNR 13.78  &
          PSNR 31.66  &
          PSNR 28.59  &
          PSNR 33.93  &
          PSNR 36.01  &
          PSNR 38.47  &
          PSNR 39.31  &
          PSNR Inf\\
           \includegraphics[width=0.12\textwidth]{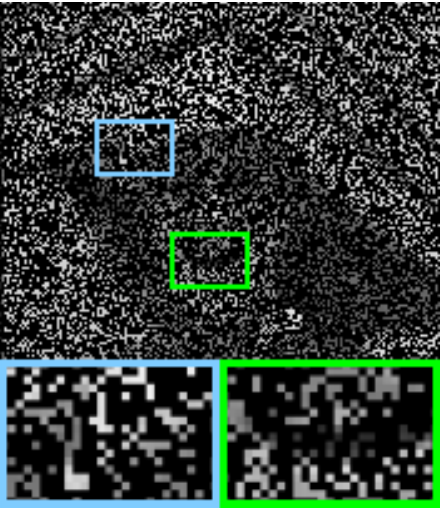}&
          \includegraphics[width=0.12\textwidth]{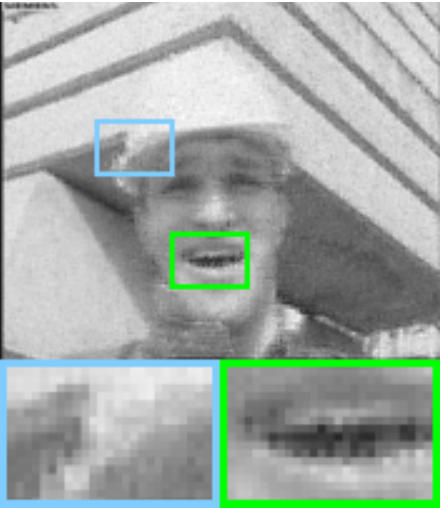}&
           \includegraphics[width=0.12\textwidth]{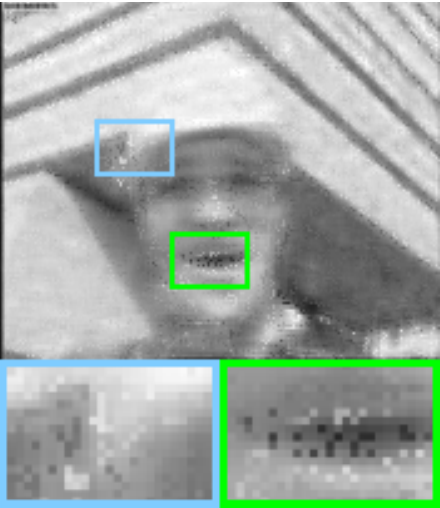}&
            \includegraphics[width=0.12\textwidth]{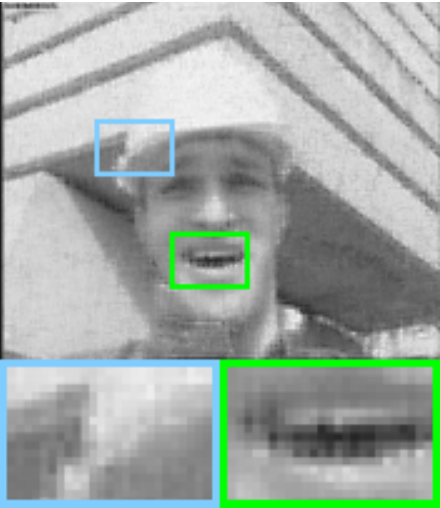}&
             \includegraphics[width=0.12\textwidth]{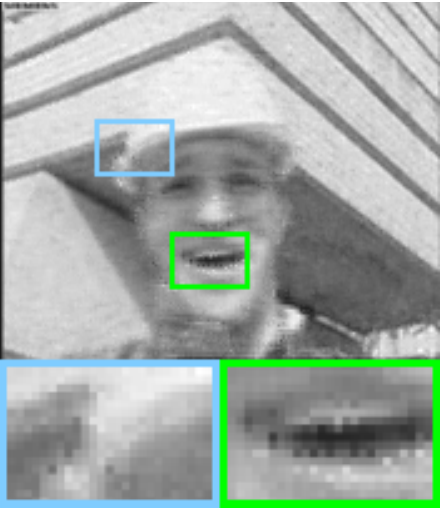}&
              \includegraphics[width=0.12\textwidth]{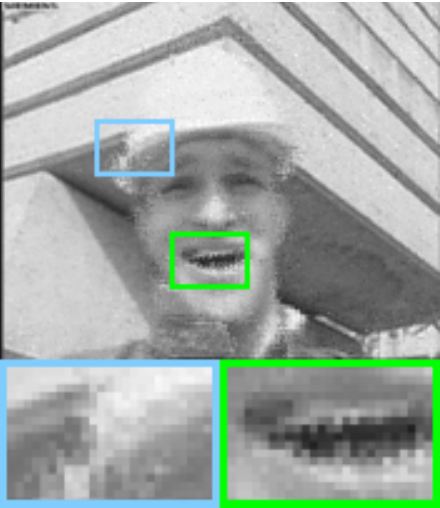}&
               \includegraphics[width=0.12\textwidth]{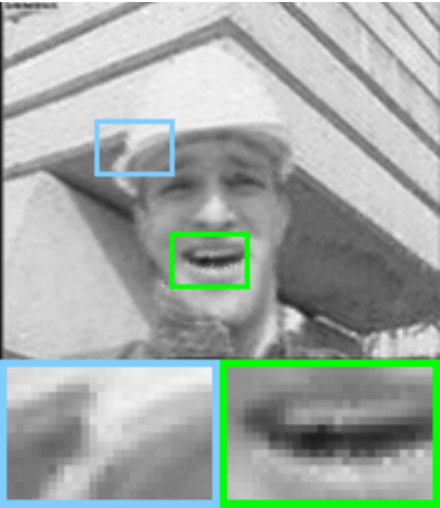}&
          \includegraphics[width=0.12\textwidth]{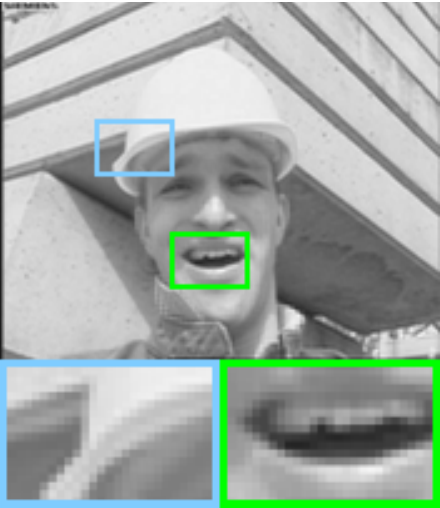}\\
      PSNR 4.70  &
      PSNR 28.93  &
      PSNR 26.30  &
      PSNR 28.01  &
      PSNR 29.70  &
      PSNR 29.65  &
      PSNR 31.74  &
      PSNR Inf\\
                 \includegraphics[width=0.12\textwidth]{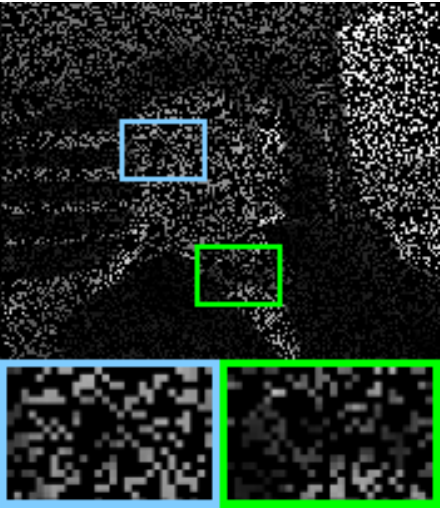}&
                \includegraphics[width=0.12\textwidth]{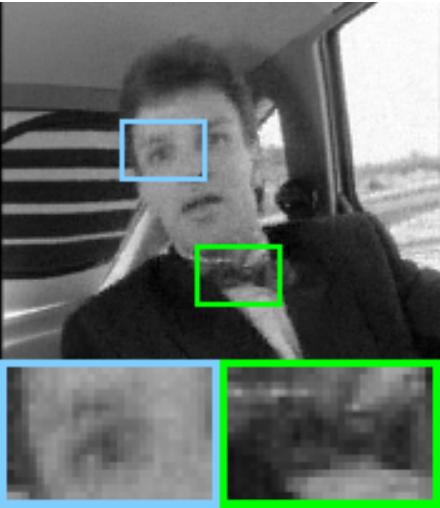}&
                 \includegraphics[width=0.12\textwidth]{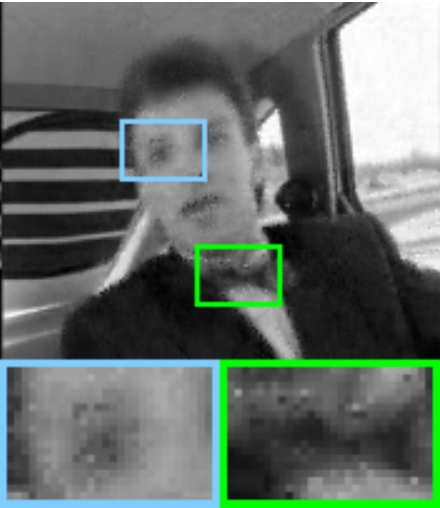}&
                  \includegraphics[width=0.12\textwidth]{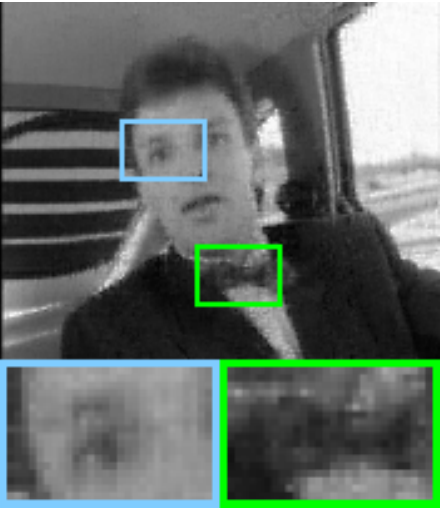}&
                   \includegraphics[width=0.12\textwidth]{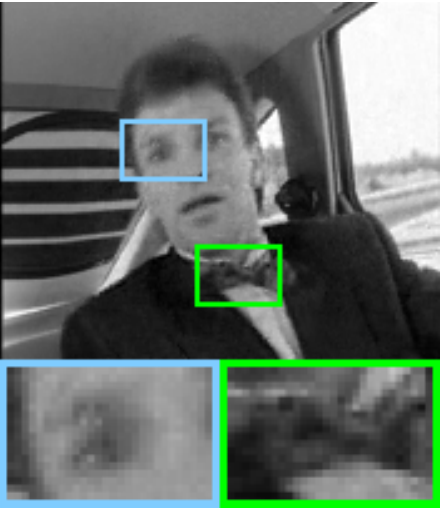}&
                    \includegraphics[width=0.12\textwidth]{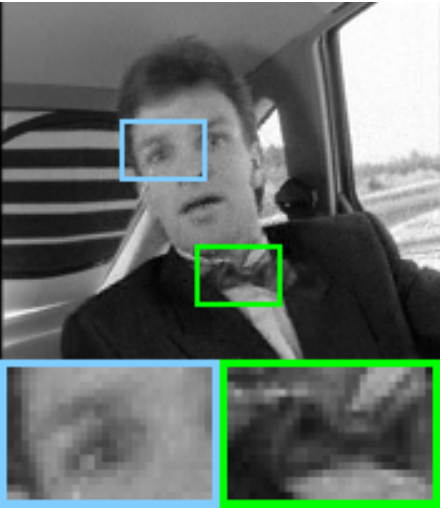}&
                     \includegraphics[width=0.12\textwidth]{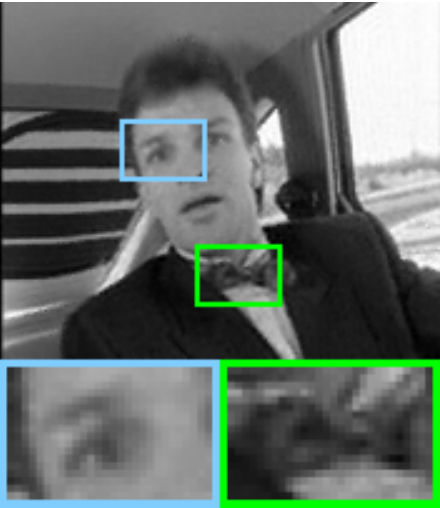}&
                \includegraphics[width=0.12\textwidth]{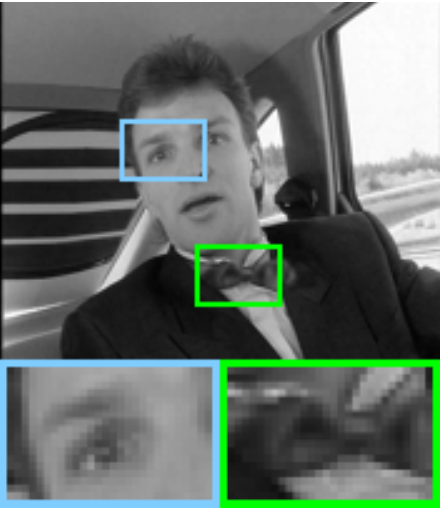}\\
      PSNR 7.89  &
      PSNR 30.58  &
      PSNR 28.44  &
      PSNR 28.37  &
      PSNR 31.72  &
      PSNR 31.96  &
      PSNR 32.36  &
      PSNR Inf\\
          Observed &DCTNN\cite{CVPR_19} &TRLRF\cite{TRLRF}&FTNN\cite{FTNN}&FCTN\cite{FCTN}&HLRTF\cite{HLRTF}&LRTFR & Original\\
          \end{tabular}
          \end{center}
          \vspace{-0.5cm}
          \caption{The results of multi-dimensional image inpainting by different methods on MSIs {\it Toys} and {\it Flowers} (SR = 0.1) and videos {\it Foreman} and {\it Carphone} (SR = 0.3).\label{fig_completion_2}}\vspace{-0.3cm}
          \end{figure*}
\begin{figure*}[t]
    \scriptsize
    \setlength{\tabcolsep}{0.9pt}
    \begin{center}
    \begin{tabular}{cccccccc}
        \includegraphics[width=0.12\textwidth]{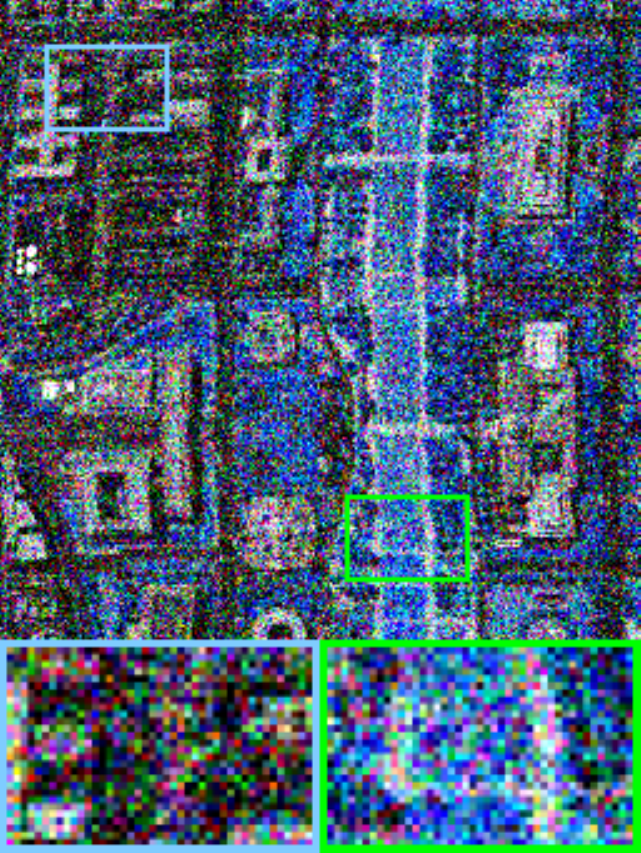}&
         \includegraphics[width=0.12\textwidth]{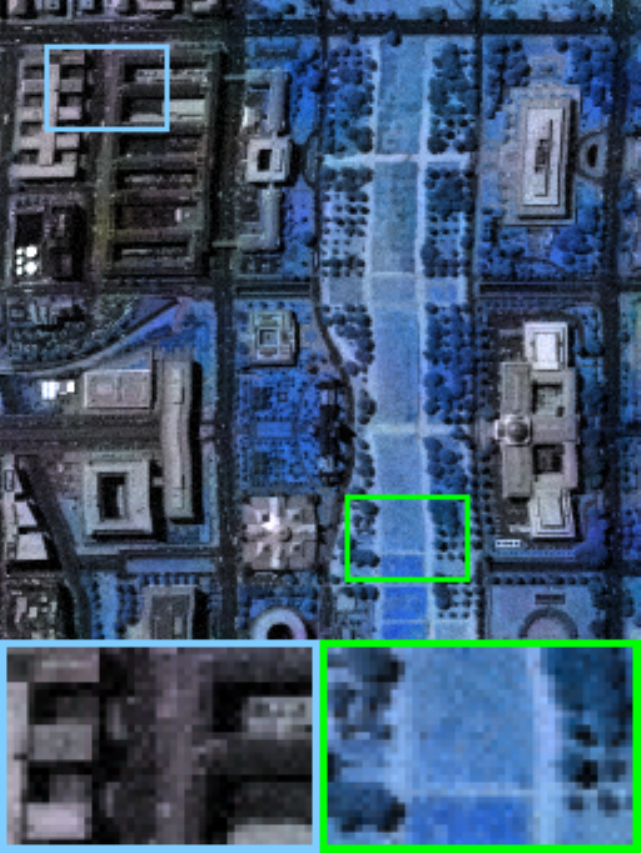}&
        \includegraphics[width=0.12\textwidth]{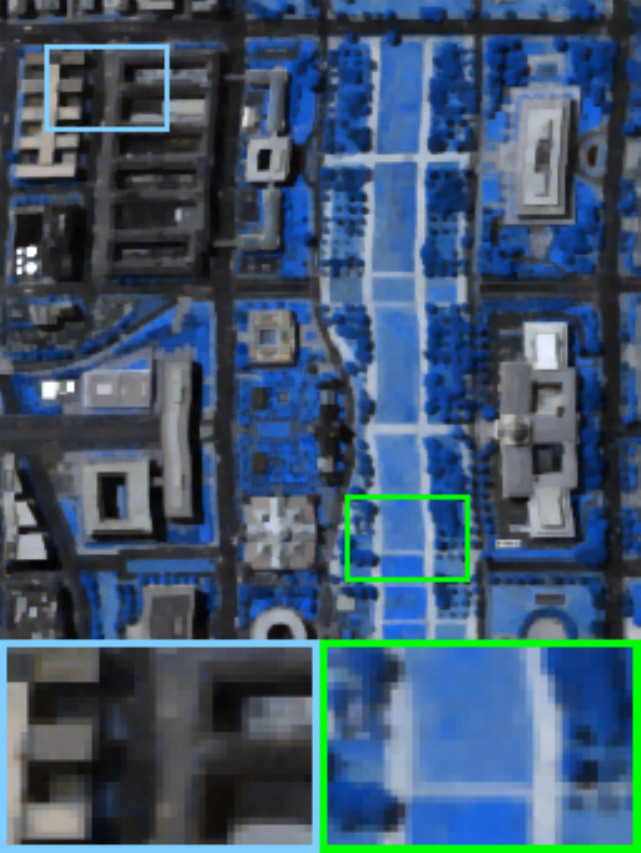}&
         \includegraphics[width=0.12\textwidth]{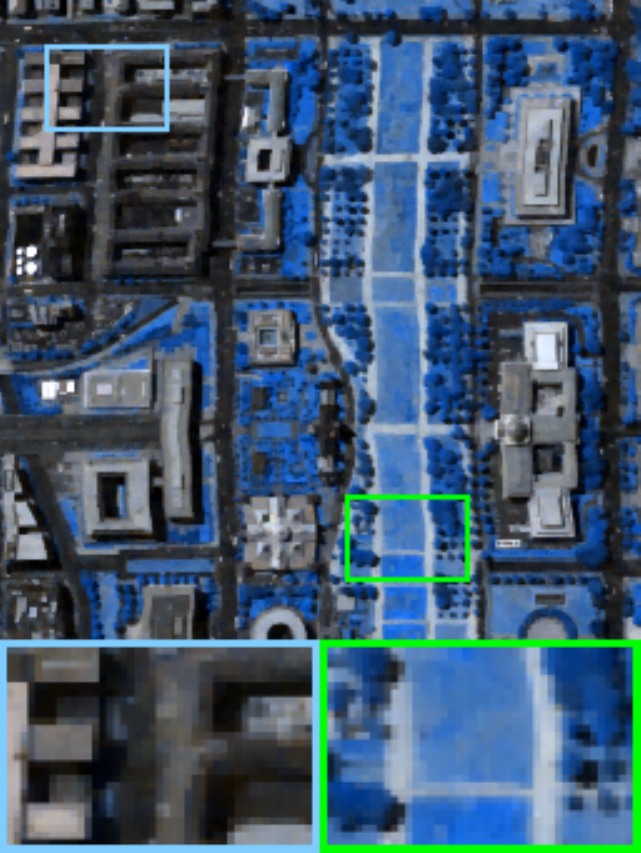}&
          \includegraphics[width=0.12\textwidth]{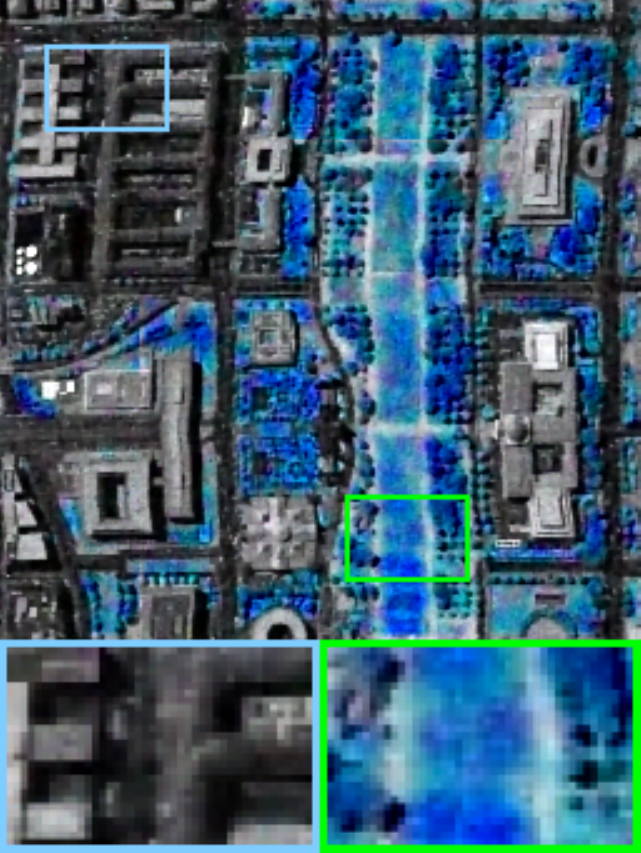}&
           \includegraphics[width=0.12\textwidth]{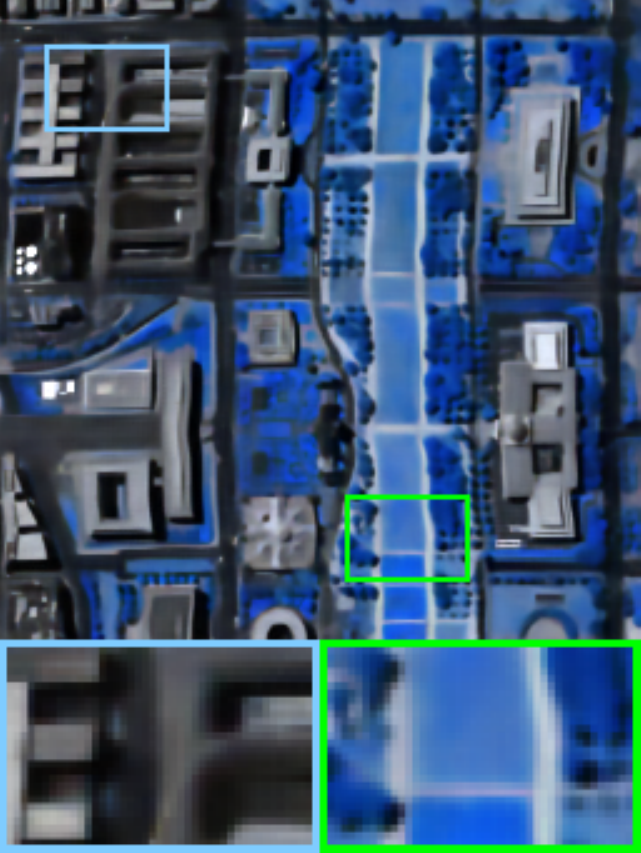}&
            \includegraphics[width=0.12\textwidth]{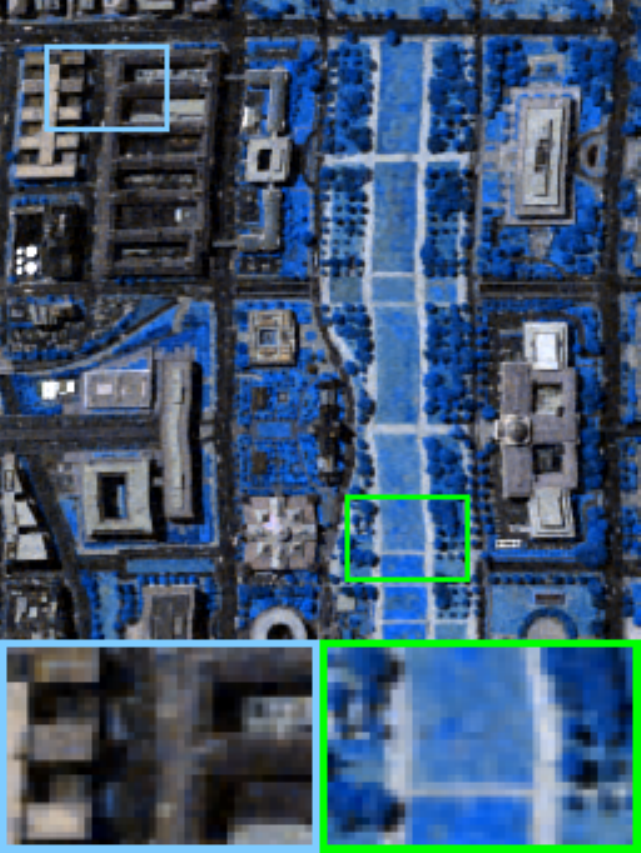}&
            \includegraphics[width=0.12\textwidth]{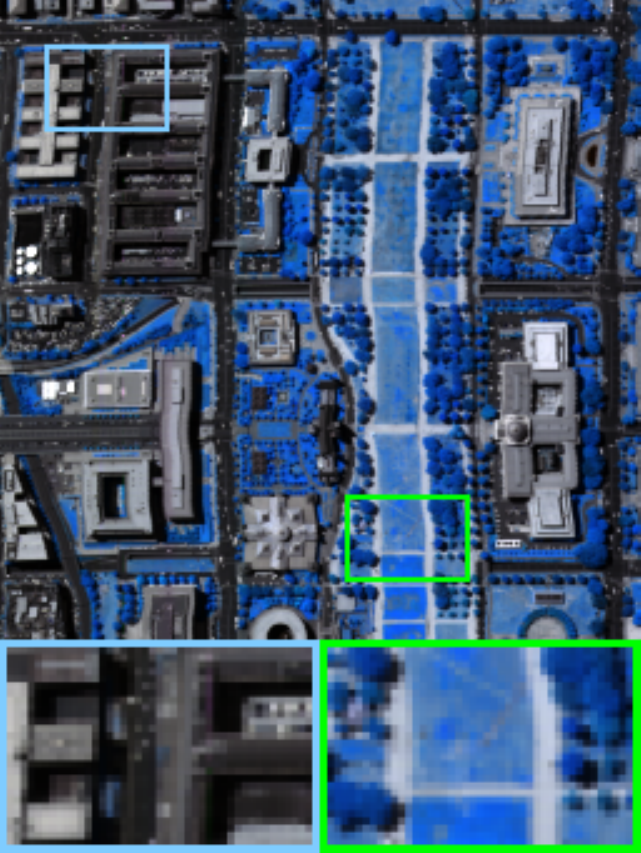}\\
PSNR 15.65  &
PSNR 28.36  &
PSNR 29.95  &
PSNR 32.43  &
PSNR 29.04  &
PSNR 28.90  &
PSNR 33.28  &
PSNR Inf\\
        \includegraphics[width=0.12\textwidth]{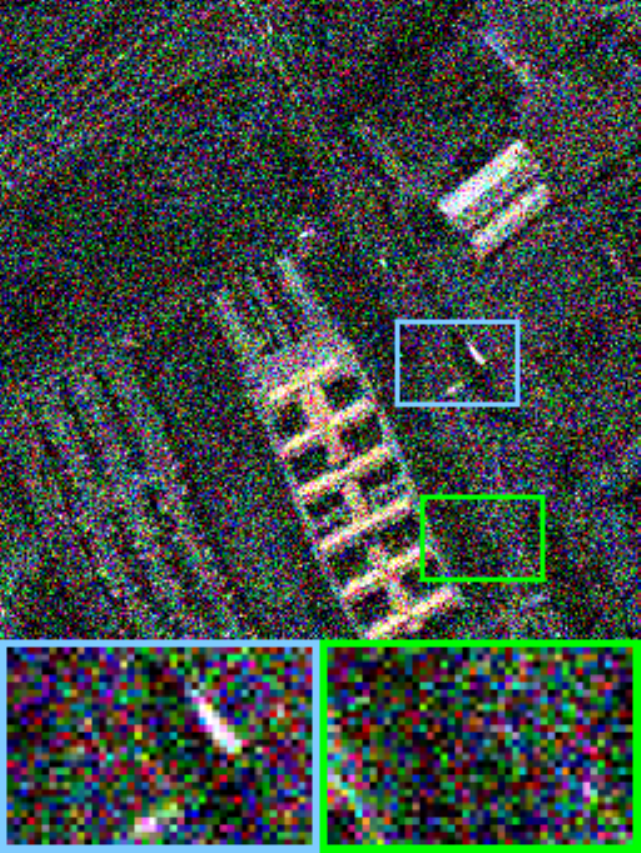}&
         \includegraphics[width=0.12\textwidth]{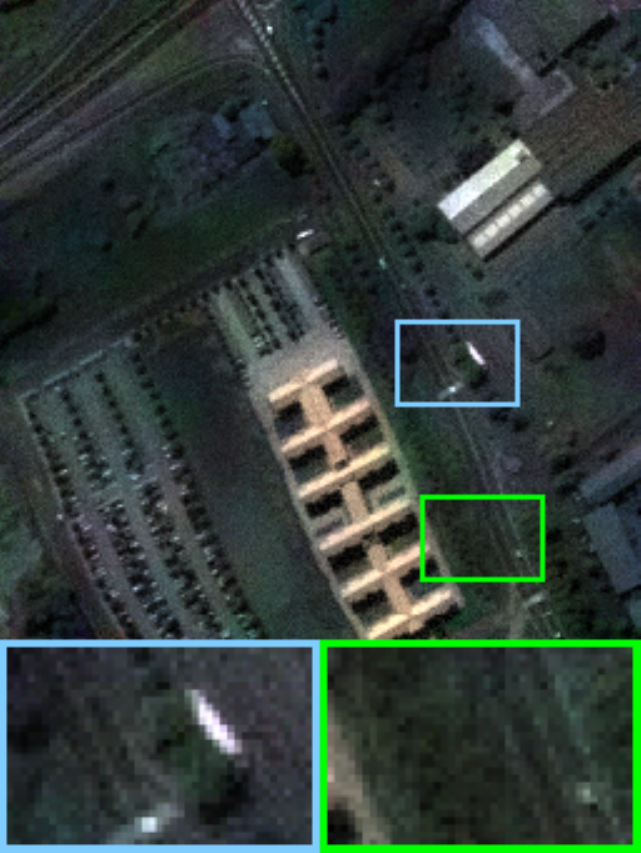}&
        \includegraphics[width=0.12\textwidth]{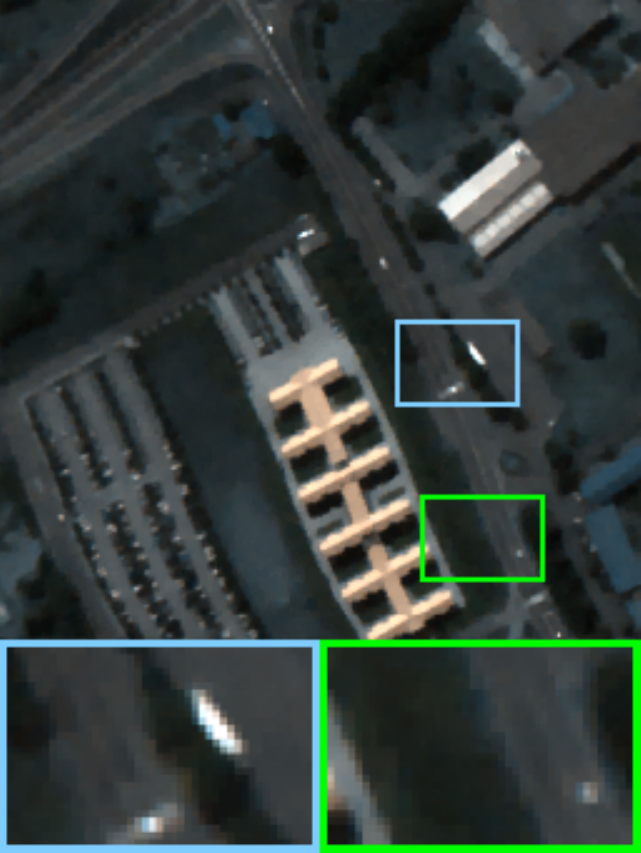}&
         \includegraphics[width=0.12\textwidth]{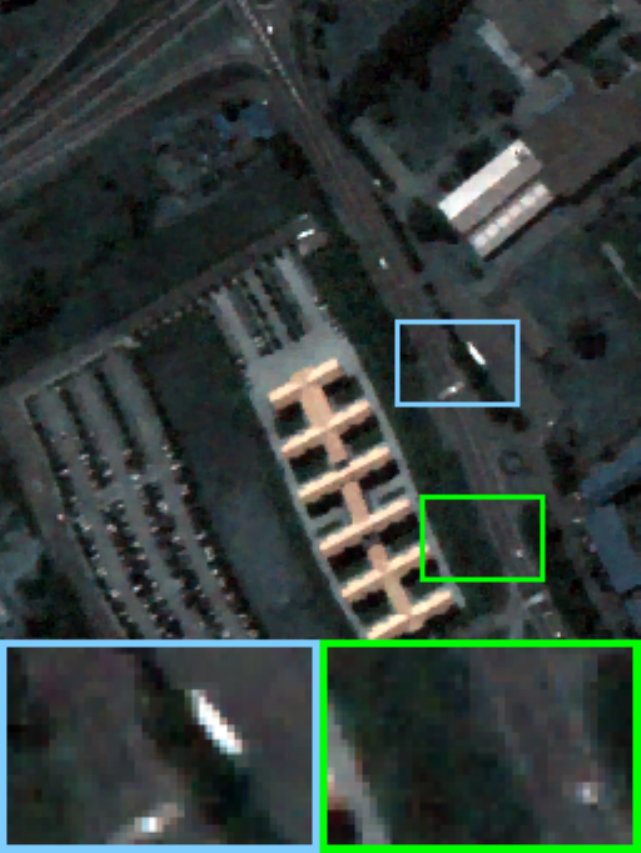}&
          \includegraphics[width=0.12\textwidth]{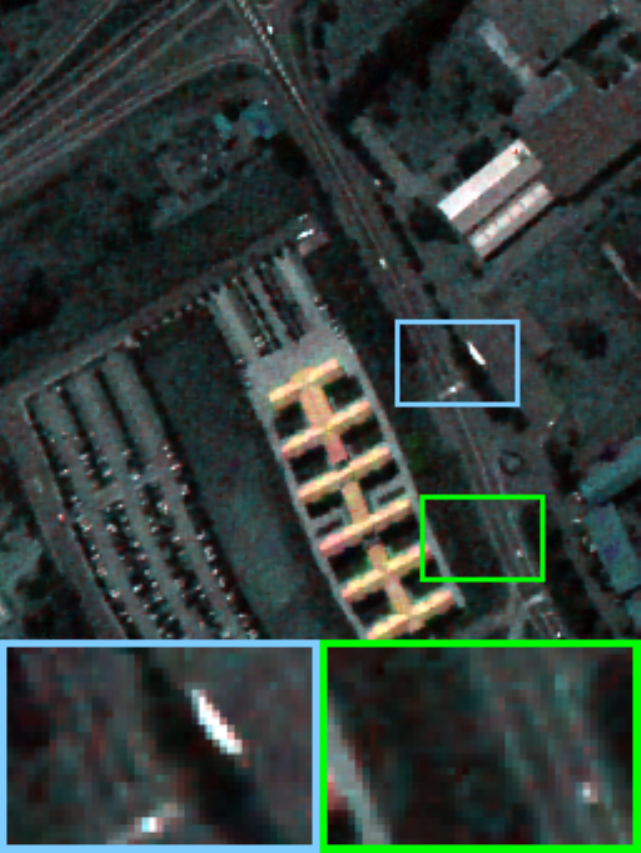}&
           \includegraphics[width=0.12\textwidth]{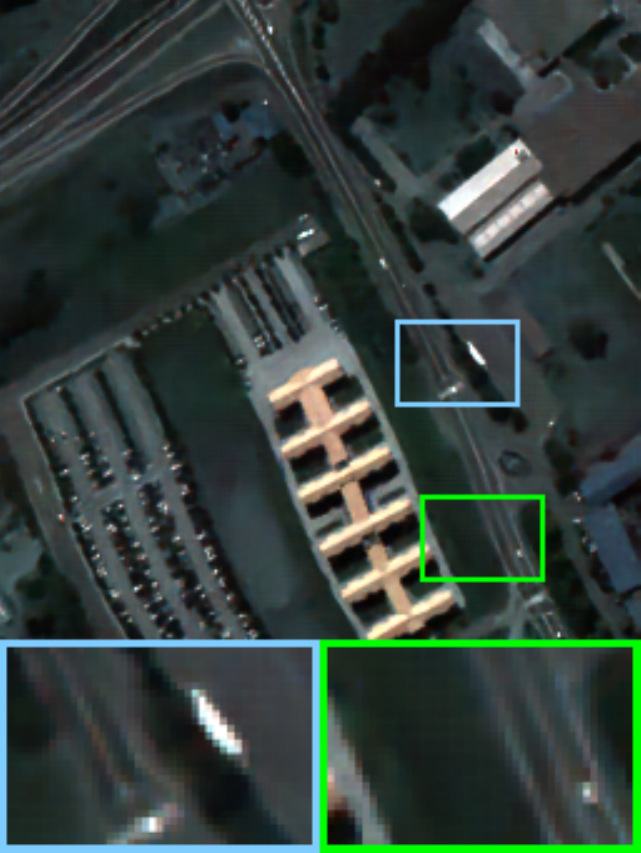}&
            \includegraphics[width=0.12\textwidth]{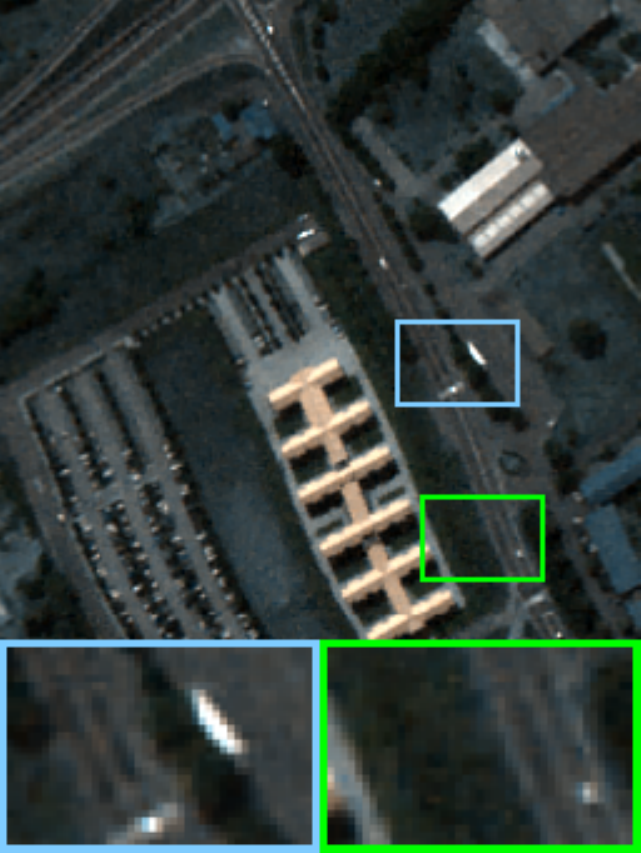}&
            \includegraphics[width=0.12\textwidth]{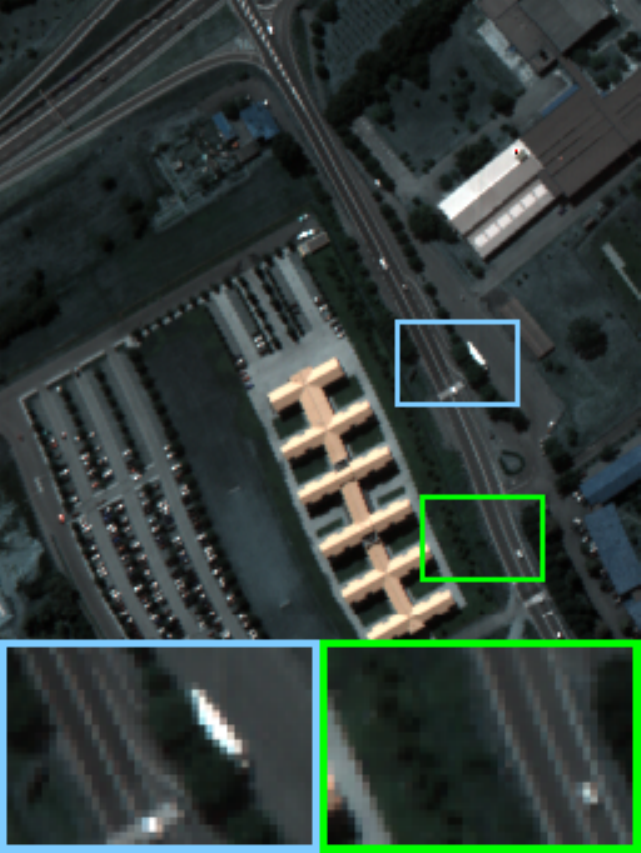}\\
PSNR 15.47  &
PSNR 27.63  &
PSNR 31.67  &
PSNR 31.43  &
PSNR 30.30  &
PSNR 33.25  &
PSNR 32.39  &
PSNR Inf\\
      Observed &LRMR\cite{LRMR}&LRTDTV\cite{LRTDTV}&E3DTV\cite{E3DTV}&HSID-CNN\cite{HSIDCNN}&SDeCNN\cite{SDeCNN}&LRTFR & Original\\
    \end{tabular}
    \end{center}
    \vspace{-0.3cm}
    \caption{The results of multispectral image denoising by different methods on HSI {\it WDC mall} and {\it University} (Case 1).\label{fig_denoising}}\vspace{-0.5cm}
    \end{figure*}
\begin{table*}[t]
    \caption{The average quantitative results by different methods for multispectral image denoising. The {\bf best} and \underline{second-best} values are highlighted. (PSNR $\uparrow$, SSIM $\uparrow$, and NRMSE $\downarrow$)\label{tab_denoising}}\vspace{-0.4cm}
      \begin{center}
      \scriptsize
      \setlength{\tabcolsep}{2.5pt}
      \begin{spacing}{1.2}
      \begin{tabular}{clccccccccccccccc}
      \toprule
      \multicolumn{2}{c}{Case}&\multicolumn{3}{c}{Case 1}&\multicolumn{3}{c}{Case 2}&\multicolumn{3}{c}{Case 3}&\multicolumn{3}{c}{Case 4}&\multicolumn{3}{c}{Case 5}\\
      \cmidrule{1-17}
      Data&Method&PSNR &SSIM &NRMSE \;\; &PSNR &SSIM &NRMSE \;\; &PSNR &SSIM &NRMSE \;\; &PSNR &SSIM&NRMSE \;\;&PSNR &SSIM&NRMSE \\
  \midrule
  \multirow{7}*{\tabincell{c}{
  HSIs\\ {\it WDC mall}\\{(256$\times$256$\times$191)}\\{\it University}\\{(256$\times$256$\times$103)}}}
&Observed&{15.56}&{0.267}&{0.667}\;\;&{13.88}&{0.239}&{0.670}\;\;&{13.65}&{0.231}&{0.672}\;\;&{13.85}&{0.235}&{0.681}\;\;&{13.60}&{0.226}&{0.685}\\
&LRMR&{27.99}&{0.872}&{0.192}\;\;&{30.45}&{0.925}&{0.141}\;\;&{29.81}&{0.917}&{0.149}\;\;&{29.50}&{0.917}&{0.163}\;\;&{29.03}&{0.909}&{0.170} \\
&LRTDTV&{30.81}&{0.919}&{0.138}\;\;&{33.15}&{0.951}&{0.103}\;\;&{32.96}&{0.950}&{0.105}\;\;&{32.64}&{0.947}&{0.114}\;\;&{32.58}&{0.946}&{0.113} \\
&E3DTV&\underline{31.93}&\underline{0.937}&\underline{0.122}\;\;&\underline{34.11}&\underline{0.962}&\underline{0.097}\;\;&\underline{33.49}&\underline{0.957}&\underline{0.103}\;\;&\underline{33.47}&\underline{0.959}&\underline{0.107}\;\;&\underline{33.35}&\underline{0.957}&\underline{0.108} \\
&HSID-CNN&{29.67}&{0.905}&{0.160}\;\;&{25.42}&{0.827}&{0.224}\;\;&{24.15}&{0.800}&{0.255}\;\;&{25.21}&{0.817}&{0.237}\;\;&{23.98}&{0.787}&{0.270} \\
&SDeCNN&{31.07}&{0.916}&{0.134}\;\;&{26.12}&{0.848}&{0.206}\;\;&{24.97}&{0.847}&{0.230}\;\;&{26.33}&{0.843}&{0.209}\;\;&{25.00}&{0.838}&{0.239} \\
&LRTFR&\bf{32.83}&\bf{0.947}&\bf{0.107}\;\;&\bf{34.51}&\bf{0.964}&\bf{0.087}\;\;&\bf{34.17}&\bf{0.962}&\bf{0.091}\;\;&\bf{34.43}&\bf{0.963}&\bf{0.088}\;\;&\bf{34.09}&\bf{0.961}&\bf{0.096} \\
  \midrule
  \multirow{7}*{\tabincell{c}{
  MSIs\\ {\it Cups}\\{\it Fruits}\\{(256$\times$256$\times$31)}}}
&Observed&{15.61}&{0.145}&{0.658}\;\;&{13.86}&{0.133}&{0.667}\;\;&{13.62}&{0.129}&{0.671}\;\;&{13.78}&{0.133}&{0.680}\;\;&{13.54}&{0.128}&{0.683}\\

&LRMR&{28.58}&{0.791}&{0.213}\;\;&{31.38}&{0.877}&{0.144}\;\;&{30.72}&{0.862}&{0.153}\;\;&{30.36}&{0.872}&{0.164}\;\;&{29.77}&{0.856}&{0.172} \\

&LRTDTV&{33.54}&{0.932}&{0.125}\;\;&{34.96}&{0.934}&\underline{0.105}\;\;&\underline{34.51}&{0.925}&\underline{0.109}\;\;&{34.34}&{0.935}&\underline{0.116}\;\;&{33.82}&{0.923}&\underline{0.124}\\

&E3DTV&{32.85}&{0.929}&{0.143}\;\;&\underline{34.97}&\underline{0.952}&{0.124}\;\;&{34.47}&\underline{0.941}&{0.127}\;\;&\underline{34.69}&\bf{0.961}&{0.133}\;\;&\underline{34.38}&\bf{0.952}&{0.134}\\

&HSID-CNN&{30.41}&{0.877}&{0.157}\;\;&{25.61}&{0.676}&{0.258}\;\;&{24.39}&{0.646}&{0.297}\;\;&{24.90}&{0.662}&{0.279}\;\;&{23.89}&{0.632}&{0.313}\\

&SDeCNN&\underline{33.98}&\underline{0.936}&\underline{0.113}\;\;&{28.13}&{0.797}&{0.198}\;\;&{25.27}&{0.698}&{0.269}\;\;&{28.41}&{0.800}&{0.198}\;\;&{25.03}&{0.690}&{0.277} \\

&LRTFR&\bf{34.51}&\bf{0.955}&\bf{0.103}\;\;&\bf{35.92}&\bf{0.958}&\bf{0.086}\;\;&\bf{35.07}&\bf{0.945}&\bf{0.095}\;\;&\bf{35.83}&\underline{0.957}&\bf{0.087}\;\;&\bf{35.18}&\underline{0.945}&\bf{0.095}\\
\midrule
  \multirow{7}*{\tabincell{c}{
  MSIs\\ {\it Bin}\\{\it Board}\\{(256$\times$256$\times$32)}}}
&Observed&{15.43}&{0.183}&{0.610}\;\;&{13.97}&{0.166}&{0.617}\;\;&{13.73}&{0.159}&{0.622}\;\;&{13.87}&{0.163}&{0.631}\;\;&{13.70}&{0.158}&{0.633} \\

&LRMR&{27.84}&{0.785}&{0.164}\;\;&{31.44}&{0.894}&\underline{0.105}\;\;&{30.70}&{0.884}&{0.115}\;\;&{30.36}&{0.886}&{0.124}\;\;&{29.91}&{0.877}&{0.128}\\

&LRTDTV&\underline{30.32}&{0.929}&\underline{0.126}\;\;&\underline{31.89}&{0.944}&\underline{0.105}\;\;&\underline{31.68}&{0.941}&\underline{0.106}\;\;&\underline{31.03}&{0.933}&\underline{0.123}\;\;&\underline{31.06}&{0.931}&\underline{0.120} \\

&E3DTV&{30.11}&{0.928}&{0.136}\;\;&{31.11}&\underline{0.948}&{0.131}\;\;&{31.11}&\bf{0.953}&{0.125}\;\;&{30.76}&\underline{0.950}&{0.133}\;\;&{30.85}&\bf{0.952}&{0.128} \\

&HSID-CNN&{28.43}&{0.859}&{0.149}\;\;&{25.10}&{0.787}&{0.204}\;\;&{23.93}&{0.759}&{0.231}\;\;&{24.57}&{0.770}&{0.223}\;\;&{23.66}&{0.747}&{0.242} \\

&SDeCNN&{29.91}&\underline{0.931}&{0.129}\;\;&{26.68}&{0.889}&{0.173}\;\;&{24.89}&{0.828}&{0.205}\;\;&{26.92}&{0.889}&{0.173}\;\;&{25.02}&{0.821}&{0.208}\\

&LRTFR&\bf{32.55}&\bf{0.938}&\bf{0.097}\;\;&\bf{34.23}&\bf{0.955}&\bf{0.078}\;\;&\bf{33.54}&\underline{0.948}&\bf{0.085}\;\;&\bf{34.16}&\bf{0.954}&\bf{0.079}\;\;&\bf{33.59}&\underline{0.950}&\bf{0.084}\\
  \bottomrule
  \end{tabular}
  \end{spacing}
  \end{center}
  \vspace{-0.7cm}
  \end{table*}    
  \begin{figure*}[t]
      \scriptsize
      \setlength{\tabcolsep}{0.9pt}
      \begin{center}
      \begin{tabular}{cccccccc}
          \includegraphics[width=0.12\textwidth]{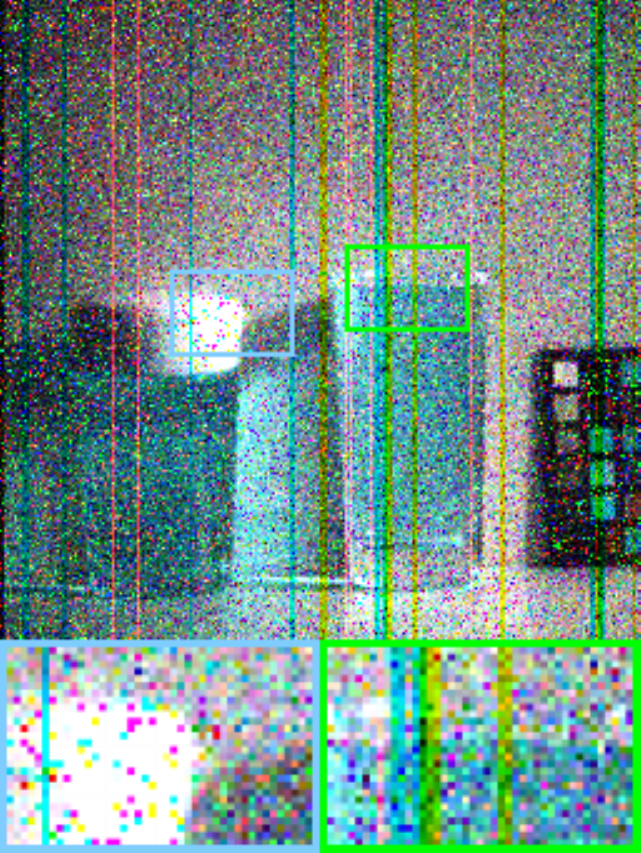}&
           \includegraphics[width=0.12\textwidth]{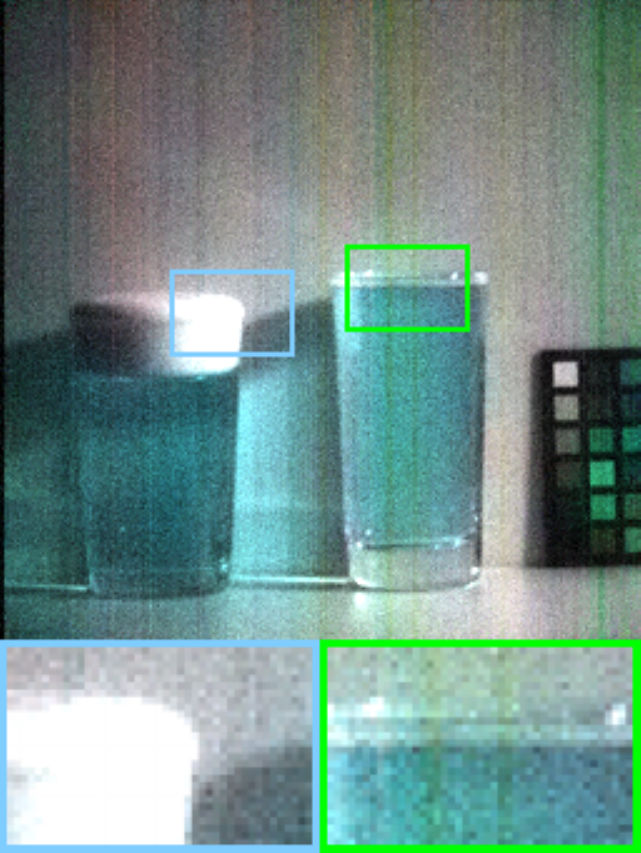}&
          \includegraphics[width=0.12\textwidth]{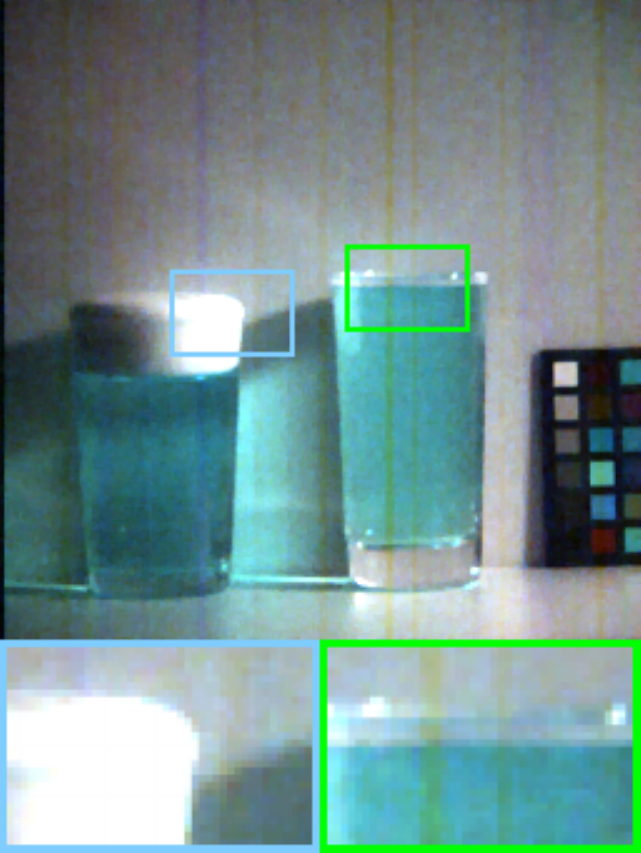}&
           \includegraphics[width=0.12\textwidth]{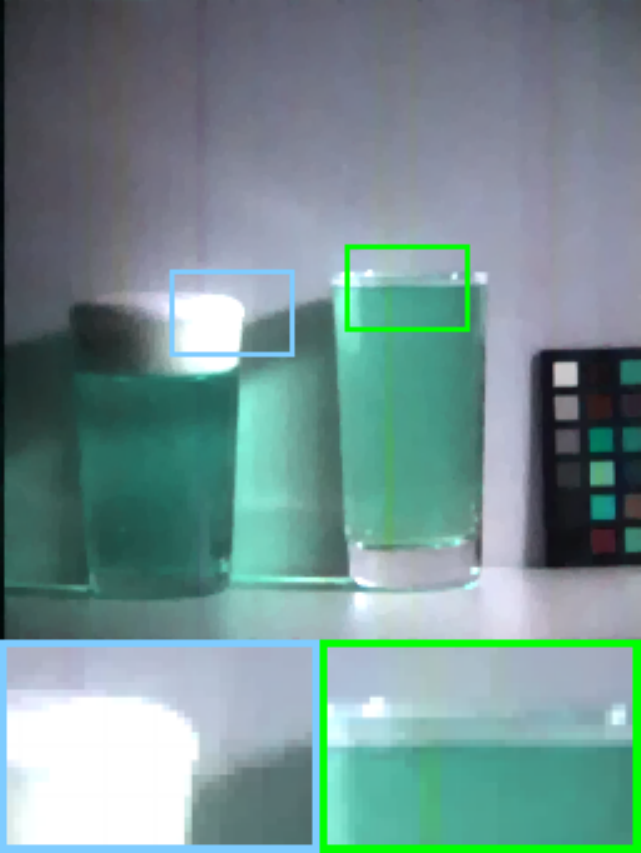}&
            \includegraphics[width=0.12\textwidth]{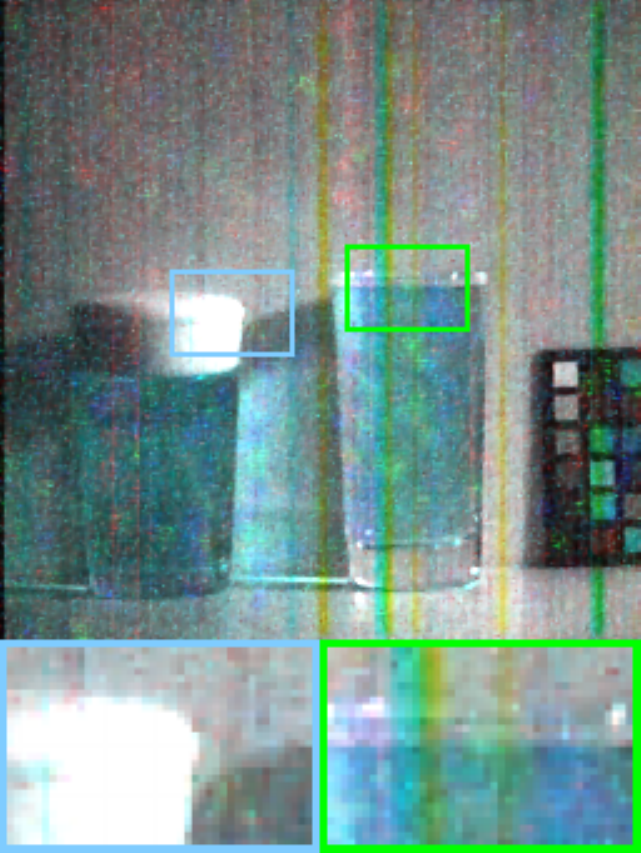}&
             \includegraphics[width=0.12\textwidth]{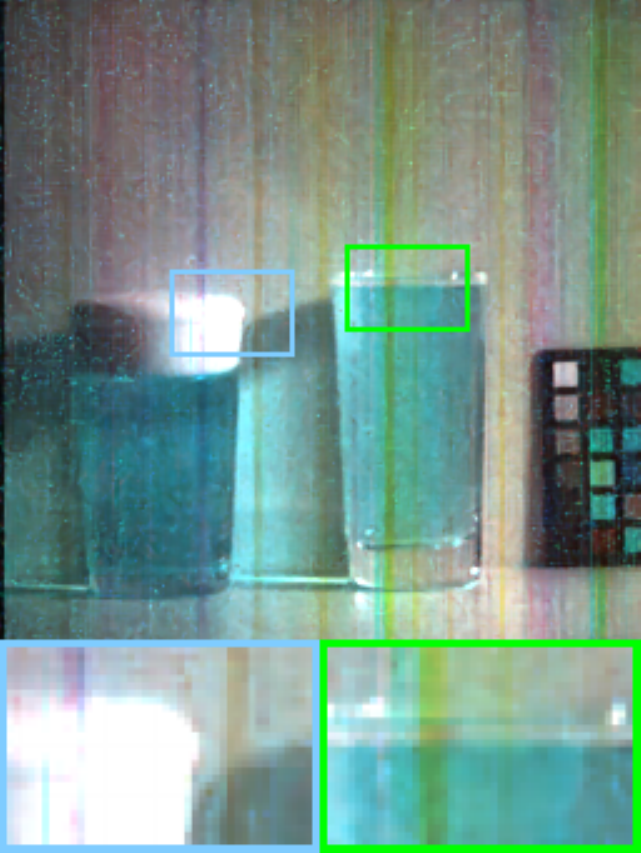}&
              \includegraphics[width=0.12\textwidth]{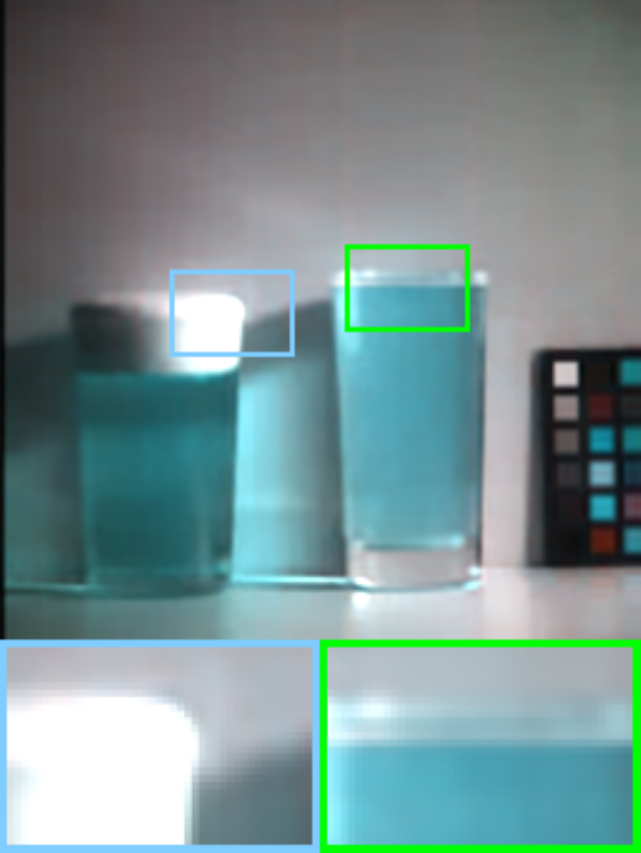}&
              \includegraphics[width=0.12\textwidth]{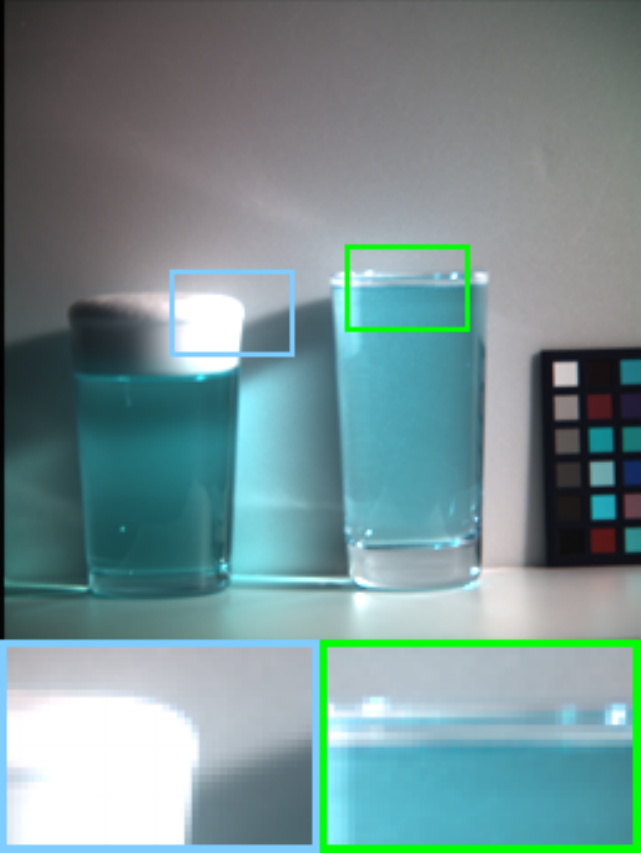}\\
PSNR 14.03  &
PSNR 29.03  &
PSNR 33.78  &
PSNR 34.76  &
PSNR 24.96  &
PSNR 26.77  &
PSNR 36.21  &
PSNR Inf\\
          \includegraphics[width=0.12\textwidth]{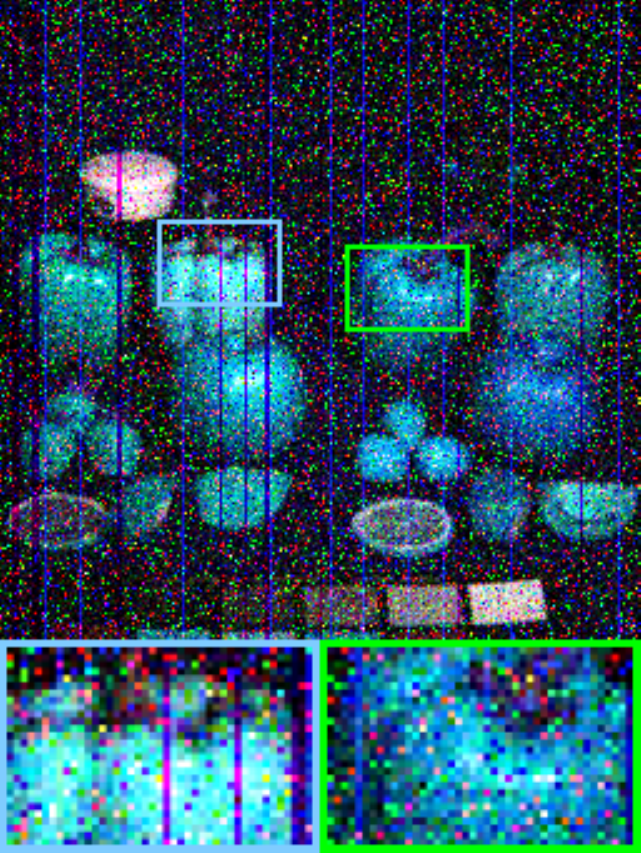}&
           \includegraphics[width=0.12\textwidth]{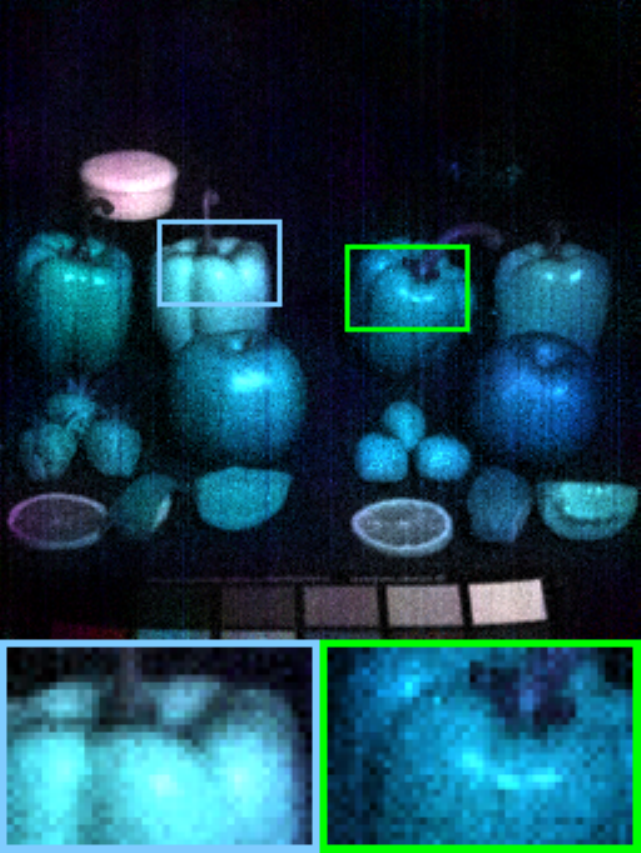}&
          \includegraphics[width=0.12\textwidth]{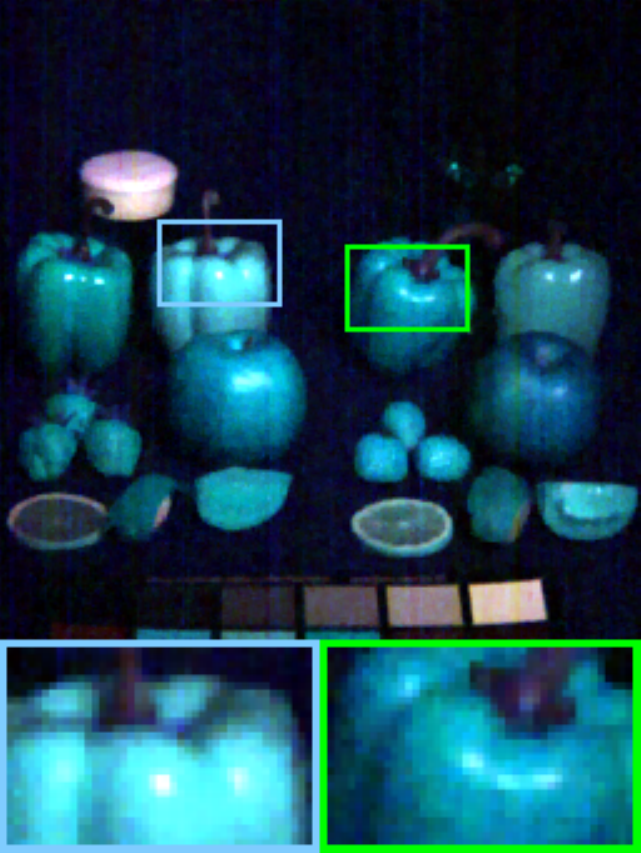}&
           \includegraphics[width=0.12\textwidth]{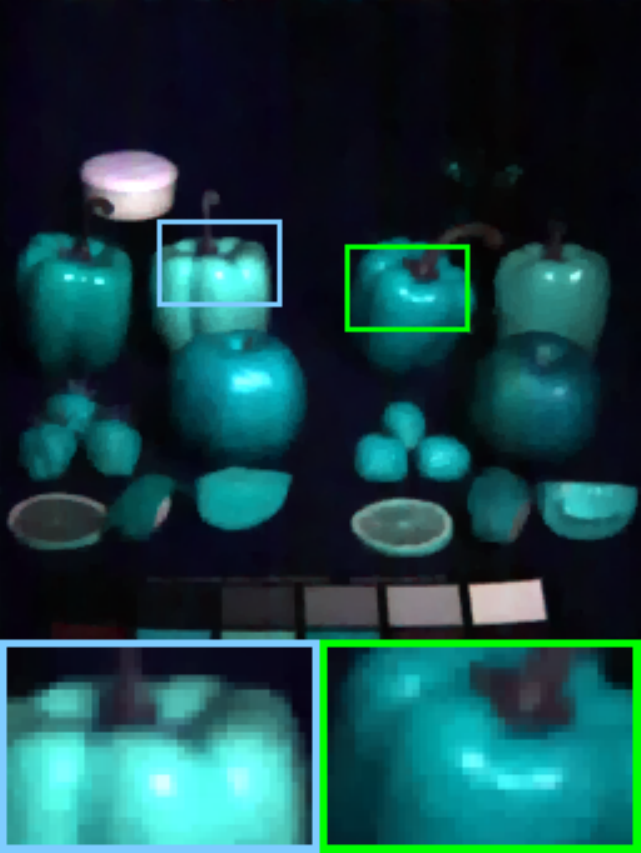}&
            \includegraphics[width=0.12\textwidth]{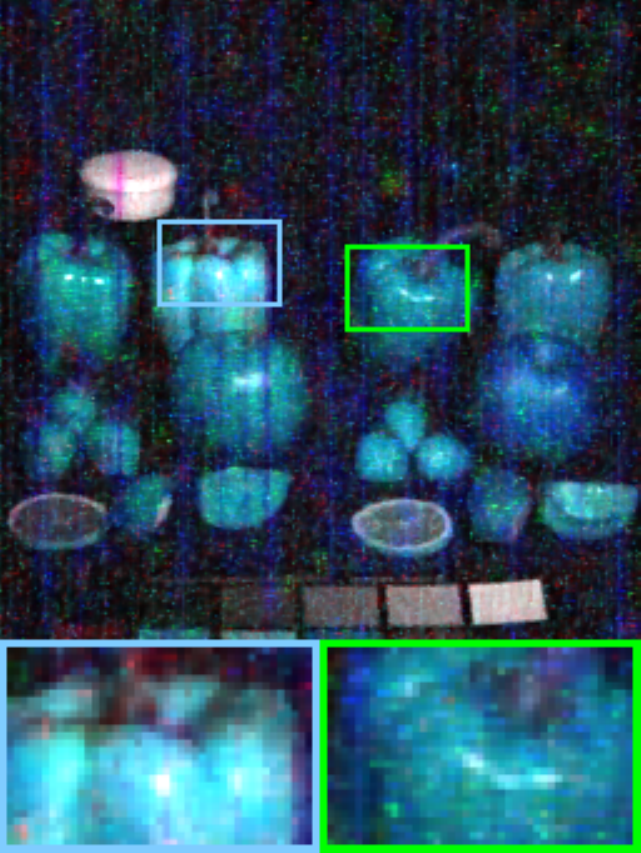}&
             \includegraphics[width=0.12\textwidth]{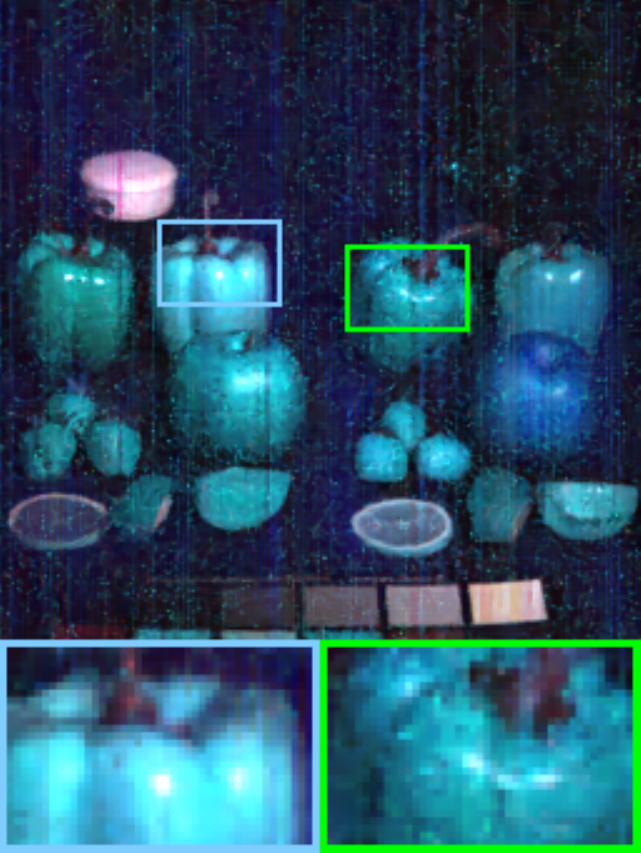}&
              \includegraphics[width=0.12\textwidth]{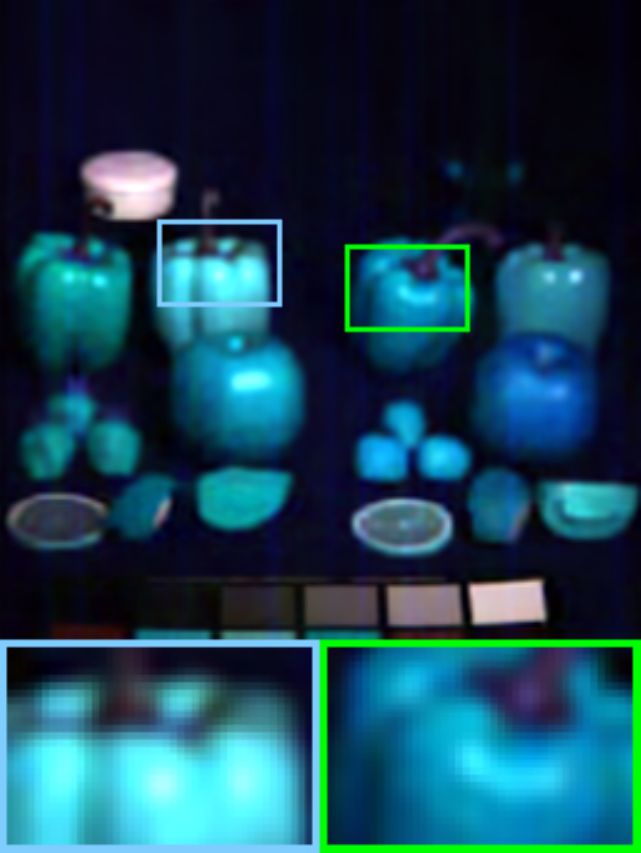}&
              \includegraphics[width=0.12\textwidth]{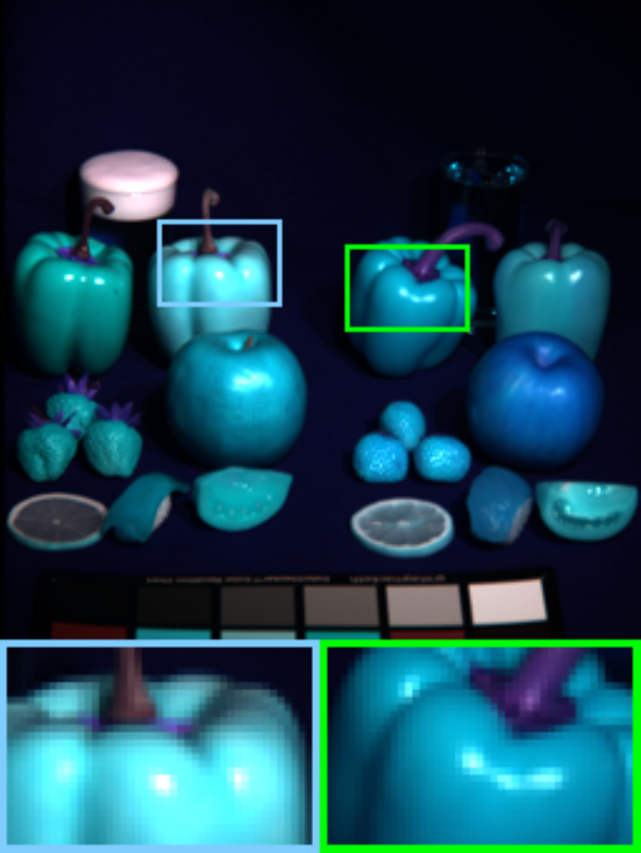}\\
PSNR 13.05  &
PSNR 30.52  &
PSNR 33.86  &
PSNR 34.00  &
PSNR 22.82  &
PSNR 23.28  &
PSNR 34.15  &
PSNR Inf\\
          \includegraphics[width=0.12\textwidth]{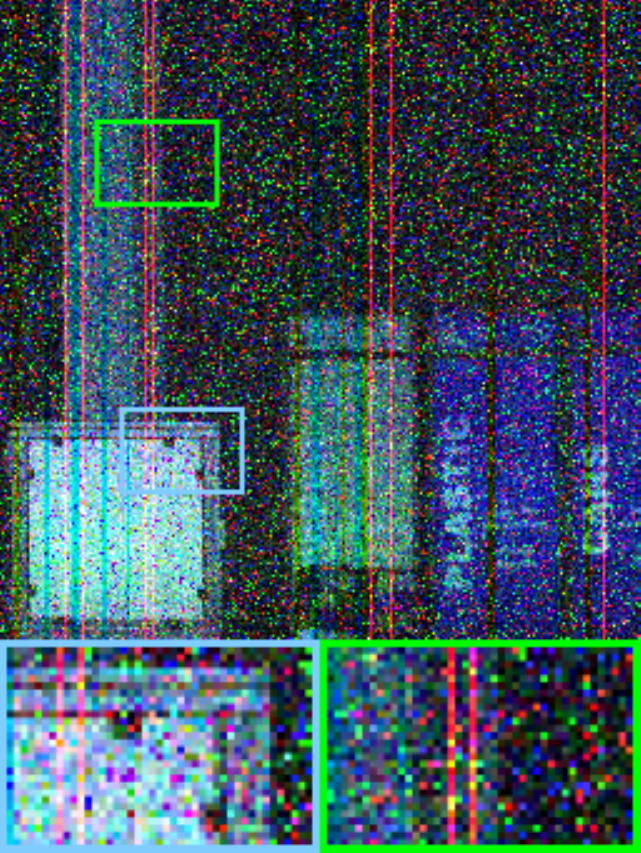}&
           \includegraphics[width=0.12\textwidth]{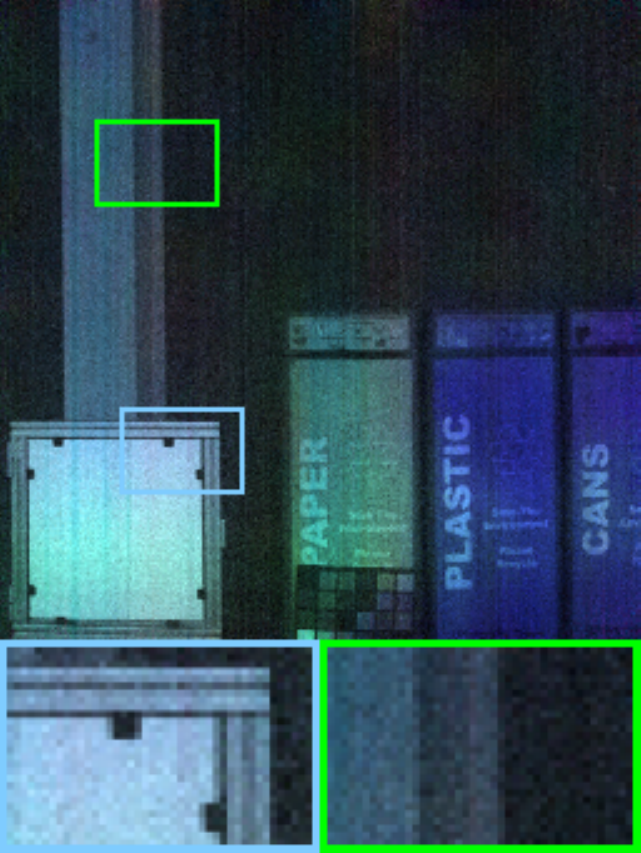}&
          \includegraphics[width=0.12\textwidth]{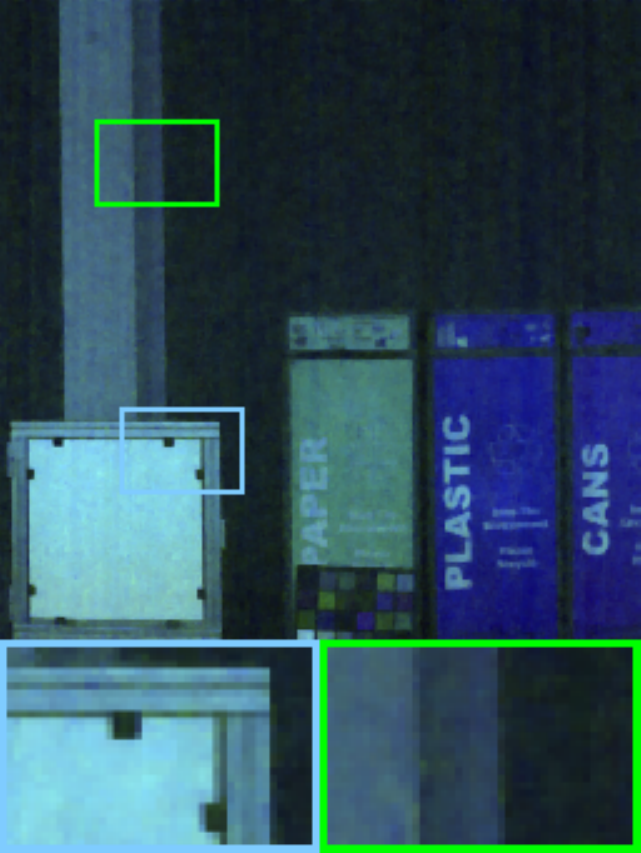}&
           \includegraphics[width=0.12\textwidth]{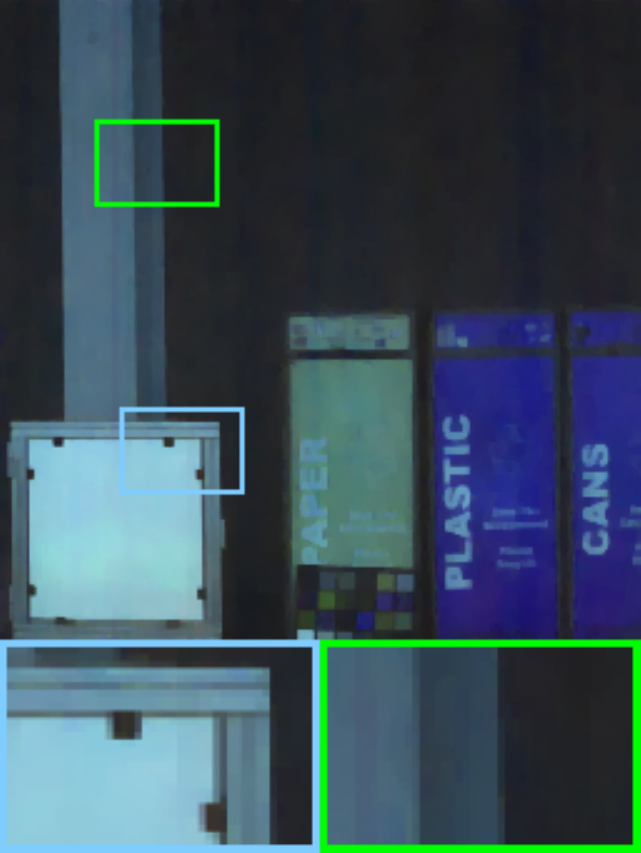}&
            \includegraphics[width=0.12\textwidth]{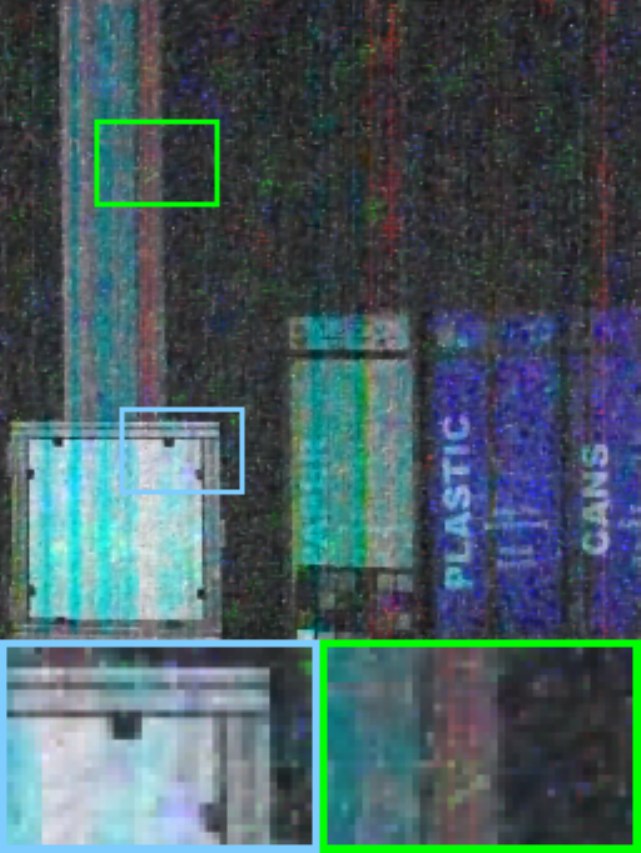}&
             \includegraphics[width=0.12\textwidth]{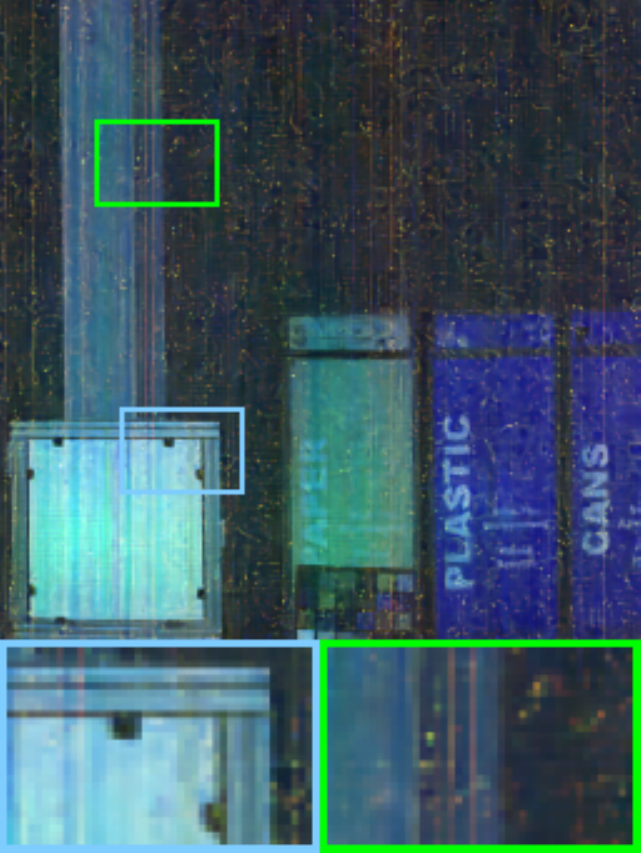}&
              \includegraphics[width=0.12\textwidth]{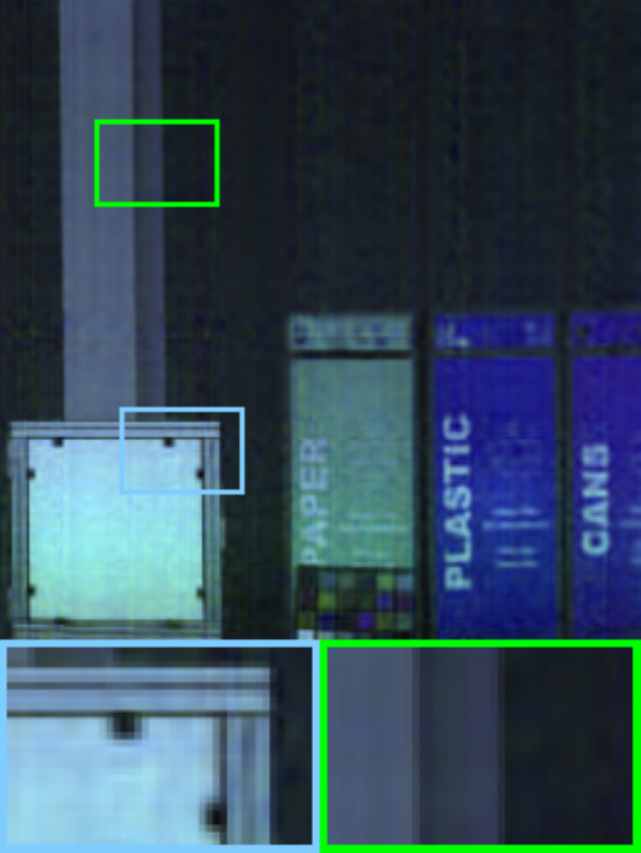}&
              \includegraphics[width=0.12\textwidth]{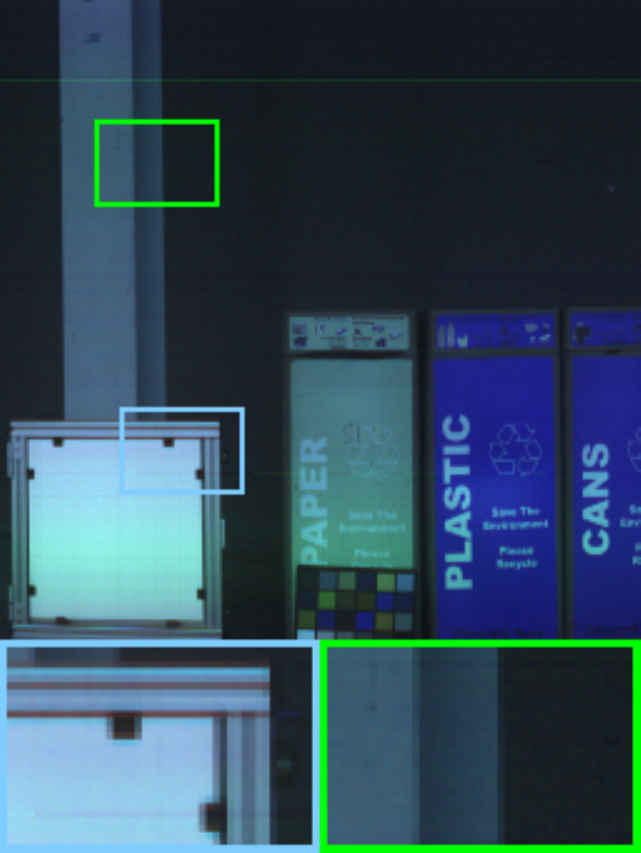}\\
PSNR 13.66  &
PSNR 30.05  &
PSNR 31.30  &
PSNR 31.00  &
PSNR 23.79  &
PSNR 25.12  &
PSNR 34.55  &
PSNR Inf\\
          \includegraphics[width=0.12\textwidth]{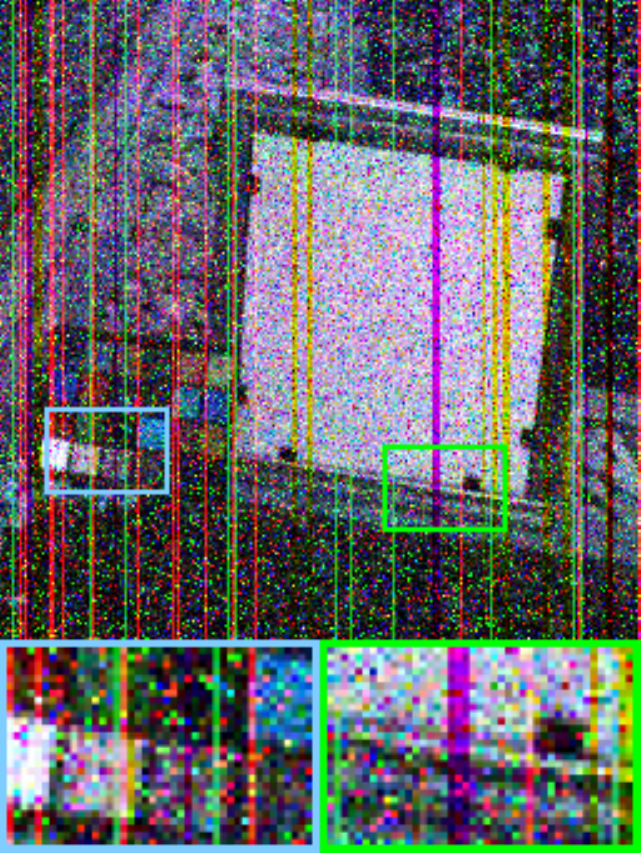}&
           \includegraphics[width=0.12\textwidth]{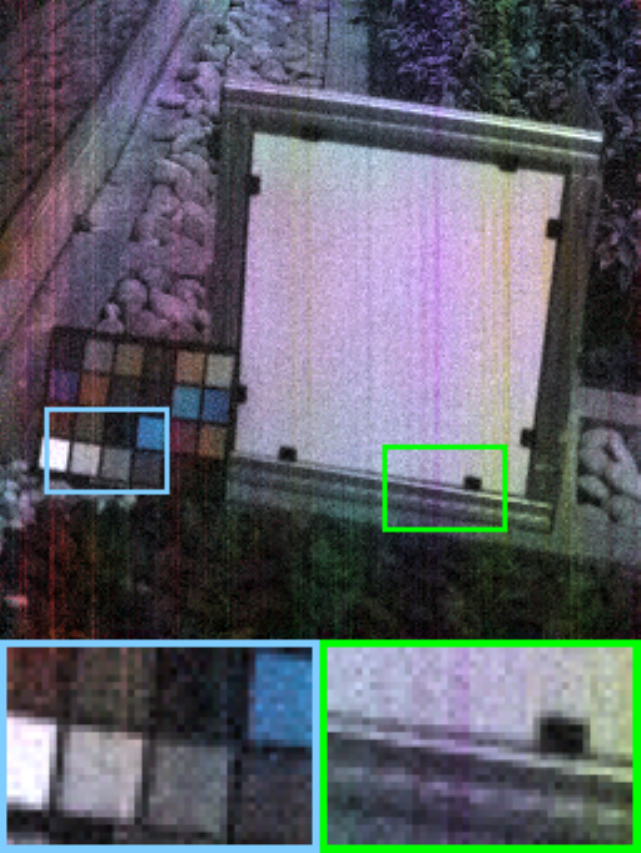}&
          \includegraphics[width=0.12\textwidth]{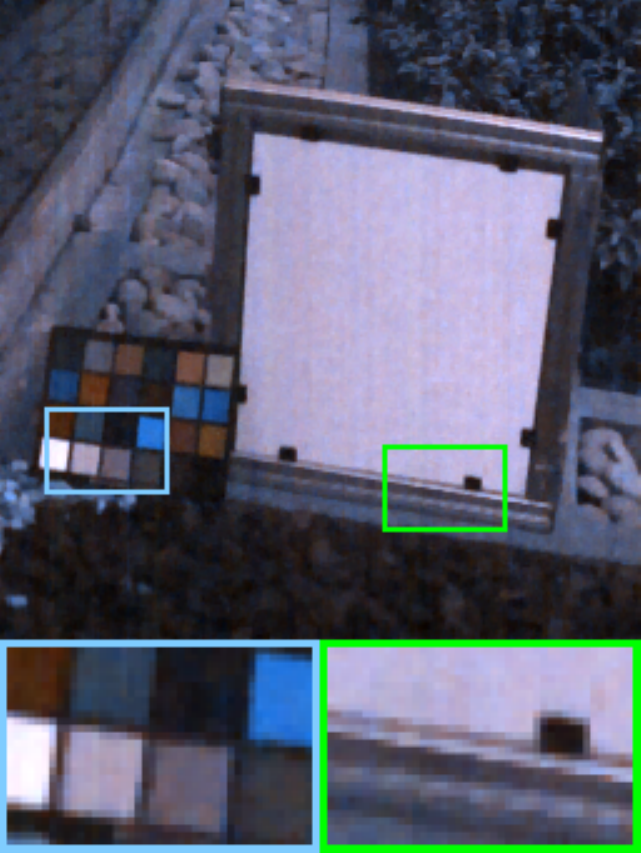}&
           \includegraphics[width=0.12\textwidth]{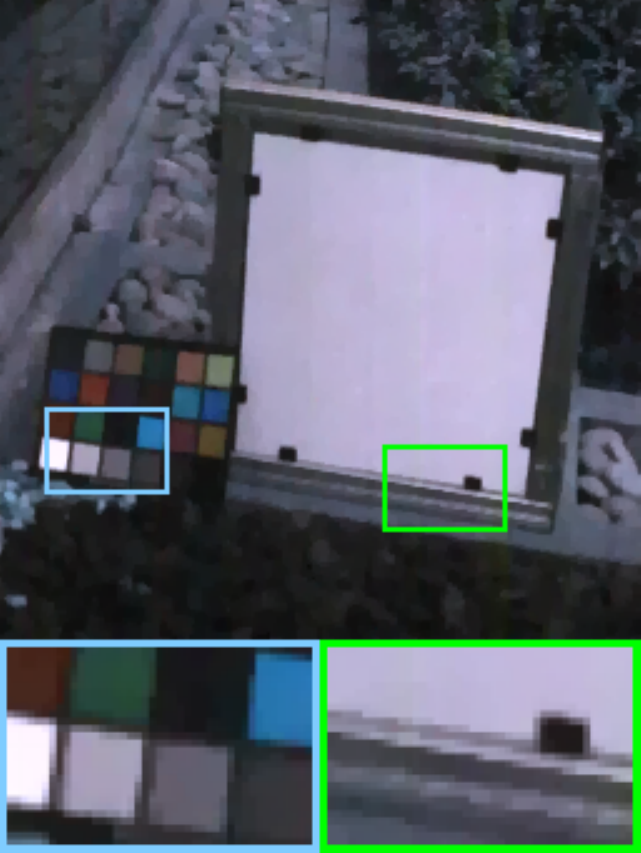}&
            \includegraphics[width=0.12\textwidth]{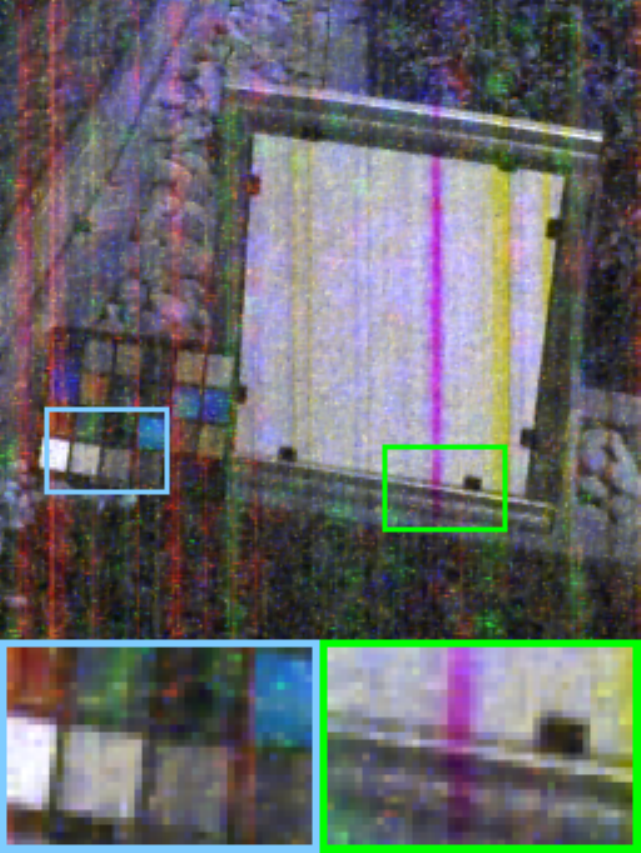}&
             \includegraphics[width=0.12\textwidth]{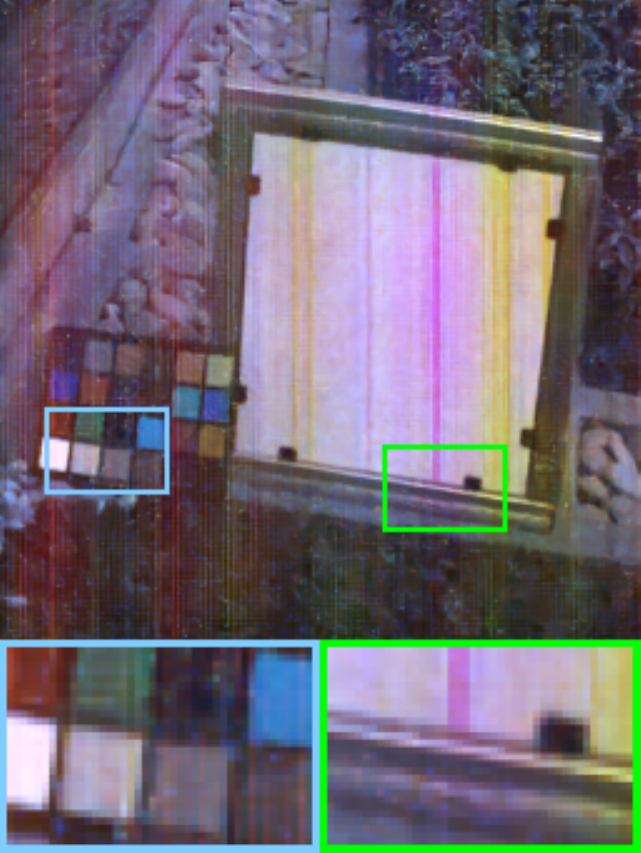}&
              \includegraphics[width=0.12\textwidth]{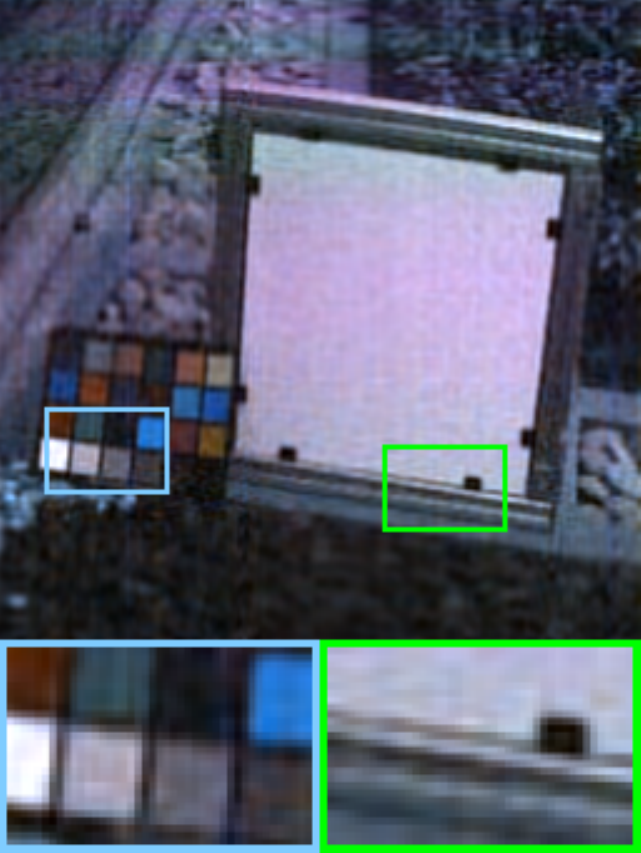}&
              \includegraphics[width=0.12\textwidth]{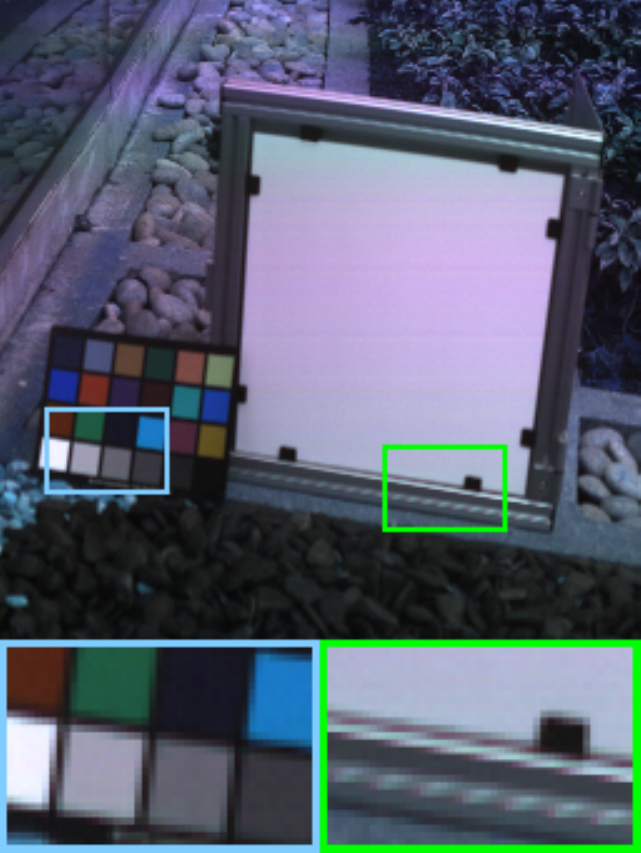}\\
PSNR 13.73  &
PSNR 29.77  &
PSNR 30.82  &
PSNR 30.70  &
PSNR 23.54  &
PSNR 24.92  &
PSNR 32.63  &
PSNR Inf\\
      Observed &LRMR\cite{LRMR}&LRTDTV\cite{LRTDTV}&E3DTV\cite{E3DTV}&HSID-CNN\cite{HSIDCNN}&SDeCNN\cite{SDeCNN}&LRTFR & Original\\
      \end{tabular}
      \end{center}
      \vspace{-0.3cm}
      \caption{The results of multispectral image denoising by different methods on MSIs {\it Cups}, {\it Fruits}, {\it Bin}, and {\it Board} (Case 5).\label{fig_denoising_2}}
      \vspace{-0.2cm}
      \end{figure*}
\begin{table*}[h]
    \caption{The quantitative results by different HPO methods. We report the results on eight binary classification tasks using Gaussian kernel-based SVM with hyperparameters suggested by different HPO methods. The {\bf best} and \underline{second-best} values are highlighted. (ACA $\uparrow$ and ARA $\uparrow$)\label{tab_HPO}}\vspace{-0.4cm}
      \begin{center}
      \scriptsize
      \setlength{\tabcolsep}{4pt}
      \begin{spacing}{1.2}
      \begin{tabular}{clcccccccccccccccc}
      \toprule
      \multicolumn{2}{c}{Data variance}& \multicolumn{2}{c}{\;\;0.05}& \multicolumn{2}{c}{\;\;0.1}& \multicolumn{2}{c}{\;\;0.15}& \multicolumn{2}{c}{\;\;0.2}& \multicolumn{2}{c}{\;\;0.25}& \multicolumn{2}{c}{\;\;0.3}& \multicolumn{2}{c}{\;\;0.35}& \multicolumn{2}{c}{\;\;0.4}\\
      \midrule
    \multicolumn{2}{c}{Method}&\;\;ACA&ARA   &\;\;ACA&ARA&\;\;ACA&ARA&\;\;ACA&ARA&\;\;ACA&ARA&\;\;ACA&ARA&\;\;ACA&ARA&\;\;ACA&ARA\\
  \midrule
    \multirow{6}*{Meshgrid}&
RS (Observed)&93.68&94.53&\;\;88.69&95.22&\;\;84.98&93.88&\;\;83.07&\underline{94.61}&\;\;74.88&90.39&\;\;76.04&94.27&\;\;73.57&89.91&\;\;70.39&89.74\\
~&DCTNN &93.81&92.58&\;\;88.47&95.26&\;\;83.76&87.65&\;\;81.96&90.81&\;\;74.42&85.34&\;\;74.01&85.70&\;\;73.68&87.18&\;\;68.27&74.42\\
   ~&TRLRF &94.06&\bf96.90&\;\;87.64&93.80&\;\;84.96&93.43&\;\;80.70&86.15&\;\;74.49&85.22&\;\;74.89&90.37&\;\;73.73&88.30&\;\;70.69&89.32\\
   ~&FTNN& 93.74&94.56&\;\;88.48&93.91&\;\;83.82&87.88&\;\;81.86&88.93&\;\;75.26&89.97&\;\;75.96&93.84&\;\;72.19&86.15&\;\;70.56&89.18\\
   ~&FCTN & 91.46&90.52&\;\;83.61&80.87&\;\;82.84&87.15&\;\;75.23&66.05&\;\;72.08&77.47&\;\;72.24&78.08&\;\;71.96&80.81&\;\;67.01&72.13\\
   ~&HLRTF &93.25&90.60&\;\;88.48&94.56&\;\;83.98&84.52&\;\;81.82&89.90&\;\;74.93&89.25&\;\;75.74&92.56&\;\;73.63&89.02&\;\;70.40&86.66\\
     \midrule
   \multirow{3}*{\tabincell{c}{Beyond\\meshgrid}}&LRTFR& \underline{94.10}&96.14&\;\;\underline{88.84}&\underline{96.16}&\;\;\underline{85.41}&95.08&\;\;83.09&\bf94.72&\;\;75.50&89.71&\;\;75.66&91.35&\;\;73.79&87.93&\;\;70.68&91.80\\
  ~&LRTFR ($\times 2$)& \bf 94.14&\underline{96.35}&\;\;88.79&95.92&\;\;\bf85.46&\bf95.76&\;\;\bf83.18&94.48&\;\;\bf76.06&\underline{91.88}&\;\;\bf76.18&\bf94.89&\;\;\underline{74.05}&\underline{90.04}&\;\;\underline{70.77}&\underline{91.94}\\
  ~&LRTFR ($\times 4$) &\bf 94.14&96.33&\;\;\bf88.87&\bf96.23&\;\;85.39&\underline{95.44}&\;\;\underline{83.12}&94.30&\;\;\underline{76.04}&\bf92.10&\;\;\underline{76.07}&\underline{94.44}&\;\;\bf74.14&\bf90.95&\;\;\bf70.78&\bf92.04\\
  \midrule
\multicolumn{2}{c}{Grid Search (Ground-Truth)}&94.79&100.0&\;\;90.03&100.0&\;\;86.28&100.0&\;\;84.41&100.0&\;\;77.18&100.0&\;\;77.19&100.0&\;\;75.43&100.0&\;\;72.14&100.0\\
    \bottomrule
  \end{tabular}
  \end{spacing}
  \end{center}
  \vspace{-0.4cm}
  \end{table*} 
\subsection{Multi-Dimensional Image Inpainting Results}
The multi-dimensional image inpainting is a typical data recovery problem on meshgrid with the original resolution. To validate the effectiveness of our LRTFR on the original meshgrid, we compare our method with state-of-the-art low-rank tensor-based methods DCTNN \cite{CVPR_19}, TRLRF \cite{TRLRF}, FTNN \cite{FTNN}, FCTN \cite{FCTN}, and HLRTF \cite{HLRTF}. The testing data include color images\footnote{\url{http://sipi.usc.edu/database/database.php}}, MSIs in the CAVE dataset\footnote{\url{https://www.cs.columbia.edu/CAVE/databases/multispectral/}} \cite{CAVE}, and videos\footnote{\url{http://trace.eas.asu.edu/yuv/}}. We consider random missing with sampling rates (SRs) 0.1, 0.2, and 0.3.\par 
The quantitative and qualitative results of multi-dimensional image inpainting are shown in Table \ref{tab_completion} and Figs. \ref{fig_completion}-\ref{fig_completion_2}. It can be seen that our LRTFR obtains the best results both quantitatively and qualitatively, which reveals the superiority of our LRTFR over classical low-rank tensor representations. Specifically, the recovered images of LRTFR are cleaner and smoother than other compared discrete low-rank tensor methods, which is mainly because the proposed LRTFR implicitly encodes the Lipschitz smoothness into the continuous representation, making the recovered images have better visual qualities. Meanwhile, the recovered results of LRTFR capture more fine details of the images, which validates the high representation abilities of our method owing to the approximation abilities of MLPs. Since our LRTFR concurrently encodes the low-rankness and smoothness into the representation, it is reasonable that our method obtains such promising recovery results on meshgrid.
\subsection{Multispectral Image Denoising Results}
The MSI denoising is another challenging data recovery problem on the original meshgrid. We compare our LRTFR with low-rank matrix/tensor-based methods LRMR \cite{LRMR}, LRTDTV\cite{LRTDTV}, and E3DTV\cite{E3DTV}. Meanwhile, we include two supervised deep learning-based methods HSID-CNN \cite{HSIDCNN} and SDeCNN \cite{SDeCNN} into comparisons. We use the best pre-trained models of HSID-CNN and SDeCNN for testing. The testing data includes MSIs in the CAVE dataset \cite{CAVE} and other two hyperspectral images (HSIs)\footnote{\url{http://sipi.usc.edu/database/database.php}}. We consider several different noisy cases. Case 1 contains Gaussian noise with standard deviation 0.2. Case 2 contains Gaussian noise with standard deviation 0.1 and sparse noise with SR 0.1. Case 3 contains the same noise of Case 2 plus stripe noise\cite{TGRS_Liu_1} in 40\% of spectral bands. Case 4 contains the same noise of Case 2 plus deadline noise\cite{TGRS_Liu_1} in 50\% of spectral bands. Case 5 contains the same noise of Case 4 plus stripe noise\cite{TGRS_Liu_1} in 40\% of spectral bands.\par 
The results of MSI denoising are shown in Table \ref{tab_denoising} and Figs. \ref{fig_denoising}-\ref{fig_denoising_2}. From Table \ref{tab_denoising}, we can see that LRTFR is the most stable method among the tested denoising algorithms in terms of different noisy cases and different data. Especially, LRTFR outperforms delicately designed TV-based methods LRTDTV and E3DTV, which validates the superiority of our combined global-local smooth regularizations. In Figs. \ref{fig_denoising}-\ref{fig_denoising_2}, we can observe that our method can totally remove complex noise. As compared, other denoising methods sometimes do not totally remove the mixed noise. Moreover, from the zoom-in boxes of Fig. \ref{fig_denoising}, we can see that other model-based methods (LRTDTV and E3DTV) may produce over-smoothness. In contrast, our method preserves the image details better. The deep learning methods HSID-CNN and SDeCNN have relatively stable performances for Gaussian noise, but suffer from domain gap between training and testing samples when dealing with mixed noise. As compared, our method is a model-based method that implicitly encodes different prior information, which delivers more stable performances for different types of noise.
 \subsection{Hyperparameter Optimization Results}\label{sec_HPO}
The HPO can be elegently modeled as the low-rank tensor completion problem\cite{HPO_PAMI}. As aforementioned, it is even more interesting to exploit the benefits of conducting HPO in the continuous domain by using our LRTFR. Specifically, we can use LRTFR to conduct HPO beyond the original meshgrid resolution and investigate the corresponding effect.\par
Following the experimental settings of previous SOTA work \cite{HPO_PAMI} along this research line, we consider the classification problem using Gaussian kernel-based support vector machine (SVM), which has two hyperparameters---the regularization parameter $C$ and the kernel parameter $\sigma$. The search sets are $\{3^{i}|i=-15:1:15\}$ for $C$ and $\{2^{i}|i=-15:1:15\}$ for $\sigma$, and thus there are $31\times 31=961$ candidate configurations. We use Gaussian distribution to generate classification datasets. The variance of the distribution is traversed in $\{0.05,0.1,0.15,\cdots,0.4\}$ to construct eight tasks with different difficulty levels. In each task, we generate 16 classification datasets. Specifically, we generate four Gaussian distributions with the same variance and random means and divide them into two groups to form a binary classification dataset. Each dataset has 100 training points and 500 testing points. Therefore, by conducting grid search (GS) for all configurations and datasets we form a tensor of size $31\times 31\times 16$ for each task (variance). The first two dimensions ($31\times31$) indicate the number of candidate configurations and the third dimension ($16$) indicates the number of datasets. By following \cite{HPO_PAMI}, we use a low SR of 0.01 to sample the tensor slice corresponding to the new dataset and use the SR of 0.1 to sample the tensor slices corresponding to historical datasets (see detailed explanations in \cite{HPO_PAMI}), which forms the observed incompleted tensor. We repeat the sampling process 16 times, where each dataset is chosen as the new dataset once and the others are historical datasets. It results in 16 observed incompleted tensors and we report the average HPO results on the new dataset.\par
We report three results of our method, including the standard tensor completion using LRTFR, and the super-resolution results using the learned continuous representation and evenly spaced sampling, termed as LRTFR ($\times2$)/($\times4$). The super-resolution results offer more candidate configurations and gives the predicted accuracy even though our method only sees the observed data on the original meshgrid. We compare our method with random search (RS)\cite{RS}, which corresponds to the recommendation results using the observed tensor. Meanwhile, we use state-of-the-art tensor completion methods DCTNN \cite{CVPR_19}, TRLRF \cite{TRLRF}, FTNN \cite{FTNN}, FCTN \cite{FCTN}, and HLRTF \cite{HLRTF} as baselines.\par 
The results of HPO are shown in Table \ref{tab_HPO}. Compared to classical tensor completion methods, our LRTFR is more effective for suggesting a suitable configuration. Moreover, LRTFR ($\times2$)/($\times4$) incline to attain better performances than LRTFR, which reveals that searching the hyperparameter value beyond the original meshgrid resolution is helpful to obtain a better recommendation result. The underlying reason is that the optimal configuration probably does not lie in the given meshgrid, thus expanding the search regions can effectively improve the recommendation results. These results show the superiority of our continuous representation over discrete meshgrid-based tensor completion methods. 
\subsection{Point Cloud Upsampling Results}
Then, we consider the point cloud upsampling problem to show the effectiveness of our method beyond meshgrid. Standard low-rank tensor-based methods can not be applied to point cloud upsampling since they are not suitable for representing the unordered point cloud beyond meshgrid. As compared, our LRTFR is suitable to represent the point cloud since it learns a continuous representation of data. \par 
We consider five deep learning-based methods as baselines, including MSN\cite{MSN_AAAI}, SnowflakeNet\cite{SnowflakeNet}, SMOG\cite{SMOG}, NeuralPoints\cite{NeuralPoints}, and SAPCU\cite{SAPCU}. We use the best pre-trained models provided by the authors. We remark that SAPCU\cite{SAPCU} is an INR-based method and thus the comparisons could certainly reveal the superiority of our LRTFR over INR. We adopt different datasets, including those in the ShapeNet benchmark\cite{ShapeNet} ({\it Table}, {\it Boat}, {\it Lamp}, and {\it Sofa}), the Stanford {\it Bunny}\footnote{\url{https://graphics.stanford.edu/data/3Dscanrep/}}, and three hand-crafted shapes ({\it Doughnut}, {\it Sphere}, and {\it Heart}). We use random sampling to downsample the original point clouds such that the downsampled point cloud has a number of points less than 500.\par
The results for point cloud upsampling are shown in Table \ref{tab_PCC} and Figs. \ref{PCC_1}-\ref{PCC_2}. From the quantitative comparisons in Table \ref{tab_PCC}, we can observe that LRTFR outperforms other competitors in most cases, which reveals the effectiveness of our method for continuous representation. From Fig. \ref{PCC_1}, we can see that our method obtains satisfactory results on the ShapeNet dataset. Other methods like MSN and SAPCU also have stable performances, but their accuracy is less than ours. Note that MSN and SnowflakeNet were trained on the ShapeNet dataset, and thus their performances on the point clouds in Fig. \ref{PCC_1} are relatively satisfactory. However, in Fig. \ref{PCC_2}, the testing samples are out-of-distribution of the training domain of these deep learning methods. Therefore, they attain relatively worse generalization performances on these wild datasets. As compared, our method does not depend on training data and is more stable since it depends on the low-rank regularization, where the learned SDF lies in a low-rank manifold to ensure the robust performance for different datasets. The results validate the effectiveness of our method for representing the continuous SDF structure, which can not be achieved by other meshgrid-based low-rank representations.    
  \begin{figure*}[t]
      \scriptsize
      \setlength{\tabcolsep}{0.9pt}
      \begin{center}
      \begin{tabular}{cccccccc}
          \includegraphics[width=0.12\textwidth]{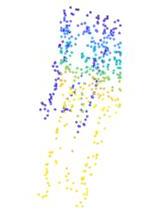}&
          \includegraphics[width=0.12\textwidth]{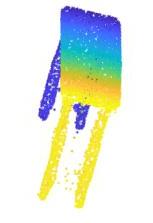}&
          \includegraphics[width=0.12\textwidth]{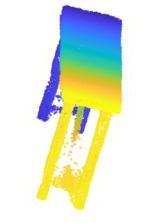}&
          \includegraphics[width=0.12\textwidth]{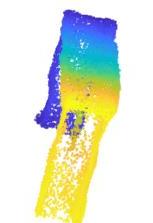}&
          \includegraphics[width=0.12\textwidth]{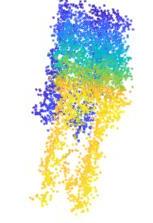}&
          \includegraphics[width=0.12\textwidth]{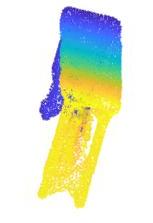}&
          \includegraphics[width=0.12\textwidth]{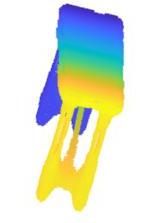}&
          \includegraphics[width=0.12\textwidth]{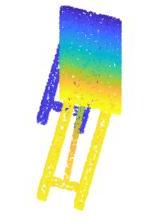}\\
          \includegraphics[width=0.12\textwidth]{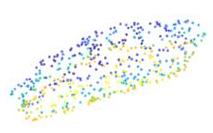}&
          \includegraphics[width=0.12\textwidth]{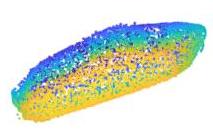}&
          \includegraphics[width=0.12\textwidth]{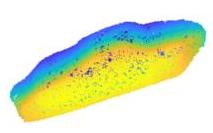}&
          \includegraphics[width=0.12\textwidth]{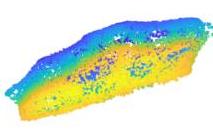}&
          \includegraphics[width=0.12\textwidth]{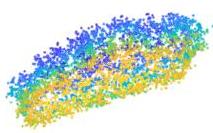}&
          \includegraphics[width=0.12\textwidth]{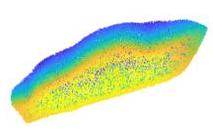}&
          \includegraphics[width=0.12\textwidth]{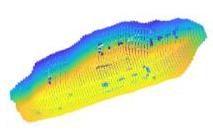}&
          \includegraphics[width=0.12\textwidth]{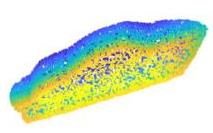}\\
          \includegraphics[width=0.12\textwidth]{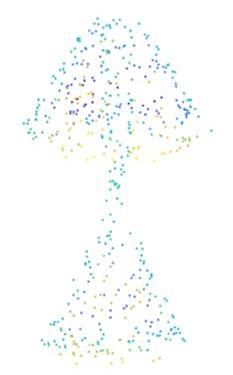}&
          \includegraphics[width=0.12\textwidth]{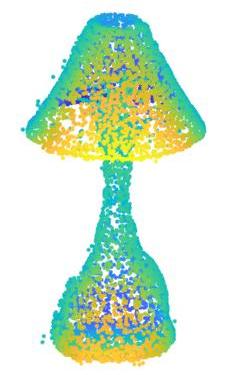}&
          \includegraphics[width=0.12\textwidth]{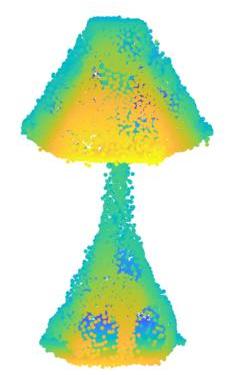}&
          \includegraphics[width=0.12\textwidth]{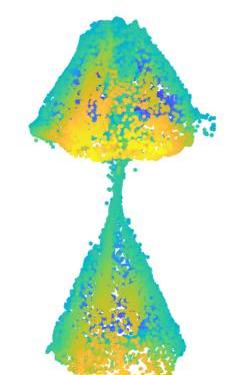}&
          \includegraphics[width=0.12\textwidth]{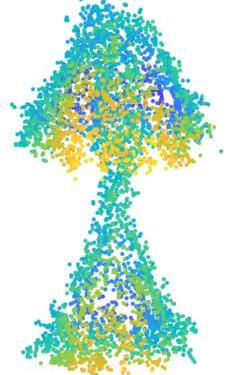}&
          \includegraphics[width=0.12\textwidth]{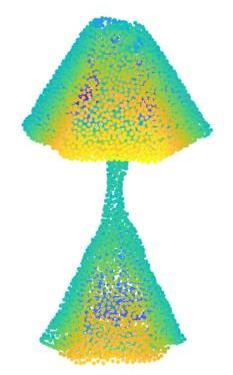}&
          \includegraphics[width=0.12\textwidth]{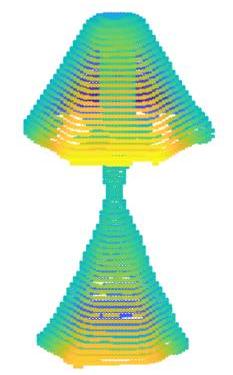}&
          \includegraphics[width=0.12\textwidth]{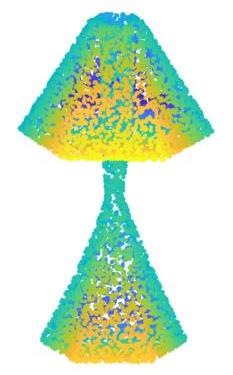}\\
          \includegraphics[width=0.12\textwidth]{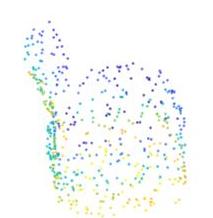}&
          \includegraphics[width=0.12\textwidth]{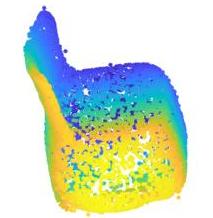}&
          \includegraphics[width=0.12\textwidth]{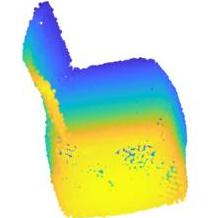}&
          \includegraphics[width=0.12\textwidth]{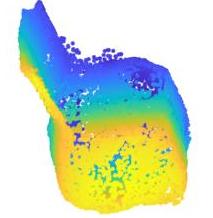}&
          \includegraphics[width=0.12\textwidth]{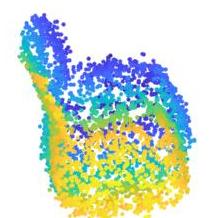}&
          \includegraphics[width=0.12\textwidth]{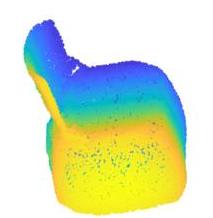}&
          \includegraphics[width=0.12\textwidth]{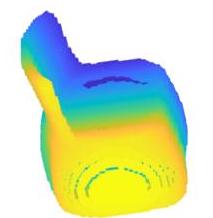}&
          \includegraphics[width=0.12\textwidth]{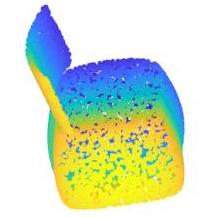}\\
 Observed&MSN\cite{MSN_AAAI}&SnowflakeNet\cite{SnowflakeNet}&SMOG\cite{SMOG}&NeuralPoints\cite{NeuralPoints}&SAPCU\cite{SAPCU}&LRTFR & Original\\
      \end{tabular}
      \end{center}
      \vspace{-0.3cm}
      \caption{The results of point cloud upsampling by different methods on {\it Table}, {\it Vessel}, {\it Lamp}, and {\it Sofa} in the ShapeNet dataset\cite{ShapeNet}.\label{PCC_1}}
      \vspace{-0.2cm}
      \end{figure*}
               \begin{table*}[t]
          \caption{The quantitative results by different methods for point cloud upsampling. The {\bf best} and \underline{second-best} values are highlighted. (CD $\downarrow$ and F-Socre $\uparrow$)\label{tab_PCC}}\vspace{-0.4cm}
            \begin{center}
            \scriptsize
            \setlength{\tabcolsep}{2.5pt}
            \begin{spacing}{1.2}
            \begin{tabular}{lcccccccccccccccc}
            \toprule
            Data&\multicolumn{2}{c}{\it Table}&\multicolumn{2}{c}{\it Vessel}&\multicolumn{2}{c}{\it Lamp}&\multicolumn{2}{c}{\it Sofa}&\multicolumn{2}{c}{\it Bunny}&\multicolumn{2}{c}{\it Doughnut}&\multicolumn{2}{c}{\it Sphere}&\multicolumn{2}{c}{\it Heart}\\
        \midrule
        Method&CD&F-Socre&\;\;CD&F-Socre&\;\;CD&F-Socre&\;\;CD&F-Socre&\;\;CD&F-Socre&\;\;CD&F-Socre&\;\;CD&F-Socre&\;\;CD&F-Socre\\
        \midrule
        Observed&0.0187&    0.3448&\;\;    0.0180&    0.3866&\;\;    0.0185&    0.3677&\;\;    0.0221&    0.2976&\;\;  0.1098&    0.2086&\;\;    0.1882&    0.1863&\;\;    0.0451&    0.7379&\;\;    0.0959&    0.2201\\
        MSN &0.0164&    0.6918&\;\;    0.0154&    0.7768&\;\;    0.0167&    0.7188&\;\;    0.0183&    0.6267&\;\;    0.1624&    0.4157&\;\;    0.3639&    0.1408&\;\;    0.1134&    0.5529&\;\;    0.2491&    0.3145\\
        SnowflakeNet& \bf0.0106&    \bf0.9068&\;\;    \underline{0.0096}&    \underline{0.9493}&\;\;    \underline{0.0126}&    \underline{0.8511}&\;\;    \bf0.0125&   \bf0.8668&\;\;    {0.2017}&    0.3320&\;\;    \underline{0.1248}&    \underline{0.5194}&\;\;    \underline{0.0228}&   {0.9891}&\;\;   {0.0729}&   {0.8268}\\
        SMOG&0.0288&    0.4819&\;\;    0.0151&    0.7769&\;\;    0.0206&    0.5783&\;\;    0.0215&    0.6163&\;\;  0.2843&    0.2380&\;\;    0.1333&    0.5036&\;\;    0.0392&    0.9642&\;\;    0.1245&    0.4662\\
        NeuralPoints& 0.0283&    0.3900&\;\;    0.0227&    0.5272&\;\;    0.0246&    0.4800&\;\;    0.0227&    0.5083&\;\;   0.0558&    0.9045&\;\;    0.2456&    0.3316&\;\;    0.0342&    0.9559&\;\;    0.0856&    0.7326\\
          SAPCU& 0.0194&    0.7234&\;\;    0.0152&    0.8848&\;\;    0.0161&    0.7719&\;\;   \underline{0.0148}&    0.8085&\;\;    \underline{0.0518}&   \underline{0.9187}&\;\;    0.1329&    0.4142&\;\;    0.0271&    \underline{0.9974}&\;\;    \underline{0.0526}& \underline{0.9130}\\
        LRTFR&\underline{0.0113}&    \underline{0.8453}&\;\;    \bf0.0090&    \bf0.9723&\;\;    \bf0.0111&    \bf0.9360&\;\;   \bf{0.0125}&    \underline{0.8387}&\;\;    \bf0.0445&    \bf0.9784&\;\;    \bf0.1113&    \bf0.5871&\;\;    \bf0.0200&     \bf1.0000&\;\;    \bf0.0463&    \bf0.9824\\
        \bottomrule
        \end{tabular}
        \end{spacing}
        \end{center}
        \vspace{-0.5cm}
        \end{table*}    
        \begin{figure*}[t]
            \scriptsize
            \setlength{\tabcolsep}{0.9pt}
            \begin{center}
            \begin{tabular}{cccccccc}
                \includegraphics[width=0.12\textwidth]{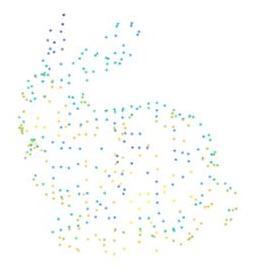}&
                \includegraphics[width=0.12\textwidth]{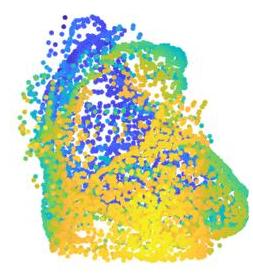}&
                \includegraphics[width=0.12\textwidth]{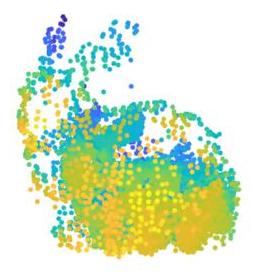}&
                \includegraphics[width=0.12\textwidth]{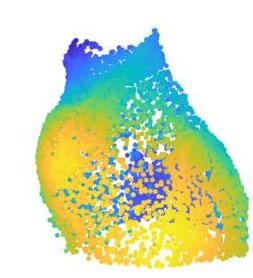}&
                \includegraphics[width=0.12\textwidth]{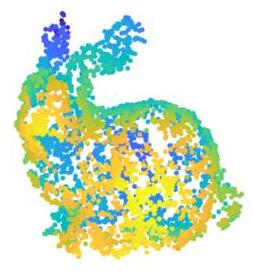}&
                \includegraphics[width=0.12\textwidth]{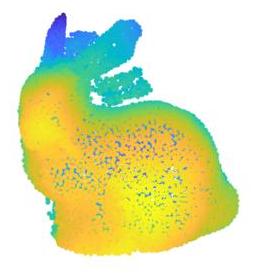}&
                \includegraphics[width=0.12\textwidth]{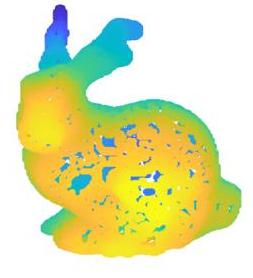}&
                \includegraphics[width=0.12\textwidth]{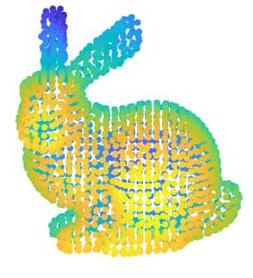}\\
                          \includegraphics[width=0.12\textwidth]{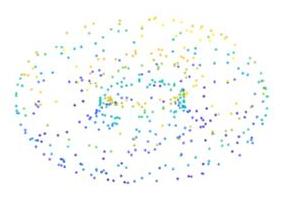}&
                          \includegraphics[width=0.12\textwidth]{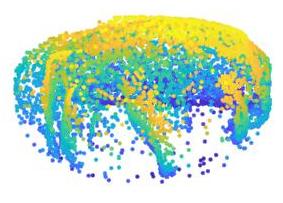}&
                          \includegraphics[width=0.12\textwidth]{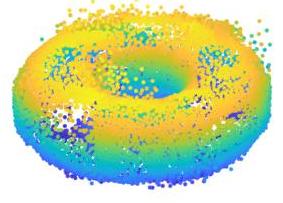}&  \includegraphics[width=0.12\textwidth]{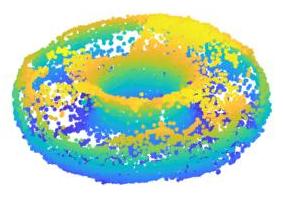}&
                                    \includegraphics[width=0.12\textwidth]{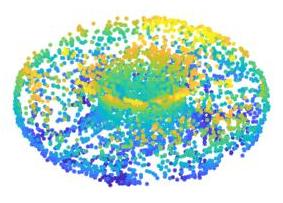}&
                \includegraphics[width=0.12\textwidth]{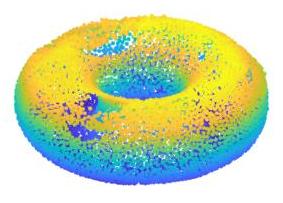}&
                          \includegraphics[width=0.12\textwidth]{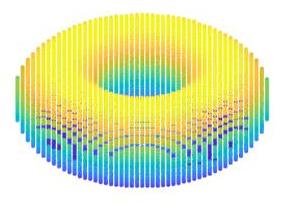}&
                          \includegraphics[width=0.12\textwidth]{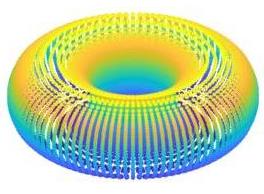}\\
                \includegraphics[width=0.12\textwidth]{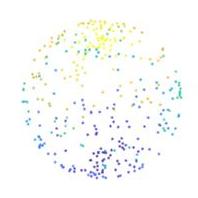}&
                \includegraphics[width=0.12\textwidth]{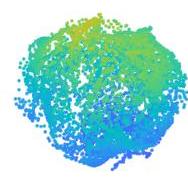}&
                \includegraphics[width=0.12\textwidth]{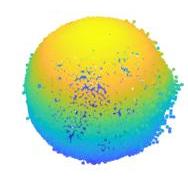}&
                \includegraphics[width=0.12\textwidth]{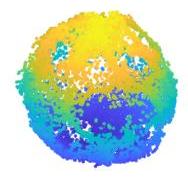}&
                \includegraphics[width=0.12\textwidth]{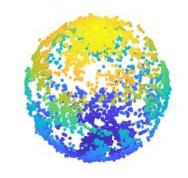}&
                \includegraphics[width=0.12\textwidth]{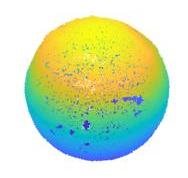}&
                \includegraphics[width=0.12\textwidth]{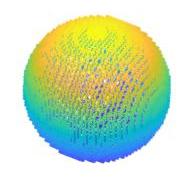}&
                \includegraphics[width=0.12\textwidth]{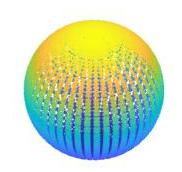}\\
                \includegraphics[width=0.12\textwidth]{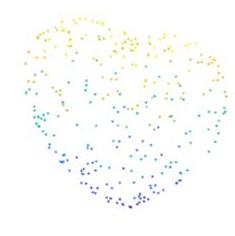}&
                \includegraphics[width=0.12\textwidth]{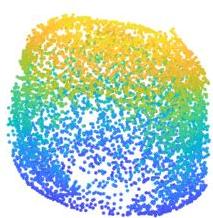}&
                \includegraphics[width=0.12\textwidth]{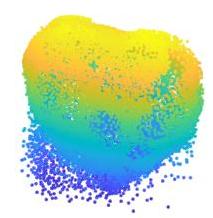}&
                \includegraphics[width=0.12\textwidth]{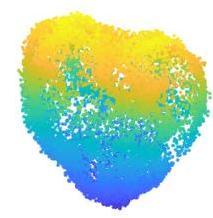}&
                \includegraphics[width=0.12\textwidth]{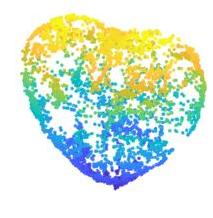}&
                \includegraphics[width=0.12\textwidth]{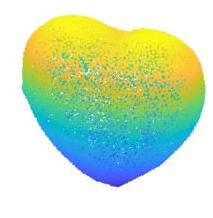}&
                \includegraphics[width=0.12\textwidth]{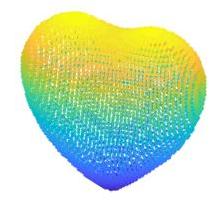}&
                \includegraphics[width=0.12\textwidth]{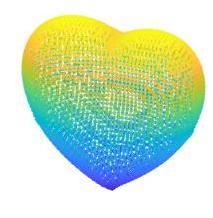}\\
       Observed&MSN\cite{MSN_AAAI}&SnowflakeNet\cite{SnowflakeNet}&SMOG\cite{SMOG}&NeuralPoints\cite{NeuralPoints}&SAPCU\cite{SAPCU}&LRTFR & Original\\
            \end{tabular}
            \end{center}
            \vspace{-0.3cm}
            \caption{The results of point cloud upsampling by different methods on {\it Bunny}, {\it Doughnut}, {\it Sphere}, and {\it Heart}.\label{PCC_2}}
            \vspace{-0.1cm}
         \end{figure*}
\section{Discussions}\label{Sec_dis}
\subsection{Influences of Tensor Factorization}
The tensor Tucker factorization is a core building block in our LRTFR. However, our method can be readily extended to different tensor function factorizations. Here, we compare the CP factorization\cite{SIAM_review} and the Tucker factorization in our method for multi-dimensional image inpainting; see Fig. \ref{com_CP} (d)-(e). It can be seen that Tucker function factorization shows certain advantageous performances. This is possibly attributed to the better tensor structure preserving capability by the former Tucker representation manner. We will more deeply investigate this issue in future research.
\subsection{Influences of Activation Functions}
Since we use MLPs to parameterize the factor functions of LRTFR, the selection of activation function in the MLPs warrants discussion. Inspired by recent study of INR\cite{sine}, which shows that the periodic sine activation function can capture natural signals' complex structures and fine details, we adopt the sine activation function in the MLP to learn the LRTFR to help obtain a more realistic continuous representation. To validate the effectiveness of the sine activation function, we compare it with ReLU, LeakyReLU, and Tanh activation functions; see Fig. \ref{com_CP} (a)-(c), (e). The results show that the sine activation function can help recover the signal much better than other activation functions, which are consistent with the results in existing literatures\cite{sine,piGAN}. 
\subsection{Influences of Hyperparameters}\label{Sec_hyper}
Selecting suitable hyperparameter values is a necessary step in our method. These hyperparameters include the weight decay of the Adam optimizer (denoted by $w$), the hyperparameter of the sine activation $\omega_0$, the {\rm F\text{-}rank} ($r_1,r_2,r_3$), and the depth of the MLP (denoted by $d$). Meanwhile, there are two regularization parameters $\gamma_1,\gamma_2$ in the denoising model (\ref{loss_denoising}) and the point cloud upsampling model (\ref{loss_point}). To comprehensively analyze the influences of different hyperparameters on the performances of our method, we change the value of each hyperparameter and fix the others and report the corresponding results. The results are shown in Fig. \ref{fig_hyper} and Fig. \ref{fig_hyper_2}. We remark that different tasks require different values of hyperparameters, and thus the testing ranges of a hyperparameter are inconsistent for different tasks. From the results we can see that our method is relatively robust w.r.t. these hyperparameters since it can obtain satisfactory performances for a wide range of values. This makes our method relatively easy to be applicable in real scenarios. Moreover, from Fig. \ref{fig_hyper_2} we can see that the adopted regularizations (e.g., the TV regularization in (\ref{loss_denoising})) are effective to boost the performance of our method with suitable hyperparameter values, which reveals the compatibility of our method with other proven techniques to enhance performance. Anyway, how to build easy and automatic hyperparameter tuning strategies to make our method more flexible and adaptable to diverse scenarios still requires more endeavor of our future research.
\begin{figure*}[t]
    \scriptsize
    \setlength{\tabcolsep}{0.9pt}
    \begin{center}
    \begin{tabular}{ccccccc}
     \includegraphics[width=0.12\textwidth]{planep2.pdf}&
     \includegraphics[width=0.12\textwidth]{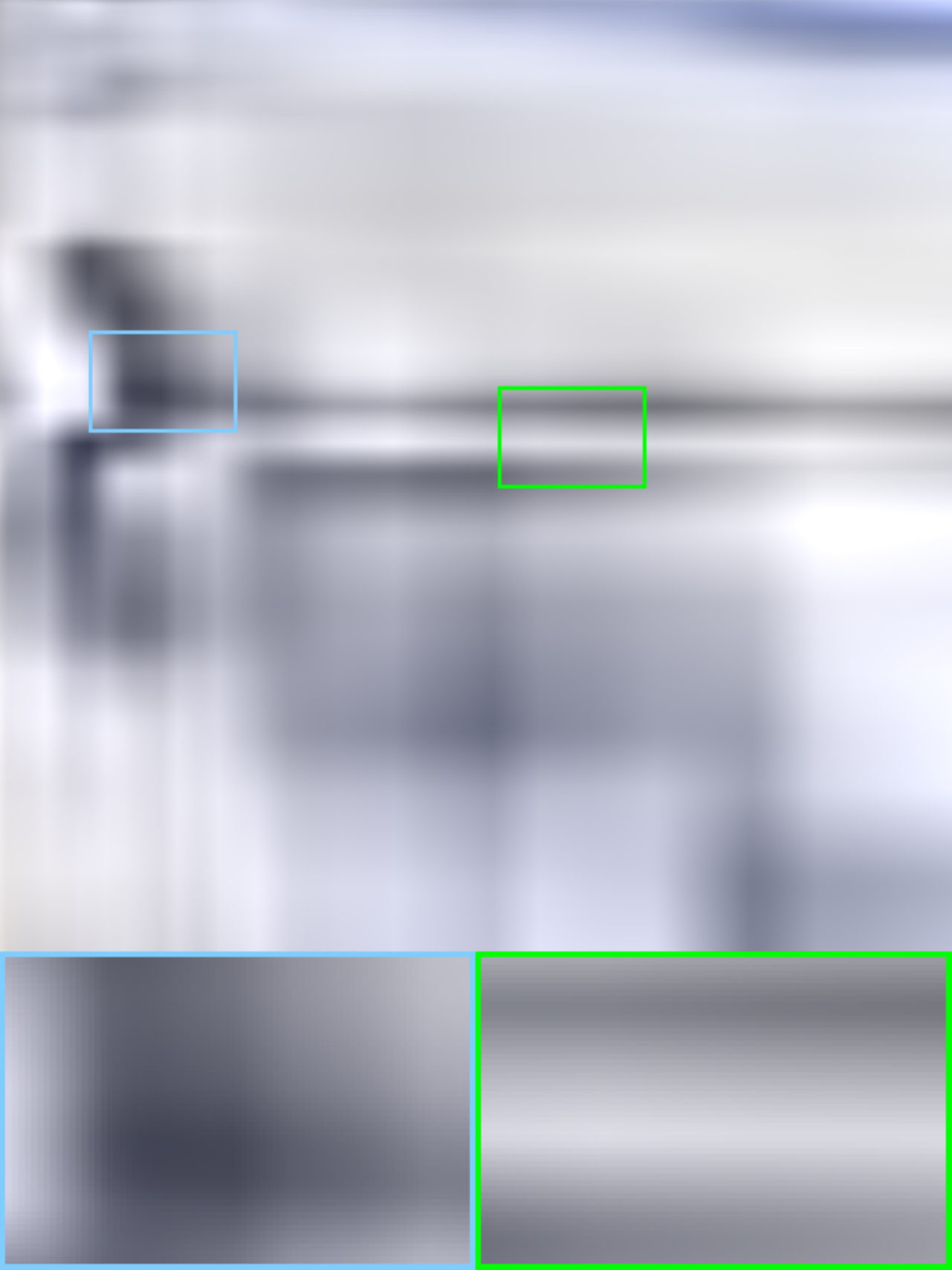}&
     \includegraphics[width=0.12\textwidth]{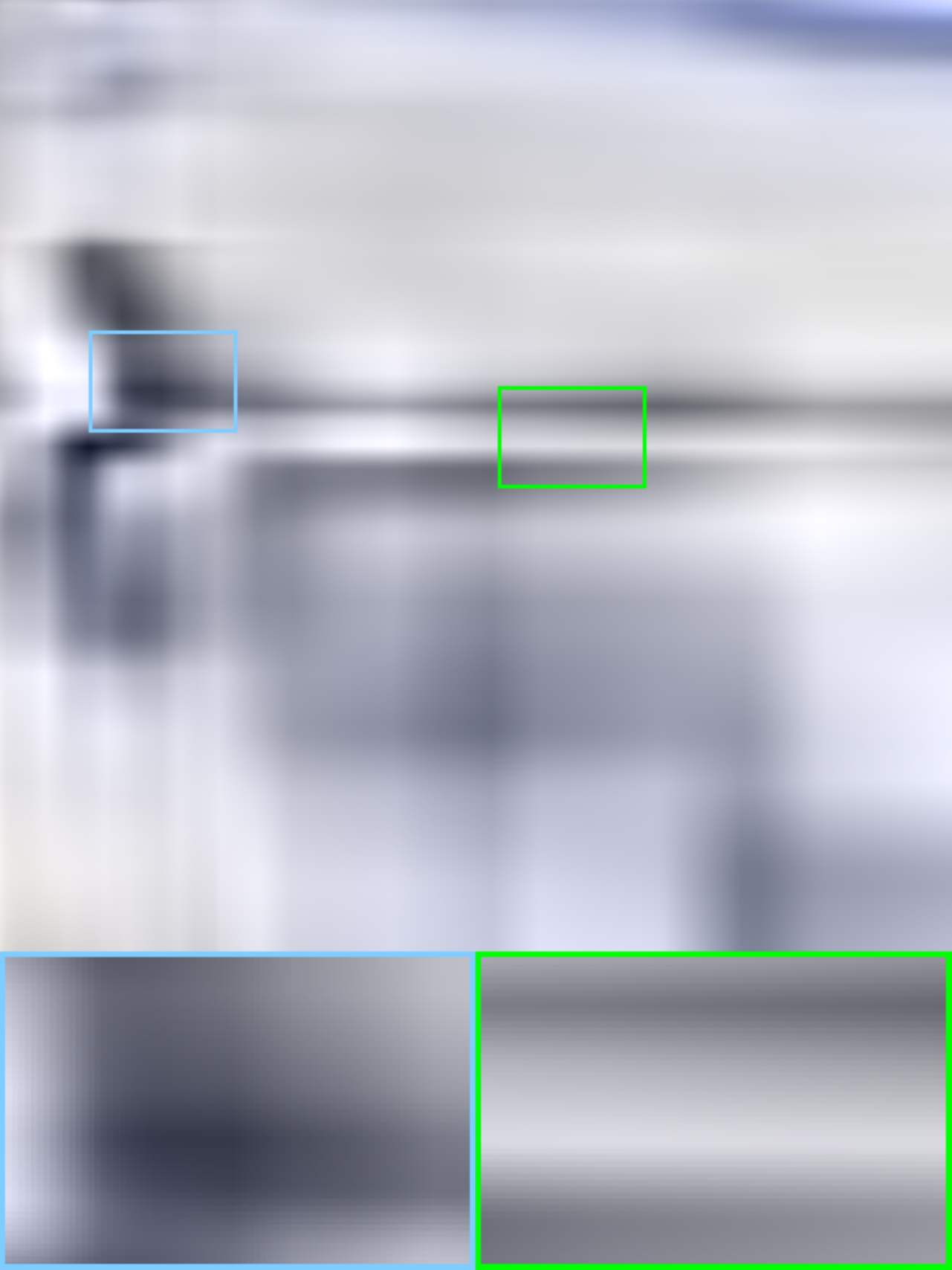}&
     \includegraphics[width=0.12\textwidth]{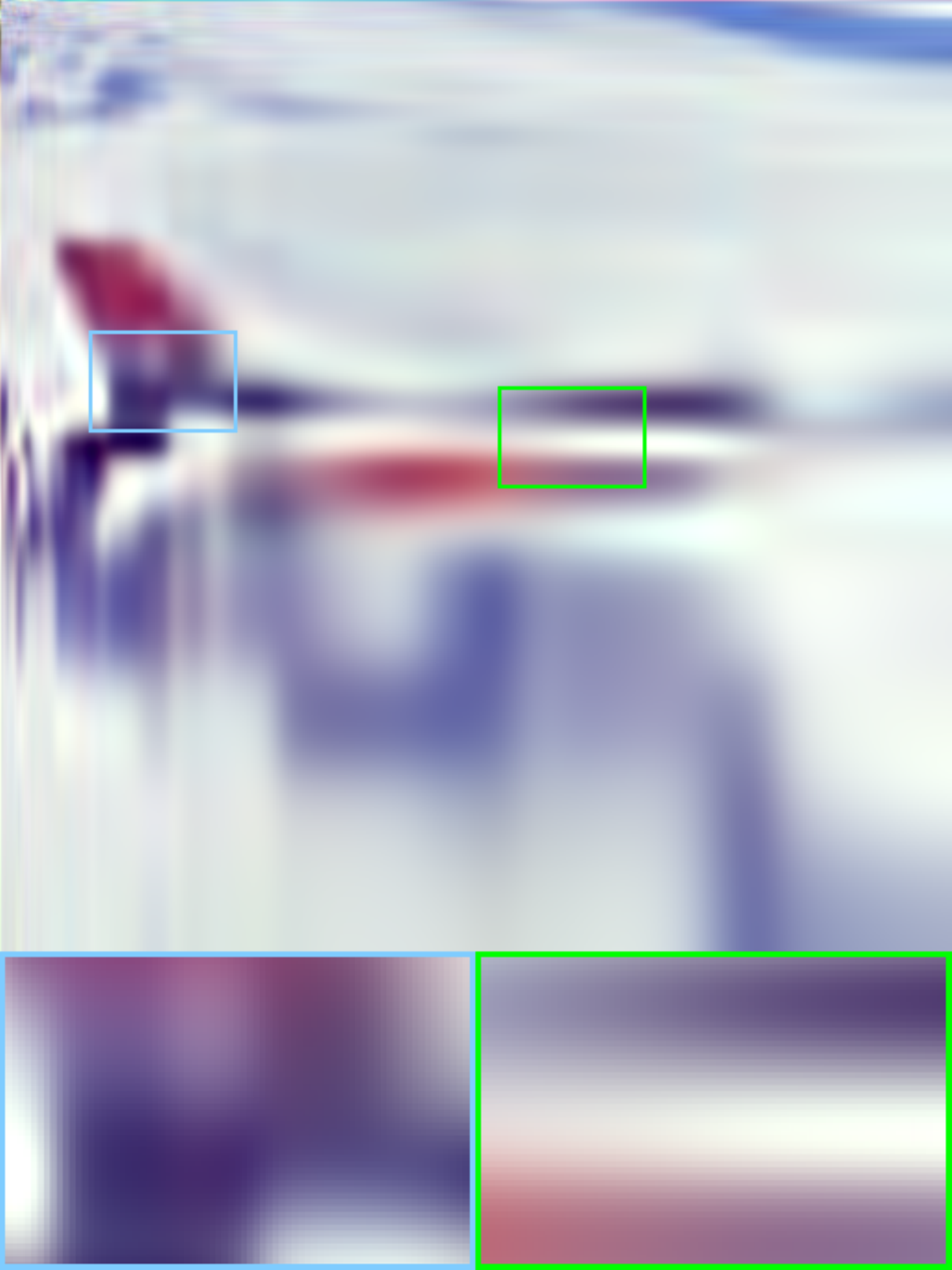}&
     \includegraphics[width=0.12\textwidth]{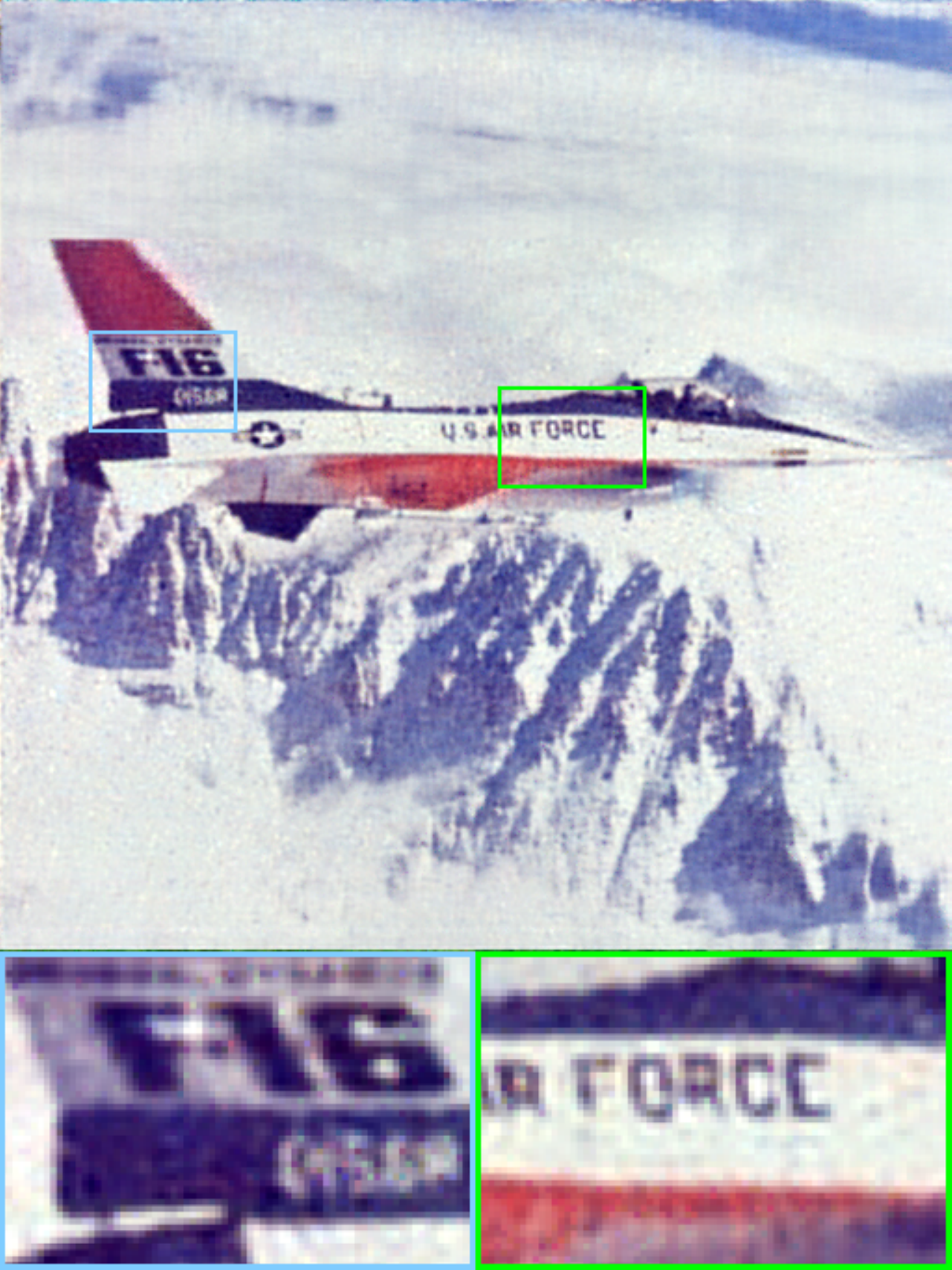}&
     \includegraphics[width=0.12\textwidth]{planep2NCLRF.pdf}&
     \includegraphics[width=0.12\textwidth]{planep2gt.pdf}\\
     Observed&(a)&(b)&(c)&(d)&(e) &Original\\
     \vspace{-0.1cm}
     PSNR 3.66 &PSNR 18.64 &PSNR 18.71 &PSNR 20.07 &PSNR 28.77 &PSNR 29.75 &PSNR Inf\\
    \end{tabular}
    \end{center}
    \vspace{-0.3cm}
    \caption{The results of LRTFR with different MLP activation functions and factorization strategies for inpainting ({\it Plane} SR = 0.2). (a) ReLU + Tucker function factorization. (b) LeakyReLU + Tucker function factorization. (c) Tanh + Tucker function factorization. (d) Sine + CP function factorization. (e) Sine + Tucker function factorization (Our proposed).\label{com_CP}}
    \vspace{-0.1cm}
    \end{figure*}
\begin{figure*}[t]
  \scriptsize
  \setlength{\tabcolsep}{0.9pt}
    \begin{center}
    \begin{tabular}{c}
  \includegraphics[width=1\textwidth]{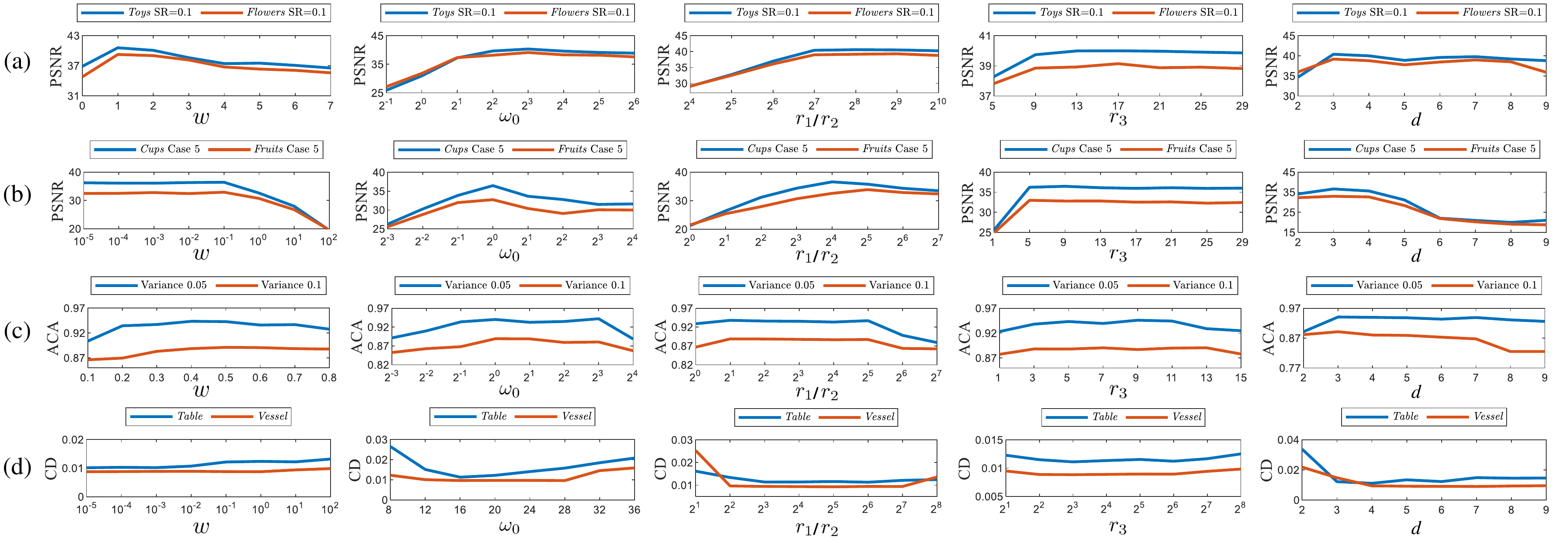}
    \end{tabular}
    \end{center}
  \vspace{-0.4cm}
  \caption{The quantitative performances of LRTFR w.r.t. different values of hyperparameters (The weight decay $w$, the hyperparameter of the sine activation function $\omega_0$, the {\rm F\text{-}rank} $(r_1,r_2,r_3)$, and the number of MLP layers $d$) for (a) multi-dimensional image inpainting, (b) MSI denoising, (c) hyperparameter optimization, and (d) point cloud upsampling.\label{fig_hyper}\vspace{-0.2cm}}
\end{figure*} 
       \begin{figure}[t]
  \scriptsize
  \setlength{\tabcolsep}{0.9pt}
    \begin{center}
    \begin{tabular}{c}
  \includegraphics[width=0.44\textwidth]{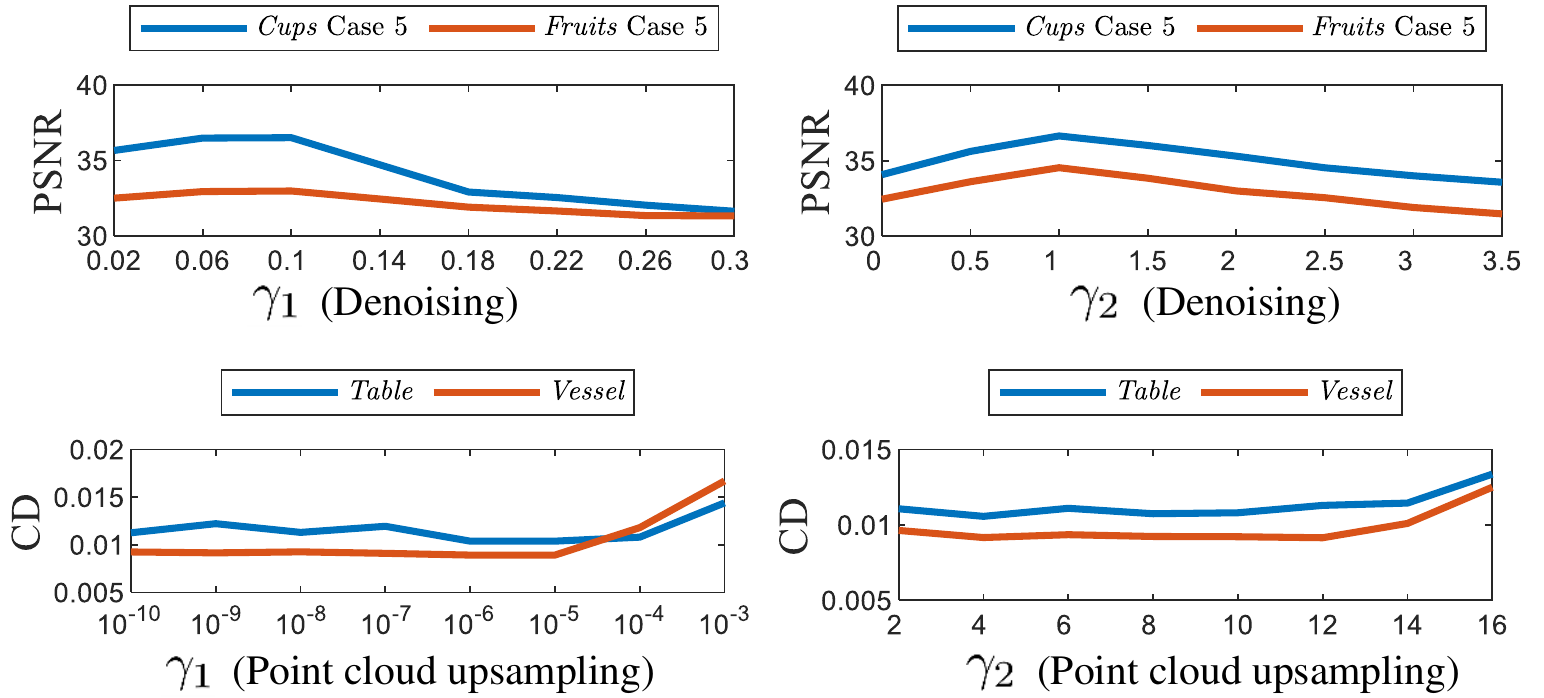}
    \end{tabular}
    \end{center}
  \vspace{-0.4cm}
  \caption{The quantitative performances of LRTFR w.r.t. different values of regularization parameters in the denoising model (\ref{loss_denoising}) and point cloud upsampling model (\ref{loss_point}).\label{fig_hyper_2}\vspace{-0.5cm}}
\end{figure}
\section{Conclusions}
In this work, we have proposed a multi-dimensional data continuous representation, termed as LRTFR. We have formulated the continuous representation as a low-rank tensor function, and developed fundamental concepts for tensor functions including F-rank and tensor function factorization. This then forms the first low-rank representation that can continuously represent data. We theoretically justified the low-rank and smooth regularizations hidden in the model, which makes LRTFR effective and efficient for continuous representation. Extensive experiments on multi-dimensional data recovery problems including multi-dimensional image inpainting, denoising, HPO, and point cloud upsampling have validated the broad applicability and superiority of our method as compared with state-of-the-art methods. The suggested continuous representation is a potential tool for multi-dimensional data processing and analysis that can be applied to more tasks in the future, e.g., hyperspectral fusion and blind image super-resolution.
\bibliographystyle{ieeetr}
\bibliography{aaai22}
\newpage
\section{Proof of Theorem 2}
\subsection{Proof of Theorem 2 (i)}
\begin{proof} 
We divide the proof into three parts. In {Part 1}, we establish the linear rerpesentation of each mode-1 function ($f(x,y,z)$ with fixed $y,z$ and variable $x$) using some mode-1 basis functions. {Part 2} gives the mode-1 low-rank function factorization, and proves that such factorization keeps the F-rank in all modes. In {Part 3}, we repeat the procedures in {Parts 1-2} for mode-2 and mode-3 to form the final low-rank tensor function factorization.
\subsubsection*{Part 1. Linear Representation of Mode-1 Function}
Suppose that ${\rm F\text{-}rank}[f]=(r_1,r_2,r_3)$ with $r_i<\infty$ ($i=1,2,3$). Let
\begin{equation}
R_1:=\{{\rm rank}({\bf T}^{(1)})|{\tt fold}_1({\bf T}^{(1)})\in S[f]\}.
\end{equation}
Then $r_1=\sup R_1$. We can see that $R_1$ is a closed set. Thus there exists ${\mathcal X}\in S[f]$ such that ${\rm rank}({\bf X}^{(1)})=r_1$. Similarly, there exist ${\mathcal Y},{\mathcal Z}\in S[f]$ such that ${\rm rank}({\bf Y}^{(2)})=r_2$ and ${\rm rank}({\bf Z}^{(3)})=r_3$. Suppose the coordinate vectors of sampling ${\mathcal X}$ from $f(\cdot)$ are ${\bf x}^{\mathcal X}$, ${\bf y}^{\mathcal X}$, and ${\bf z}^{\mathcal X}$. Define the same coordinate vectors for $\mathcal Y$ and $\mathcal Z$ (${\bf x}^{\mathcal Y}$, ${\bf y}^{\mathcal Y}$, ${\bf z}^{\mathcal Y}$, ${\bf x}^{\mathcal Z}$, ${\bf y}^{\mathcal Z}$, ${\bf z}^{\mathcal Z}$), and consider their concatenations:
\begin{equation}
\left\{
\begin{aligned}
&{\bf x}=[{\bf x}^{\mathcal X},{\bf x}^{\mathcal Y},{\bf x}^{\mathcal Z}]\in{X_f}^{n_1},\\
&{\bf y}=[{\bf y}^{\mathcal X},{\bf y}^{\mathcal Y},{\bf y}^{\mathcal Z}]\in{Y_f}^{n_2},\\
&{\bf z}=[{\bf z}^{\mathcal X},{\bf z}^{\mathcal Y},{\bf z}^{\mathcal Z}]\in{Z_f}^{n_3},\\
\end{aligned}
\right.
\end{equation}
where we introduce $n_i$s ($i=1,2,3$) to denote the sizes of these vectors. Define the tensor ${\mathcal T}\in{\mathbb R}^{n_1\times n_2\times n_3}$ as 
\begin{equation}\label{eq_T}
{\mathcal T}_{(i,j,k)}=f({\bf x}_{(i)},{\bf y}_{(j)},{\bf z}_{(k)}),\;\forall\; i,j,k.
\end{equation}
It is easy to see that ${\rm rank}_T({\mathcal T})=(r_1,r_2,r_3)$. Each column of ${\bf T}^{(1)}$ is a mode-1 fiber of $\mathcal T$, so we can denote the $r_1$ column basis vectors of ${\bf T}^{(1)}$ by ${\mathcal T}_{(:,j_1,k_1)},{\mathcal T}_{(:,j_2,k_2)},\cdots,{\mathcal T}_{(:,j_{r_1},k_{r_1})}$ with certain $j_l$s and $k_l$s ($l=1,2,\cdots,r_1$).\par 
Now let us consider any $y\in Y_f$ and $z\in Z_f$. We define a new tensor ${\mathcal U}\in{\mathbb R}^{n_1\times (n_2+1)\times (n_3+1)}$ as ${\mathcal U}_{(i,j,k)}=f({\bf x}_{(i)},{\bf y}_{(j)}',{\bf z}_{(k)}')$, where ${\bf y}'=[{\bf y},y]$ and ${\bf z}'=[{\bf z},z]$ are concatenations of vectors. We can see that ${\mathcal U}\in S[f]$, hence ${\rm rank}({\bf U}^{(1)})\leq r_1$. Meanwhile, from the definition of ${\mathcal U}$ we see that ${\bf T}^{(1)}$ is a sub-matrix of ${\bf U}^{(1)}$, which deduces $r_1={\rm rank}({\bf T}^{(1)})\leq {\rm rank}({\bf U}^{(1)})$. Thus we have ${\rm rank}({\bf U}^{(1)})={\rm rank}({\bf T}^{(1)})=r_1$, and ${\mathcal T}_{(:,j_1,k_1)},{\mathcal T}_{(:,j_2,k_2)},\cdots,{\mathcal T}_{(:,j_{r_1},k_{r_1})}$ are $r_1$ basis vectors of the column space of ${\bf U}^{(1)}$. Note that ${\mathcal U}_{(:,n_2+1,n_3+1)}$ is a column of ${\bf U}^{(1)}$. Hence, it can be linearly represented by these basis vectors with a unique coefficient vector ${\bf c}\in{\mathbb R}^{r_1}$, i.e., 
\begin{equation}\label{represent_1}
{\mathcal U}_{(:,n_2+1,n_3+1)}=\sum_{l=1}^{r_1}{\bf c}_{(l)}{\mathcal T}_{(:,j_l,k_l)}.
\end{equation}
Meanwhile, ${\mathcal U}_{(i,n_2+1,n_3+1)}=f({\bf x}_{(i)},y,z)$ for any $i\in\{1,2,\cdots,n_1\}$. Thus, the following equality holds: 
\begin{equation}
\begin{split}
f({\bf x}_{(i)},y,z) = \sum_{l=1}^{r_1}{\bf c}_{(l)}{\mathcal T}_{(i,j_l,k_l)}
\overset{\text{(\ref{eq_T})}}{=}\sum_{l=1}^{r_1}{\bf c}_{(l)}f({\bf x}_{(i)},{\bf y}_{(j_l)},{\bf z}_{(k_l)}).
\end{split}
\end{equation}
Next, we will generalize this conclusion from ${\bf x}_{(i)}\in\{{\bf x}_{(1)},{\bf x}_{(2)},\cdots,{\bf x}_{(n_1)}\}$ to any $x\in X_f$, i.e., we will show that
\begin{equation}\label{eq_f}
\begin{split}
f(x,y,z) =\sum_{l=1}^{r_1}{\bf c}_{(l)}f(x,{\bf y}_{(j_l)},{\bf z}_{(k_l)})
\end{split}
\end{equation}
holds for any $x\in X_f$.\par
Given $x\in X_f$, consider the following tensor: ${\mathcal M}\in{\mathbb R}^{(n_1+1)\times (n_2+1)\times (n_3+1)}$, where ${\mathcal M}_{(i,j,k)} = f({\bf x}_{(i)}',{\bf y}_{(j)}',{\bf z}_{(k)}')$ and ${\bf x}'=[{\bf x},x]$. We can see that ${\mathcal M}\in S[f]$ and ${\bf T}^{(1)}$ is a sub-matrix of ${\bf M}^{(1)}$. Hence ${\rm rank}({\bf M}^{(1)})=r_1$. It is easy to verify that ${\mathcal M}_{(:,j_1,k_1)},{\mathcal M}_{(:,j_2,k_2)},\cdots,{\mathcal M}_{(:,j_{r_1},k_{r_1})}$ are $r_1$ basis vectors of ${\bf M}^{(1)}$ and ${\mathcal M}_{(:,n_2+1,n_3+1)}$ is a column of ${\bf M}^{(1)}$. So it admits
\begin{equation}\label{eq_M}
{\mathcal M}_{(:,n_2+1,n_3+1)}=\sum_{l=1}^{r_1}{\bf b}_{(l)}{\mathcal M}_{(:,j_l,k_l)}, 
\end{equation}
where $\bf b$ is the coefficient vector. It derives 
\begin{equation}
{\mathcal M}_{(1:n_1,n_2+1,n_3+1)}=\sum_{l=1}^{r_1}{\bf b}_{(l)}{\mathcal M}_{(1:n_1,j_l,k_l)}.
\end{equation}
From the definition of $\mathcal M$ we see that ${\mathcal M}_{(1:n_1,n_2+1,n_3+1)} = {\mathcal U}_{(:,n_2+1,n_3+1)}$ and ${\mathcal M}_{(1:n_1,j_l,k_l)}={\mathcal T}_{(:,j_l,k_l)}$ for any $l\in\{1,2,\cdots,r_1\}$. Thus we have
\begin{equation}\label{represent_2}
{\mathcal U}_{(:,n_2+1,n_3+1)}=\sum_{l=1}^{r_1}{\bf b}_{(l)}{\mathcal T}_{(:,j_l,k_l)}.
\end{equation}
From Eqs. (\ref{represent_1}), (\ref{represent_2}) and the uniqueness of the coefficient vector, we get that ${\bf b}={\bf c}$. Hence, we have
\begin{equation}
{\mathcal M}_{(:,n_2+1,n_3+1)}=\sum_{l=1}^{r_1}{\bf c}_{(l)}{\mathcal M}_{(:,j_l,k_l)}
\end{equation}
by following Eq. (\ref{eq_M}). Note that ${\mathcal M}_{(n_1+1,n_2+1,n_3+1)}=f(x,y,z)$ and ${\mathcal M}_{(n_1+1,j_l,k_l)}=f(x,{\bf y}_{(j_l)},{\bf z}_{(k_l)})$, which concludes that Eq. (\ref{eq_f}) is true for any $x\in X_f$. This gives the linear representation form of the mode-1 function $f(x,y,z)$ (with fixed $y,z$ and variable $x$) using some basis functions $f(x,{\bf y}_{(j_l)},{\bf z}_{(k_l)})$ ($l=1,2,\cdots,r_1$).\\
\subsubsection*{Part 2. Mode-1 Function Factorization and F-rank Invariance}
We define the factor function $f_x(\cdot):X_f\rightarrow{\mathbb R}^{r_1}$ as
\begin{equation}
\begin{split}
f_x(x):=[f(x,{\bf y}_{(j_1)},{\bf z}_{(k_1)}),f(x,&{\bf y}_{(j_2)},{\bf z}_{(k_2)}),\cdots,\\&f(x,{\bf y}_{(j_{r_1})},{\bf z}_{(k_{r_1})})].
\end{split}
\end{equation}
Also, define the tensor function $h(\cdot):\{1,2,\cdots,r_1\}\times Y_f\times Z_f\rightarrow{\mathbb R}$ as
\begin{equation}
h(i,y,z):=({c}(y,z))_{(i)},
\end{equation}
where the mapping ${c}(y,z):Y_f\times Z_f\rightarrow {\mathbb R}^{r_1}$ returns the coefficient vector $\bf c$ in Eq. (\ref{eq_f}) with the given coordinates $y\in Y_f$ and $z\in Z_f$ (The coefficient vector $\bf c$ is depend on the coordinates $y$ and $z$). $({c}(y,z))_{(i)}$ denotes the $i$-th element of ${c}(y,z)$. From {Part 1} we can see that for any $(x,y,z)\in D_f=X_f\times Y_f\times Z_f$, it holds that 
\begin{equation}
f(x,y,z) = \sum_{l=1}^{r_1}h(l,y,z)\;(f_x(x))_{(l)},
\end{equation}
where $(f_x(x))_{(l)}$ denotes the $l$-th element of $f_x(x)$. It can be re-written in a tensor-matrix product form:
\begin{equation}\label{fac_1}
f(x,y,z) = h(:,y,z)\times_1f_x(x),
\end{equation}
where 
\begin{equation}
h(:,y,z):=[h(1,y,z),h(2,y,z),\cdots,h(r_1,y,z)]^T\in{\mathbb R}^{r_1\times 1\times 1}.
\end{equation}
We call this factorization of $f(\cdot)$ as the mode-1 low-rank function factorization. Next, we will show that this factorization has F-rank invariant property, i.e., ${\rm F\text{-}rank}[f]={\rm F\text{-}rank}[h]$. First, it is easy to see that $({\rm F\text{-}rank}[h])_{(1)}\leq r_1$ since the definition domain of $h(\cdot)$ in the first dimension is $\{1,2,\cdots,r_1\}$. Then, consider the tensor ${\mathcal H}\in{\mathbb R}^{r_1\times n_2\times n_3}\cap S[h]$ defined by ${\mathcal H}_{(i,j,k)}=h(i,{\bf y}_{(j)},{\bf z}_{(k)})$. Also consider the matrix
\begin{equation}
{\bf U}=
\begin{pmatrix}
f_x({\bf x}_{(1)})\\
f_x({\bf x}_{(2)})\\
\cdots\\
f_x({\bf x}_{(n_1)})\\
\end{pmatrix}
\in{\mathbb R}^{{n_1}\times r_1}.
\end{equation}
Then we have ${\mathcal T}={\mathcal H}\times_1 {\bf U}$, hence ${\bf T}^{(1)}={\bf U}{\bf H}^{(1)}$, inducing that $r_1={\rm rank}({\bf T}^{(1)})\leq{\rm rank}({\bf H}^{(1)})$. So we have $({\rm F\text{-}rank}[h])_{(1)}= r_1$.\par 
Next we show that $({\rm F\text{-}rank}[h])_{(2)}= r_2$. Define the matrix
\begin{equation}
\hat{\bf U}=
\begin{pmatrix}
{\bf U}^T&&&\\
&{\bf U}^T&&\\
&&\ddots\\
&&&{\bf U}^T\\
\end{pmatrix}
\in{\mathbb R}^{r_1n_3\times n_1n_3}.
\end{equation}
We can see that ${\bf T}^{(2)}={\bf H}^{(2)}\hat{\bf U}$. Hence, $r_2={\rm rank}({\bf T}^{(2)})\leq{\rm rank}({\bf H}^{(2)})$, and thus $({\rm F\text{-}rank}[h])_{(2)}\geq r_2$.
\par
Consider any ${\mathcal A}\in S[h]$. Suppose its corresponding coordinate vectors are ${\bf x}^{\mathcal A}\in{X_h}^{m_1}$, ${\bf y}^{\mathcal A}\in{Y_h}^{m_2}$, and ${\bf z}^{\mathcal A}\in{Z_h}^{m_3}$. Define the matrix
\begin{equation}
\begin{scriptsize}
{\bf F}=
{\begin{pmatrix}
\big{(}f_x({\bf x}_{(1)})\big{)}_{{\bf x}^{\mathcal A}_{(1)}}
&
\big{(}f_x({\bf x}_{(1)})\big{)}_{{\bf x}^{\mathcal A}_{{(2)}}}&\cdots&
\big{(}f_x({\bf x}_{(1)})\big{)}_{{\bf x}^{\mathcal A}_{{(m_1)}}}\\

\big{(}f_x({\bf x}_{(2)})\big{)}_{{\bf x}^{\mathcal A}_{(1)}}
&
\big{(}f_x({\bf x}_{(2)})\big{)}_{{\bf x}^{\mathcal A}_{(2)}}&\cdots&
\big{(}f_x({\bf x}_{(2)})\big{)}_{{\bf x}^{\mathcal A}_{(m_1)}}\\

\vdots&\vdots&\ddots&\vdots\\
\big{(}f_x({\bf x}_{(n_1)})\big{)}_{{\bf x}^{\mathcal A}_{(1)}}
&
\big{(}f_x({\bf x}_{(n_1)})\big{)}_{{\bf x}^{\mathcal A}_{(2)}}&\cdots&
\big{(}f_x({\bf x}_{(n_1)})\big{)}_{{\bf x}^{\mathcal A}_{(m_1)}}\\
\end{pmatrix}}
\end{scriptsize}
\end{equation}
with size ${n_1\times m_1}$. Let ${\mathcal B}={\mathcal A}\times_1 {\bf F}$, we can see that ${\mathcal B}\in S[f]$. Also, we have ${\bf B}^{(2)}={\bf A}^{(2)}\hat{\bf F}$, where
\begin{equation}
\hat{\bf F}=
\begin{pmatrix}
{\bf F}^T&&&\\
&{\bf F}^T&&\\
&&\ddots\\
&&&{\bf F}^T\\
\end{pmatrix}
\in{\mathbb R}^{m_1m_3\times n_1m_3}.
\end{equation}
From the definition of $f_x(\cdot)$, we can observe that $\bf F$ is full column rank, and thus $\hat{\bf F}$ is full row rank. Multiplying a full row rank matrix in the RHS does not change the rank. Therefore, ${\rm rank}({\bf B}^{(2)})={\rm rank}({\bf A}^{(2)})$ holds. Since ${\mathcal B}\in S[f]$, we have ${\rm rank}({\bf B}^{(2)})\leq r_2$ and thus ${\rm rank}({\bf A}^{(2)})\leq r_2$. Due to the arbitrariness of $\mathcal A$, we have $({\rm F\text{-}rank}[h])_{(2)} \leq r_2$, which deduces $({\rm F\text{-}rank}[h])_{(2)} = r_2$.\par
In a similar way, we can show that $({\rm F\text{-}rank}[h])_3 = r_3$, which concludes the F-rank invariance of mode-1 low-rank function factorization.
\subsubsection*{Part 3. Low-Rank Tensor Function Factorization}
Now we have shown that ${\rm F\text{-}rank}[h]=(r_1,r_2,r_3)$. Thus, similar to the mode-1 low-rank function factorization, $h(\cdot)$ can be factorized along mode-2, i.e., there exist a factor function $f_y(\cdot):Y_f\rightarrow{\mathbb R}^{r_2}$ and a tensor function $h'(\cdot):\{1,2,\cdots,r_1\}\times \{1,2,\cdots,r_2\}\times Z_f\rightarrow{\mathbb R}$ such that 
\begin{equation}\label{fac_2}
h(x,y,z) = h'(x,:,z)\times_2f_y(y)
\end{equation}
holds for any $(x,y,z)$, where 
\begin{equation}
h'(x,:,z):=[h'(x,1,z),h'(x,2,z),\cdots,h'(x,r_2,z)]\in{\mathbb R}^{1\times r_2\times 1}.
\end{equation}
Similar to {Part 2}, we can show that $h'(\cdot)$ is F-rank invariant, i.e., ${\rm F\text{-}rank}[h']=(r_1,r_2,r_3)$. So we can apply the same low-rank function factorization in mode-3 and get that there exist a factor function $f_z(\cdot):Z_f\rightarrow{\mathbb R}^{r_3}$ and a tensor function $h''(\cdot):\{1,2,\cdots,r_1\}\times \{1,2,\cdots,r_2\}\times \{1,2,\cdots,r_3\}\rightarrow{\mathbb R}$ such that 
\begin{equation}\label{fac_3}
h'(x,y,z) = h''(x,y,:)\times_3f_z(z)
\end{equation}
holds for any $(x,y,z)$, where 
\begin{equation}
h''(x,y,:):=[h''(x,y,1),h''(x,y,2),\cdots,h''(x,y,r_3)]\in{\mathbb R}^{1\times 1\times r_3}.
\end{equation}
Define the core tensor ${\mathcal C}\in{\mathbb R}^{r_1\times r_2\times r_3}$ by ${\mathcal C}_{(i,j,k)}=h''(i,j,k)$ for any $(i,j,k)$. Then for any $(v_1,v_2,v_3)\in D_f$, from Eqs. (\ref{fac_1}), (\ref{fac_2}), and (\ref{fac_3}) we have
\begin{equation}
f(v_1,v_2,v_3)={\mathcal C}\times_3 f_z(v_3)\times_2 f_y(v_2)\times_1 f_x(v_1).
\end{equation}
One can verify that tensor-matrix product admits commutative law, which completes the proof of Theorem 2 {\bf (i)}.
\end{proof}
\subsection{Proof of Theorem 2 (ii)}
\begin{proof}
For any ${\mathcal G}\in S[g]$, suppose its sampling coordinate vectors are ${\bf x}^{\mathcal G}\in{X_g}^{n_1}$, ${\bf y}^{\mathcal G}\in{Y_g}^{n_2}$, and ${\bf z}^{\mathcal G}\in{Z_g}^{n_3}$, i.e., 
\begin{equation}
{\mathcal G}_{(i,j,k)}=g({\bf x}^{\mathcal G}_{(i)},{\bf y}^{\mathcal G}_{(j)},{\bf z}^{\mathcal G}_{(k)}),\;\forall\; i,j,k.
\end{equation}
Define the factor matrices
\begin{equation}
\left\{
\begin{aligned}
&{\bf U}=(g_x({\bf x}^{\mathcal G}_{(1)}),g_x({\bf x}^{\mathcal G}_{(2)})\cdots,g_x({\bf x}^{\mathcal G}_{(n_1)}))^T\in{\mathbb R}^{n_1\times r_1},\\
&{\bf V}=(g_y({\bf y}^{\mathcal G}_{(1)}),g_y({\bf y}^{\mathcal G}_{(2)})\cdots,g_y({\bf y}^{\mathcal G}_{(n_2)}))^T\in{\mathbb R}^{n_2\times r_2},\\
&{\bf W}=(g_z({\bf z}^{\mathcal G}_{(1)}),g_z({\bf z}^{\mathcal G}_{(2)})\cdots,g_z({\bf z}^{\mathcal G}_{(n_3)}))^T\in{\mathbb R}^{n_3\times r_3}.\\
\end{aligned}
\right.
\end{equation}
Then we have ${\mathcal G}={\mathcal C}\times_1 {\bf U}\times_2{\bf V}\times_3{\bf W}$. From Theorem 1 {\bf (ii)} we get that $({\rm rank}_T({\mathcal G}))_{(i)}\leq r_i$ ($i=1,2,3$). Thus we have $({\rm F\text{-}rank}[g])_{(i)}\leq r_1$ ($i=1,2,3$). 
\end{proof}
\section{Proof of Theorem 3}
\begin{proof}
For any $(x_1,y_1,z_1),(x_2,y_2,z_2)\in D_f$, direct calculation yields
\begin{equation}
\begin{split}
&|f(x_1,y_1,z_1)-f(x_2,y_1,z_1)|\\
=&|{\mathcal C}\times_1 f_{\theta_x}(x_1)\times_2f_{\theta_y}(y_1)\times_3f_{\theta_z}(z_1)-{\mathcal C}\times_1 f_{\theta_x}(x_x)\times_2\\
&f_{\theta_y}(y_1)\times_3f_{\theta_z}(z_1)|\\
=&|({\mathcal C}\times_2f_{\theta_y}(y_1)\times_3f_{\theta_z}(z_1))\times_1(f_{\theta_x}(x_1)-f_{\theta_x}(x_2))|\\
\leq&\lVert{\mathcal C}\times_2f_{\theta_y}(y_1)\times_3f_{\theta_z}(z_1)\rVert_{\ell_1}\lVert f_{\theta_x}(x_1)-f_{\theta_x}(x_2)\rVert_{\ell_1}\\
\leq&\eta\lVert f_{\theta_y}(y_1)\rVert_{\ell_1}\lVert f_{\theta_z}(z_1)\rVert_{\ell_1}\lVert f_{\theta_x}(x_1)-f_{\theta_x}(x_2)\rVert_{\ell_1}.\\
\end{split}
\end{equation}
Note that $\sigma(\cdot)$ is Lipschitz continuous, i.e., $|\sigma(x)-\sigma(y)|\leq \kappa|x-y|$ holds for any $x,y$, and letting $y=0$ derives $|\sigma(x)|\leq \kappa|x|$ since $\sigma(0)=\sin(0)=0$. Thus we have
\begin{equation}
\begin{split}
\lVert f_{\theta_y}(y_1)\rVert_{\ell_1}=&\lVert{\bf H}^y_d(\sigma({\bf H}^y_{d-1}\cdots{\bf H}^y_2(\sigma( {\bf H}^y_1y_1))))\rVert_{\ell_1}\\
\leq&\eta\lVert\sigma({\bf H}^y_{d-1}\cdots{\bf H}^y_2(\sigma ({\bf H}_1y_1)))\rVert_{\ell_1}\\
\leq&\eta\kappa\lVert{\bf H}^y_{d-1}\cdots{\bf H}^y_2(\sigma ({\bf H}_1y_1))\rVert_{\ell_1}\\
&\cdots\\
\leq&\eta^d\kappa^{d-1}|y_1|,\\
\end{split}
\end{equation}
where $\{{\bf H}^y_i\}_{i=1}^d$ denote the weight matrices of $f_{\theta_y}(\cdot)$. Similarly, we have $\lVert f_{\theta_z}(z_1)\rVert_{\ell_1}\leq\eta^d\kappa^{d-1}|z_1|$. Meanwhile, it holds that
\begin{equation}
\begin{split}
&\lVert f_{\theta_x}(x_1)-f_{\theta_x}(x_2)\rVert_{\ell_1}\\
=&\lVert{\bf H}^x_d(\sigma({\bf H}^x_{d-1}\cdots{\bf H}^x_2(\sigma( {\bf H}^x_1x_1))))-{\bf H}^x_d(\sigma({\bf H}^x_{d-1}\cdots\\&\;\;{\bf H}^x_2(\sigma( {\bf H}^x_1x_2))))\rVert_{\ell_1}\\
=&\lVert{\bf H}^x_d(\sigma({\bf H}^x_{d-1}\cdots{\bf H}^x_2(\sigma( {\bf H}^x_1x_1)))-\sigma({\bf H}^x_{d-1}\cdots\\&\;\;{\bf H}^x_2(\sigma( {\bf H}^x_1x_2))))\rVert_{\ell_1}\\
\leq&\eta\lVert\sigma({\bf H}^x_{d-1}\cdots{\bf H}^x_2(\sigma( {\bf H}^x_1x_1)))-\sigma({\bf H}^x_{d-1}\cdots\\&\;\;{\bf H}^x_2(\sigma( {\bf H}^x_1x_2)))\rVert_{\ell_1}\\
\leq&\eta\kappa\lVert{\bf H}^x_{d-1}\cdots{\bf H}^x_2(\sigma( {\bf H}^x_1x_1))-{\bf H}^x_{d-1}\cdots\\&\;\;{\bf H}^x_2(\sigma( {\bf H}^x_1x_2))\rVert_{\ell_1}\\
&\cdots\\
\leq&\eta^d\kappa^{d-1}|x_1-x_2|,
\end{split}
\end{equation}
where $\{{\bf H}^x_i\}_{i=1}^d$ denote the weight matrices of $f_{\theta_x}(\cdot)$. Combining the above inequalities, we have
\begin{equation}
\begin{split}
|f(x_1,y_1,z_1)-f(x_2,y_1,z_1)|\leq &\eta^{3d+1}\kappa^{3d-3}|y_1||z_1||x_1-x_2|\\
\leq&\eta^{3d+1}\kappa^{3d-3}\zeta^2|x_1-x_2|.
\end{split}
\end{equation}
In a similar way we can show that 
\begin{equation}
\left\{
\begin{aligned}
&|f(x_1,y_1,z_1)-f(x_1,y_2,z_1)|\leq&\eta^{3d+1}\kappa^{3d-3}\zeta^2|y_1-y_2|\\
&|f(x_1,y_1,z_1)-f(x_1,y_1,z_2)|\leq&\eta^{3d+1}\kappa^{3d-3}\zeta^2|z_1-z_2|.
\end{aligned}
\right.
\end{equation}
The proof is completed.
\end{proof}
\end{document}